\documentclass[letterpaper,twocolumn,10pt]{article}
\usepackage{usenix}
\usepackage{tikz}
\usepackage{amsmath}
\usepackage{microtype}
\usepackage{graphicx}
\usepackage{subfigure}
\usepackage{booktabs} % for professional tables
\usepackage{float}

\usepackage{minitoc}
\usepackage{times}
\usepackage{soul}
\usepackage{url}
\usepackage{booktabs}
\usepackage{multirow}
\usepackage{xcolor}         % colors
\usepackage[dvipsnames]{xcolor}

\usepackage{enumitem}
\usepackage[utf8]{inputenc} % allow utf-8 input
\usepackage[T1]{fontenc}    % use 8-bit T1 fonts
\usepackage{hyperref}       % hyperlinks
\usepackage{url}            % simple URL typesetting
\usepackage{booktabs}       % professional-quality tables
\usepackage{amsfonts}       % blackboard math symbols
\usepackage{nicefrac}       % compact symbols for 1/2, etc.
\usepackage{microtype}      % microtypography
\usepackage{xcolor}         % colors
\usepackage[table]{xcolor}  % color support for tables
\definecolor{LightGray}{gray}{0.9} % define LightGray color define LightGray color
\usepackage[ruled,linesnumbered]{algorithm2e}
\usepackage[numbers]{natbib}
\usepackage[dvipsnames]{xcolor}
\usepackage{microtype}
\usepackage{graphicx}
\usepackage{subfigure}
\usepackage{booktabs} % for professional tables
\usepackage{natbib}

\usepackage{minitoc}
\usepackage{times}
\usepackage{soul}
\usepackage{url}
\usepackage{booktabs}
\usepackage{multirow}

\usepackage{hyperref}
\usepackage{url}
\usepackage{times}  % DO NOT CHANGE THIS
\usepackage{helvet}  % DO NOT CHANGE THIS
\usepackage{courier}  % DO NOT CHANGE THIS
\usepackage{graphicx} % DO NOT CHANGE THIS
\usepackage{amsthm}
\usepackage{mathdots}
\usepackage{amssymb}
\usepackage{amsmath}
\usepackage{amsfonts}
\usepackage{color}
\usepackage{wrapfig}

\usepackage{amsmath}
\usepackage{amssymb}
\usepackage{mathtools}
\usepackage{amsthm}

\usepackage[capitalize,noabbrev]{cleveref}

\usepackage{amsthm}
\usepackage{mdframed}
\theoremstyle{plain}
\newmdenv[
  topline=true,
  bottomline=true,
  rightline=true,
  leftline=true,
  linecolor=black,
  linewidth=0.8pt,
  backgroundcolor=white,
  skipabove=10pt,
  skipbelow=10pt
]{boxedtheorem}

\newcommand{\lnconv}{\mathbin{\overset{\mathrm{logcvx}}{\sim}}}

\newtheorem{reptheorem}{Theorem}

\newtheorem{theorem}{\bf{Theorem}}
\newtheorem{definition}{\bf{Definition}}

\newtheorem{lemma}{\bf{Lemma}}
\newtheorem{corollary}{\bf Corollary}

\newtheorem{reproposition}{\bf{Proposition}}
\newtheorem{proposition}{\bf{Proposition}}
\usepackage{amsthm}   % For the theorem environment
\usepackage{tcolorbox} % For creating colored boxes
\usepackage{xr} % or \usepackage{xr-hyper}
\usepackage{microtype}      % microtypography
\usepackage{minitoc} 
\usepackage{lipsum}

\newcommand{\DEL}[1]{\iffalse #1 \fi}

\newcommand{\rd}{\color{BrickRed}}
\newcommand{\bl}{\color{RoyalBlue}}

\begin{document}
\title{Interpolation-Based Optimization for Enforcing 
$\ell_p$-Norm Metric Differential Privacy in Continuous and Fine-Grained Domains}

%for single author (just remove % characters)
\author{
{\rm Chenxi Qiu}\\
University of North Texas
% copy the following lines to add more authors
% \and
% {\rm Name}\\
%Name Institution
} % end author


%Telephone: (800) 555--1212, Fax: (888) 555--1212}


\maketitle

\begin{abstract}
\emph{Metric Differential Privacy (mDP)} generalizes Local Differential Privacy (LDP) by adapting privacy guarantees based on pairwise distances, enabling context-aware protection and improved utility. While existing optimization-based methods reduce utility loss effectively in coarse-grained domains, optimizing mDP in fine-grained or continuous settings remains challenging due to the computational cost of constructing dense perterubation matrices and satisfying pointwise constraints. 

In this paper, we propose an \emph{interpolation-based} framework for optimizing $\ell_p$-norm mDP in such domains. Our approach optimizes perturbation distributions at a sparse set of anchor points and interpolates distributions at non-anchor locations via log-convex combinations, which provably preserve mDP. To address privacy violations caused by naive interpolation in high-dimensional spaces, we decompose the interpolation process into a sequence of one-dimensional steps and derive a corrected formulation that enforces $\ell_p$-norm mDP by design. We further explore joint optimization over perturbation distributions and privacy budget allocation across dimensions. Experiments on real-world location datasets demonstrate that our method offers rigorous privacy guarantees and competitive utility in fine-grained domains, outperforming baseline mechanisms. \looseness = -1
\end{abstract}


\section{Introduction}

\label{sec:intro}
Privacy-preserving data sharing is increasingly important in applications such as location-based services (LBSs), mobility prediction, and user modeling. These applications highly depend on fine-grained data representations, which are also required to comply with stringent privacy constraints. Standard mechanisms like \emph{Local Differential Privacy (LDP)}~\cite{Duchi-FOCS2013} enforce uniform privacy across all input pairs, often introducing excessive noise and degrading utility, especially in spatial, continuous, or structured domains. \emph{Metric Differential Privacy (mDP)}~\cite{Chatzikokolakis-PETS2013} addresses this limitation by relaxing LDP through a distance-aware formulation: It requires stronger indistinguishability for nearby records and permits weaker guarantees for distant ones. This flexibility enables improved privacy-utility trade-offs and has been applied to protect geo-location data~\cite{Andres-CCS2013}, word embeddings~\cite{feyisetan-WTNLP2021}, speech~\cite{Han-ICME2020}, and image data~\cite{fan-ICME2019, Chen-CVPR2021}. \looseness = -1

\DEL{
\begin{figure}[t]

\begin{minipage}{0.45\textwidth}
\centering
\subfigure{
\includegraphics[width=1.00\textwidth]{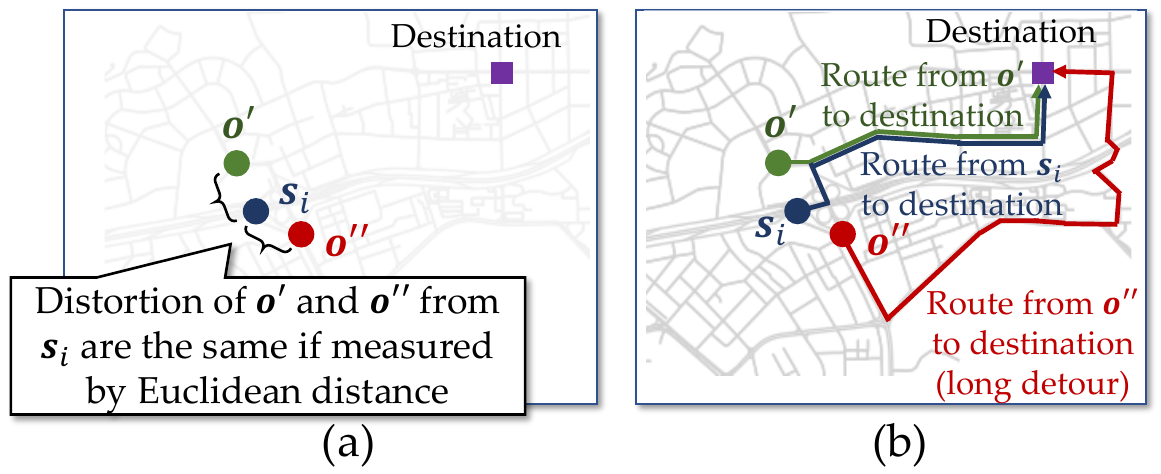}}

\caption{Example: Utility losses are different when a location is perturbed in different directions.}
\label{fig:sensitivito_example}
\end{minipage}

\end{figure}}

% To address these limitations, \emph{metric Differential Privacy (mDP)}~\cite{Chatzikokolakis-PETS2013} was introduced as a generalization of LDP that enables more nuanced levels of indistinguishability between inputs. Instead of applying a uniform privacy guarantee, mDP incorporates a distance metric to modulate the strength of protection: inputs that are close under the metric must remain highly indistinguishable, while those that are farther apart may be more readily distinguished. This distance-aware relaxation enables privacy mechanisms to inject less noise, thereby improving utility while still offering meaningful privacy guarantees. This enhancement broadens the flexibility and applicability of LDP across various data domains, including geo-location perturbation in LBSs~\cite{Andres-CCS2013, Yu-NDSS2017, Shokri-PoPETs2015, Chatzikokolakis-PoPETs2015} and text perturbation in natural language processing (NLP)~\cite{ImolaUAI2022,Carvalho2021TEMHU}.
Since the introduction of mDP~\cite{Chatzikokolakis-PETS2013}, numerous \emph{pre-defined noise mechanisms} have been proposed to enforce distance-based privacy guarantees. Notably, the planar Laplace mechanism~\cite{Andres-CCS2013} achieves $\ell_2$-based geo-indistinguishability by adding two-dimensional noise scaled to the sensitivity of location queries, while the Exponential Mechanism (EM)~\cite{Chatzikokolakis-PoPETs2015} selects outputs based on a fixed utility function that favors locations closer to the true input. These methods are efficient and well-suited for continuous and fine-grained domains. However, their fixed noise distributions often lead to suboptimal privacy-utility trade-offs, as they fail to account for direction-dependent or context-specific variations in utility.

% Consider LBS for vehicles as an example, where the utility of a perturbed location is tied to its accuracy in estimating travel costs. As illustrated in Fig.~\ref{fig:sensitivito_example}(a)(b), two perturbed locations $\mathbf{y}'$ and $\mathbf{y}''$ may lie at the same distance from the actual location $\mathbf{x}_i$, i.e., $d(\mathbf{x}_i, \mathbf{y}') = d(\mathbf{x}_i, \mathbf{y}'')$, yet incur different utility losses due to the structure of the road network. The route from $\mathbf{y}'$ closely resembles the optimal route from $\mathbf{x}_i$, whereas the route from $\mathbf{y}''$ requires a detour, resulting in a significantly higher utility loss for travel cost estimation. However, both the planar Laplace mechanism and the standard EM assign the same probabilities to $\mathbf{y}'$ and $\mathbf{y}''$—the former depends solely on distance, and the latter typically assumes a distance-based utility function. As a result, these mechanisms fail to capture direction-dependent or context-sensitive variations in utility, highlighting the limitations of fixed noise distributions in real-world spatial applications.

% fail to provide optimal utility in high-dimensional or irregular domains. Specifically, the Laplace mechanism injects symmetric noise regardless of the underlying geometry or directionality of utility loss, and EM relies on handcrafted utility functions whose performance is highly sensitive to parameter tuning. Moreover, both mechanisms apply , which may be suboptimal for applications where privacy sensitivity varies across locations or attributes.

To address the limitations of pre-defined noise mechanisms, recent work has explored \emph{optimization-based} approaches, most notably \emph{linear programming (LP)}, to directly optimize the perturbation distribution by minimizing expected utility loss subject to mDP constraints~\cite{Bordenabe-CCS2014,ImolaUAI2022,Qiu-IJCAI2024}. However, LP-based formulations are typically restricted to coarse-grained domains: solving them over continuous or fine-grained spaces is computationally prohibitive. A common workaround is to discretize the domain (e.g., using uniform grids or road-map features), which improves tractability but weakens formal guarantees: discretization can overestimate distances between nearby records, thereby loosening the effective mDP constraints and overlooking fine-grained privacy leakage~\cite{Chatzikokolakis-PETS2017}.

To balance utility and efficiency, hybrid methods combine optimization with pre-defined noise. For instance, \emph{Bayesian remapping}~\cite{Chatzikokolakis-PETS2017} post-processes outputs of a fixed mechanism via Bayes’ rule to improve utility; \emph{ConstOPT}~\cite{ImolaUAI2022} and \emph{LR-Geo}~\cite{Qiu-PETS2025} reduce complexity by constraining the search space (e.g., locality- or structure-aware parameterizations). Nonetheless, these approaches have notable limitations: Bayesian remapping depends on fixed noise priors and cannot guarantee global optimality, while ConstOPT and LR-Geo \emph{still incur} substantial computational overhead and remain tied to discretization, which limits scalability and can compromise strict mDP enforcement in fine-grained or continuous domains.

 % \looseness = -1
\begin{figure}[t]
\begin{minipage}{0.45\textwidth}
\centering
\subfigure{
\includegraphics[width=1.00\textwidth]{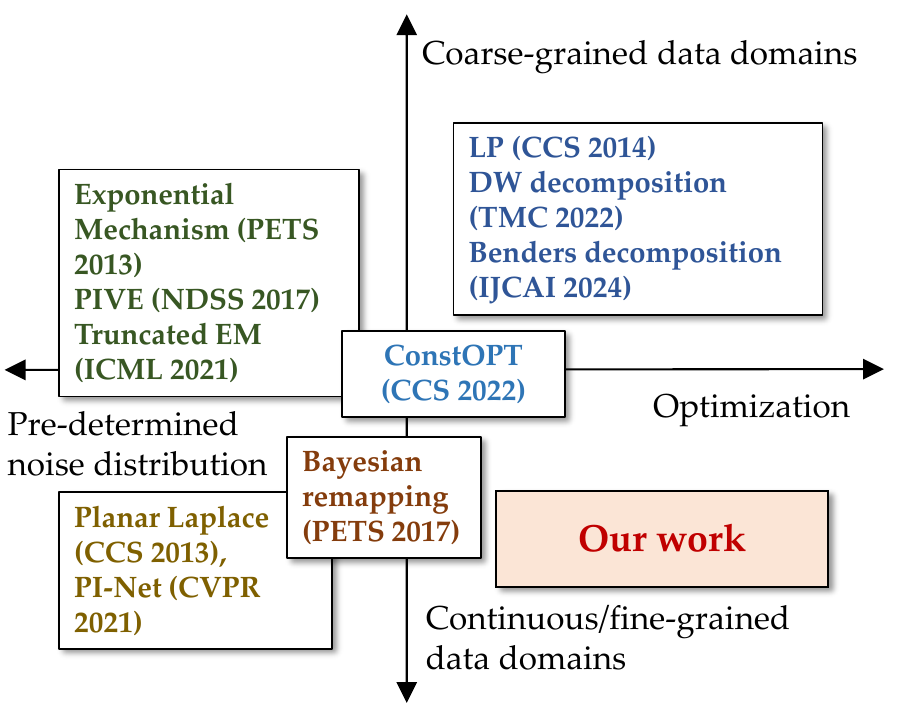}}
\caption{Related works vs. our work. 
\newline \small Example works in the figure: EM~\cite{Chatzikokolakis-PETS2013}, Planar Laplace~\cite{Andres-CCS2013}, LP~\cite{Bordenabe-CCS2014}, PI-Net (Laplace)~\cite{Chen-CVPR2021}, PIVE (EM-based)~\cite{Yu-NDSS2017}, Danzig-Wolfe (DW) decomposition (LP-based)~\cite{Qiu-TMC2022},  Benders decomposition (LP-based)~\cite{Qiu-IJCAI2024}, ConstOPT (EM+LP)~\cite{ImolaUAI2022}, Bayesian Remapping~\cite{Chatzikokolakis-PETS2017}, Truncated EM~\cite{Carvalho2021TEMHU}.}
\label{fig:relatedwork}
\end{minipage}
\end{figure}
Fig.~\ref{fig:relatedwork} situates prior mDP methods along two axes, domain granularity and mechanism design, and shows that most either (i) rely on pre-defined perturbation distributions suited to continuous spaces or (ii) perform optimization only after coarse discretization. This reveals a clear gap: scalable, optimization-based mechanisms that operate efficiently in continuous/fine-grained domains.

% Compared to traditional DP, optimizing mDP presents extra challenges, owing to varied privacy needs between neighboring records and diverse sensitivity of utility loss to perturbation across different directions and magnitudes~\cite{Qiu-TMC2022}. Most LP-based existing works, as listed in Fig. \ref{fig:relatedwork}, focus on optimizing perturbation strategies for discrete datasets. In these studies, if the secret dataset is initially in a continuous space, each record is approximated to a predetermined discrete data point. For example, using a grid map to represent locations, the region is divided into grid cells, with the center of each cell approximating the location of all points within that cell. This approximation reduces computational load but fails to ensure mDP in continuous space or finer-grained representations, as the distance between approximated points might exceed the actual distance between original records, relaxing mDP constraints.

 %, including spatial data analysis~\cite{Andres-CCS2013}, language model  or Mahatton distance ($\ell_1$-norm) for spatial data or max-norm ($\ell_\infty$-norm) for multi-attribute frequency estimation~\cite{Arcolezi-VLDB2023}. 

\paragraph{Our contributions.} In this paper, we take a first step toward addressing this gap by introducing a new \emph{interpolation-based} framework for optimizing mDP in continuous or fine-grained domains under general $\ell_p$-norm distance metrics (e.g., $\ell_1$, $\ell_2$, $\ell_\infty$). The flexibility of $\ell_p$-norms allows our method to accommodate diverse applications.
Our framework consists of three main steps: (1) partitioning the $N$-dimensional secret domain into disjoint cells and selecting their corners as \emph{anchor records} to approximate the full domain; (2) solving an \emph{Anchor Perturbation Optimization (APO)} problem to compute the optimal perturbation probabilities at these anchor points; and (3) interpolating the perturbation probabilities for non-anchor records using a \emph{log-convex function}, specifically a weighted geometric mean of the anchor probabilities, which ensures smooth and privacy-preserving transitions across the domain. \looseness=-1

% {\rd Notably, directly applying log-convex interpolation in high-dimensional domains does not guarantee $\ell_p$-norm mDP. This is due to the convexity of the $\ell_p$-norm for $p \geq 1$, which can cause the interpolated log-probabilities to violate the \emph{Lipschitz bound} required by $\ell_p$-norm mDP, formally defined in Eq. (\ref{eq:Lipschitzbound})). To resolve this issue, we decompose the interpolation process into $N$ \emph{one-dimensional log-convex interpolations}, each aligned with a coordinate axis (\textbf{Definition~\ref{def:logconvex_1d}}). To ensure that the combined mechanism satisfies an overall $(\epsilon, d_p)$-mDP guarantee, we introduce a \emph{dimension-wise composition strategy} that allocates per-dimension privacy budgets $\{\epsilon_\ell\}_{\ell=1}^N$, subject to a sufficient condition formalized in \textbf{Theorem~\ref{thm:composition}}. The privacy guarantee of each one-dimensional interpolator is established in \textbf{Propositions~\ref{prop:intral}} and \textbf{\ref{prop:across}}, and their composition via a dimension-wise product structure (\textbf{Definition~\ref{def:PPI}}) results in a valid high-dimensional interpolation mechanism, as proven in \textbf{Theorem~\ref{thm:AIPOfeasible}}. \looseness =-1}

Naively applying log-convex interpolation in high dimensions does not by itself enforce $\ell_p$-mDP: because the $\ell_p$ norm is convex for $p \geq 1$, combining anchor-wise bounds during interpolation can enlarge the effective distance and thus violate the \emph{global Lipschitz condition} required by mDP~\cite{koufogiannis2015}. We avoid this by decomposing the $N$-dimensional interpolation into axis-aligned, one-dimensional log-convex steps (\textbf{Definition~\ref{def:logconvex_1d}}). For each axis we establish a one-dimensional Lipschitz bound (\textbf{Propositions~\ref{prop:intral}–\ref{prop:across}}), then distribute a total privacy budget across coordinates using a dimension-wise composition rule (\textbf{Theorem~\ref{thm:composition}}; e.g., $\sum_{\ell}\epsilon_{\ell}^{p/(p-1)}\leq\epsilon^{p/(p-1)}$ for $p>1$). We combine the per-axis interpolants through a product construction (\textbf{Definition~\ref{def:PPI}}) and show that the resulting mechanism satisfies a global Lipschitz guarantee with respect to the $\ell_p$ metric (\textbf{Theorem~\ref{thm:AIPOfeasible}}). After normalization, this yields a valid high-dimensional perturbation mechanism that preserves the target mDP guarantee up to a small constant-factor slack quantified in our analysis.

Moreover, to minimize utility loss within this framework, we formulate the \emph{Anchor Perturbation Optimization (APO)} problem, which jointly optimizes the perturbation probabilities at anchor points and the allocation of dimension-wise privacy budgets. Due to the non-convex nature of the APO objective, we propose a tractable linear approximation, termed \emph{Approx-APO}, and provide theoretical bounds on its optimality gap relative to the original formulation. 

We evaluate our method on real-world road network datasets from Rome, New York City, and London~\cite{openstreetmap}, considering both the $\ell_2$ norm (Euclidean distance) and the $\ell_1$ norm (Manhattan distance) for measuring spatial proximity. Experimental results show that our approach enforces strictly stronger privacy guarantees (0\% mDP violations) than coarse-grained LP-based mechanisms~\cite{Bordenabe-CCS2014, ImolaUAI2022}, while consistently achieving lower utility loss compared to pre-defined noise mechanisms (e.g., EM~\cite{Chatzikokolakis-PoPETs2015, Carvalho2021TEMHU}, Laplace~\cite{Andres-CCS2013}, and \emph{Truncated EM}~\cite{Carvalho2021TEMHU}), and hybrid methods (e.g., ConstOPT~\cite{ImolaUAI2022} and Bayesian remapping~\cite{Chatzikokolakis-PETS2017}). 
 % Our method maintains formal mDP guarantees and scales efficiently to large, continuous domains. 
 
Our main contribution can be summarized as follows: 
\looseness = -1

\begin{itemize}[left=0.2em, labelsep=0.5em]
    \item We propose a novel interpolation-based framework for enforcing $\ell_p$-norm mDP in continuous and fine-grained domains, bridging the gap between rigorous privacy guarantees and practical utility.
    
    \item We introduce a dimension-wise composition strategy for $\ell_p$-mDP and design a log-convex interpolation mechanism, proving its theoretical validity under $(\epsilon, d_p)$-mDP (\textbf{Theorems~\ref{thm:composition}–\ref{thm:AIPOfeasible}} and \textbf{Proposition \ref{prop:AIPOfeasible}}).
    
    \item We formulate the \emph{Anchor Perturbation Optimization (APO)} problem, which jointly optimizes anchor perturbations and privacy budget allocation, and propose a linear approximation with provable optimality bounds.
    
    \item Across mobility datasets from Rome, New York City, and London, our method shows zero observed mDP violations in fine-grained settings and consistently achieves lower utility loss than competing approaches.
\end{itemize}

The remainder of the paper is organized as follows. Section~\ref{sec:preliminary} introduces the preliminaries of mDP optimization. Section~\ref{sec:framework} outlines the overall framework, while Sections~\ref{sec:1Dinterpolation} and~\ref{sec:NDinterpolation} present the algorithmic design for one-dimensional and multi-dimensional interpolation, respectively. Section~\ref{sec:experiments} evaluates the performance of the proposed algorithm. Section~\ref{sec:relatedworks} discusses related work, and Section~\ref{sec:conclusions} concludes the paper.

\section{Preliminaries}
\label{sec:preliminary}

In this section, we introduce the \emph{formal definition} of the $\ell_p$-norm mDP (\textbf{Section \ref{subsec:optimization}}), describe the \emph{perturbation optimization framework} (\textbf{Section \ref{subsec:optimization}}), and highlight the \emph{limitations} of existing optimization-based approaches (\textbf{Section \ref{subsec:limitations}}). A summary of the key notations used throughout the paper is provided in \textbf{Appendix~\ref{sec:mathnotitions}}.

\subsection{$\ell_p$-norm mDP}
\label{subsec:mDPdef}

\emph{Local Differential Privacy (LDP)} enforces the same level of \emph{indistinguishability} for every pair of inputs, regardless of how similar they are~\cite{Duchi-FOCS2013}. \emph{Metric differential privacy} (mDP), also called \emph{Lipschitz privacy}~\cite{koufogiannis2015}, $d_{\mathcal X}$-privacy~\cite{Feyisetan-WDSM2020}, or \emph{smooth DP}~\cite{Dharangutte-AAAI2023}, generalizes this idea by tying the privacy guarantee to a distance over the input space: nearby inputs must be harder to distinguish than far-apart ones. Originally proposed for location privacy~\cite{Andres-CCS2013}, mDP naturally extends to high-dimensional continuous domains via distance-aware privacy control.

We consider a continuous (or fine-grained) secret domain $\mathcal{X} \subseteq \mathbb{R}^{N}$, where each record $\mathbf{x}_a \in \mathcal{X}$ is a vector $\mathbf{x}_a = [x_{a,1},\ldots,x_{a,N}]$. Similarity between $\mathbf{x}_a,\mathbf{x}_b \in \mathcal X$ is measured by the $\ell_p$ distance
\begin{equation}
d_p(\mathbf{x}_a,\mathbf{x}_b) \coloneqq \Big(\sum_{\ell=1}^{N} |x_{a,\ell}-x_{b,\ell}|^{p}\Big)^{1/p},
\end{equation}
which yields a flexible family of metrics (e.g., $\ell_1$, $\ell_2$, $\ell_\infty$) to capture task-specific sensitivity.

\begin{definition}[Lipschitz bound and continuity w.r.t.\ $d_p$]
Let $f:\mathcal X\to\mathbb R$ and let $d_p$ be the $\ell_p$ distance on $\mathbb R^N$ with $p\in[1,\infty]$.
\begin{itemize}
    \item \textbf{Pairwise Lipschitz bound.} We say $f$ satisfies an \emph{$(\epsilon,d_p)$-Lipschitz bound} between $\mathbf{x}_a,\mathbf{x}_b\in\mathcal X$ if
    \begin{equation} \label{eq:Lipschitzbound}
    |f(\mathbf{x}_a)-f(\mathbf{x}_b)| \le \epsilon\, d_p(\mathbf{x}_a,\mathbf{x}_b).
    \end{equation}
    \item \textbf{Lipschitz continuity.} We say $f$ is \emph{$(\epsilon,d_p)$-Lipschitz continuous} if the bound in~\eqref{eq:Lipschitzbound} holds for all $\mathbf{x}_a,\mathbf{x}_b\in\mathcal X$.
\end{itemize}
\end{definition}

% We reserve calligraphic $\mathcal M$ for randomized mechanisms $\mathcal M:\mathcal X\to \mathcal Y$. For each $\mathbf{y}\in\mathcal Y$, we apply Theorem~\ref{thm:composition} to the function $f_\mathbf{y}(\mathbf{x})\coloneqq\ln \Pr[\mathcal \mathcal{M}(\mathbf{x})=\mathbf{y}]$. 

\begin{definition}[$(\epsilon, d_p)$-metric differential privacy (mDP)] \label{def:metricDP} Let $\mathcal{X} \subseteq \mathbb{R}^N$ be the secret (input) domain and let $\mathcal{Y}$ denote the perturbation (output) domain. A randomized mechanism $\mathcal{M}: \mathcal{X} \rightarrow \mathcal{Y}$ is said to satisfy \emph{$(\epsilon, d_p)$-mDP} (or \emph{$\ell_p$-norm mDP}) if, 
\begin{itemize} 
\item for each element $\mathbf{y} \in \mathcal{Y}$, the log-probability function $f_\mathbf{y}(\mathbf{x})\coloneqq\ln \Pr\left[\mathcal{M}(\mathbf{x}) = \mathbf{y}\right]$ of the output is $(\epsilon, d_p)$-{\em Lipschitz continuous}, i.e., $\forall \mathbf{x}_a, \mathbf{x}_b \in \mathcal{X}$ 
\begin{equation} 
\label{eq:DP} 
\left|\ln \Pr\left[\mathcal{M}(\mathbf{x}_a) = \mathbf{y}\right] - \ln \Pr\left[\mathcal{M}(\mathbf{x}_b) = \mathbf{y}\right]\right| \leq \epsilon d_p(\mathbf{x}_a, \mathbf{x}_b), 
\end{equation} 
\item and $\forall \mathbf{x}\in \mathcal{X}$, the \emph{normalization constraint} of its perturbation probability is satisfied, i.e., 
\begin{equation} 
% \textstyle 
\sum_{\mathbf{y}\in \mathcal{Y}}\Pr\left[\mathcal{M}(\mathbf{x})  = \mathbf{y}\right] = 1. 
\end{equation} 
\end{itemize} 
\end{definition} 
% Intuitively, $(\epsilon, d_p)$-mDP ensures that similar inputs lead to statistically similar outputs, such that the difference in output log-probabilities is bounded by the input distance scaled by $\epsilon$. 
\paragraph{Threat Model.}
We assume an adversary that observes the obfuscated output $\mathbf y\in\mathcal Y$ released by the mechanism $\mathcal M$ and attempts to infer the user's true input $\mathbf x\in\mathcal X$. 
The adversary may have arbitrary auxiliary knowledge about the domain $\mathcal X$ (e.g., road networks, prior distributions) and the mechanism $\mathcal{M}$, but does not control $\mathcal{M}$. 
The $(\epsilon,d_p)$-mDP guarantee ensures that the likelihood ratio between any two possible inputs $\mathbf x_a,\mathbf x_b$ is bounded in proportion to their distance $d_p(\mathbf x_a,\mathbf x_b)$, thereby limiting the adversary’s ability to distinguish between nearby inputs.

% Specifically, when $p \to \infty$, the $\ell_p$-norm reduces to the \emph{maximum absolute difference}: $\|x - x'\|_{\infty} = \max_\ell |x_\ell - x'_\ell|$, and the privacy guarantee of $\ell_p$-norm MDP simplifies to the standard LDP guarantee: $\frac{\Pr[\mathcal{M}(x) = y]}{\Pr[\mathcal{M}(x') = y]} \leq e^{\epsilon \max_\ell |x_\ell - x'_\ell|} = e^{\epsilon}$. 

\subsection{Perturbation Discretization and Optimization}
\label{subsec:optimization}
While pre-defined noise mechanisms can enforce mDP over continuous domains, they do not explicitly optimize for utility. Optimization-based methods address this by minimizing utility loss, but exact solutions in continuous or fine-grained spaces are computationally intractable. To enable tractable optimization, prior work discretizes the input domain into a finite set of representative points. For instance,~\cite{Bordenabe-CCS2014, Wang-WWW2017} partition geographic regions into uniform grids and represent each cell by its centroid. Similarly,~\cite{Qiu-TMC2022, Qiu-IJCAI2024} project raw locations onto road network features (e.g., intersections and junctions) to define a discrete optimization space. \looseness = -1

Formally, let $\hat{\mathcal{X}} \subset \mathcal{X}$ denote the finite set of representative records obtained by discretizing the continuous domain $\mathcal{X}$. Let $p(\mathbf{x})$ be the prior distribution over $\mathcal{X}$, representing the probability that the secret record is located at $\mathbf{x} \in \mathcal{X}$. We consider the case that perturbation domain $\mathcal{Y}$ is discrete, and model the randomized mechanism $\mathcal{M}$ as a stochastic \emph{perturbation matrix} $\mathbf{Z} = \left\{ z(\mathbf{y}_k \mid \hat{\mathbf{x}}_i) \right\}_{(\hat{\mathbf{x}}_i, \mathbf{y}_k) \in \hat{\mathcal{X}} \times \mathcal{Y}}$, where each entry $z(\mathbf{y}_k \mid \hat{\mathbf{x}}_i)$ denotes the probability of reporting output $\mathbf{y}_k \in \mathcal{Y}$ given input $\hat{\mathbf{x}}_i \in \hat{\mathcal{X}}$, i.e., $z(\mathbf{y}_k \mid \hat{\mathbf{x}}_i) = \Pr\left[ \mathcal{M}(\hat{\mathbf{x}}_i) = \mathbf{y}_k \right]$. The Lipschitz bound constraint in Eq.~\eqref{eq:DP} is then enforced over all pairs of records in the discretized domain $\hat{\mathcal{X}}$, and corresponds to the following set of linear constraints:
\begin{equation}
\label{eq:DPdiscrete}
z(\mathbf{y}_k \mid \hat{\mathbf{x}}_i) - e^{\epsilon d_p(\hat{\mathbf{x}}_i, \hat{\mathbf{x}}_j)} z(\mathbf{y}_k \mid \hat{\mathbf{x}}_j) \leq 0, ~ \forall \hat{\mathbf{x}}_i, \hat{\mathbf{x}}_j \in \hat{\mathcal{X}}, ~\forall \mathbf{y}_k \in \mathcal{Y}.
\end{equation}
We define $\mathcal{L}(\hat{\mathbf{x}}_i,\mathbf{y}_k)$ as the utility loss incurred when the mechanism reports $\mathbf{y}_k$ while the true input lies in the region of $\mathcal{X}$ represented by the discretized point $\hat{\mathbf{x}}_i$; let $p(\hat{\mathbf{x}}_i)$ denote the prior probability that the true input lies in that region (i.e., the probability mass of the cell corresponding to $\hat{\mathbf{x}}_i$, with $\sum_i p(\hat{\mathbf{x}}_i)=1$). Then, the expected utility loss caused by the perturbation matrix $\mathbf{Z}$ can be represented by 
\begin{equation}
% \textstyle 
\mathcal{L}(\mathbf{Z}) = \sum_{\hat{\mathbf{x}}_i\in \hat{\mathcal{X}}}\sum_{\mathbf{y}_k \in \mathcal{Y}} p(\hat{\mathbf{x}}_i) z(\mathbf{y}_k|\hat{\mathbf{x}}_i) \mathcal{L}(\hat{\mathbf{x}}_i, \mathbf{y}_k).
\end{equation}
Consequently, the goal of the \emph{mDP optimization} problem is to minimize the expected utility loss $\mathcal{L}(\mathbf{Z})$ subject to the constraints imposed by mDP (\emph{Lipschitz bound} and \emph{normalization} constraints). This can be formulated as the following \emph{linear programming (LP)} problem:
\begin{eqnarray}
\label{eq:LPobjective}
\min && \mathcal{L}(\mathbf{Z}) \\
\text{s.t.} && z(\mathbf{y}_k \mid \hat{\mathbf{x}}_i) - e^{\epsilon d_p(\hat{\mathbf{x}}_i, \hat{\mathbf{x}}_j)} z(\mathbf{y}_k \mid \hat{\mathbf{x}}_j) \leq 0, \nonumber \\
&& \forall \hat{\mathbf{x}}_i, \hat{\mathbf{x}}_j \in \hat{\mathcal{X}}, \forall \mathbf{y}_k \in \mathcal{Y} ~\text{(Lipschitz bound)} \\
\label{eq:unitmeasure}
&& \sum_{\mathbf{y}_k \in \mathcal{Y}} z(\mathbf{y}_k \mid \hat{\mathbf{x}}_i) = 1, \forall \hat{\mathbf{x}}_i \in \hat{\mathcal{X}} ~\text{(Normalization)} \\
\label{eq:nonnegativity}
&& z(\mathbf{y}_k \mid \hat{\mathbf{x}}_i) \geq 0, ~\forall \hat{\mathbf{x}}_i \in \hat{\mathcal{X}}, ~\forall \mathbf{y}_k \in \mathcal{Y} % ~\text{(Non-negativity constraint)}
\end{eqnarray}
where the \textit{non-negativity constraint} in Eq.~\eqref{eq:nonnegativity} enforces that each individual probability is non-negative~\cite{probability}.

% We let $\mathbf{x}_i = [z_{i, 1}, ..., z_{i,K}]$ $(i =1,..., K)$ denote the record $\mathbf{x}_i$'s \emph{perturbation vector}, i.e., the probability distribution of its perturbed records. 

% The decision variables in the LP problem in Eq. (\ref{eq:LPobjective})-(\ref{eq:LPconstraint1}) are the perturbation matrix $\mathbf{Z}$, including $O(NK)$ decision variables (entries), constrained by $O(N^2 K)$ linear constraints. Such a high complexity makes this LP computation framework hard to apply in large-scale mDP applications~\cite{Pappachan-EDBT2023}.

\DEL{
\begin{wrapfigure}{r}{0.41\textwidth}

\begin{minipage}{0.41\textwidth}
\centering
\subfigure{
\includegraphics[width=1.00\textwidth]{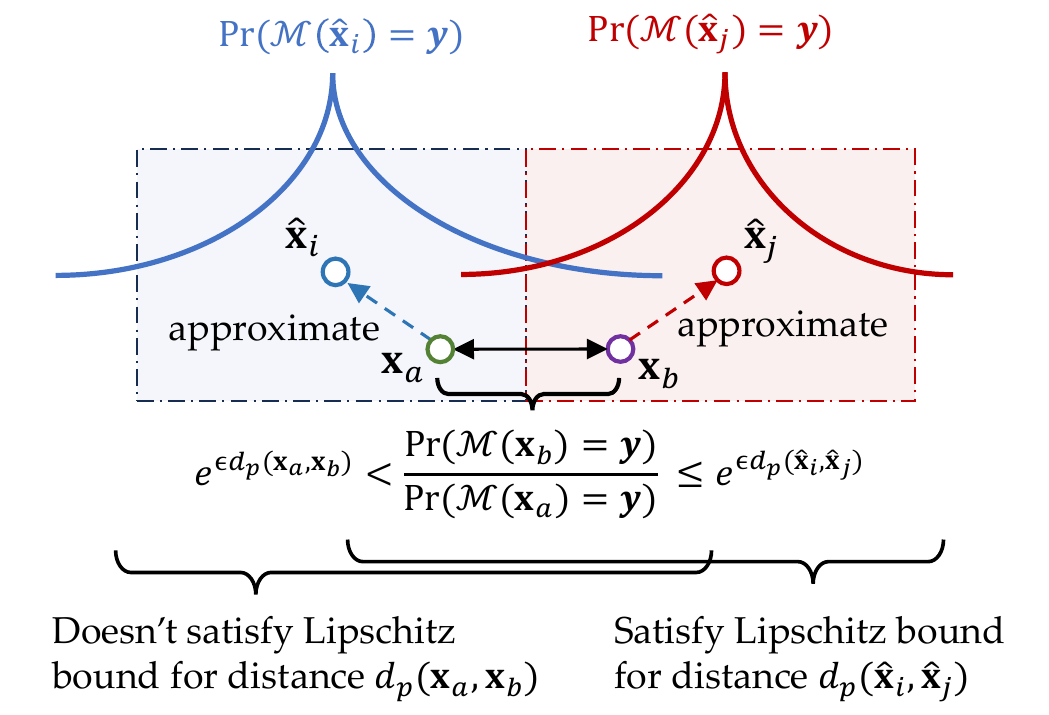}}

\caption{mDP violation due to approximated distance.}
\label{fig:DPlimit}
\end{minipage}

\end{wrapfigure}}

\subsection{Limitations of Discretization-Based mDP Enforcement} 
\label{subsec:limitations}
While discretization-based methods significantly reduce the computational cost of solving LPs, they do not ensure strict compliance with mDP in continuous domains or high-resolution settings. The key limitation arises from distance overestimation: discretized representative points (e.g., grid cell centers) may misrepresent the true distance between original inputs, especially when records lie near cell boundaries. This can lead to relaxed mDP constraints that are satisfied by the mechanism, but fail to hold for the underlying continuous/fine-grained domain. % , resulting in unintended privacy violations.

\begin{figure}[t]
\centering
\begin{minipage}{0.45\textwidth}
\centering
  \subfigure{
\includegraphics[width=1.00\textwidth]{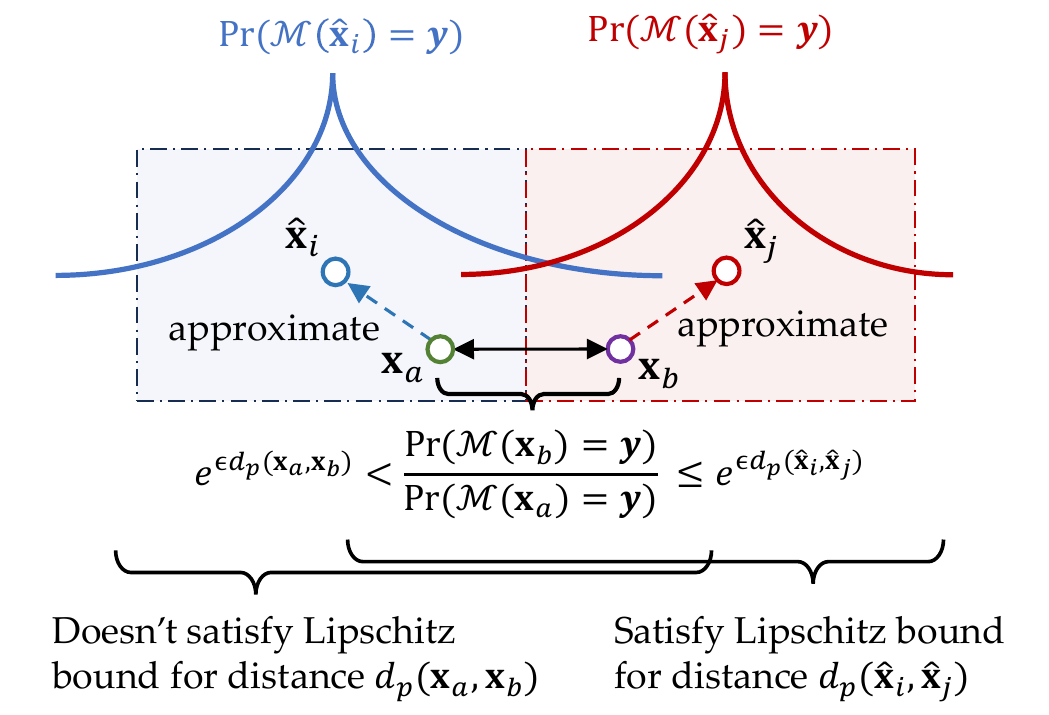}}

\caption{mDP based on approximated distance.}
\label{fig:DPlimit}
\end{minipage}

\end{figure}
Fig.~\ref{fig:DPlimit} illustrates this limitation. 
Suppose that a user moves from location $\mathbf{x}_a$ to a nearby location $\mathbf{x}_b$ in an adjacent grid cell, 
and let $\hat{\mathbf{x}}_i$ and $\hat{\mathbf{x}}_j$ denote the corresponding cell centers. 
Prior methods~\cite{Bordenabe-CCS2014, Wang-WWW2017, Pappachan-EDBT2023} discretize the domain by 
approximating each location with its cell center, so that $\Pr\!\left[\mathcal{M}(\mathbf{x}) = \mathbf{y}\right]
= \Pr\!\left[\mathcal{M}(\hat{\mathbf{x}}) = \mathbf{y}\right]$, 
and they enforce $\frac{\Pr\!\left[\mathcal{M}(\hat{\mathbf{x}}_i) = \mathbf{y}\right]}
     {\Pr\!\left[\mathcal{M}(\hat{\mathbf{x}}_j) = \mathbf{y}\right]}
\leq e^{\epsilon d_p(\hat{\mathbf{x}}_i, \hat{\mathbf{x}}_j)},
~\forall \mathbf{y} \in \mathcal{Y}$.
However, near cell boundaries one typically has 
$d_p(\hat{\mathbf{x}}_i, \hat{\mathbf{x}}_j) > d_p(\mathbf{x}_a, \mathbf{x}_b)$. 
Thus, the center-based enforcement is \emph{weaker} than the true constraint: $\frac{\Pr\!\left[\mathcal{M}(\mathbf{x}_a) = \mathbf{y}\right]}
     {\Pr\!\left[\mathcal{M}(\mathbf{x}_b) = \mathbf{y}\right]}
\leq e^{\epsilon d_p(\mathbf{x}_a, \mathbf{x}_b)}$ in Definition \ref{def:metricDP}. 
Consequently, there may exist some $\mathbf{y}$ such that
$\frac{\Pr\left[\mathcal{M}(\mathbf{x}_a) = \mathbf{y}\right]}
     {\Pr\left[\mathcal{M}(\mathbf{x}_b) = \mathbf{y}\right]}
\leq e^{\epsilon d_p(\hat{\mathbf{x}}_i, \hat{\mathbf{x}}_j)}$ but $> e^{\epsilon d_p(\mathbf{x}_a, \mathbf{x}_b)}$, which violates the Lipschitz bound between the true locations $\mathbf{x}_{a}$ and $\mathbf{x}_{b}$. 
In essence, mDP requires that small changes in the input domain induce only 
$e^{\epsilon d_p}$-bounded changes in output probabilities, 
an invariant that coarse discretization may fail to preserve.

% \looseness = -1

\section{Framework}
\label{sec:framework}

% In this section, we introduce our method to minimize utility loss while achieving mDP in a continuous secret data domain. % Given that optimizing the perturbation probabilities for every possible point in a continuous space involves an infinite number of decision variables, we propose first optimizing the perturbation probabilities for a finite subset of records and then interpolating the probabilities for the remaining records using a pre-determined function.

\begin{figure*}[t]
\centering
\begin{minipage}{1.00\textwidth}
\centering
  \subfigure{
\includegraphics[width=0.92\textwidth]{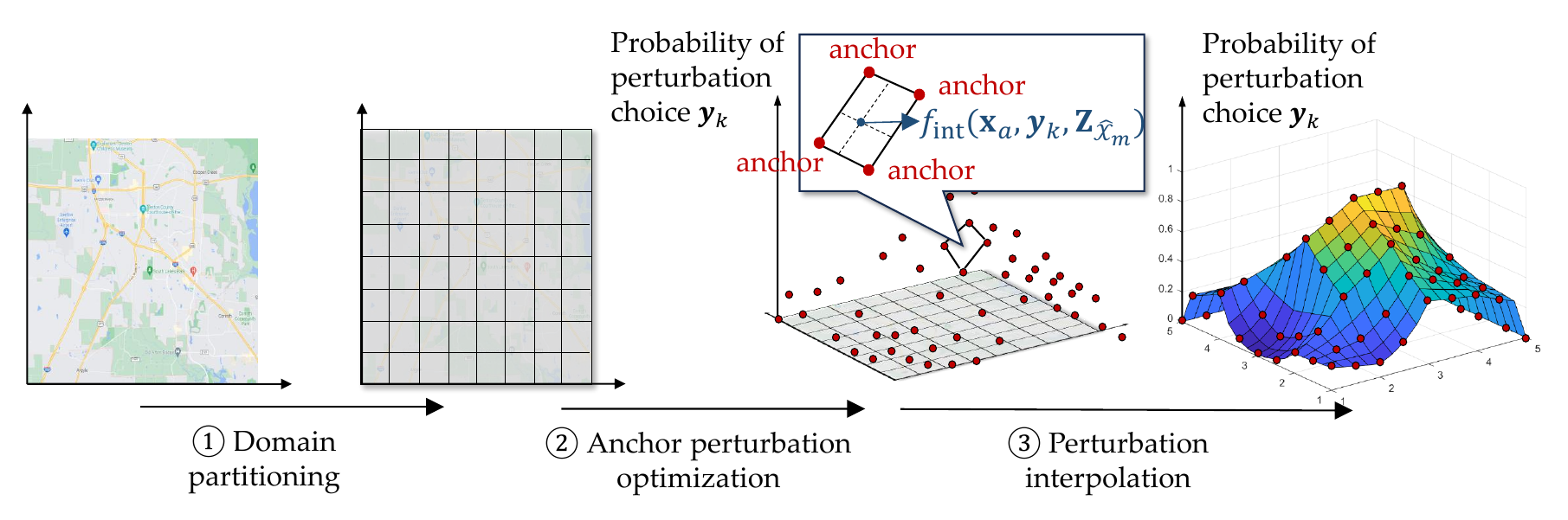}}

\caption{Framework of our method %(the anchor/interpolation cartoon) for visual intuition.
}
\label{fig:framework}
\end{minipage}

\end{figure*}

In this section, we introduce our interpolation-based framework for optimizing perturbation mechanisms under $\ell_p$-norm mDP in continuous or fine-grained secret domains. As illustrated in Fig.~\ref{fig:framework}, the framework consists of three main steps:

\begin{itemize}[left=0.7em, labelsep=0.5em]
    \item[\textbf{\textcircled{1}}] \textbf{Domain partitioning}: The secret domain $\mathcal{X}$ is partitioned into cells, with anchor records placed at each cell corner to serve as representative inputs for optimization. \looseness=-1
    
    \item[\textbf{\textcircled{2}}] \textbf{Anchor perturbation optimization}: Perturbation distributions are optimized at anchor locations to minimize expected utility loss, subject to mDP constraints over the anchor set.
    
    \item[\textbf{\textcircled{3}}] \textbf{Perturbation interpolation}: For non-anchor records, perturbation probabilities are inferred via interpolation from anchor distributions, while preserving mDP guarantees.
\end{itemize}

\noindent \textbf{Intuition behind interpolation design.} In Step~\textcircled{3}, let $\mathcal{C}_m$ denote a cell in the partitioned domain, and let $\hat{\mathcal{X}}_m$ denote the set of anchor points within $\mathcal{C}_m$. Suppose the perturbation distributions at these anchor points, denoted by $\mathbf{Z}_{\hat{\mathcal{X}}_m} = \{z(\mathbf{y}_k\mid\hat{\mathbf{x}}_j)\}_{(\hat{\mathbf{x}}_j, \mathbf{y}_k) \in \hat{\mathcal{X}}_m \times \mathcal{Y}}$, have already been optimized. The goal of the interpolation function $f_{\mathrm{int}}$ is to assign perturbation probabilities to each non-anchor record $\mathbf{x}_a \in \mathcal{C}_m$ such that mDP is satisfied between $\mathbf{x}_a$ and any other record in the domain $\mathcal{X}$. 
%\begin{equation}
%z(\mathbf{y}_k \mid \mathbf{x}_a) = f_{\mathrm{int}}(\mathbf{x}_a, \mathbf{y}_k, \mathbf{Z}_{\hat{\mathcal{X}}_m}) \quad \forall \mathbf{y}_k \in \mathcal{Y}.
%\end{equation}

As defined by the Lipschitz bound (Eq. (\ref{eq:DP}) in Definition~\ref{def:metricDP}), $(\epsilon, d_p)$-mDP requires that the \emph{log-probability} of outputs varies at most linearly with the $\ell_p$-distance between inputs. This motivates the use of \emph{log-convex interpolation}, which linearly interpolates log-probabilities to ensure smooth transitions in the output distribution. As a natural first attempt, we define the log-probability at a non-anchor point $\mathbf{x}_a \in \mathcal{C}_m$ as a convex combination of those at the anchor points: 
$$\ln z(\mathbf{y}_k \mid \mathbf{x}_a) = \sum_{\hat{\mathbf{x}} \in \hat{\mathcal{X}}_m} \lambda_{\hat{\mathbf{x}}, \mathbf{x}_a} \ln z(\mathbf{y}_k \mid \hat{\mathbf{x}}),$$ 
where the convex coefficients $\lambda_{\hat{\mathbf{x}}, \mathbf{x}_a} \geq 0$ sum to one and reflect the relative position of $\mathbf{x}_a$ within $\mathcal{C}_m$. Ideally, such interpolation would preserve mDP, i.e., if a reference point $\mathbf{x}_b$ satisfies mDP with respect to each anchor in $\hat{\mathcal{X}}_m$, then it should also satisfy mDP with respect to the interpolated point $\mathbf{x}_a$. Formally, we would desire:
\begin{equation}
\label{eq:d_pineq}
\bigl| \underbrace{\textstyle  \sum_{\hat{\mathbf{x}} \in \hat{\mathcal{X}}_m} \lambda_{\hat{\mathbf{x}}, \mathbf{x}_a} \ln z(\mathbf{y}_k \mid \hat{\mathbf{x}})}_{\text{Interpolated probability for $\mathbf{x}_a$}} - \ln z(\mathbf{y}_k \mid \mathbf{x}_b) \bigr| 
\leq \epsilon \cdot d_p(\mathbf{x}_a, \mathbf{x}_b).
\end{equation}
However, this inequality does not generally hold in high dimensions due to the convexity of the $\ell_p$-norm for $p \geq 1$. Specifically, the convex combination of anchor distances $\sum \lambda_{\hat{\mathbf{x}}, \mathbf{x}_a} d_p(\hat{\mathbf{x}}, \mathbf{x}_b)$ may exceed $d_p(\mathbf{x}_a, \mathbf{x}_b)$, leading to interpolated log-probabilities that violate the Lipschitz bound required by mDP. A formal analysis of this problem is provided in \textbf{Appendix~\ref{subsec:discussion:d_p}}. \looseness = -1

To overcome this issue, we factor the $N$-dimensional interpolation into a sequence of \emph{one-dimensional, log-convex interpolations}, applied independently along each coordinate axis. In one dimension, the Lipschitz bound in Eq.~\eqref{eq:d_pineq} holds exactly (i.e., for $N=1$), as formally established in \textbf{Propositions~\ref{prop:intral}} and \textbf{\ref{prop:across}}. We then construct the $N$-dimensional mechanism by multiplicatively composing the
per-axis interpolators in \textbf{Definition \ref{def:logconvex_nd}} and normalizing to obtain a valid joint perturbation
distribution in \textbf{Definition \ref{def:PPI}}. The correctness of this composition is established by \textbf{Theorem~\ref{thm:AIPOfeasible}} and \textbf{Proposition~\ref{prop:AIPOfeasible}}, which respectively show that the resulting mechanism is \emph{$(\epsilon,d_p)$-Lipschitz continuous} and satisfies \emph{$(\epsilon,d_p)$-mDP} over the entire $N$-dimensional domain $\mathcal{X}$.

Next, we first introduce the one-dimensional interpolation primitive in \textbf{Section~\ref{sec:1Dinterpolation}}, and then extend this construction to the multi-dimensional setting in \textbf{Section~\ref{sec:NDinterpolation}}, following the three-step procedure outlined in Fig.~\ref{fig:framework} (Steps~\textcircled{1}–\textcircled{3}).

\begin{figure*}[t]
\centering
\hspace{0.00in}
\begin{minipage}{1.00\textwidth}
\centering
  \subfigure[Intra-interval validity]{
\includegraphics[width=0.375\textwidth, height = 0.175\textheight]{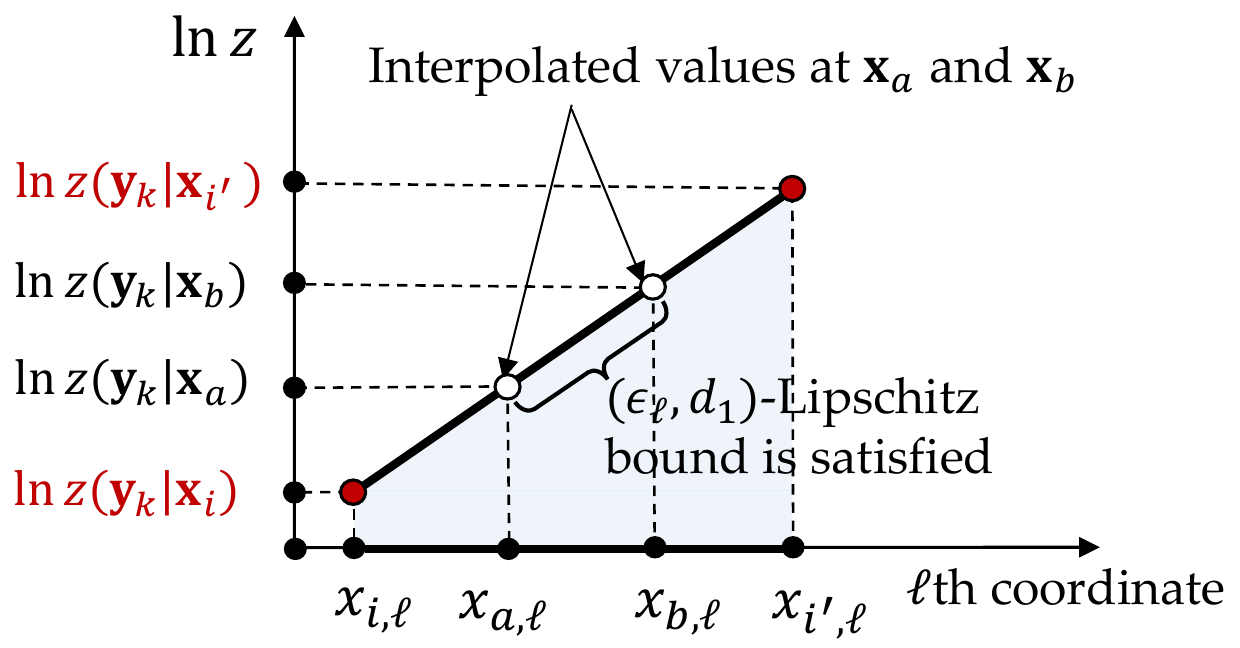}}
\hspace{0.2in}
  \subfigure[Across-interval validity]{
\includegraphics[width=0.47\textwidth, height = 0.19\textheight]{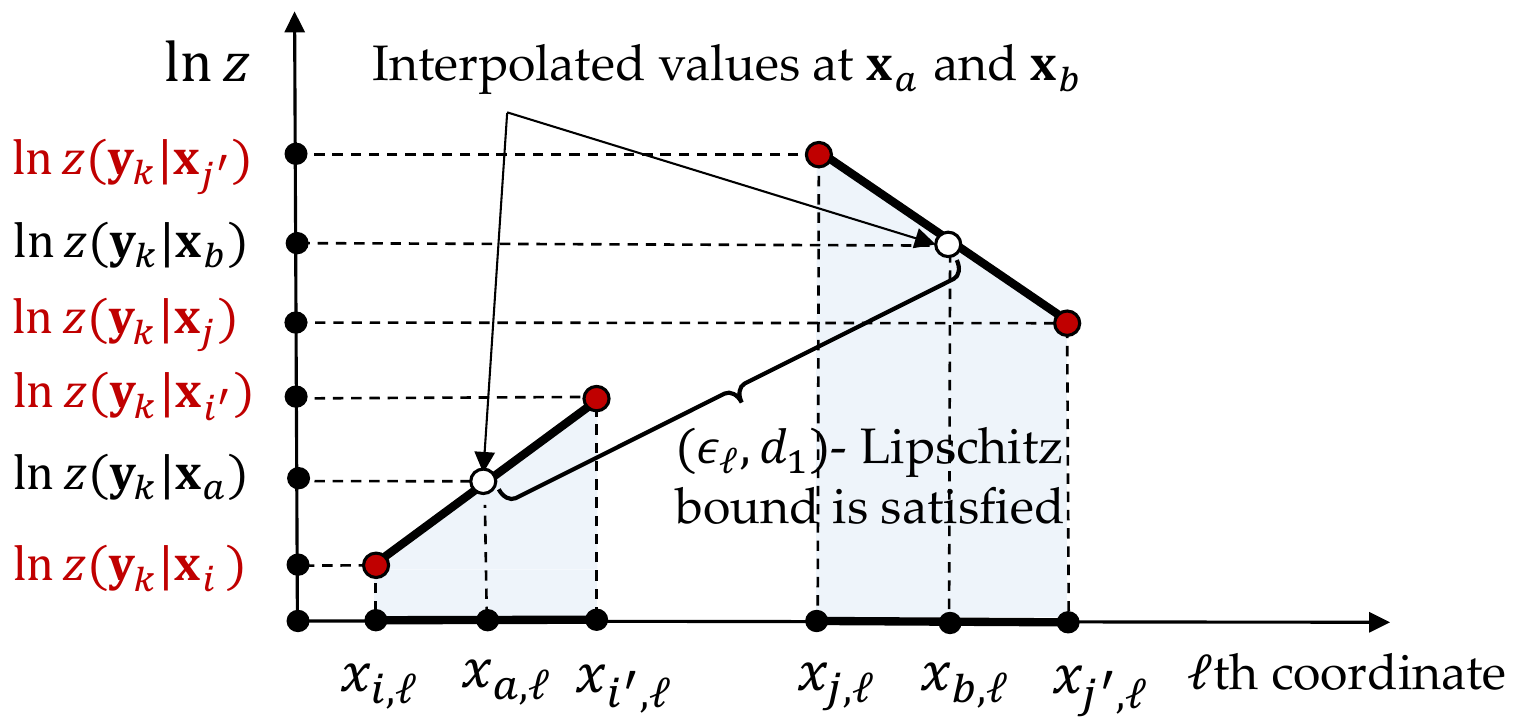}}

\end{minipage}
\caption{Illustration of Proposition \ref{prop:intral} and Proposition \ref{prop:across}.}
\label{fig:property}

\end{figure*}

% Next, we detail the above three steps. Although perturbation interpolation is presented as the third step, it fundamentally constrains the design of the previous two: the structure of the partition must support valid interpolation, and the optimized anchor perturbations must satisfy per-dimension mDP constraints to ensure privacy-preserving interpolation. Therefore, for clarity and logical flow, we first introduce the interpolation function (Step \textcircled{3}) in \textbf{Section \ref{subsec:interpolation}}, then present the domain partitioning (Step \textcircled{1}) in \textbf{Section \ref{subsec:partitioning}}  and the optimization strategy (Step \textcircled{2}) in \textbf{Section \ref{subsec:anchoropt}}. \looseness = -1

% \subsection{Interpolation Function Design (Step \textcircled{3})}
% \label{subsec:interpolation}
\section{One-Dimensional Interpolation and Privacy Composition}
\label{sec:1Dinterpolation}

\DEL{
\begin{equation}
\epsilon_\ell^{\frac{p}{p-1}} = \frac{1}{N} \epsilon^{\frac{p}{p-1}}
\end{equation}
\begin{equation}
\epsilon_\ell = N^{\frac{p-1}{p}} \epsilon
\end{equation}
\begin{equation}
|\ln x_\ell - \ln x'_\ell| \leq \epsilon_\ell d_{x_\ell, x'_\ell} = N^{\frac{p-1}{p}} \epsilon d_{x_\ell, x'_\ell}
\end{equation}}

% We define the one-dimensional log-convex interpolation method: 

\begin{definition}[\textbf{One-Dimensional Log-Convex Interpolation}]
\label{def:logconvex_1d}
Let $\mathbf{x}_i, \mathbf{x}_{i'} \in \mathcal{X}$ be two records that differ only in their $\ell$th coordinate, with $x_{i', \ell} = x_{i, \ell} + \Delta_\ell$ for some $\Delta_\ell > 0$. For any intermediate point $\mathbf{x}_a$ such that $x_{a,\ell} \in [x_{i,\ell}, x_{i',\ell}]$ and all other coordinates match $\mathbf{x}_i$ and $\mathbf{x}_{i'}$, define the convex coefficient $\lambda_{\mathbf{x}_i, \mathbf{x}_a}^{\ell}= \frac{x_{i',\ell} - x_{a,\ell}}{x_{i',\ell} - x_{i,\ell}}$. Then, the log-probability at $\mathbf{x}_a$ is given by the log-convex interpolation:
\begin{equation}
\label{eq:logconvex}
\ln z(\mathbf{y}_k \mid \mathbf{x}_a) = \lambda_{\mathbf{x}_i, \mathbf{x}_a}^{\ell} \ln z(\mathbf{y}_k \mid \mathbf{x}_i) + (1 - \lambda_{\mathbf{x}_i, \mathbf{x}_a}^{\ell}) \ln z(\mathbf{y}_k \mid \mathbf{x}_{i'}),
\end{equation}
which is written as $z(\mathbf{y}_k \mid \mathbf{x}_a) \lnconv \left(z(\mathbf{y}_k \mid \mathbf{x}_i),  z(\mathbf{y}_k \mid \mathbf{x}_{i'})\right)$. 
\end{definition}

\noindent
In \textbf{Propositions~\ref{prop:intral}} and \textbf{\ref{prop:across}}, we prove that the interpolation mechanism preserves $(\epsilon_\ell, d_1)$-Lipschitz continuity; that is, it is $(\epsilon_\ell, d_1)$-Lipschitz within each one-dimensional interval and across adjacent intervals.
\begin{proposition}[\textbf{Intra-Interval Validity}]
\label{prop:intral}
Let $\mathbf{x}_i$ and $\mathbf{x}_{i'}$ be two records that differ only in their $\ell$th coordinate, with $x_{i, \ell} < x_{i', \ell}$, and suppose their corresponding log-perturbation probabilities $\ln z(\mathbf{y}_k \mid \mathbf{x}_i)$ and $\ln z(\mathbf{y}_k \mid \mathbf{x}_{i'})$ satisfy the $(\epsilon, d_1)$-Lipschitz bound. Then, for any $\mathbf{x}_a, \mathbf{x}_b \in \mathcal{X}$ such that $x_{a, \ell}, x_{b, \ell} \in [x_{i, \ell}, x_{i', \ell}]$ and all other coordinates are identical to those of $\mathbf{x}_i$, if the interpolated values $\hat{z}(\mathbf{y}_k \mid \mathbf{x}_a)$ and $\hat{z}(\mathbf{y}_k \mid \mathbf{x}_b)$ are calculated by 
\begin{align}
\hat{z}(\mathbf{y}_k \mid \mathbf{x}_a) \lnconv\left(z(\mathbf{y}_k \mid \mathbf{x}_i), z(\mathbf{y}_k \mid \mathbf{x}_{i'})\right), \\
\hat{z}(\mathbf{y}_k \mid \mathbf{x}_b)\lnconv\left(z(\mathbf{y}_k \mid \mathbf{x}_i), z(\mathbf{y}_k \mid \mathbf{x}_{i'})\right),
\end{align}
then $\ln \hat{z}(\mathbf{y}_k \mid \mathbf{x}_a)$ and $\ln \hat{z}(\mathbf{y}_k \mid \mathbf{x}_b)$ also satisfy the $(\epsilon, d_1)$-Lipschitz bound between $\mathbf{x}_a$ and $\mathbf{x}_b$.
This property is illustrated in Fig.~\ref{fig:property}(a).
\end{proposition}

\begin{proposition}[\textbf{Across-Interval Validity}]
\label{prop:across}
Let $\mathbf{x}_i, \mathbf{x}_{i'}, \mathbf{x}_j, \mathbf{x}_{j'}$ be four records that differ only in their $\ell$th coordinate, with $x_{i, \ell} < x_{i', \ell} \leq x_{j, \ell} < x_{j', \ell}$. Suppose that each pair of log-perturbation probabilities $\ln z(\mathbf{y}_k \mid \mathbf{x}_i)$, $\ln z(\mathbf{y}_k \mid \mathbf{x}_{i'})$, $\ln z(\mathbf{y}_k \mid \mathbf{x}_j)$, and $\ln z(\mathbf{y}_k \mid \mathbf{x}_{j'})$ satisfy $(\epsilon, d_1)$-Lipschitz bound.

Let $\mathbf{x}_a$ and $\mathbf{x}_b$ be two additional records that differ from the above four points only in the $\ell$th coordinate, with $x_{a,\ell} \in [x_{i,\ell}, x_{i',\ell}]$ and $x_{b,\ell} \in [x_{j,\ell}, x_{j',\ell}]$. If the corresponding interpolated values $\hat{z}(\mathbf{y}_k \mid \mathbf{x}_a)$ and $\hat{z}(\mathbf{y}_k \mid \mathbf{x}_b)$ are defined via log-convex interpolation as:
\begin{eqnarray}
&& \hat{z}(\mathbf{y}_k \mid \mathbf{x}_a) \lnconv \left(z(\mathbf{y}_k \mid \mathbf{x}_i), z(\mathbf{y}_k \mid \mathbf{x}_{i'})\right) \\
&&  
\hat{z}(\mathbf{y}_k \mid \mathbf{x}_b) \lnconv \left(z(\mathbf{y}_k \mid \mathbf{x}_j), z(\mathbf{y}_k \mid \mathbf{x}_{j'})\right),
\end{eqnarray}
then the pair $\left(\hat{z}(\mathbf{y}_k \mid \mathbf{x}_a), \hat{z}(\mathbf{y}_k \mid \mathbf{x}_b)\right)$ satisfies $(\epsilon, d_1)$-Lipschitz bound between $\mathbf{x}_a$ and $\mathbf{x}_b$. This property is illustrated in Fig.~\ref{fig:property}(b).
\end{proposition}

% \textbf{Propositions \ref{prop:intral}} and \textbf{\ref{prop:across}} establish that one-dimensional log-convex interpolation preserves $(\epsilon_\ell, d_1)$--Lipschitz bound along dimension $\ell$. While in multi-dimensional domains, inputs can differ along multiple coordinates simultaneously. To extend our interpolation framework to these settings while preserving $(\epsilon, d_p)$--Lipschitz bound, we need to coordinate the privacy leakage along each dimension. \textbf{Theorem~\ref{thm:composition}} formalizes this requirement by characterizing how per-coordinate privacy budgets $\{\epsilon_\ell\}_{\ell=1}^N$ can be composed to ensure that the overall mechanism satisfies $(\epsilon, d_p)$--Lipschitz bound, providing the theoretical foundation for building high-dimensional interpolations.

\textbf{Propositions~\ref{prop:intral}} and \textbf{\ref{prop:across}} establish that the one-dimensional log-convex interpolation is $(\epsilon_\ell, d_1)$-Lipschitz along coordinate $\ell$. In multi-dimensional domains, however, inputs typically differ along multiple coordinates simultaneously. To extend our interpolation framework while preserving an overall $(\epsilon, d_p)$-Lipschitz guarantee, we must coordinate the privacy leakage across dimensions via per-coordinate budgets $\{\epsilon_\ell\}_{\ell=1}^N$. \textbf{Theorem~\ref{thm:composition}} formalizes this by specifying how the $\{\epsilon_\ell\}$ compose to yield a mechanism that satisfies the $(\epsilon, d_p)$-Lipschitz condition on $\mathcal{X}$, thereby providing the theoretical foundation for our high-dimensional interpolation scheme.

% To ensure that the overall mechanism satisfies the original $\ell_p$-norm mDP constraint, we introduce \textbf{Theorem~\ref{thm:composition}}, which formalizes how the global privacy budget $\epsilon$ can be distributed across dimensions via per-axis budgets $\{\epsilon_\ell\}_{\ell=1}^N$ under a dimension-wise composition rule.

% a key theoretical result that underpins our design. Specifically, we show that if a mechanism satisfies $(\epsilon_\ell, d_1)$-mDP along each individual dimension, then it satisfies an overall $(\epsilon, d_p)$-mDP guarantee when the coordinate-specific privacy budgets $\epsilon_\ell$ satisfy a suitable composition condition. This \emph{dimension-wise composition property}, formalized in \textbf{Theorem~\ref{thm:composition}}, provides the theoretical foundation for both the \emph{anchor perturbation optimization} (Step \textcircled{2}) and the \emph{interpolation design} (Step \textcircled{3}) of our method. \looseness = -1

% \begin{boxedtheorem}
\begin{theorem}
[\textbf{Dimension-wise Composition for Lipschitz Bound Condition}]
\label{thm:composition}
Let $f:\mathcal X\to\mathbb R$ be a mechanism that interpolates values in an $N$-dimensional space. Suppose that for each $\ell \in \{1, \dots, N\}$, $f$ satisfies $(\epsilon_\ell, d_1)$-Lipschitz bound when the input records differ only in the $\ell$th coordinate. If the parameters $\epsilon_1, \dots, \epsilon_N$ satisfy the following budget composition condition:
\begin{eqnarray}
\label{eq:budgetcompo1}
\textstyle
&& \sum_{\ell=1}^{N} \epsilon_\ell^{\frac{p}{p-1}} \leq \epsilon^{\frac{p}{p-1}}, \quad \text{for } p > 1, 
\\  \label{eq:budgetcompo2}
\text{and} && 
\max_{\ell} \epsilon_\ell \leq \epsilon, \quad \text{for } p = 1,
\end{eqnarray}
then $f$ is $(\epsilon, d_p)$-Lipschitz continuous.
\end{theorem}
% \end{boxedtheorem}

\begin{proof}[Proof Sketch]
Fix $\mathbf{x}_a,\mathbf{x}_b\in\mathcal{X}$ and let $\Delta=\mathbf{x}_b-\mathbf{x}_a$, where $\boldsymbol{\Delta} = [\Delta_1, \ldots, \Delta_N]$.
Construct an axis-aligned path that updates one coordinate at a time:
$\mathbf{x}^{(0)}=\mathbf{x}_a$ and 
$\mathbf{x}^{(\ell)}=\mathbf{x}^{(\ell-1)}+\Delta_{\ell}\mathbf{e}_{\ell}$, 
where $\mathbf{e}_{\ell}$ is the $\ell$th basis vector.
Given that each one-dimensional step satisfies,
for all $\mathbf{y}_k \subseteq\mathcal{Y}$,

\begin{equation}
\left|f(\mathbf{x}^{(\ell)}) - f(\mathbf{x}^{(\ell-1)})\right|
\leq  \epsilon_{\ell}|\Delta_{\ell}|.
\end{equation}
Summing these bounds along the path gives

\begin{equation}
\textstyle \left|f(\mathbf{x}_b) - f(\mathbf{x}_a)\right|
\leq \sum_{\ell=1}^N \epsilon_\ell|\Delta_\ell|.
\end{equation}
Let $p\in[1,\infty]$ and $q$ be its dual exponent ($1/p+1/q=1$). By Hölder’s inequality, $\sum_{\ell=1}^N \epsilon_\ell|\Delta_\ell|
\leq  \|\boldsymbol{\epsilon}\|_{q}\|\Delta\|_{p}$. 
Setting $\epsilon := \|\boldsymbol{\epsilon}\|_{q}$ and 
$d_p(\mathbf{x}_a,\mathbf{x}_b) := \|\mathbf{x}_b-\mathbf{x}_a\|_{p}$ yields $\sum_{\ell=1}^N \epsilon_\ell|\Delta_\ell| \leq \epsilon  d_p(\mathbf{x}_a,\mathbf{x}_b)$, and hence $\left|f(\mathbf{x}_b) - f(\mathbf{x}_a)\right| \leq \epsilon  d_p(\mathbf{x}_a,\mathbf{x}_b)$, which is the desired $(\epsilon,d_p)$-Lipschitz bound.
A complete proof is provided in \textbf{Appendix~\ref{subsec:proof:thm:composition}}.
\end{proof}

\noindent \textbf{Remark.} \textbf{Theorem~\ref{thm:composition}} ensures preservation of coordinate-wise Lipschitz bounds but does not guarantee that the interpolated values form a valid probability distribution
(i.e., summing to one over $\mathcal{Y}$). In applications where $\mathcal{M}$ represents a data perturbation mechanism (where the normalization constraint is required), we therefore apply a normalization step after interpolation to restore validity. In our construction, normalization can increase the effective Lipschitz bound (and thus the mDP budget) by up to a factor of $2$, which will be further discussed in \textbf{Section \ref{subsec:step3}}.  

\medskip
In \textbf{Corollary \ref{cor:mDP}}, we instantiate \textbf{Theorem~\ref{thm:composition}} for a
normalized data-perturbation mechanism $\mathcal M:\mathcal X\to \mathcal Y$. 
\begin{corollary}[\textbf{Dimension-wise Composition for mDP}]
\label{cor:mDP}
Let $\mathcal{M}$ be a mechanism that perturbs data in an $N$-dimensional space. Suppose that for each $\ell \in \{1,\dots,N\}$,  $\mathcal{M}$ satisfies $(\epsilon_\ell, d_1)$-mDP when the input records differ only in the $\ell$th coordinate. If the parameters $\epsilon_1,\dots,\epsilon_N$ satisfy the budget composition condition defined in Eqs. (\ref{eq:budgetcompo1})(\ref{eq:budgetcompo2}), then $\mathcal{M}$ is $(\epsilon, d_p)$-mDP.
\end{corollary}

\paragraph{Discussion: Multi-attribute LDP as a special case of $\ell_p$-norm mDP.}
Consider $\mathcal X \subseteq \mathcal A^N$ with Hamming distance 
$d_H(\mathbf{x}_a,\mathbf{x}_b) = \sum_{\ell=1}^N \mathbf 1[x_{a,\ell} \neq x_{b,\ell}]$.
A mechanism $\mathcal M$ satisfies \emph{multi-attribute $\epsilon$-LDP} \cite{Arcolezi-VLDB2023} if, for all 
$\mathbf{x}_a,\mathbf{x}_b\in\mathcal X$ and $\mathbf y\in\mathcal Y$,
\begin{equation}
\Pr[\mathcal M(\mathbf{x}_a)=\mathbf y] \le 
e^{\epsilon  d_H(\mathbf{x}_a,\mathbf{x}_b)} 
\Pr[\mathcal M(\mathbf{x}_b)=\mathbf y],
\end{equation}
which is exactly $(\epsilon,d_H)$-mDP. 

Moreover, multi-attribute LDP can be seen as a limiting case of $\ell_p$-mDP when $p\to\infty$ on binary domains. 
Assume $\mathcal X \subseteq \{0,1\}^N$ (or more generally, that each coordinate difference is at most $1$ after rescaling). Then
\begin{equation}
d_p(\mathbf x_a,\mathbf x_b) = \|\mathbf x_a - \mathbf x_b\|_\infty 
= \max_\ell |x_{a,\ell} - x_{b,\ell}|
=\begin{cases}
1, & \mathbf x_a \neq \mathbf x_b,\\
0, & \mathbf x_a = \mathbf x_b,
\end{cases}
\end{equation}
and per-coordinate distances reduce to
\begin{equation}
d_1(x_{a,\ell},x_{b,\ell}) = |x_{a,\ell}-x_{b,\ell}| 
=\begin{cases}
1, & x_{a,\ell}\neq x_{b,\ell},\\
0, & x_{a,\ell}=x_{b,\ell}.
\end{cases}
\end{equation}
Under this metric, the pointwise mDP guarantee
\begin{equation}
\Pr[\mathcal M(\mathbf x_a)=\mathbf y]
\le  e^{\epsilon d_p(\mathbf x_a,\mathbf x_b)} \Pr[\mathcal M(\mathbf x_b)=\mathbf y]
\end{equation}
simplifies to the standard LDP bound
\begin{equation}
\Pr[\mathcal M(\mathbf x_a)=\mathbf y]
\le e^{\epsilon}\Pr[\mathcal M(\mathbf x_b)=\mathbf y]
\end{equation}
whenever $\mathbf x_a \neq \mathbf x_b$ (and equals $1$ when $\mathbf x_a=\mathbf x_b$). 

Finally, as $p \to \infty$, the exponent $p/(p-1)$ in the budget composition rule in Eqs. (\ref{eq:budgetcompo1})(\ref{eq:budgetcompo2}) tends to $1$, and hence  
\begin{equation}
\sum_{\ell=1}^N \epsilon_\ell^{\frac{p}{p-1}} \leq  
\epsilon^{\frac{p}{p-1}}
\quad \Longrightarrow \quad 
\sum_{\ell=1}^N \epsilon_\ell \leq  \epsilon,
\end{equation}
which is exactly the familiar sequential composition condition. 
Thus, $\ell_p$-norm mDP recovers multi-attribute LDP in the $p\to\infty$ limit on binary domains (equivalently, under the Hamming metric).

\section{Multi-Dimensional Interpolation}
\label{sec:NDinterpolation}

Having established one-dimensional interpolation and the corresponding privacy composition property in Section~\ref{sec:1Dinterpolation}, we now turn to extending our framework to multi-dimensional domains. Directly applying log-convex interpolation in higher dimensions can violate $(\epsilon, d_p)$-mDP due to the geometric properties of $\ell_p$-norms. Based on \textbf{Theorem \ref{thm:composition}}, we adopt a coordinate-wise approach: we interpolate along each dimension and then combine the results using a carefully designed composition rule to ensure global privacy guarantees. For clarity, we present the anchor perturbation optimization (Step~\textcircled{2}) before introducing the interpolation function (Step~\textcircled{3}), as the optimization depends on the specific interpolation structure.

\subsection{Step \textcircled{1} - Domain Partitioning} 
\label{subsec:partitioning}

According to \textbf{Theorem~\ref{thm:composition}}, $\ell_p$-norm mDP can be enforced by bounding privacy leakage separately along each coordinate. To support this dimension-wise composition, we partition the secret domain $\mathcal{X}$ into $M$ non-overlapping $N$-orthotopes $\mathcal{C}_1, \ldots, \mathcal{C}_M$, axis-aligned hyperrectangles that generalize rectangles to $N$ dimensions. This coordinate-aligned structure ensures that neighboring anchors differ in only one dimension, enabling efficient log-convex interpolation while preserving $(\epsilon, d_p)$-mDP. % In addition, orthotopic partitioning guarantees full domain coverage, supports scalable optimization, and avoids geometric inconsistencies introduced by non-axis-aligned divisions.

% More precisely, let $i_m = [i_1, \ldots, i_N]$ denote the index vector of an $N$-orthotope $\mathcal{C}_m$. We define its base corner as $\hat{\mathbf{x}}_{i_m} = [\hat{x}_{i_1,1}, \ldots, \hat{x}_{i_N,N}]$, and let $\Delta_\ell$ denote the side length along dimension $\ell$, forming the side-length vector $\boldsymbol{\Delta} = [\Delta_1, \ldots, \Delta_N]$. The full set of $2^N$ corner points of $\mathcal{C}_m$ is then given by: $\hat{\mathcal{X}}_m = \left\{ \hat{\mathbf{x}}_{i_m} + \boldsymbol{\gamma} \odot \boldsymbol{\Delta} \;\middle|\; \boldsymbol{\gamma} \in \{0,1\}^N \right\}$, where $\boldsymbol{\gamma}$ is a binary indicator vector and $\odot$ denotes element-wise multiplication. Each corner has coordinate $\hat{x}_{i_\ell, \ell} + \gamma_\ell \Delta_\ell$ along dimension $\ell$. We refer to the set $\hat{\mathcal{X}}_m$ as the \emph{anchor set} of $\mathcal{C}_m$, and define the full anchor set across all orthotopes as $\hat{\mathcal{X}} = \bigcup_m \hat{\mathcal{X}}_m$, which serves as the representative support for interpolation and optimization.

More precisely, for each $N$-orthotope $\mathcal{C}_m$, we let $\hat{\mathbf{x}}_{i_m} = [\hat{x}_{i_m,1}, \ldots, \hat{x}_{i_m,N}]$ denote its base (minimum) corner, and let $\boldsymbol{\Delta} = [\Delta_1, \ldots, \Delta_N]$ represent the side lengths along each dimension. The full set of $2^N$ corner points of $\mathcal{C}_m$ is given by: 
$\hat{\mathcal{X}}_m = \left\{ \hat{\mathbf{x}}_{i_m} + \boldsymbol{\gamma} \odot \boldsymbol{\Delta} \;\middle|\; \boldsymbol{\gamma} \in \{0,1\}^N \right\}$, where $\boldsymbol{\gamma}$ is a binary indicator vector and $\odot$ denotes element-wise multiplication. Each corner point thus has coordinate $\hat{x}_{i_m,\ell} + \gamma_\ell \Delta_\ell$ along dimension $\ell$. We refer to $\hat{\mathcal{X}}_m$ as the \emph{anchor set} of $\mathcal{C}_m$, and define the full anchor set as $\hat{\mathcal{X}} = \bigcup_m \hat{\mathcal{X}}_m$, which serves as the support for interpolation and optimization. \looseness = -1

\begin{definition}[Axis-Aligned Anchor Neighbors]
\label{def:anchorneighbor}
Let $\hat{\mathbf{x}}_i, \hat{\mathbf{x}}_j \in \hat{\mathcal{X}}_m$ be two anchor points within the same $N$-orthotope $\mathcal{C}_m$. We say that $\hat{\mathbf{x}}_i$ and $\hat{\mathbf{x}}_j$ are \textbf{$\ell$-axis neighbors} (or simply \textbf{$\ell$-neighbors}) if they differ only along the $\ell$th coordinate; that is, $\hat{x}_{i,\ell} \neq \hat{x}_{j,\ell}$ and $\hat{x}_{i,\ell'} = \hat{x}_{j,\ell'}$ for all $\ell' \neq \ell$.
\end{definition}

\subsection{Step \textcircled{3} - Perturbation Interpolation}
\label{subsec:step3}

% With the domain partitioned into $N$-orthotopes in Step~\textcircled{1}, we now describe how to interpolate perturbation distributions for non-anchor records using the optimized anchor points. For each record located within a cell, we combine the perturbation probabilities from its $2^N$ corner anchors using a structured interpolation scheme. This scheme applies log-convex interpolation along each dimension separately and merges the results into a single distribution using a multiplicative form. % The resulting mechanism yields smooth, consistent transitions over the domain while maintaining rigorous $(\epsilon, d_p)$-mDP guarantees.  The interpolation function is formally defined in \textbf{Definition~\ref{def:PPI}}, with privacy preservation established in \textbf{Theorem~\ref{thm:AIPOfeasible}}. \looseness = -1

With the domain partitioned into $N$-orthotopes (Step~\textcircled{1}), we interpolate perturbation distributions for non-anchor records using their corner anchors. For a record inside a cell, we first construct an \emph{unnormalized} multi-dimensional log-convex interpolant $f_{\mathrm{int}}$ by applying log-convex interpolation \emph{separately} along each coordinate and composing the results multiplicatively (\textbf{Definition~\ref{def:logconvex_nd}}).
Because these interpolated values might not sum to one over $\mathcal{Y}$, we then obtain a valid probability distribution by \emph{normalizing} them, yielding the \emph{normalized} multi-dimensional log-convex interpolation $\overline{f}_{\mathrm{int}}$ (\textbf{Definition~\ref{def:PPI}}).
The \emph{Lipschitz continuity} of the interpolated (log-)values is established by \textbf{Theorem~\ref{thm:AIPOfeasible}}, and the $(\epsilon,d_p)$-mDP guarantee for the normalized mechanism follows from \textbf{Proposition~\ref{prop:AIPOfeasible}}.

\begin{definition}[\textbf{Unnormalized Multi-Dimensional Log-Convex Interpolation $f_{\mathrm{int}}$}]
\label{def:logconvex_nd}
Let $\hat{\mathcal{X}}_m$ denote the set of $2^N$ anchor points at the corners of an $N$-dimensional orthotope $\mathcal{C}_m$, and let $\mathbf{Z}_{\hat{\mathcal{X}}_m}$ represent their corresponding perturbation distributions. For any point $\mathbf{x}_a \in \mathcal{C}_m$ and output $\mathbf{y}_k \in \mathcal{Y}$, the interpolated value is defined as:

\begin{equation}
\label{eq:logconvex_nd_}
\ln f_{\mathrm{int}}(\mathbf{x}_a, \mathbf{y}_k, \mathbf{Z}_{\hat{\mathcal{X}}_m}) = 
\sum_{\boldsymbol{\gamma} \in \{0,1\}^N} w(\boldsymbol{\gamma}) \ln z\big(\mathbf{y}_k \mid \hat{\mathbf{x}}_{i_m} + \boldsymbol{\gamma} \odot \boldsymbol{\Delta}\big),
\end{equation}
or equivalently: 
\begin{equation}
\label{eq:logconvex_nd}
\textstyle  f_{\mathrm{int}}(\mathbf{x}_a, \mathbf{y}_k, \mathbf{Z}_{\hat{\mathcal{X}}_m}) = 
\displaystyle\prod_{\boldsymbol{\gamma} \in \{0,1\}^N} z\big(\mathbf{y}_k \mid \hat{\mathbf{x}}_{i_m} + \boldsymbol{\gamma} \odot \boldsymbol{\Delta}\big)^{w(\boldsymbol{\gamma})}.
\end{equation}
Here, the weight function $w(\boldsymbol{\gamma})$, defined as 
\begin{equation}
\textstyle 
w(\boldsymbol{\gamma}) = \prod_{\ell=1}^N \left[(1 - \gamma_\ell)\lambda_{\hat{\mathbf{x}}_{i_m}, \mathbf{x}_a}^\ell + \gamma_\ell \left(1 - \lambda_{\hat{\mathbf{x}}_{i_m}, \mathbf{x}_a}^\ell\right)\right], 
\end{equation}
represents how each anchor point’s distribution contributes to the interpolated distribution at the non-anchor point $\mathbf{x}_a$, with the weights reflecting the relative position of the point within its cell. 
\end{definition}

% \begin{boxedtheorem}
\begin{theorem}
[Correctness of Log-Convex Interpolation $f_{\mathrm{int}}$]
\label{thm:AIPOfeasible}
Given that $(\epsilon_\ell, d_1)$-Lipschitz bound holds between each pair of $\ell$-neighbors in $\hat{\mathcal{X}}$ and $\{\epsilon_\ell\}_{\ell=1}^N$ satisfy the privacy budget composition condition formalized in Eq. (\ref{eq:budgetcompo1})(\ref{eq:budgetcompo2}), the use of the interpolation function $f_{\mathrm{int}}$ (defined by Eq.~(\ref{eq:logconvex_nd_})) guarantees that any two interpolated values within the entire secret data domain $\mathcal{X}$ satisfy $(\epsilon, d_p)$-Lipschitz bound.
\end{theorem}
%\begin{theorem}
%\label{thm:AIPOfeasible}
%Given that $(\epsilon, \xi_p)$-mDP between each pair of anchors $\mathbf{x}_i, \mathbf{x}_j \in \hat{\mathcal{X}}$ are satisfied, then using the interpolation function $f_{\mathrm{int}}$ (defined by Eq. (\ref{eq:PPI})) ensures that any two interporlated records in $\mathcal{X}$ satisfy $(\epsilon, d_p)$-mDP. 
%\end{theorem}
% \end{boxedtheorem}
\begin{proof}[Proof Sketch]
% We aim to show that the interpolation function $f_{\mathrm{int}}$ yields a mechanism that satisfies $(\epsilon, d_p)$-mDP over the entire domain $\mathcal{X}$. The key idea is to verify that, for any pair of records $\mathbf{x}_a, \mathbf{x}_b \in \mathcal{X}$, the interpolated perturbation distributions satisfy the Lipschitz constraint defined in \textbf{Definition~\ref{def:metricDP}}.
We begin by proving that if two records $\mathbf{x}_a, \mathbf{x}_b \in \mathcal{X}$ differ only in a single coordinate $\ell$, then their interpolated perturbation probabilities under $f_{\mathrm{int}}$ satisfy $(\epsilon_\ell, d_1)$-mDP along that dimension. Then by applying the dimension-wise composition theorem (\textbf{Theorem~\ref{thm:composition}}), we establish that the composed mechanism satisfies $(\epsilon, d_p)$-mDP globally. The full proof is provided in \textbf{Appendix~\ref{prop:thm:AIPOfeasible}}.
\end{proof}

While the unnormalized log-convex interpolant preserves coordinate-wise Lipschitz continuity, its outputs do not necessarily lie on the probability simplex. To obtain a valid perturbation mechanism, we normalize these values (\textbf{Definition~\ref{def:PPI}}), which restores a proper probability distribution at the cost of increasing the effective Lipschitz bound by at most a factor of~2 (\textbf{Proposition \ref{prop:AIPOfeasible}}). 

\begin{definition}[\textbf{Normalized Multi-Dimensional Log-Convex Interpolation} $\overline{f}_{\mathrm{int}}$]
\label{def:PPI}
Given a point $\mathbf{x}_a$ within cell $\mathcal{C}_m$ and a perturbation candidate $\mathbf{y}_k \in \mathcal{Y}$, the normalized interpolated probability is defined as:

\begin{equation}
\label{eq:PPI}
\overline{f}_{\mathrm{int}}(\mathbf{x}_a, \mathbf{y}_k, \mathbf{Z}_{\hat{\mathcal{X}}_m}) = 
\frac{f_{\mathrm{int}}(\mathbf{x}_a, \mathbf{y}_k, \mathbf{Z}_{\hat{\mathcal{X}}_m})}
{\sum_{\mathbf{y}_j \in \mathcal{Y}} f_{\mathrm{int}}(\mathbf{x}_a, \mathbf{y}_j, \mathbf{Z}_{\hat{\mathcal{X}}_m})},
\end{equation}
where the denominator normalizes the interpolated scores over all possible outputs, ensuring that the resulting distribution is valid: $\sum_{\mathbf{y}_k \in \mathcal{Y}} \overline{f}_{\mathrm{int}}(\mathbf{x}_a, \mathbf{y}_k, \mathbf{Z}_{\hat{\mathcal{X}}_m}) = 1$.
\end{definition}
% The interpolation function $f_{\text{int}}$ computes the perturbation probability for any non-anchor record $\mathbf{x}_a \in \mathcal{C}_m$ by aggregating the contributions of all $2^N$ anchor points at the corners of the corresponding $N$-orthotope, using a weighted geometric mean. These weights are derived from the product of convex coefficients across all dimensions, as specified in Eq.~\eqref{eq:PPI}. This construction ensures that the interpolated perturbation probabilities satisfy $(\epsilon, d_p)$-mDP over the entire continuous domain, as formally proven in \textbf{Theorem~\ref{thm:AIPOfeasible}}.

%In \textbf{Theorem \ref{thm:AIPOfeasible}}, we prove the correctness of the interpolation function if $(\epsilon_\ell, d_1)$-mDP is satisfied between each pair of $\ell$-neighbors in $\hat{\mathcal{X}}$. 
\begin{proposition}
\label{prop:AIPOfeasible}
Given that any pair of real records $\mathbf{x}_a \in \mathcal{X}_m$ and $\mathbf{x}_b \in \mathcal{X}_{m'}$ with their perturbation probabilities interpolated by 
\begin{eqnarray}
&& z(\mathbf{y}_k \mid \mathbf{x}_a) = \overline{f}_{\mathrm{int}}(\mathbf{x}_a, \mathbf{y}_k, \mathbf{Z}_{\hat{\mathcal{X}}_m})\\
&& z(\mathbf{y}_k \mid \mathbf{x}_b) = \overline{f}_{\mathrm{int}}(\mathbf{x}_b, \mathbf{y}_k, \mathbf{Z}_{\hat{\mathcal{X}}_{m'}})
\end{eqnarray} 
then their perturbation probabilities satisfy $(2\epsilon, d_p)$-mDP.  
\end{proposition}
\begin{proof}[Proof Sketch]
The unnormalized interpolant $f_{\text{int}}$ already satisfies
pointwise $(\epsilon,d_p)$-mDP, i.e., $
\frac{f_{\mathrm{int}}(\mathbf{x}_a, \mathbf{y}_k, \mathbf{Z}_{\hat{\mathcal{X}}_m})}{f_{\mathrm{int}}(\mathbf{x}_b, \mathbf{y}_k, \mathbf{Z}_{\hat{\mathcal{X}}_m})}
\;\le\; e^{\epsilon d_p(\mathbf{x}_a,\mathbf{x}_b)}$. 
Summing over $\mathbf{y}_j$ shows that the partition function 
$Z(\mathbf{x})=\sum_{\mathbf{y}_j\in \mathcal{Y}} f_{\mathrm{int}}(\mathbf{x}_a, \mathbf{y}_j, \mathbf{Z}_{\hat{\mathcal{X}}_m})$ obeys the same multiplicative bounds. The normalized mechanism is $z(\mathbf{y}_k|\mathbf{x}_a) = \frac{f_{\text{int}}(\mathbf{x}_a,\mathbf{y}_k)}{Z(\mathbf{x}_a)}$. Thus, the ratio of normalized probabilities is the product of the
per-class ratio and the inverse normalizer ratio, giving $\frac{z(\mathbf{y}_k|\mathbf{x}_a)}{z(\mathbf{y}_k|\mathbf{x}_b)}
\;\le\; e^{2\epsilon d_p(\mathbf{x}_a,\mathbf{x}_b)}$. 
Hence, the normalized mechanism is $(2\epsilon,d_p)$-mDP.
The detailed proof can be found in \textbf{Appendix \ref{subsec:proof:prop:AIPOfeasible}.} 
\end{proof}

% The intuition of Eq. (\ref{eq:PPI}) is to let $z(\mathbf{y}_k|\mathbf{x}_a)$ be closer to $z(\mathbf{y}_k|\hat{\mathbf{x}}_j)$ (corresponding to a higher value of $\lambda_{a,j}$) when $d_p(\mathbf{x}_a, \hat{\mathbf{x}}_j)$ is smaller, due to the tighter mDP constraints between $z(\mathbf{y}_k|\mathbf{x}_a)$ and $z(\mathbf{y}_k|\hat{\mathbf{x}}_j)$. 

% It is worth noting that \textbf{Theorem~\ref{thm:AIPOfeasible}} is established under the assumption that $(\epsilon_\ell, d_1)$-mDP holds for each pair of neighboring anchors, which imposes a stricter constraint than the original $(\epsilon, d_p)$-mDP. In Appendix~\ref{subsec:discussion:d_p}, we formally demonstrate that enforcing $(\epsilon, d_p)$-mDP between each anchor and a reference point does not guarantee that the interpolated records satisfy global $(\epsilon, d_p)$-mDP. This limitation stems from the convexity of the $d_p$ distance function and the implications of Jensen’s inequality.

% $e^{- \epsilon d_p(\mathbf{x}_a, \mathbf{x}_b)} \leq \frac{\sum_{\mathbf{y}_j \in \mathcal{Y}} z(\mathbf{y}_k \mid \mathbf{x}_a)}{\sum_{\mathbf{y}_j \in \mathcal{Y}} z(\mathbf{y}_k \mid \mathbf{x}_b)} \leq e^{\epsilon d_p(\mathbf{x}_a, \mathbf{x}_b)}$.

\subsection{Step \textcircled{2} - Anchor Perturbation Optimization (APO)}
\label{subsec:anchoropt}

The goal of APO is to jointly optimize (1) the per-dimension privacy budgets $\{\epsilon_\ell\}_{\ell=1}^N$ under a global privacy budget composition constraint (formalized in Eq. (\ref{eq:budgetcompo1})(\ref{eq:budgetcompo2}) in \textbf{Theorem \ref{thm:composition}}), and (2) the perturbation distributions of each pair of $\ell$-neighbor in the anchor set $\hat{\mathcal{X}}$ satisfy $(\epsilon_\ell, d_1)$-mDP constraints along each dimension $\ell$, 
so as to minimize the expected utility loss over the secret domain. Here, we %use $\mathcal{L}(\mathbf{x}, \mathbf{y})$ to denote the utilty loss caused by $\mathbf{y}$ given the real secret record $\mathbf{x} \in \mathcal{X}$. The
define the expected utility loss of the secret data within each $N$-orthotope $\mathcal{C}_m$ by
\normalsize
% \small 
\begin{eqnarray}
% \nonumber 
\mathcal{L}(\mathbf{Z}_{\hat{\mathcal{X}}_m}) 
% &=& \textstyle\sum_{\mathbf{y}_k \in \mathcal{Y}} \mathbb{E}_{\mathbf{x}\sim \mathcal{C}_m, \mathbf{y}_k \sim \mathcal{M}(\mathbf{x}, \mathbf{Z}_{\hat{\mathcal{X}}_m})} \left(\mathcal{L}(\mathbf{x}, \mathbf{y}_k)\right) \\ 
\label{eq:loss}
= \sum_{\mathbf{y}_k \in \mathcal{Y}} \int_{\mathcal{C}_m}\overline{f}_{\mathrm{int}}(\mathbf{x}, \mathbf{y}_k, \mathbf{Z}_{\hat{\mathcal{X}}_m}) p(\mathbf{x}) \mathcal{L}(\mathbf{x}, \mathbf{y}_k) \mathrm{d}\mathbf{x}. 
% \\  &=& \sum_{\mathbf{y}_k \in \mathcal{Y}} \int_{\mathcal{C}_m} \frac{\prod_{\boldsymbol\gamma\in \left\{0,1\right\}^N} z_{(i_m+\boldsymbol\gamma),k}^{\prod_{\ell=1}^N\lambda^{(i_m+\gamma_\ell)}}}{\sum_{\mathbf{y}_j \in \mathcal{Y}}\prod_{\boldsymbol\gamma\in \left\{0,1\right\}^N} z_{(i_m+\boldsymbol\gamma),j}^{\prod_{\ell=1}^N\lambda^{(i_m+\gamma_\ell)}}} p(\mathbf{x}) \mathcal{L}(\mathbf{x}, \mathbf{y}_k) \mathrm{d}\mathbf{x} 
\end{eqnarray}
\normalsize
\begin{definition}[\textbf{Anchor Perturbation Optimization (APO)}]
\label{def:APO}
The APO problem jointly optimizes the anchor perturbation distributions and the per-dimension privacy budgets to minimize the total expected utility loss $\sum_{m=1}^M \mathcal{L}(\mathbf{Z}_{\hat{\mathcal{X}}_m})$, subject to privacy and probability constraints. Formally:
\begin{eqnarray}
\label{eq:APO:obj}
\min && % \textstyle 
\sum_{m=1}^M \mathcal{L}({\bl \mathbf{Z}_{\hat{\mathcal{X}}_m}}) \\ 
\label{eq:APO:privacybudget}
\text{s.t.} && \textstyle \sum_{\ell=1}^N {\rd \epsilon_\ell}^{\frac{p}{p-1}} \leq \left(\frac{\epsilon}{2}\right)^{\frac{p}{p-1}}, \quad \text{when } p > 1,  \\
&& 
\max_{\ell} {\rd\epsilon_\ell} \leq \frac{\epsilon}{2}, \quad \text{when } p = 1 \\ \nonumber 
&& {\bl z(\mathbf{y}_k \mid \hat{\mathbf{x}}_i)} - e^{{\rd \epsilon_\ell} \Delta_\ell} {\bl z(\mathbf{y}_k \mid \hat{\mathbf{x}}_j)} \leq 0, \\
\label{eq:APO:mDP}
&& \text{for each pair $\ell$-neighbor $\hat{\mathbf{x}}_i, \hat{\mathbf{x}}_j\in \hat{\mathcal{X}}$}, \forall \ell \\
\label{eq:APO:normalization}
&& \textstyle  \sum_{\mathbf{y}_k \in \mathcal{Y}} {\bl z(\mathbf{y}_k \mid \hat{\mathbf{x}}_i)} = 1,~ \forall \hat{\mathbf{x}}_i \in \hat{\mathcal{X}}, \\ \label{eq:APO:nonnagative}
&&  {\bl z(\mathbf{y}_k \mid \hat{\mathbf{x}}_i)} \geq 0, \forall \hat{\mathbf{x}}_i \in \hat{\mathcal{X}}, ~\forall \mathbf{y}_k \in \mathcal{Y}.  
\end{eqnarray}
Here, Eq.~\eqref{eq:APO:privacybudget} enforces the \emph{privacy budget composition constraint}, Eq.~\eqref{eq:APO:mDP} imposes $(\epsilon_\ell, d_1)$-mDP constraints between each pair of $\ell$-neighbor anchors, and Eqs.~\eqref{eq:APO:normalization}–(\ref{eq:APO:nonnagative}) define the \emph{normalization and non-negativity constraints}.
\end{definition}

\begin{table*}[ht]
\centering
\small 
% \footnotesize 
% \scriptsize
\caption{Statistics of Road Network Datasets and Grid Partitioning}
\label{tab:dataset-stats}
\begin{tabular}{lccccc}
\toprule
\textbf{City} & \textbf{Bounding Box (SW -- NE)} & \textbf{\#Nodes} & \textbf{\#Edges} & \textbf{Grid Size} & \textbf{Cell Size} \\
\hline
Rome, Italy   & (41.66, 12.24) -- (42.10, 12.81) & 43,160 & 89,739  & $15 \times 20$ & $3.18km \times 3.18km$ % 3.1786
\\
NYC, USA      & (40.50, 73.70) -- (40.91, 74.25) & 55,353 & 139,638 & $10 \times 20$ & $4.58km \times 4.58km$ \\
London, UK    & (51.29, -0.51) -- (51.69, 0.28)  & 12,820 & 299,524 & $10 \times 20$ & $4.52km \times 4.52km$ \\
\bottomrule
\end{tabular}
\end{table*}

The APO formulation introduces a nontrivial coupling between two sets of decision variables: the per-dimension privacy budgets ${\rd \{\epsilon_\ell\}_{\ell=1}^N}$ and the perturbation matrix ${\bl \mathbf{Z}_{\hat{\mathcal{X}}} = \{ z(\mathbf{y}_k \mid \hat{\mathbf{x}}_i) \}_{(\hat{\mathbf{x}}_i, \mathbf{y}_k) \in \hat{\mathcal{X}} \times \mathcal{Y}}}$. Since the feasible space of $\mathbf{Z}_{\hat{\mathcal{X}}}$ is constrained by the choice of $\{\epsilon_\ell\}$, this coupling complicates joint optimization. In our experiments (\textbf{Section~\ref{sec:experiments}}), we focus on a 2D $\ell_2$-norm setting, where the set of feasible privacy budget allocations $\{(\epsilon_1, \epsilon_2)\}$ must satisfy the constraint $\epsilon_1^2 + \epsilon_2^2 \leq \epsilon^2/4$. This defines a quarter-circle region in the first quadrant. To search for the optimal allocation, we discretize this curve by sampling values of $\epsilon_1$, compute the corresponding $\epsilon_2 = \sqrt{\epsilon^2/4 - \epsilon_1^2}$, and evaluate the resulting utility loss.  This grid search remains tractable in low dimensions. For higher-dimensional settings, we discuss potential scalable extensions based on decomposition techniques in \textbf{Appendix~\ref{sec:discussion:budgetalloc}}. \looseness = -1

Another practical obstacle in solving the APO problem is the form of the expected‐loss objective
$\mathcal{L}\bigl(\mathbf{Z}_{\mathcal{X}}\bigr)
       = \sum_{m}\mathcal{L}\bigl(\mathbf{Z}_{\hat{\mathcal{X}}_{m}}\bigr)$, where for any fixed cell $\mathcal{C}_{m}$ the integrand
$\overline{f}_{\mathrm{int}}(\mathbf{x}_a, \mathbf{y}_k, \mathbf{Z}_{\hat{\mathcal{X}}_m})$
is a \emph{ratio of products of exponentials}; hence the resulting integral in
Eq.~\eqref{eq:loss} is analytically intractable and costly to evaluate numerically. Therefore, we replace this exact probability by a weighted geometric mean of the anchor probabilities and thereby convert the cell‑level loss into a linear form.

\begin{proposition}[\textbf{Linear Surrogate for Utility Loss}]
\label{prop:linapprox}
Let $\mathbf{x} = \hat{\mathbf{x}}_{i_m} + \boldsymbol{\lambda}  \odot \boldsymbol{\Delta} \in \mathcal{C}_m$ be a non-anchor point, where $\boldsymbol{\lambda} = [\lambda_{\hat{\mathbf{x}}_{i_m}, \mathbf{x}}^{1}, \ldots, \lambda_{\hat{\mathbf{x}}_{i_m}, \mathbf{x}}^{N}] \in [0,1]^N$ are the convex interpolation weights. Approximating the perturbation probability in the objective function via weighted geometric interpolation yields: \looseness = -1
\begin{eqnarray}
\small && \Pr\big[\mathcal{M}(\mathbf{x}; \mathbf{Z}_{\hat{\mathcal{X}}_m}) = \mathbf{y}_k\big]
\\ 
\nonumber &\approx& % \textstyle 
\sum_{\boldsymbol{\gamma} \in \{0,1\}^N}
\left( \prod_{\ell=1}^N \big( (1 - \gamma_\ell)\lambda_{\hat{\mathbf{x}}_{i_m}, \mathbf{x}}^\ell + \gamma_\ell(1 - \lambda_{\hat{\mathbf{x}}_{i_m}, \mathbf{x}}^\ell) \big) \right) \\
&\times& z\big(\mathbf{y}_k \mid \hat{\mathbf{x}}_{i_m} + \boldsymbol{\gamma} \odot \boldsymbol{\Delta}\big).
\end{eqnarray}
Under this approximation, the expected utility loss over cell $\mathcal{C}_m$ admits a linear surrogate: $
\tilde{\mathcal{L}}(\mathbf{Z}_{\hat{\mathcal{X}}_m}) = \left\langle \tilde{\mathbf{C}}_{\hat{\mathcal{X}}_m}, \mathbf{Z}_{\hat{\mathcal{X}}_m} \right\rangle$, where $\tilde{\mathbf{C}}_{\hat{\mathcal{X}}_m} = \left\{ \tilde{c}(\hat{\mathbf{x}}_i, \mathbf{y}_k) \right\}_{(\hat{\mathbf{x}}_i, \mathbf{y}_k) \in \hat{\mathcal{X}}_m \times \mathcal{Y}}$ is a constant coefficient matrix that depends only on the prior distribution $p(\mathbf{x})$ and the pointwise utility loss $\mathcal{L}(\mathbf{x}, \mathbf{y}_k)$. Here, $\langle A, B \rangle$ denotes the standard Frobenius inner product between two matrices of the same shape: $
\langle A, B \rangle = \sum_{i,k} A_{i,k} \cdot B_{i,k}$. A closed-form derivation of $\tilde{\mathbf{C}}_{\hat{\mathcal{X}}_m}$ is provided in \textbf{Appendix~\ref{subsec:proof:prop:linapprox}}.
\end{proposition}

\begin{definition}[\textbf{Approximate Anchor Perturbation Optimization (Approx‑APO)}]
\label{def:Approx-APO}
Approx‑APO retains all linear $(\epsilon_\ell, d_1)$-mDP and probability constraints from the original APO formulation, but replaces the expected utility loss objective $\sum_{m=1}^M \mathcal{L}(\mathbf{Z}_{\hat{\mathcal{X}}_m})$ with its linear surrogate: 
\begin{eqnarray}
% \textstyle 
\sum_{m=1}^M \tilde{\mathcal{L}}(\mathbf{Z}_{\hat{\mathcal{X}}_m}) 
= \sum_{m=1}^M \left\langle \tilde{\mathbf{C}}_{\hat{\mathcal{X}}_m}, \mathbf{Z}_{\hat{\mathcal{X}}_m} \right\rangle,     
\end{eqnarray} 
as defined in \textbf{Proposition~\ref{prop:linapprox}}.
\end{definition}

Approx‑APO can be solved efficiently using standard linear programming solvers~\cite{MATLABlinprog}, making it scalable to larger domains. However, since the surrogate objective is an approximation of the original expected utility loss, the resulting solution may not be optimal for the original APO. To assess its quality, we derive a universal lower bound on the optimal utility loss of the full APO formulation (see \textbf{Appendix~\ref{sec:gapanalysis}}). This bound serves as a benchmark for evaluating how closely the Approx‑APO solution approaches the true optimum.

\DEL{
\begin{figure*}[t]
\centering
\hspace{0.00in}
\begin{minipage}{1.00\textwidth}
  \subfigure[Rome]{
\includegraphics[width=0.30\textwidth, height = 0.18\textheight]{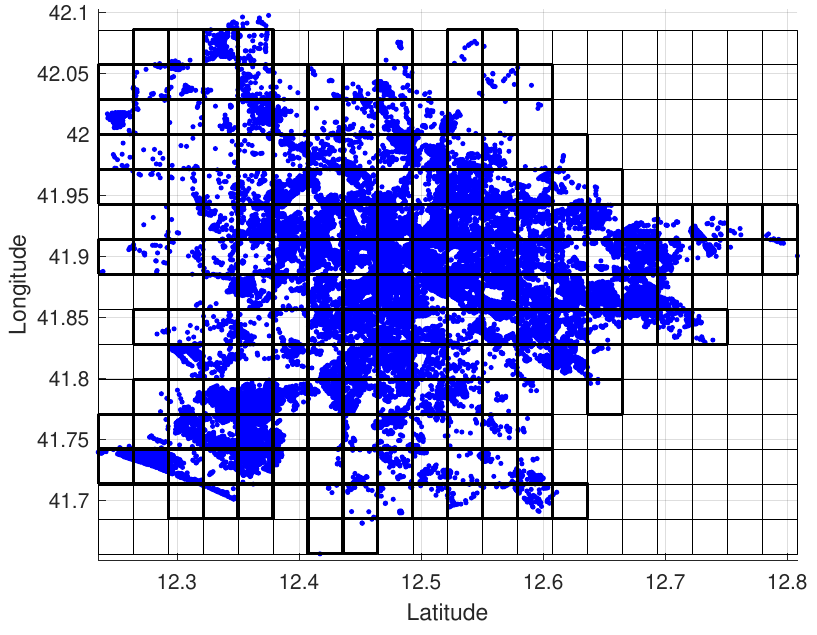}}
  \subfigure[London]{
\includegraphics[width=0.30\textwidth, height = 0.18\textheight]{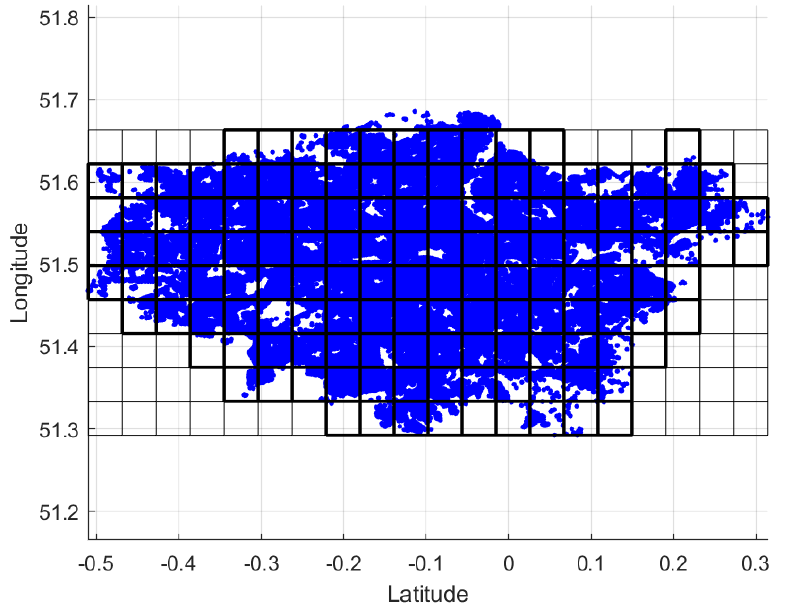}}
  \subfigure[New York City]{
\includegraphics[width=0.30\textwidth, height = 0.18\textheight]{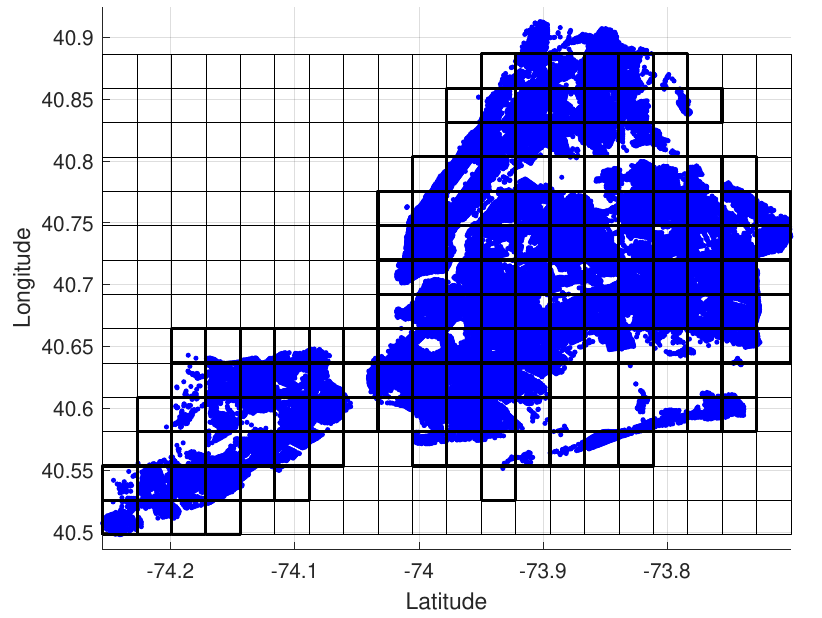}}
\end{minipage}
\caption{Spatial distribution of locations and grid maps for the three cities.}
\label{fig:exp:roadmap}
\end{figure*}}

\begin{table*}[t]
\label{Tb:exp:speedlimits}
\centering
% \small 
\footnotesize 
% \scriptsize
\begin{tabular}{p{1.50cm} | p{1.25cm} p{1.25cm} p{1.25cm} p{1.25cm} p{1.25cm} p{1.25cm} p{1.25cm} p{1.25cm} p{1.25cm}} 
\toprule
\multicolumn{9}{c }{Rome road map}\\ 
  \hline
\multicolumn{2}{c }{Method} & $\epsilon = 0.2$& $\epsilon = 0.4$& $\epsilon = 0.6$& $\epsilon = 0.8$& $\epsilon = 1.0$& $\epsilon = 1.2$& $\epsilon = 1.4$& $\epsilon = 1.6$\\ 
 \hline
 \hline
Pre-defined& EM & 0.00$\pm$0.00 & 0.00$\pm$0.00 & 0.00$\pm$0.00 & 0.00$\pm$0.00 & 0.00$\pm$0.00 & 0.00$\pm$0.00 & 0.00$\pm$0.00 & 0.00$\pm$0.00 \\
\cline{2-2}
Noise & Laplace & 0.00$\pm$0.00 & 0.00$\pm$0.00 & 0.00$\pm$0.00 & 0.00$\pm$0.00 & 0.00$\pm$0.00 & 0.00$\pm$0.00 & 0.00$\pm$0.00 & 0.00$\pm$0.00 \\
\cline{2-2}
Distribution & TEM & 0.00$\pm$0.00 & 0.00$\pm$0.00 & 0.00$\pm$0.00 & 0.00$\pm$0.00 & 0.00$\pm$0.00 & 0.00$\pm$0.00 & 0.00$\pm$0.00 & 0.00$\pm$0.00 \\
\cline{1-2}
Hybrid & RMP & 0.00$\pm$0.00 & 0.00$\pm$0.00 & 0.00$\pm$0.00 & 0.00$\pm$0.00 & 0.00$\pm$0.00 & 0.00$\pm$0.00 & 0.00$\pm$0.00 & 0.00$\pm$0.00 \\
\cline{2-2}
Method &  COPT  & 1.17$\pm$0.72 \cellcolor{red!10} & 3.14$\pm$2.04 \cellcolor{red!10} & 1.12$\pm$0.54 \cellcolor{red!10} & 0.99$\pm$0.39 \cellcolor{red!10} & 0.90$\pm$0.32 \cellcolor{red!10} & 0.82$\pm$0.27 \cellcolor{red!10} & 0.76$\pm$0.23 \cellcolor{red!10} & 0.70$\pm$0.20 \cellcolor{red!10} \\
\cline{1-2}
\multicolumn{2}{c }{LP} & 1.76$\pm$0.14 \cellcolor{red!10} & 1.55$\pm$0.13 \cellcolor{red!10} & 1.32$\pm$0.11 \cellcolor{red!10} & 1.06$\pm$0.10 \cellcolor{red!10} & 1.21$\pm$0.05 \cellcolor{red!10} & 1.11$\pm$0.04 \cellcolor{red!10} & 0.95$\pm$0.03 \cellcolor{red!10} & 0.88$\pm$0.03 \cellcolor{red!10} \\
\hline
 \multicolumn{2}{c}{AIPO-R  } & 7.24$\pm$0.51 & 6.91$\pm$0.38 & 4.80$\pm$0.30 & 3.84$\pm$0.17 & 2.70$\pm$0.04 & 0.00$\pm$0.00 & 0.00$\pm$0.00 & 0.00$\pm$0.00 \\
%  LB & 13.61$\pm$2.39 & 7.29$\pm$1.05 & 6.62$\pm$0.79 & 5.71$\pm$0.87 & 4.72$\pm$0.50 & 3.75$\pm$0.49 & 3.24$\pm$1.09 & 3.10$\pm$2.81 \\
\toprule
\rowcolor{lightgray!30}
\multicolumn{2}{c}{\textbf{AIPO}$^\dagger$} & 0.00$\pm$0.00 & 0.00$\pm$0.00 & 0.00$\pm$0.00 & 0.00$\pm$0.00 & 0.00$\pm$0.00 & 0.00$\pm$0.00 & 0.00$\pm$0.00 & 0.00$\pm$0.00 \\
\toprule
\multicolumn{9}{c }{London road map}\\ 
  \hline
\multicolumn{2}{c }{Method} & $\epsilon = 0.2$& $\epsilon = 0.4$& $\epsilon = 0.6$& $\epsilon = 0.8$& $\epsilon = 1.0$& $\epsilon = 1.2$& $\epsilon = 1.4$& $\epsilon = 1.6$\\ 
 \hline
 \hline
Pre-defined& EM & 0.00$\pm$0.00 & 0.00$\pm$0.00 & 0.00$\pm$0.00 & 0.00$\pm$0.00 & 0.00$\pm$0.00 & 0.00$\pm$0.00 & 0.00$\pm$0.00 & 0.00$\pm$0.00 \\
\cline{2-2}
Noise & Laplace & 0.00$\pm$0.00 & 0.00$\pm$0.00 & 0.00$\pm$0.00 & 0.00$\pm$0.00 & 0.00$\pm$0.00 & 0.00$\pm$0.00 & 0.00$\pm$0.00 & 0.00$\pm$0.00 \\
\cline{2-2}
Distribution & TEM & 0.00$\pm$0.00 & 0.00$\pm$0.00 & 0.00$\pm$0.00 & 0.00$\pm$0.00 & 0.00$\pm$0.00 & 0.00$\pm$0.00 & 0.00$\pm$0.00 & 0.00$\pm$0.00 \\
\cline{1-2}
Hybrid  & RMP & 0.00$\pm$0.00 & 0.00$\pm$0.00 & 0.00$\pm$0.00 & 0.00$\pm$0.00 & 0.00$\pm$0.00 & 0.00$\pm$0.00 & 0.00$\pm$0.00 & 0.00$\pm$0.00 \\
\cline{2-2}
Method &  COPT & 0.72$\pm$0.69 \cellcolor{red!10} & 2.43$\pm$1.10 \cellcolor{red!10} & 0.64$\pm$0.29 \cellcolor{red!10} & 0.63$\pm$0.26 \cellcolor{red!10} & 0.59$\pm$0.24 \cellcolor{red!10} & 0.55$\pm$0.21 \cellcolor{red!10} & 0.51$\pm$0.19 \cellcolor{red!10} & 0.48$\pm$0.17 \cellcolor{red!10} \\
\cline{1-2}
\multicolumn{2}{c }{LP} & 1.35$\pm$0.10 \cellcolor{red!10} & 1.35$\pm$0.09 \cellcolor{red!10} & 1.10$\pm$0.06 \cellcolor{red!10} & 0.82$\pm$0.05 \cellcolor{red!10} & 0.72$\pm$0.04 \cellcolor{red!10} & 0.58$\pm$0.03 \cellcolor{red!10} & 0.65$\pm$0.04 \cellcolor{red!10} & 1.86$\pm$0.17 \cellcolor{red!10}\\
\hline
 \multicolumn{2}{c}{AIPO-R  }  & 7.82$\pm$0.25 & 9.26$\pm$0.32 & 5.77$\pm$0.30 & 4.21$\pm$0.14 & 1.75$\pm$0.02 & 1.40$\pm$0.05 & 1.08$\pm$0.06 & 0.89$\pm$0.08 \\
 
% LB & 19.65$\pm$2.98 & 19.78$\pm$2.66 & 17.33$\pm$1.50 & 14.39$\pm$2.00 & 14.72$\pm$1.14 & 14.42$\pm$0.89 & 12.89$\pm$0.72 & 7.16$\pm$0.35 \\
\toprule
\rowcolor{lightgray!30} 
\multicolumn{2}{c}{\textbf{AIPO}$^\dagger$} & 0.00$\pm$0.00 & 0.00$\pm$0.00 & 0.00$\pm$0.00 & 0.00$\pm$0.00 & 0.00$\pm$0.00 & 0.00$\pm$0.00 & 0.00$\pm$0.00 & 0.00$\pm$0.00 \\
\toprule
\multicolumn{9}{c }{New York City road map}\\ 
  \hline
\multicolumn{2}{c }{Method} & $\epsilon = 0.2$& $\epsilon = 0.4$& $\epsilon = 0.6$& $\epsilon = 0.8$& $\epsilon = 1.0$& $\epsilon = 1.2$& $\epsilon = 1.4$& $\epsilon = 1.6$\\  
 \hline
 \hline
Pre-defined& EM & 0.00$\pm$0.00 & 0.00$\pm$0.00 & 0.00$\pm$0.00 & 0.00$\pm$0.00 & 0.00$\pm$0.00 & 0.00$\pm$0.00 & 0.00$\pm$0.00 & 0.00$\pm$0.00 \\
\cline{2-2}
Noise & Laplace & 0.00$\pm$0.00 & 0.00$\pm$0.00 & 0.00$\pm$0.00 & 0.00$\pm$0.00 & 0.00$\pm$0.00 & 0.00$\pm$0.00 & 0.00$\pm$0.00 & 0.00$\pm$0.00 \\
\cline{2-2}
Distribution & TEM & 0.00$\pm$0.00 & 0.00$\pm$0.00 & 0.00$\pm$0.00 & 0.00$\pm$0.00 & 0.00$\pm$0.00 & 0.00$\pm$0.00 & 0.00$\pm$0.00 & 0.00$\pm$0.00 \\
\cline{1-2}
Hybrid  & RMP & 0.00$\pm$0.00 & 0.00$\pm$0.00 & 0.00$\pm$0.00 & 0.00$\pm$0.00 & 0.00$\pm$0.00 & 0.00$\pm$0.00 & 0.00$\pm$0.00 & 0.00$\pm$0.00 \\
\cline{2-2}
Method &  COPT  & 0.97$\pm$0.48 \cellcolor{red!10} & 4.10$\pm$0.84 \cellcolor{red!10} & 0.74$\pm$0.39 \cellcolor{red!10} & 0.72$\pm$0.41 \cellcolor{red!10} & 0.68$\pm$0.41 \cellcolor{red!10} & 0.65$\pm$0.40 \cellcolor{red!10} & 0.61$\pm$0.38 \cellcolor{red!10} & 0.57$\pm$0.37 \cellcolor{red!10}\\
\cline{1-2}
\multicolumn{2}{c }{LP}  & 1.28$\pm$0.08 \cellcolor{red!10} & 1.27$\pm$0.07 \cellcolor{red!10} & 1.00$\pm$0.05 \cellcolor{red!10} & 1.12$\pm$0.08 \cellcolor{red!10} & 0.88$\pm$0.05 \cellcolor{red!10} & 0.77$\pm$0.06 \cellcolor{red!10} & 0.64$\pm$0.05 \cellcolor{red!10} & 1.76$\pm$0.12 \cellcolor{red!10} \\
\hline
 \multicolumn{2}{c}{AIPO-R  }  & 10.99$\pm$0.52 & 6.84$\pm$0.36 & 6.01$\pm$0.37 & 4.22$\pm$0.29 & 2.07$\pm$0.11 & 1.35$\pm$0.06 & 1.06$\pm$0.07 & 0.90$\pm$0.08 \\
 
% bound & 26.87$\pm$6.30 & 25.95$\pm$6.43 & 22.55$\pm$1.35 & 20.80$\pm$2.06 & 18.15$\pm$2.20 & 15.41$\pm$0.75 & 8.94$\pm$1.21 & 8.88$\pm$0.79 \\
\toprule
\rowcolor{lightgray!30}
\multicolumn{2}{c}{\textbf{AIPO}$^\dagger$} & 0.00$\pm$0.00 & 0.00$\pm$0.00 & 0.00$\pm$0.00 & 0.00$\pm$0.00 & 0.00$\pm$0.00 & 0.00$\pm$0.00 & 0.00$\pm$0.00 & 0.00$\pm$0.00 \\
\hline
    \end{tabular}
\caption{mDP violation ratio (Mean$\pm$1.96$\times$standard deviation).}
\label{tab:privacy2norm}
\end{table*}
\section{Case Study: Location Privacy in Navigation and Spatial Crowdsourcing}
% \section{Empirical Evaluation}
\label{sec:experiments}

This section presents a case study of mDP in the context of location privacy protection. Such problems commonly arise when users must approach a target location to obtain a service or complete a task. Representative examples include (i) navigation services~\cite{Qiu-TMC2022}, where users query for routes while concealing their true location, and (ii) spatial crowdsourcing~\cite{Wang-WWW2017}, where participants contribute geo-tagged data under privacy constraints.

To investigate this setting, we evaluate the proposed \emph{Anchor-Interpolated Privacy Optimization (AIPO)} algorithm on real-world road-network datasets and compare it against representative baselines. The evaluation covers three dimensions: 
(i) \emph{privacy} (\textbf{Section \ref{subsec:mainexp:privacy}}), measured by violations of the $(\epsilon,d_p)$-mDP constraints; 
(ii) \emph{utility loss} (\textbf{Section \ref{subsec:mainexp:UL}}), quantified by expected service loss under different distance metrics; and 
(iii) \emph{computational efficiency} (\textbf{Section \ref{subsec:mainexp:time}}), assessed via runtime performance. 
The results show that AIPO enforces strict privacy with zero observed mDP violations, consistently reduces utility loss relative to existing methods, and achieves lower runtime compared to other optimization-based approaches. We begin by describing the experimental setup in \textbf{Section \ref{subsec:mainexp:settings}}, including details on the \emph{datasets}, \emph{computational resources}, and \emph{baseline methods}.

\subsection{Experiment Settings}
\label{subsec:mainexp:settings}

\noindent \textbf{Datasets.} We conducted experiments on road network datasets from three major cities: \emph{Rome, Italy}, \emph{New York City (NYC), USA}, and \emph{London, UK}. Each dataset comprises nodes representing intersections, junctions, and other key points in the urban road infrastructure, with edges corresponding to actual road segments. The data were obtained from OpenStreetMap~\cite{openstreetmap}. To support our interpolation-based method, we discretize each city's geographic area into a uniform grid map. Table~\ref{tab:dataset-stats} provides a summary of geographic boundaries, node and edge counts, grid configurations, and cell sizes for each dataset. 

\noindent \textbf{Distance metric.} Unless otherwise stated, all main results are reported under the Euclidean distance $d_2$ (i.e., the $\ell_2$ norm) when evaluating $(\epsilon,d_p)$-mDP and utility. 
For completeness, we also include a full set of results under the Manhattan distance $d_1$ (i.e., the $\ell_1$ norm) in  \textbf{Appendix \ref{subsec:exp:1norm}}: utility loss (Table~\ref{tab:ULnorm1}), privacy compliance (Table~\ref{tab:privacy1norm}), and runtime (Table~\ref{tab:timenorm1}). These $\ell_1$ results corroborate the conclusions drawn from the $\ell_2$ setting.

% Experiments are conducted on real-world road network datasets from three urban regions: \emph{Rome}, \emph{New York City (NYC)}, and \emph{London}, containing 43{,}160, 55{,}353, and 12{,}820 discrete locations, respectively.

% Fig.~\ref{fig:exp:roadmap}(a)(b)(c) illustrates the spatial distribution of locations for the three cities.

% We compare AIPO against a range of $(\epsilon, d_p)$-mDP-compliant baselines, including \emph{pre-defined noise distribution} methods such as the \emph{Laplace Mechanism (Laplace)}~\cite{Andres-CCS2013}, \emph{Exponential Mechanism (EM)}~\cite{Chatzikokolakis-PoPETs2015}, and \emph{Truncated EM (TEM)}~\cite{Carvalho2021TEMHU}; an LP-based optimization method~\cite{Bordenabe-CCS2014}; and hybrid approaches that combine EM with optimization, such as \emph{COPT}~\cite{ImolaUAI2022} and \emph{Bayesian Remapping (RMP)}~\cite{Chatzikokolakis-PETS2017}. 

% 

\begin{table*}[t]
\centering
% \small 
\footnotesize 
% \scriptsize
\begin{tabular}{p{1.50cm} |  p{1.25cm} p{1.25cm} p{1.25cm} p{1.25cm} p{1.25cm} p{1.25cm} p{1.25cm} p{1.25cm} p{1.25cm}} 
\toprule
% \multicolumn{1}{ c|  }{  }& \multicolumn{9}{c}{}\\ \cline{2-10}
\multicolumn{9}{c }{Rome road map}\\ 
 \hline
\multicolumn{2}{c }{Method} & $\epsilon = 0.2$& $\epsilon = 0.4$& $\epsilon = 0.6$& $\epsilon = 0.8$& $\epsilon = 1.0$& $\epsilon = 1.2$& $\epsilon = 1.4$& $\epsilon = 1.6$\\ 
 \hline
 \hline
Pre-defined& EM & 8.71$\pm$0.78 & 8.70$\pm$1.13 & 8.65$\pm$1.28 & 8.62$\pm$1.38 & 8.58$\pm$1.45 & 8.56$\pm$1.51 & 8.54$\pm$1.55 & 8.52$\pm$1.58 \\
\cline{2-2}
Noise & Laplace & 8.71$\pm$0.71 & 8.48$\pm$1.00 & 8.46$\pm$1.40 & 8.45$\pm$1.69 & 8.44$\pm$1.90 & 8.43$\pm$2.05 & 8.42$\pm$2.16 & 8.40$\pm$2.24 \\
\cline{2-2}
Distribution & TEM & 8.85$\pm$2.71 & 8.95$\pm$3.19 & 8.66$\pm$2.44 & 8.64$\pm$1.83 & 8.66$\pm$1.11 & 8.66$\pm$0.69 & 8.66$\pm$0.27 & 8.62$\pm$0.22 \\
\cline{1-2}
Hybrid  & RMP & 5.94$\pm$0.25 & 4.96$\pm$0.45 & 4.28$\pm$0.36 & 3.85$\pm$0.26 & 3.58$\pm$0.21 & 3.40$\pm$0.19 & 3.28$\pm$0.18 & 3.19$\pm$0.18 \\
\cline{2-2}
Method &  COPT  & 7.99$\pm$1.53 & 7.95$\pm$1.04 & 8.33$\pm$1.50 & 8.29$\pm$1.57 & 8.27$\pm$1.59 & 8.25$\pm$1.61 & 8.24$\pm$1.60 & 8.24$\pm$1.60 \\
\cline{1-2}
\multicolumn{2}{c }{LP} & 4.25$\pm$0.41 & 2.97$\pm$0.11 & 2.56$\pm$0.03 & 2.45$\pm$0.07 & 2.43$\pm$0.03 & 2.42$\pm$0.02 & 2.42$\pm$0.01 & 2.42$\pm$0.01 \\
\hline
\multicolumn{2}{c}{AIPO-R} & 5.19$\pm$0.23 & 3.97$\pm$0.21 & 3.34$\pm$0.17 & 3.01$\pm$0.12 & 2.81$\pm$0.07 & 2.66$\pm$0.03 & 2.56$\pm$0.05 & 2.50$\pm$0.01 \\
\cline{1-2}
\multicolumn{2}{c}{LB}  & 1.82$\pm$0.01 & 1.73$\pm$0.01 & 1.73$\pm$0.00 & 1.73$\pm$0.01 & 1.73$\pm$0.00 & 1.73$\pm$0.01 & 1.73$\pm$0.00 & 1.73$\pm$0.00 \\
\hline
\rowcolor{lightgray!30}
\multicolumn{2}{c}{\textbf{AIPO}$^\dagger$} & \textbf{5.68$\pm$0.34} & \textbf{4.65$\pm$0.45} & \textbf{4.02$\pm$0.22} & \textbf{3.63$\pm$0.08} & \textbf{3.38$\pm$0.37} & \textbf{3.14$\pm$0.25} & \textbf{2.99$\pm$0.11} & \textbf{2.88$\pm$0.19} \\
\toprule
\multicolumn{9}{c }{London road map}\\ 
 \hline
\multicolumn{2}{c }{Method} & $\epsilon = 0.2$& $\epsilon = 0.4$& $\epsilon = 0.6$& $\epsilon = 0.8$& $\epsilon = 1.0$& $\epsilon = 1.2$& $\epsilon = 1.4$& $\epsilon = 1.6$\\  
 \hline
 \hline
Pre-defined& EM & 7.69$\pm$0.74 & 7.44$\pm$1.30 & 7.30$\pm$1.58 & 7.20$\pm$1.73 & 7.14$\pm$1.82 & 7.09$\pm$1.88 & 7.05$\pm$1.93 & 7.02$\pm$1.97 \\
\cline{2-2}
Noise & Laplace & 8.65$\pm$0.93 & 8.63$\pm$0.92 & 8.60$\pm$0.91 & 8.56$\pm$0.90 & 8.51$\pm$0.88 & 8.42$\pm$0.83 & 8.29$\pm$0.71 & 8.11$\pm$0.43 \\
\cline{2-2}
Distribution & TEM & 8.01$\pm$2.22 & 7.72$\pm$2.27 & 7.87$\pm$1.94 & 7.96$\pm$1.11 & 7.98$\pm$0.62 & 7.99$\pm$0.35 & 7.99$\pm$0.12 & 7.98$\pm$0.13 \\
\cline{1-2}
Hybrid  & RMP & 5.86$\pm$0.21 & 5.07$\pm$0.39 & 4.49$\pm$0.41 & 4.09$\pm$0.37 & 3.83$\pm$0.32 & 3.65$\pm$0.29 & 3.53$\pm$0.26 & 3.44$\pm$0.24 \\
\cline{2-2}
Method &  COPT  & 8.06$\pm$2.24 & 8.06$\pm$2.22 & 8.06$\pm$2.18 & 8.07$\pm$2.15 & 8.07$\pm$2.10 & 8.07$\pm$2.04 & 8.11$\pm$1.19 & 7.35$\pm$0.93 \\
\cline{1-2}
\multicolumn{2}{c }{LP} & 4.19$\pm$0.24 & 2.92$\pm$0.13 & 2.56$\pm$0.14 & 2.47$\pm$0.11 & 2.45$\pm$0.07 & 2.44$\pm$0.09 & 2.44$\pm$0.03 & 2.44$\pm$0.05 \\
\hline
\multicolumn{2}{c}{AIPO-R}  & 4.97$\pm$0.21 & 3.94$\pm$0.13 & 3.22$\pm$0.08 & 2.83$\pm$0.03 & 2.59$\pm$0.01 & 2.43$\pm$0.04 & 2.31$\pm$0.03 & 2.24$\pm$0.05 \\
\cline{1-2}
\multicolumn{2}{c}{LB}  & 1.51$\pm$0.05 & 1.44$\pm$0.02 & 1.44$\pm$0.03 & 1.44$\pm$0.01 & 1.44$\pm$0.00 & 1.44$\pm$0.01 & 1.44$\pm$0.00 & 1.44$\pm$0.00 \\
\hline
\rowcolor{lightgray!30}
\multicolumn{2}{c}{\textbf{AIPO}$^\dagger$} & \textbf{5.42$\pm$0.76} & \textbf{4.50$\pm$0.26} & \textbf{3.90$\pm$0.17} & \textbf{3.43$\pm$0.17} & \textbf{3.11$\pm$0.11} & \textbf{2.88$\pm$0.09} & \textbf{2.71$\pm$0.21} & \textbf{2.58$\pm$0.13} \\
\toprule
\multicolumn{9}{c }{New York City road map}\\ 
 \hline
\multicolumn{2}{c }{Method} & $\epsilon = 0.2$& $\epsilon = 0.4$& $\epsilon = 0.6$& $\epsilon = 0.8$& $\epsilon = 1.0$& $\epsilon = 1.2$& $\epsilon = 1.4$& $\epsilon = 1.6$\\  
 \hline
 \hline
Pre-defined& EM & 13.96$\pm$1.59 & 13.95$\pm$2.38 & 13.88$\pm$2.78 & 13.80$\pm$3.06 & 13.73$\pm$3.27 & 13.69$\pm$3.44 & 13.65$\pm$3.57 & 13.64$\pm$3.66 \\
\cline{2-2}
Noise & Laplace & 13.75$\pm$1.95 & 13.62$\pm$2.52 & 13.48$\pm$2.58 & 13.41$\pm$2.63 & 13.37$\pm$2.72 & 13.36$\pm$2.82 & 13.35$\pm$2.92 & 13.35$\pm$3.00 \\
\cline{2-2}
Distribution & TEM & 13.62$\pm$3.79 & 13.53$\pm$4.02 & 13.77$\pm$3.00 & 13.98$\pm$1.96 & 13.95$\pm$1.11 & 13.92$\pm$0.64 & 13.83$\pm$0.23 & 13.72$\pm$0.17 \\
\cline{1-2}
Hybrid  & RMP & 7.69$\pm$0.37 & 5.58$\pm$0.37 & 4.55$\pm$0.24 & 4.00$\pm$0.24 & 3.69$\pm$0.28 & 3.50$\pm$0.31 & 3.38$\pm$0.32 & 3.29$\pm$0.32 \\
\cline{2-2}
Method &  COPT  & 8.19$\pm$1.63 & 13.39$\pm$2.91 & 13.72$\pm$3.30 & 13.64$\pm$3.55 & 13.62$\pm$3.64 & 13.61$\pm$3.72 & 13.63$\pm$3.75 & 13.64$\pm$3.77 \\
\cline{1-2}
\multicolumn{2}{c }{LP} & 4.80$\pm$0.12 & 3.14$\pm$0.07 & 2.67$\pm$0.11 & 2.56$\pm$0.21 & 2.53$\pm$0.09 & 2.52$\pm$0.04 & 2.52$\pm$0.07 & 2.52$\pm$0.01 \\
\hline
\multicolumn{2}{c}{AIPO-R}  & 6.10$\pm$0.21 & 4.24$\pm$0.05 & 3.42$\pm$0.08 & 2.97$\pm$0.12 & 2.73$\pm$0.03 & 2.59$\pm$0.01 & 2.49$\pm$0.01 & 2.37$\pm$0.09 \\
\cline{1-2}
\multicolumn{2}{c}{LB}  & 2.18$\pm$0.11 & 2.04$\pm$0.02 & 2.03$\pm$0.01 & 2.03$\pm$0.01 & 2.03$\pm$0.00 & 2.03$\pm$0.00 & 2.03$\pm$0.01 & 2.03$\pm$0.00 \\
\hline
\rowcolor{lightgray!30}
\multicolumn{2}{c}{\textbf{AIPO}$^\dagger$} & \textbf{7.14$\pm$0.32} & \textbf{5.26$\pm$0.21} & \textbf{4.32$\pm$0.11} & \textbf{3.73$\pm$0.09} & \textbf{3.36$\pm$0.12} & \textbf{3.10$\pm$0.17} & \textbf{2.91$\pm$0.36} & \textbf{2.77$\pm$0.05} \\
\bottomrule
    \end{tabular}
\caption{Utility loss across different perturbation methods (Mean$\pm$1.96$\times$standard deviation; see Table \ref{tab:add:UL2norm} in the appendix for the full version).}
\label{tab:UL2norm}
\end{table*}

\noindent \textbf{Experiments Compute Resources.} Our experiments were conducted on a workstation equipped with an Intel Core i9-13900F processor (24 cores, 2.00–5.60 GHz), 32 GB of DDR5 memory (4800 MHz), and an NVIDIA GeForce RTX 4090 GPU with 24 GB of GDDR6X VRAM. Linear programs were solved using the MATLAB Optimization Toolbox function \texttt{linearprog}~\cite{MATLABlinprog}.

\noindent \textbf{Representative baselines.} We list several representative baselines that achieve $\epsilon$-mDP:
\begin{itemize}
    \item [(1)] \emph{Pre-defined Noise Distribution Methods}, including \emph{Exponential Mechanism (EM)}~\cite{Chatzikokolakis-PoPETs2015}, \emph{Planar Laplace Mechanism (Laplace)}~\cite{Andres-CCS2013}, and \emph{Truncated Exponential Mechanism (TEM)}~\cite{Carvalho2021TEMHU}. 
    \item [(2)] \emph{Linear programming (LP)~\cite{Bordenabe-CCS2014}}, which minimizes expected utility under mDP constraints via a LP on a \emph{discretized} domain; this approximation can misestimate pairwise distances and weaken mDP enforcement in fine-grained settings.  
% While LP yields an optimal mechanism, it suffers from poor scalability when covering the entire secret data domain $\mathcal{X}$, resulting in $|\mathcal{X}|^2$ decision variables. 
    \item [(3)] \emph{Hybrid Methods}, including \emph{ConstOPTMech (COPT)~\cite{ImolaUAI2022}}, which combines LP with EM to balance utility and scalability under mDP, and \emph{Bayesian Remapping (RMP)~\cite{Chatzikokolakis-PETS2017}}, which is a post-processing technique that enhances utility without compromising mDP.  \looseness = -1
    \item [(4)] \emph{AIPO-Relaxed (AIPO-R)}. AIPO-R is a relaxed variant of the proposed method that directly enforces the $(\epsilon, d_p)$-mDP constraint using pairwise distances between anchors, without decomposing the privacy budget across dimensions.
\end{itemize}

% In this section, we empirically evaluate the performance of our proposed method, referred to as \emph{AIPO} (Anchor-Interpolated Privacy Optimization). Experiments are conducted on real-world road network datasets from three urban regions: \emph{Rome}, \emph{New York City (NYC)}, and \emph{London}, containing 43{,}160, 55{,}353, and 12{,}820 discrete locations, respectively. We compare AIPO against a range of $(\epsilon, d_p)$-mDP-compliant baselines, including \emph{pre-defined noise distribution} methods such as the \emph{Laplace Mechanism (Laplace)}~\cite{Andres-CCS2013}, \emph{Exponential Mechanism (EM)}~\cite{Chatzikokolakis-PoPETs2015}, and \emph{Truncated EM (TEM)}~\cite{Carvalho2021TEMHU}; an LP-based optimization method~\cite{Bordenabe-CCS2014}; and hybrid approaches that combine EM with optimization, such as \emph{COPT}~\cite{ImolaUAI2022} and \emph{Bayesian Remapping (RMP)}~\cite{Chatzikokolakis-PETS2017}. 

\subsection{Privacy Evaluation}
\label{subsec:mainexp:privacy}
To evaluate the empirical compliance of each mechanism with $(\epsilon, d_p)$-mDP, we introduce the \emph{perturbation probability ratio (PPR)} as a diagnostic metric. We randomly sample $1{,}000$ locations from the input domain, denoted by $\mathcal{S}$, and for each pair $\mathbf{x}, \mathbf{x}' \in \mathcal{S}$ and output $\mathbf{y}_k$, compute
\begin{equation}
\label{eq:PPR}
\mathrm{PPR}(\mathbf{x}, \mathbf{x}', \mathbf{y}_k) \;=\; 
\frac{\left| \ln z(\mathbf{y}_k \mid \mathbf{x}) - \ln z(\mathbf{y}_k \mid \mathbf{x}') \right|}{d_p(\mathbf{x}, \mathbf{x}')}.
\end{equation}
A violation of the $(\epsilon, d_p)$-mDP constraint is recorded whenever $\mathrm{PPR}(\mathbf{x}, \mathbf{x}', \mathbf{y}_k) > \epsilon$ for any output $\mathbf{y}_k$.

Table~\ref{tab:privacy2norm} reports empirical violation ratios under privacy budgets $\epsilon \in \{0.2, 0.4, \dots, 1.6\}\,\text{km}^{-1}$. \emph{AIPO} attains $0\%$ violations across all datasets and budgets, corroborating the correctness of its dimension-wise composition and log-convex interpolation. In contrast, \emph{LP} and \emph{COPT} exhibit nonzero violation ratios because they enforce constraints over discretized representatives, thereby \emph{approximating} pairwise distances; such approximations can overestimate true continuous distances and relax the effective mDP constraints, missing privacy leakage at finer granularity. Pre-defined Noise Distribution mechanisms (e.g., Laplace, EM, TEM) do not incur violations but achieve this via heavier randomization, as reflected in their utility (detailed in Section \ref{subsec:mainexp:UL}). 

Complementing the aggregate ratios, the distributional analysis in \textbf{Section~\ref{subsec:exp:privacy}} (Fig.~\ref{fig:PLdistRome}–Fig.~\ref{fig:PLdistNYC}) shows that AIPO’s PPR values concentrate well below $\epsilon$ with tight tails, whereas LP-based and hybrid methods yield broader spreads with noticeable mass near (and occasionally beyond) the threshold; these patterns are consistent in Rome, London, and NYC.

The relaxed variant, \emph{AIPO-R}, exhibits higher violation ratios. Unlike AIPO, it enforces mDP only between anchor points and does not guarantee compliance in interpolated regions; consequently, violations arise in areas between anchors, especially under sparse anchoring or in higher-dimensional settings. A formal discussion is provided in \textbf{Appendix~\ref{subsec:discussion:d_p}}.

\begin{table*}[t]
\centering
%\small 
\footnotesize 
% \scriptsize
\begin{tabular}{ %c | c | c | c | c | c | c | c | c 
p{0.95cm} p{1.45cm} p{1.45cm} p{1.45cm} p{1.45cm} p{1.45cm} p{1.45cm} p{1.45cm} p{1.45cm}
} 
\toprule
\multicolumn{9}{c }{Rome road map}\\ 
 \hline
Method & $\epsilon = 0.2$& $\epsilon = 0.4$& $\epsilon = 0.6$& $\epsilon = 0.8$& $\epsilon = 1.0$& $\epsilon = 1.2$& $\epsilon = 1.4$& $\epsilon = 1.6$\\ 
 \hline
 \hline
% EM & 0.01$\pm$0.00 & 0.01$\pm$0.00 & 0.01$\pm$0.00 & 0.01$\pm$0.00 & 0.01$\pm$0.00 & 0.01$\pm$0.00 & 0.01$\pm$0.00 & 0.01$\pm$0.00 \\
% RMP & 0.02$\pm$0.00 & 0.02$\pm$0.00 & 0.02$\pm$0.00 & 0.02$\pm$0.00 & 0.02$\pm$0.00 & 0.02$\pm$0.00 & 0.02$\pm$0.01 & 0.02$\pm$0.00 \\
% Laplace & 0.01$\pm$0.00 & 0.01$\pm$0.00 & 0.01$\pm$0.00 & 0.01$\pm$0.00 & 0.01$\pm$0.00 & 0.01$\pm$0.00 & 0.01$\pm$0.00 & 0.01$\pm$0.00 \\
% TEM & 0.26$\pm$0.04 & 0.28$\pm$0.02 & 0.28$\pm$0.02 & 0.29$\pm$0.02 & 0.30$\pm$0.02 & 0.32$\pm$0.02 & 0.32$\pm$0.02 & 0.34$\pm$0.02 \\
COPT & 147.4$\pm$18.2 & 138.7$\pm$10.3 & 140.4$\pm$9.2 & 135.8$\pm$6.2 & 137.2$\pm$6.6 & 136.7$\pm$8.2 & 136.2$\pm$4.2 & 137.6$\pm$5.0 \\
LP & 210.8$\pm$296.7 & 122.1$\pm$19.3 & 345.5$\pm$323.4 & 342.1$\pm$242.8 & 399.0$\pm$158.7 & 318.1$\pm$138.7 & 428.2$\pm$223.9 & 366.3$\pm$249.4 \\
\toprule
\rowcolor{lightgray!30}
$\textbf{AIPO}^\dagger$ & 29.3$\pm$10.9 & 20.0$\pm$5.8 & 28.6$\pm$5.6 & 64.7$\pm$159.9 & 10.0$\pm$2.4 & 8.3$\pm$2.6 & 5.1$\pm$0.9 & 8.0$\pm$19.3 \\
\toprule
\multicolumn{9}{c }{London road map}\\ 
 \hline
Method & $\epsilon = 0.2$& $\epsilon = 0.4$& $\epsilon = 0.6$& $\epsilon = 0.8$& $\epsilon = 1.0$& $\epsilon = 1.2$& $\epsilon = 1.4$& $\epsilon = 1.6$\\ 
\hline
\hline
% EM & 0.01$\pm$0.00 & 0.01$\pm$0.00 & 0.01$\pm$0.00 & 0.01$\pm$0.00 & 0.01$\pm$0.00 & 0.01$\pm$0.00 & 0.01$\pm$0.00 & 0.01$\pm$0.00 \\
% RMP & 0.09$\pm$0.02 & 0.09$\pm$0.03 & 0.09$\pm$0.02 & 0.08$\pm$0.02 & 0.09$\pm$0.02 & 0.09$\pm$0.02 & 0.09$\pm$0.02 & 0.08$\pm$0.02 \\
% Laplace & 0.10$\pm$0.01 & 0.09$\pm$0.01 & 0.09$\pm$0.01 & 0.10$\pm$0.01 & 0.10$\pm$0.01 & 0.10$\pm$0.02 & 0.09$\pm$0.00 & 0.09$\pm$0.00 \\
% TEM & 0.55$\pm$0.04 & 0.56$\pm$0.06 & 0.57$\pm$0.04 & 0.56$\pm$0.08 & 0.55$\pm$0.05 & 0.55$\pm$0.03 & 0.56$\pm$0.02 & 0.56$\pm$0.02 \\
COPT & 285.1$\pm$115.0 & 274.6$\pm$110.9 & 272.9$\pm$114.4 & 272.2$\pm$113.2 & 271.5$\pm$118.1 & 269.7$\pm$117.8 & 266.4$\pm$109.7 & 269.4$\pm$103.2 \\
LP & 159.3$\pm$48.7 & 106.9$\pm$32.2 & 125.7$\pm$100.3 & 99.9$\pm$41.3 & 111.8$\pm$57.6 & 104.7$\pm$51.3 & 101.5$\pm$34.3 & 180.7$\pm$208.6 \\
\toprule
\rowcolor{lightgray!30}
$\textbf{AIPO}^\dagger$ & 78.2$\pm$13.9 & 77.0$\pm$22.4 & 78.5$\pm$24.2 & 77.4$\pm$13.8 & 80.7$\pm$16.1 & 75.0$\pm$26.4 & 62.3$\pm$8.7 & 63.9$\pm$10.9 \\
\toprule
\multicolumn{9}{c }{New York City road map}\\ 
 \hline
Method & $\epsilon = 0.2$& $\epsilon = 0.4$& $\epsilon = 0.6$& $\epsilon = 0.8$& $\epsilon = 1.0$& $\epsilon = 1.2$& $\epsilon = 1.4$& $\epsilon = 1.6$\\ 
 \hline
 \hline
% EM & 0.01$\pm$0.00 & 0.01$\pm$0.00 & 0.01$\pm$0.00 & 0.01$\pm$0.00 & 0.01$\pm$0.00 & 0.01$\pm$0.00 & 0.01$\pm$0.00 & 0.01$\pm$0.00 \\
% RMP & 0.05$\pm$0.01 & 0.06$\pm$0.03 & 0.05$\pm$0.01 & 0.05$\pm$0.00 & 0.05$\pm$0.00 & 0.05$\pm$0.01 & 0.05$\pm$0.02 & 0.05$\pm$0.00 \\
% Laplace & 0.07$\pm$0.05 & 0.07$\pm$0.04 & 0.08$\pm$0.04 & 0.08$\pm$0.04 & 0.09$\pm$0.06 & 0.10$\pm$0.07 & 0.08$\pm$0.04 & 0.08$\pm$0.07 \\
% TEM & 0.26$\pm$0.02 & 0.27$\pm$0.03 & 0.27$\pm$0.02 & 0.28$\pm$0.03 & 0.28$\pm$0.02 & 0.27$\pm$0.02 & 0.27$\pm$0.02 & 0.28$\pm$0.02 \\
COPT & 157.3$\pm$12.3 & 159.4$\pm$16.2 & 154.6$\pm$14.5 & 153.1$\pm$15.6 & 155.2$\pm$5.1 & 151.3$\pm$11.9 & 152.5$\pm$12.6 & 153.5$\pm$9.6 \\
LP & 303.8$\pm$140.8 & 324.1$\pm$218.8 & 263.0$\pm$141.9 & 418.6$\pm$202.1 & 265.6$\pm$178.5 & 284.8$\pm$203.5 & 393.7$\pm$304.8 & 403.2$\pm$174.9 \\
\toprule
\rowcolor{lightgray!30}
$\textbf{AIPO}^\dagger$ & 75.1$\pm$25.9 & 56.3$\pm$12.9 & 39.9$\pm$14.9 & 29.3$\pm$7.5 & 24.3$\pm$7.7 & 21.3$\pm$3.8 & 27.6$\pm$44.7 & 17.3$\pm$3.9 \\
\hline
    \end{tabular}
\caption{Computation time of different perturbation methods (Mean$\pm$1.96$\times$standard deviation).}
\label{tab:timenorm2}
\end{table*}

\subsection{Utility Loss Comparison}
\label{subsec:mainexp:UL}
Given a true vehicle position $\mathbf{x}_i$ and a task destination $\mathbf{x}_{\mathrm{task}}$, the loss incurred by releasing the perturbed location $\mathbf{y}_k$ is defined as the absolute difference between the corresponding shortest-path lengths, $\bigl|\text{path}(\mathbf{x}_i,\mathbf{x}_{\mathrm{task}}) - \text{path}(\mathbf{y}_k,\mathbf{x}_{\mathrm{task}})\bigr|$. 
Aggregating over the prior distribution $p(\mathbf{x}_{\mathrm{task}})$ of task locations $\mathcal{Q}$ yields the expected utility loss:
\begin{equation}
\label{eq:UL}
% \textstyle 
\mathcal{L}(\mathbf{x}_i,\mathbf{y}_k) =
\sum_{\mathbf{x}_{\mathrm{task}} \in \mathcal{Q}}
p(\mathbf{x}_{\mathrm{task}})
\bigl|\text{path}(\mathbf{x}_i,\mathbf{x}_{\mathrm{task}})
       - \text{path}(\mathbf{y}_k,\mathbf{x}_{\mathrm{task}})\bigr|.
\end{equation}
To compute $\text{path}(\cdot,\cdot)$, we apply Dijkstra’s algorithm~\cite{Algorithm} to determine the shortest-path distance between the origin and destination nodes.

Table~\ref{tab:UL2norm} reports the utility loss of our interpolation mechanism AIPO versus baseline perturbation methods on three road-network datasets:

\noindent \textbf{AIPO vs. pre-defined Noise Distribution Methods (EM, TEM, Laplace).} 
Across all three datasets, \emph{AIPO} consistently outperforms pre-defined noise distribution methods in terms of utility. 
\emph{On average, AIPO reduces utility loss by 59.03\%, 60.94\%, and 60.65\% compared to EM, Laplace, and TEM, respectively.} 
EM and TEM define perturbation probabilities using (truncated) exponential scoring functions, while Laplace adds noise sampled independently from a fixed distribution. These methods rely on global or isotropic perturbation rules that ignore the underlying geometry and the direction-dependent sensitivity of task-specific utility. Consequently, they tend to over-perturb in dense regions and under-protect in sparse ones. In contrast, AIPO optimizes perturbation distributions at a sparse set of anchor points and interpolates them across the domain via log-convex combinations, yielding smooth transitions aligned with the metric structure. % By decomposing the privacy budget across dimensions and incorporating local sensitivity into the optimization, AIPO achieves tighter privacy-utility tradeoffs.

\noindent \textbf{AIPO vs. Hybrid Methods (COPT and RMP).}  \emph{Across the three datasets, AIPO achieves an average of 59.16\% lower utility loss compared to COPT}. The limited scalability of COPT stems from its rigid LP formulation, which becomes impractical in high-resolution domains. In contrast, AIPO employs an anchor-based framework with log-convex interpolation, enabling flexible adaptation to fine-grained input spaces. By explicitly optimizing perturbation probabilities at a sparse set of anchor points, AIPO preserves the mDP guarantee while achieving improved utility.

\emph{AIPO also outperforms RMP, achieving 10.36\% lower utility loss on average across the evaluated datasets.} While RMP improves pre-defined noise mechanisms by reshaping posterior distributions to approximate mDP, its effectiveness is fundamentally limited by the quality of the initial noise distribution, which often falls short of global optimality. In contrast, AIPO formulates and solves an optimization problem directly under mDP constraints, without relying on any pre-defined perturbation mechanism, leading to a lower utility loss.

\noindent \textbf{AIPO vs. LP.} The LP method achieves lower utility loss compared to AIPO, since LP directly optimizes the perturbation distributions to minimize expected loss under distance-based constraints. 
However, this comes at the expense of privacy: the LP method only enforces the mDP constraints on a discrete grid of locations, without guaranteeing that the constraints hold for all possible input pairs in the continuous domain. As a result, while LP appears effective in terms of utility, it may produce perturbation probabilities that violate the mDP guarantee in off-grid regions, which has been demonstrated in \textbf{Table \ref{tab:privacy2norm}}.

\noindent \textbf{AIPO vs. Lower Bound (LB).} For theoretical reference, we also compare \emph{AIPO} with a universal lower bound (LB) from Proposition~\ref{prop:cell-lp-lower-bound} (Appendix~\ref{sec:gapanalysis}), which lower-bounds the minimum utility loss attainable by \emph{any} mechanism satisfying $(\epsilon,d_p)$-mDP. Empirically, \emph{AIPO} lies within $1.36\times$–$3.13\times$ of this bound across datasets and budgets, with larger gaps at tighter privacy (e.g., $\epsilon=0.4$) due to stricter constraints; the ratio narrows as $\epsilon$ increases.

\noindent \textbf{AIPO vs. AIPO-R.} Empirically, AIPO achieves higher utility loss compared to AIPO-R as AIPO-R relaxes the mDP constraint. In return, AIPO offers stronger theoretical guarantees, ensuring full compliance with $(\epsilon, d_p)$-mDP across the continuous space.

\noindent \textbf{Ablation Study: Utility loss with and without privacy-budget optimization.} Additionally, we compare \emph{AIPO} with an equal-split variant (\emph{AIPO-E}) that assigns identical per-dimension budgets (e.g., $\epsilon_1=\epsilon_2=\epsilon/\sqrt{2}$ under $\ell_2$). Across Rome, London, and NYC, \emph{AIPO} consistently achieves lower utility loss, with gains enlarging at higher $\epsilon$ (up to $7.8\%$ on NYC at $\epsilon=1.6$), indicating that learned asymmetric budgets better match directional sensitivity and improve privacy–utility trade-offs. The detailed experimental results can be found in Table~\ref{tab:privacybudget} and Fig.~\ref{fig:ULbudgetRome}--\ref{fig:ULbudgetNYC} in  \textbf{Appendix \ref{subsec:exp:privacybudget}}.

\subsection{Computation Efficiency}
\label{subsec:mainexp:time}

Table~\ref{tab:timenorm2} compares runtimes for the optimization-based methods (AIPO, LP) and the hybrid method COPT, each leveraging LP at some stage to construct perturbation mechanisms. AIPO attains a favorable utility–efficiency trade-off: it solves LPs only at a sparse set of anchors, and then uses log-convex interpolation to cover the continuous domain without a dense, full-scale program. This anchor-based design reduces both the number of decision variables and the effective constraint set, yielding stable performance across cities. Consistent with this, the grid-granularity study in \textit{AppendixD.3} (Fig.\ref{fig:gridgranularity}) shows that while finer grids monotonically reduce utility loss, runtime grows super-linearly, with gains saturating around $8$ horizontal cells for Rome/London and $10$ for NYC.

By contrast, the LP baseline incurs substantially higher cost in fine-grained settings because it solves a large program with mDP constraints over (nearly) all pairs of grid cells: the constraint count grows quadratically with grid size, and barrier/simplex iterations amplify this growth, making LP increasingly prohibitive as resolution increases.% \textbf{AIPO\mbox{-}R} is slightly faster than AIPO because it forgoes privacy-budget optimization (dimension-wise composition) and enforces mDP only between anchors. However, this simplification sacrifices global guarantees for interpolated points, leading to observable violations, as reported in Table~\ref{tab:privacy2norm}.

\section{Related Works}
\label{sec:relatedworks}

% Table~\ref{tab:mdp_summary} summarizes representative studies on mDP across various application domains, categorized by their privacy mechanisms, methodological approaches, and domain characteristics.

The earliest and most extensively studied applications of mDP focus on \emph{location privacy over grid maps}. Foundational works such as Andrés et al.~\cite{Andres-CCS2013} and Bordenabe et al.~\cite{Bordenabe-CCS2014} introduced geo-indistinguishability and linear programming (LP)-based mechanisms to enforce spatial privacy while preserving utility. These approaches rely on \emph{pre-defined noise distributions} or \emph{grid-based optimization} over discretized spatial domains. Subsequent works improve utility or scalability by incorporating personalized or group-based noise~\cite{Zhang-TBD2023, Ma-TITS2022}, truncated noise distributions~\cite{yang2021blockchain}, and adaptations to federated or blockchain settings~\cite{galli2022group}. Hybrid mechanisms such as Bayesian remapping~\cite{Chatzikokolakis-PETS2017} and utility-aware post-processing~\cite{oyaGeoIndLooking} refine pre-defined perturbations to improve performance. Yu et al.~\cite{Yu-NDSS2017} exemplify a hybrid design that combines mDP with an alternative privacy metric, "expected inference error", using context-aware noise adaptation.

A parallel line of research applies mDP to \emph{high-dimensional embeddings} in domains such as natural language processing and multimedia. In these settings, input records reside in continuous semantic vector spaces. Fernandes et al.~\cite{Fernandes2018AuthorOU} and Feyisetan et al.~\cite{Feyisetan-ICDM2019} added Laplace noise directly to text embeddings or used nearest-neighbor remapping. Carvalho et al.~\cite{Carvalho2021TEMHU} proposed a truncated exponential mechanism constrained to semantically similar candidates. For non-text data, Han et al.~\cite{Han-ICME2020} and Chen et al.~\cite{Chen-CVPR2021} adapted Laplace-based perturbations to voice and facial embeddings. These approaches often fall under \emph{pre-defined} or \emph{hybrid} categories, tailored for continuous or structured vector domains.

Recent work has shifted toward \emph{optimization-based mDP mechanisms} designed for \emph{fine-grained or continuous domains}, where more precise control over privacy-utility trade-offs is critical. Imola et al.~\cite{ImolaUAI2022} introduced ConstOPT, which combines LP with the exponential mechanism to improve scalability and utility. Qiu et al.~\cite{Qiu-IJCAI2024, Qiu-PETS2025} proposed decomposition-based perturbation frameworks, including partitioning and Benders decomposition. Other studies explore reinforcement learning~\cite{Min-TDSC2024} and bilevel optimization~\cite{Yu-TSP2023} to enforce mDP under task-specific utility objectives such as mobility prediction. These results demonstrate the versatility of \emph{customized optimization} strategies for enforcing mDP in high-resolution domains. Additionally, trajectory privacy methods such as sequential perturbation~\cite{Xiao-CCS2015} showcase the broader applicability of mDP frameworks in complex spatiotemporal contexts.

Taken together, \emph{Table~\ref{tab:mdp_summary} in Appendix} illustrates a clear evolution in the mDP literature, from pre-defined, coarse-grained perturbation mechanisms to increasingly sophisticated, optimization-based frameworks designed for continuous and fine-grained domains.

\section{Conclusion}

\label{sec:conclusions}
We propose an interpolation-based framework for optimizing $(\epsilon, d_p)$-mDP in continuous domains. By learning perturbation distributions at anchor points and applying dimension-wise log-convex interpolation, our method achieves both \emph{strict privacy guarantees} (zero mDP violations) and \emph{low utility loss}, supported by rigorous theoretical analysis and empirical evaluation on real-world datasets. This approach enables practical and theoretically grounded privacy mechanisms for fine-grained spatial and continuous data. While our method is highly effective in low-dimensional settings, its scalability is challenged in high-dimensional domains due to the exponential growth in anchor points, increased interpolation complexity, and fragmented privacy budget allocation. These limitations, common across mDP mechanisms under $\ell_p$-norms, motivate future work on bilevel programming, optimization decomposition, or adaptive partitioning. A detailed discussion is provided in \textbf{Appendix~\ref{subsec:discuss_limitations}}.

%%%%%%%%%%%%%%%%%%%%%%%%%%%%%%%%%%%%%%%%%%%%%%%%%%%%%%%%%%%%%%%%%%%%%%%%%%%%%%%
%%%%%%%%%%%%%%%%%%%%%%%%%%%%%%%%%%%%%%%%%%%%%%%%%%%%%%%%%%%%%%%%%%%%%%%%%%%%%%%
% APPENDIX

% You can have as much text here as you want. The main body must be at most $8$ pages long. For the final version, one more page can be added. If you want, you can use an appendix like this one.  

% The $\mathtt{\backslash onecolumn}$ command above can be kept in place if you prefer a one-column appendix, or can be removed if you prefer a two-column appendix.  Apart from this possible change, the style (font size, spacing, margins, page numbering, etc.) should be kept the same as the main body.
%%%%%%%%%%%%%%%%%%%%%%%%%%%%%%%%%%%%%%%%%%%%%%%%%%%%%%%%%%%%%%%%%%%%%%%%%%%%%%%
%%%%%%%%%%%%%%%%%%%%%%%%%%%%%%%%%%%%%%%%%%%%%%%%%%%%%%%%%%%%%%%%%%%%%%%%%%%%%%%

\cleardoublepage

\appendix

\section*{Ethical Considerations}
We reviewed the venue's ethics guidance, submission instructions, and ethics-document requirements. The study was conducted responsibly, and our planned next steps adhere to those principles.

\paragraph{Stakeholders.} 
The primary stakeholders of this work include: (i) the research team and the broader community of researchers working on mDP; (ii) practitioners and platform operators who may deploy mDP mechanisms in navigation, mobility analytics, spatial crowdsourcing, or related services, and the end-users of such systems; (iii) data subjects whose locations or other sensitive attributes might be protected by such mechanisms in future deployments; (iv) data and infrastructure providers such as OpenStreetMap contributors; and (v) society at large, including companies and institutions that rely on privacy-preserving data analytics, as well as marginalized or highly surveilled groups who may face disproportionate harms if location privacy is weak. For the identified stakeholders, the research process has negligible direct impact beyond advancing mDP methodology, whereas publication and deployment primarily affect practitioners, data subjects, and society by shaping how privacy mechanisms are chosen and used in practice.

\paragraph{Impacts of the research process and publication.}
Regarding the \emph{research process}, our evaluation relies solely on publicly available OpenStreetMap road-network graphs and synthetic perturbations. We do not collect, process, or attempt to infer personal or identifiable user-level information, and we do not interact with live systems. As a result, we believe that tangible harms (e.g., financial loss, exposure to disturbing content) and rights-based harms (e.g., violations of informed consent or expectations of privacy) to individuals during the research process are minimal.

Regarding the \emph{publication and potential deployment} of our results, the intended impact is positive: our interpolation-based optimization framework aims to make $\ell_p$-mDP mechanisms for continuous and fine-grained domains more scalable, analyzable, and usable, thereby strengthening protections for sensitive data. This can benefit researchers (clearer theory and reproducible baselines), practitioners (deployable and auditable mechanisms with guidance on composition and parameter selection), and data subjects (stronger protection for location and similar data in downstream applications). At the same time, there is a dual-use risk: our framework could be misused to justify overly weak privacy configurations, such as very large $\epsilon$ or poorly chosen metrics, that technically satisfy formal definitions while providing limited real-world protection, particularly for vulnerable populations. There is also a risk that deployers might over-interpret our results outside the assumptions we state (e.g., different threat models), leading to a mismatch between perceived and actual privacy.

\paragraph{Mitigations and residual risks.}
To mitigate these risks, we (i) clearly document the assumptions and limitations of our approach, including its sensitivity to $\epsilon$, metric choice, and modeling assumptions; (ii) explicitly warn against using extreme parameter settings or metrics that do not reflect meaningful notions of distance or harm for the affected population; (iii) recommend that practitioners adopt conservative privacy budgets and carefully evaluate deployments in light of relevant ethical, legal, and fairness considerations, especially in high-risk or high-stakes domains; and (iv) restrict our artifacts to code and derived or synthetic data that support defensive, privacy-preserving mechanisms, without releasing additional sensitive raw data. Despite these steps, some residual risk remains that third parties may deploy mechanisms irresponsibly or ignore our guidance; this residual risk cannot be fully eliminated.

\paragraph{Decision to conduct and publish the research.}
In deciding to conduct this research, we judged, under the lens of Beneficence and Respect for Persons, that simulation-based analysis on public infrastructure-level data, with no new collection or linkage of individual records, posed low risk relative to the potential benefit of improving privacy protection tools. Justice and Respect for Law and Public Interest are reflected in our use of publicly available, properly licensed datasets and in our aim to strengthen protections for users, including those at higher risk of surveillance or discrimination. In deciding to publish, we considered whether withholding the methods would better protect stakeholders; however, we concluded that making the techniques public, together with explicit discussion of their assumptions, limitations, potential dual use, and recommended safeguards, provides a net ethical benefit. Publication enables community scrutiny, supports more robust and transparent privacy mechanisms, and offers practitioners concrete guidance on responsible deployment.

\section*{Open Science}
We provide a detailed MATLAB implementation of the proposed AIPO framework to ensure full reproducibility of all experimental results reported in the paper. The same artifact is available at the link: \url{https://doi.org/10.5281/zenodo.17851733}. Table \ref{tab:reproduce} shows the script files and estimated running time of each figure and table. 

\begin{table}[h]
% \small
\footnotesize
\centering
\caption{Summary of reproducible tables and figures from the provided artifact.}
\label{tab:reproduce}
\begin{tabular}{p{2.2cm} p{3.9cm} p{1.1cm}}
\toprule
\textbf{Reproduced Tables and Figures} & \textbf{Script File} & \textbf{Estimated Run Time} \\
\midrule
Tables~\ref{tab:privacy2norm}–\ref{tab:timenorm2} (Main), Figures~\ref{fig:PLdistRome}–\ref{fig:PLdistNYC} \newline (Appendix) & \texttt{main\_2norm.m} & 37.0 hours \\
\midrule
Tables~\ref{tab:ULnorm1}–\ref{tab:timenorm1} \newline (Appendix) & \texttt{main\_1norm.m} & 36.0 hours \\
\midrule
Figure~\ref{fig:gridgranularity} \newline (Appendix) & \texttt{main\_granularity.m} & 15.0 hours \\
\midrule
Table~\ref{tab:privacybudget}  (Appendix), Figures~\ref{fig:ULbudgetRome}–\ref{fig:ULbudgetNYC} \newline (Appendix) & \texttt{main\_ablation\_privacybudget.m} & 7.5 hours \\
\bottomrule
\end{tabular}
\end{table}

\paragraph{(1) Code and Scripts.} The following MATLAB scripts reproduce the key tables and figures in the paper. All experiments were run on a machine with MATLAB R2024b and the Optimization Toolbox and Statistics and Machine Learning Toolbox enabled.
\begin{itemize}
    \item \texttt{main\_2norm.m}: Reproduces main results under the $\ell_2$ utility metric, including Tables~\ref{tab:privacy2norm}–\ref{tab:timenorm2} and Fig.~\ref{fig:PLdistRome}–\ref{fig:PLdistNYC}. \\
    \textit{Estimated runtime: 37.0 hours}

    \item \texttt{main\_1norm.m}: Reproduces experiments under the $\ell_1$ utility metric, including Tables~\ref{tab:ULnorm1}–\ref{tab:timenorm1}. \\
    \textit{Estimated runtime: 36.0 hours}

    \item \texttt{main\_granularity.m}: Evaluates the effect of grid granularity (Fig.~\ref{fig:gridgranularity}). \\
    \textit{Estimated runtime: 15.0 hours}

    \item \texttt{main\_ablation\_privacybudget.m}: Investigates the impact of privacy budget allocation strategies, reproducing Table~\ref{tab:privacybudget} and Fig.~\ref{fig:ULbudgetRome}–\ref{fig:ULbudgetNYC}. \\
    \textit{Estimated runtime: 7.5 hours}
\end{itemize}

\paragraph{(2) Instructions.} To reproduce the results:
\begin{enumerate}
    \item Set the working directory in MATLAB to the root folder containing the provided scripts.
    
    \item Run the desired script (e.g., \texttt{main\_2norm.m}) to generate the corresponding tables and figures.
    
    \item Output files (e.g., utiliy loss, mDP violation ratio, and runtime of different methods) will be saved in the following directories:
    \begin{itemize}
        \item \texttt{./results/1norm/}: Output of \texttt{main\_1norm.m}, including utility loss, mDP violation ratio, PPR, and runtime of all compared methods across the three datasets.
        
        \item \texttt{./results/2norm/}: Output of \texttt{main\_2norm.m}, including utility loss, mDP violation ratio, PPR, and runtime of all compared methods across the three datasets.
        
        \item \texttt{./results/granularity/}: Output of \texttt{main\_granularity.m}, including utility loss, runtime, and the number of anchors used in AIPO.
        
        \item \texttt{./results/ablation\_budget/}: \newline Output of \texttt{main\_ablation\_privacybudget.m}, including the utility loss of AIPO with and without privacy budget optimization.
    \end{itemize}
\end{enumerate}
%Each script is modularized to include:
%\begin{itemize}
%    \item Anchor perturbation optimization (Approx-APO via linear programming)
%    \item Log-convex interpolation across grid cells
%\end{itemize}

% \paragraph{Citation and Contact.} If you use this artifact in your own work, please cite our NeurIPS 2025 paper. For questions, contact the corresponding author after the review process concludes.

% \paragraph{Parameter settings.} All hyperparameters, including privacy budgets, anchor resolutions, and grid sizes, are documented in Appendix~\ref{app:exp}. The experiments were run with fixed random seeds to ensure consistency across runs.

\bibliographystyle{ACM-Reference-Format} 
\bibliography{mybib}

@article{Chatzikokolakis-PETS2017,
author = {Chatzikokolakis, Kostas and Elsalamouny, Ehab and Palamidessi, Catuscia},
year = {2017},
month = {10},
pages = {},
title = {Efficient Utility Improvement for Location Privacy},
volume = {2017},
journal = {Proceedings on Privacy Enhancing Technologies},
doi = {10.1515/popets-2017-0051}
}

@book{boyd2004convex,
  title     = {Convex Optimization},
  author    = {Stephen Boyd and Lieven Vandenberghe},
  year      = {2004},
  publisher = {Cambridge University Press},
  address   = {New York, NY, USA},
  isbn      = {9780521833783}
}

@inproceedings{feyisetan-WTNLP2021,
    title = "Private Release of Text Embedding Vectors",
    author = "Feyisetan, Oluwaseyi  and
      Kasiviswanathan, Shiva",
    editor = "Pruksachatkun, Yada  and
      Ramakrishna, Anil  and
      Chang, Kai-Wei  and
      Krishna, Satyapriya  and
      Dhamala, Jwala  and
      Guha, Tanaya  and
      Ren, Xiang",
    booktitle = "Proceedings of the First Workshop on Trustworthy Natural Language Processing",
    month = jun,
    year = "2021",
    address = "Online",
    publisher = "Association for Computational Linguistics",
    url = "https://aclanthology.org/2021.trustnlp-1.3",
    doi = "10.18653/v1/2021.trustnlp-1.3",
    pages = "15--27",
}

@inproceedings{Dharangutte-AAAI2023,
author = {Dharangutte, Prathamesh and Gao, Jie and Gong, Ruobin and Yu, Fang-Yi},
title = {Integer subspace differential privacy},
year = {2023},
isbn = {978-1-57735-880-0},
publisher = {AAAI Press},
url = {https://doi.org/10.1609/aaai.v37i6.25895},
doi = {10.1609/aaai.v37i6.25895},
booktitle = {Proceedings of the Thirty-Seventh AAAI Conference on Artificial Intelligence and Thirty-Fifth Conference on Innovative Applications of Artificial Intelligence and Thirteenth Symposium on Educational Advances in Artificial Intelligence},
articleno = {826},
numpages = {9},
series = {AAAI'23/IAAI'23/EAAI'23}
}

@INPROCEEDINGS{Duchi-FOCS2013,
  author={Duchi, John C. and Jordan, Michael I. and Wainwright, Martin J.},
  booktitle={2013 IEEE 54th Annual Symposium on Foundations of Computer Science}, 
  title={Local Privacy and Statistical Minimax Rates}, 
  year={2013},
  volume={},
  number={},
  pages={429-438},
  doi={10.1109/FOCS.2013.53}}

@book{probability,
author = {Stroock, Daniel W.},
title = {Probability Theory: An Analytic View},
year = {2010},
isbn = {0521132509},
publisher = {Cambridge University Press},
edition = {2nd}
}

@book{Algorithm,
  title = {Algorithms: Design and Analysis},
  publisher = {Oxford Univ Press},
  year = {2015},
  author = {Harsh Bhasin}
}

@inproceedings{Yu-NDSS2017,
 author = {L. Yu and L. Liu and C. Pu},
 title = {Dynamic Differential Location Privacy with Personalized Error Bounds},
 booktitle = {Proc. of IEEE NDSS},
 year = {2017},
}

@article{Zhang-TBD2023,
  author={Zhang, Pengfei and Cheng, Xiang and Su, Sen and Wang, Ning},
  journal={IEEE Transactions on Big Data}, 
  title={Task Allocation Under Geo-Indistinguishability via Group-Based Noise Addition}, 
  year={2023},
  volume={9},
  number={3},
  pages={860-877},
  doi={10.1109/TBDATA.2022.3215467}}

@inproceedings{Bordenabe-CCS2014,
 author = {N. E. Bordenabe and K. Chatzikokolakis  and C. Palamidessi},
 title = {Optimal Geo-Indistinguishable Mechanisms for Location Privacy},
 booktitle = {Proc. of ACM CCS},
 year = {2014},
 pages = {251--262},
}

@inproceedings{Chatzikokolakis-PETS2013,
author="Chatzikokolakis, Konstantinos
and Andr{\'e}s, Miguel E.
and Bordenabe, Nicol{\'a}s Emilio
and Palamidessi, Catuscia",
editor="De Cristofaro, Emiliano
and Wright, Matthew",
title="Broadening the Scope of Differential Privacy Using Metrics",
booktitle="Proc. of Privacy Enhancing Technologies",
year="2013",
publisher="Springer Berlin Heidelberg",
address="Berlin, Heidelberg",
pages="82--102",
isbn="978-3-642-39077-7"
}

@inproceedings{Wang-WWW2017,
 author = {Wang, L. and Yang, D. and Han, X. and Wang, T. and Zhang, D. and Ma, X.},
 title = {Location Privacy-Preserving Task Allocation for Mobile Crowdsensing with Differential Geo-Obfuscation},
 booktitle = {Proc. of ACM WWW},
 year = {2017},
 pages = {627--636},
 numpages = {10},
}

@inproceedings{Pappachan-EDBT2023,
  title={User Customizable and Robust Geo-Indistinguishability for Location Privacy},
  author={P. Pappachan and C. Qiu and A. Squicciarini and V. Manjunath},
  booktitle={Proc. of International Conference on Extending Database Technology (EDBT)},
  year={2023}
}

@inproceedings {Feyisetan-ICDM2019,
author = {O. Feyisetan and T. Diethe and T. Drake},
booktitle = {2019 IEEE International Conference on Data Mining (ICDM)},
title = {Leveraging Hierarchical Representations for Preserving Privacy and Utility in Text},
year = {2019},
volume = {},
issn = {},
pages = {210-219},
doi = {10.1109/ICDM.2019.00031},
url = {https://doi.ieeecomputersociety.org/10.1109/ICDM.2019.00031},
publisher = {IEEE Computer Society},
address = {Los Alamitos, CA, USA},
month = {nov}
}

@article {Chatzikokolakis-PoPETs2015,
	title = {Constructing elastic distinguishability metrics for location privacy},
	journal = {Privacy Enhancing Technologies (PoPETs)},
	volume = {2015},
	year = {2015},
	pages = {156{\textendash}170},
	url = {http://www.degruyter.com/view/j/popets.2015.2015.issue-2/popets-2015-0023/popets-2015-0023.xml},
	author = {Konstantinos Chatzikokolakis and Catuscia Palamidessi and Marco Stronati}
}

@ARTICLE{Qiu-TMC2022,
  author={C. {Qiu} and A. C. {Squicciarini} and C. {Pang} and N. {Wang} and B. {Wu}},
  journal={IEEE Transactions on Mobile Computing}, 
  title={Location Privacy Protection in Vehicle-Based Spatial Crowdsourcing via Geo-Indistinguishability}, 
  year={2022},
  volume={},
  number={},
  pages={1-1},
  doi={10.1109/TMC.2020.3037911}}

@inproceedings{Feyisetan-WDSM2020,
author = {Feyisetan, Oluwaseyi and Balle, Borja and Drake, Thomas and Diethe, Tom},
title = {Privacy- and Utility-Preserving Textual Analysis via Calibrated Multivariate Perturbations},
year = {2020},
isbn = {9781450368223},
publisher = {Association for Computing Machinery},
address = {New York, NY, USA},
url = {https://doi.org/10.1145/3336191.3371856},
doi = {10.1145/3336191.3371856},
booktitle = {Proceedings of the 13th International Conference on Web Search and Data Mining},
pages = {178–186},
numpages = {9},
location = {Houston, TX, USA},
series = {WSDM '20}
}

@INPROCEEDINGS {Chen-CVPR2021,
author = { Chen, Jia-Wei and Chen, Li-Ju and Yu, Chia-Mu and Lu, Chun-Shien },
booktitle = { 2021 IEEE/CVF Conference on Computer Vision and Pattern Recognition (CVPR) },
title = {{ Perceptual Indistinguishability-Net (PI-Net): Facial Image Obfuscation with Manipulable Semantics }},
year = {2021},
volume = {},
ISSN = {},
pages = {6474-6483},
doi = {10.1109/CVPR46437.2021.00641},
url = {https://doi.ieeecomputersociety.org/10.1109/CVPR46437.2021.00641},
publisher = {IEEE Computer Society},
address = {Los Alamitos, CA, USA},
month =Jun}

@inproceedings{Xiao-CCS2015,
author = {Xiao, Yonghui and Xiong, Li},
title = {Protecting Locations with Differential Privacy under Temporal Correlations},
year = {2015},
isbn = {9781450338325},
publisher = {Association for Computing Machinery},
address = {New York, NY, USA},
url = {https://doi.org/10.1145/2810103.2813640},
doi = {10.1145/2810103.2813640},
pages = {1298–1309},
numpages = {12},
location = {Denver, Colorado, USA},
series = {CCS '15}
}

@inproceedings{oyaGeoIndLooking,
author = {Oya, Simon and Troncoso, Carmela and P\'{e}rez-Gonz\'{a}lez, Fernando},
title = {Is Geo-Indistinguishability What You Are Looking For?},
year = {2017},
isbn = {9781450351751},
publisher = {Association for Computing Machinery},
address = {New York, NY, USA},
url = {https://doi.org/10.1145/3139550.3139555},
doi = {10.1145/3139550.3139555},
booktitle = {Proc. of the 2017 on Workshop on Privacy in the Electronic Society},
pages = {137–140},
numpages = {4},
location = {Dallas, Texas, USA},
series = {WPES '17}
}

@inproceedings{galli2022group,
title={Group privacy for personalized federated learning},
author={Filippo Galli and Sayan Biswas and Kangsoo Jung and Tommaso Cucinotta and Catuscia Palamidessi},
booktitle={Workshop on Federated Learning: Recent Advances and New Challenges (in Conjunction with NeurIPS 2022)},
year={2022},
url={https://openreview.net/forum?id=R45g30SnwsR}
}

@article{yang2021blockchain,
  title={Blockchain-based indoor location paging and answering service with truncated-geo-indistinguishability},
  author={Yang, Changxin et al.},
  journal={IET Blockchain},
  year={2021},
}

@article{Carvalho2021TEMHU,
  title={TEM: High Utility Metric Differential Privacy on Text},
  author={Ricardo Silva Carvalho and Theodore Vasiloudis and Oluwaseyi Feyisetan},
  journal={ArXiv},
  year={2021},
  volume={abs/2107.07928},
  url={https://api.semanticscholar.org/CorpusID:236034456}
}

@article{Min-TDSC2024,
  author={Minghui Min and Haopeng Zhu and Jiahao Ding and Shiyin Li and Liang Xiao and Miao Pan and Zhu Han},
  title={Personalized 3D Location Privacy Protection With Differential and Distortion Geo-Perturbation},
  year={2024},
  month={July - August},
  cdate={1722470400000},
  journal={IEEE Trans. Dependable Secur. Comput.},
  volume={21},
  number={4},
  pages={3629-3643},
  url={https://doi.org/10.1109/TDSC.2023.3335374}
}

@Inproceedings{ImolaUAI2022,
 author = {Jacob Imola and Shiva Kasiviswanathan and Stephen White and Abhinav Aggarwal and Nathanael Teissier},
 title = {Balancing utility and scalability in metric differential privacy},
 year = {2022},
 booktitle = {Proc. of UAI 2022},
}

@inproceedings{Han-ICME2020,
author = { Han, Yaowei and Li, Sheng and Cao, Yang and Ma, Qiang and Yoshikawa, Masatoshi },
booktitle = { 2020 IEEE International Conference on Multimedia and Expo (ICME) },
title = {{ Voice-Indistinguishability: Protecting Voiceprint In Privacy-Preserving Speech Data Release }},
year = {2020},
volume = {},
ISSN = {},
pages = {1-6},
doi = {10.1109/ICME46284.2020.9102875},
url = {https://doi.ieeecomputersociety.org/10.1109/ICME46284.2020.9102875},
publisher = {IEEE Computer Society},
address = {Los Alamitos, CA, USA},
month =Jul}

@article {Chatzikokolakis-PPET2015,
      author = "Konstantinos Chatzikokolakis and Catuscia Palamidessi and Marco Stronati",
      title = "Constructing elastic distinguishability metrics for location privacy",
      journal = "Proc. on Privacy Enhancing Technologies",
      year = "2015",
      publisher = "Sciendo",
      address = "Berlin",
      volume = "2015",
      number = "2",
      pages = "156 - 170"
}

@inproceedings{Qiu-PETS2025,
title={Time-Efficient Locally Relevant Geo-Location Privacy Protection}, 
author={Chenxi Qiu and Ruiyao Liu and Primal Pappachan and Anna Squicciarini and Xinpeng Xie},
year={2025},
booktitle={Proc. on Privacy Enhancing Technologies
},
}

@ARTICLE{Ma-TITS2022,
  author={Ma, Baihe and Wang, Xu and Ni, Wei and Liu, Ren Ping},
  journal={IEEE Transactions on Intelligent Transportation Systems}, 
  title={Personalized Location Privacy With Road Network-Indistinguishability}, 
  year={2022},
  volume={23},
  number={11},
  pages={20860-20872},
  doi={10.1109/TITS.2022.3179501}}

@article{Arcolezi-VLDB2023,
author = {Arcolezi, H\'{e}ber H. and Gambs, S\'{e}bastien and Couchot, Jean-Fran\c{c}ois and Palamidessi, Catuscia},
title = {On the Risks of Collecting Multidimensional Data Under Local Differential Privacy},
year = {2023},
issue_date = {January 2023},
publisher = {VLDB Endowment},
volume = {16},
number = {5},
issn = {2150-8097},
url = {https://doi.org/10.14778/3579075.3579086},
doi = {10.14778/3579075.3579086},
journal = {Proc. VLDB Endow.},
month = jan,
pages = {1126–1139},
numpages = {14}
}

@inproceedings{Andres-CCS2013,
author = {Andr\'{e}s, Miguel E. and Bordenabe, Nicol\'{a}s E. and Chatzikokolakis, Konstantinos and Palamidessi, Catuscia},
title = {Geo-indistinguishability: differential privacy for location-based systems},
year = {2013},
isbn = {9781450324779},
publisher = {Association for Computing Machinery},
address = {New York, NY, USA},
url = {https://doi.org/10.1145/2508859.2516735},
doi = {10.1145/2508859.2516735},
booktitle = {Proceedings of the 2013 ACM SIGSAC Conference on Computer \& Communications Security},
pages = {901–914},
numpages = {14},
location = {Berlin, Germany},
series = {CCS '13}
}

@inproceedings{fan-ICME2019,
   author={Fan, Liyue},
  booktitle={2019 IEEE International Conference on Multimedia and Expo (ICME)}, 
  title={Practical Image Obfuscation with Provable Privacy}, 
  year={2019},
  volume={},
  number={},
  pages={784-789},
  doi={10.1109/ICME.2019.00140}}

@article{Yu-TSP2023,
author = {Yu, Dan and Shi, Xiufang and Chai, Li and Zhang, Wen-An and Chen, Jiming},
year = {2023},
month = {01},
pages = {1-14},
title = {Balancing Localization Accuracy and Location Privacy in Mobile Cooperative Localization},
volume = {PP},
journal = {IEEE Transactions on Signal Processing},
doi = {10.1109/TSP.2023.3292505}
}

@article{Fernandes2018AuthorOU,
  title={Author Obfuscation Using Generalised Differential Privacy},
  author={Natasha Fernandes and Mark Dras and Annabelle McIver},
  journal={ArXiv},
  year={2018},
  volume={abs/1805.08866},
  url={https://api.semanticscholar.org/CorpusID:43942677}
}

@misc{openstreetmap,
    title={openstreetmap},
    howpublished = {\url{https://www.openstreetmap.org/}},
    note = {Accessed: 2020-04-07},
    year = 2020
}

@inproceedings{Qiu-IJCAI2024,
      title={Enhancing Scalability of Metric Differential Privacy via Secret Dataset Partitioning and Benders Decomposition}, 
      author={Chenxi Qiu},
      year={2024},
      booktitle={Proc. of 33rd International Joint Conference on Artificial Intelligence (IJCAI)},
}

@misc{koufogiannis2015,
      title={Optimality of the Laplace Mechanism in Differential Privacy}, 
      author={Fragkiskos Koufogiannis and Shuo Han and George J. Pappas},
      year={2015},
      eprint={1504.00065},
      archivePrefix={arXiv},
      primaryClass={cs.CR},
      url={https://arxiv.org/abs/1504.00065}, 
}

@Misc{MATLABlinprog,
howpublished = {\url{https://www.mathworks.com/help/optim/ug/linprog.html}},
note = {Accessed in January 2024},
title = {{linprog: Solve linear programming problems}},
year = {2024},
}

\cleardoublepage
\section{Math Notations}
\label{sec:mathnotitions}

\begin{table}[ht]
\centering
\caption{Summary of Mathematical Notation}
\label{tab:notation}
\small 
\begin{tabular}{p{1.5cm} p{6.2cm}}
\toprule
\textbf{Symbol} & \textbf{Description} \\
\midrule
\multicolumn{2}{l}{\textbf{Sets and Domains}} \\
\midrule
$\mathcal{X}$ & Continuous secret data domain \\
$\hat{\mathcal{X}}$ & Set of anchor records (corners of all $N$-orthotopes) \\
$\mathcal{Y}$ & Perturbed (output) data domain \\
$\mathcal{C}_m$ & The $m$-th $N$-orthotope (hyperrectangle) partition of $\mathcal{X}$ \\
$\hat{\mathcal{X}}_m$ & Anchor set for partition $\mathcal{C}_m$ \\
$\mathcal{Q}$ & Set of task locations for utility evaluation \\
\midrule
\multicolumn{2}{l}{\textbf{Variables and Vectors}} \\
\midrule
$\mathbf{x}_i, \mathbf{x}_j$ & Secret input records \\
$\mathbf{y}_k$ & Perturbed output records \\
$\hat{\mathbf{x}}_{i}$ & Anchor record indexed by $i$ \\
$\boldsymbol{\gamma} \in \{0,1\}^N$ & Binary vector to enumerate cube corners \\
$\epsilon_\ell$ & Privacy budget allocated to dimension $\ell$ \\
$\Delta_\ell$ & Length of cube edge along dimension $\ell$ \\
\midrule
\multicolumn{2}{l}{\textbf{Functions and Distributions}} \\
\midrule
$d_p(\mathbf{x}, \mathbf{x}')$ & $\ell_p$-norm distance between two records \\
$p(\mathbf{x})$ & Prior distribution over secret data domain $\mathcal{X}$ \\
$\mathcal{M}$ & Randomized perturbation mechanism \\
$z(\mathbf{y}_k \mid \mathbf{x})$ & Probability of reporting $\mathbf{y}_k$ given input $\mathbf{x}$ \\
$\mathcal{L}(\mathbf{x}, \mathbf{y}_k)$ & Task-specific utility loss between $\mathbf{x}$ and $\mathbf{y}_k$ \\
$\mathcal{L}(\mathbf{Z})$ & Expected utility loss under perturbation matrix $\mathbf{Z}$ \\
\midrule
\multicolumn{2}{l}{\textbf{Optimization and Mechanisms}} \\
\midrule
AIPO & Anchor-based Interpolation for enforcing $\ell_p$-mDP \\
AIPO-R & Relaxed variant of AIPO (enforces mDP only at anchor pairs) \\
APO & Anchor Perturbation Optimization (jointly optimizes anchors and budget) \\
Approx-APO & Linear approximation of APO with bounded optimality gap \\
$\mathbf{Z}$ & Perturbation matrix \{$z(\mathbf{y}_k \mid \mathbf{x}_i)$\} \\
\midrule
\multicolumn{2}{l}{\textbf{Constants and Parameters}} \\
\midrule
$\epsilon$ & Global privacy budget for $\ell_p$-mDP \\
$N$ & Dimensionality of the secret domain \\
$M$ & Number of $N$-orthotope partitions \\
% $\tau$ & Partitioning threshold (e.g., for anchor refinement) \\
\bottomrule
\end{tabular}
\end{table}

\section{Omitted Proofs}
\label{sec:proofs}

% \newpage 
\subsection{Proof of Proposition \ref{prop:intral}: Intra-Interval Validity}
\label{subsec:proof:prop:intral}
\begin{reproposition}[Intra-interval validity]
Let $\mathbf{x}_i$ and $\mathbf{x}_{i'}$ be two records that differ only in their $\ell$th coordinate, with $x_{i, \ell} < x_{i', \ell}$, and suppose their corresponding log-perturbation probabilities $\ln z(\mathbf{y}_k \mid \mathbf{x}_i)$ and $\ln z(\mathbf{y}_k \mid \mathbf{x}_{i'})$ satisfy the $(\epsilon, d_1)$-Lipschitz bound. Then, for any $\mathbf{x}_a, \mathbf{x}_b \in \mathcal{X}$ such that $x_{a, \ell}, x_{b, \ell} \in [x_{i, \ell}, x_{i', \ell}]$ and all other coordinates are identical to those of $\mathbf{x}_i$, if the interpolated values $\hat{z}(\mathbf{y}_k \mid \mathbf{x}_a)$ and $\hat{z}(\mathbf{y}_k \mid \mathbf{x}_b)$ are calculated by 
\begin{align}
\hat{z}(\mathbf{y}_k \mid \mathbf{x}_a) \lnconv\left(z(\mathbf{y}_k \mid \mathbf{x}_i), z(\mathbf{y}_k \mid \mathbf{x}_{i'})\right), \\
\hat{z}(\mathbf{y}_k \mid \mathbf{x}_b)\lnconv\left(z(\mathbf{y}_k \mid \mathbf{x}_i), z(\mathbf{y}_k \mid \mathbf{x}_{i'})\right),
\end{align}
then $\ln \hat{z}(\mathbf{y}_k \mid \mathbf{x}_a)$ and $\ln \hat{z}(\mathbf{y}_k \mid \mathbf{x}_b)$ also satisfy the $(\epsilon, d_1)$-Lipschitz bound between $\mathbf{x}_a$ and $\mathbf{x}_b$.
This property is illustrated in Fig.~\ref{fig:property}(a).
\end{reproposition}
\begin{proof}
The $\ell$th coordinates of $\mathbf{x}_a$ and $\mathbf{x}_b$ can be expressed as convex combinations of $\mathbf{x}_i$ and $\mathbf{x}_{i'}$, i.e.,
\begin{eqnarray}
x_{a,\ell} &=& \lambda_{\mathbf{x}_i, \mathbf{x}_a}^{\ell} x_{i,\ell} + (1 - \lambda_{\mathbf{x}_i, \mathbf{x}_a}^{\ell})(x_{i,\ell} + \Delta), \\
x_{b,\ell} &=& \lambda_{\mathbf{x}_i, \mathbf{x}_b}^{\ell} x_{i,\ell} + (1 - \lambda_{\mathbf{x}_i, \mathbf{x}_b}^{\ell})(x_{i,\ell} + \Delta).
\end{eqnarray}
Then, their $\ell_1$ distance is:
\begin{equation}
d_1(\mathbf{x}_a, \mathbf{x}_b) = |x_{a,\ell} - x_{b,\ell}| = |\lambda_{\mathbf{x}_i, \mathbf{x}_a}^{\ell} - \lambda_{\mathbf{x}_i, \mathbf{x}_b}^{\ell}| \cdot \Delta.
\end{equation}
From the log-convexity assumption:
\begin{eqnarray}
&& \ln z(\mathbf{y}_k \mid \mathbf{x}_a) 
\\ \nonumber 
&=& \lambda_{\mathbf{x}_i, \mathbf{x}_a}^{\ell} \ln z(\mathbf{y}_k \mid \mathbf{x}_i) + (1 - \lambda_{\mathbf{x}_i, \mathbf{x}_a}^{\ell}) \ln z(\mathbf{y}_k \mid \mathbf{x}_{i'}) \\
&& \ln z(\mathbf{y}_k \mid \mathbf{x}_b) \\ \nonumber
&=& \lambda_{\mathbf{x}_i, \mathbf{x}_b}^{\ell} \ln z(\mathbf{y}_k \mid \mathbf{x}_i) + (1 - \lambda_{\mathbf{x}_i, \mathbf{x}_b}^{\ell}) \ln z(\mathbf{y}_k \mid \mathbf{x}_{i'})
\end{eqnarray}

Subtracting the two expressions:
\begin{eqnarray}
\nonumber && \ln z(\mathbf{y}_k \mid \mathbf{x}_a) - \ln z(\mathbf{y}_k \mid \mathbf{x}_b)
\\ \nonumber
&=& (\lambda_{\mathbf{x}_i, \mathbf{x}_a}^{\ell} - \lambda_{\mathbf{x}_i, \mathbf{x}_b}^{\ell})(\ln z(\mathbf{y}_k \mid \mathbf{x}_i) - \ln z(\mathbf{y}_k \mid \mathbf{x}_{i'})) \\
&=& \frac{d_1(\mathbf{x}_a, \mathbf{x}_b)}{\Delta} \cdot (\ln z(\mathbf{y}_k \mid \mathbf{x}_i) - \ln z(\mathbf{y}_k \mid \mathbf{x}_{i'})).
\end{eqnarray}

Since $z(\cdot \mid \mathbf{x}_i)$ and $z(\cdot \mid \mathbf{x}_{i'})$ satisfy $(\epsilon_\ell, d_1)$-mDP:
\begin{equation}
|\ln z(\mathbf{y}_k \mid \mathbf{x}_i) - \ln z(\mathbf{y}_k \mid \mathbf{x}_{i'})| \leq \epsilon_\ell \cdot \Delta,
\end{equation}
from which we obtain:
\begin{equation}
|\ln z(\mathbf{y}_k \mid \mathbf{x}_a) - \ln z(\mathbf{y}_k \mid \mathbf{x}_b)| \leq \epsilon_\ell \cdot d_1(\mathbf{x}_a, \mathbf{x}_b).
\end{equation}
Hence, the perturbation distributions satisfy $(\epsilon_\ell, d_1)$-Lipschitz bound between $\mathbf{x}_a$ and $\mathbf{x}_b$. The proof is complete.
\end{proof}

\subsection{Proof of Proposition \ref{prop:across}: Across-Interval Validity}
\begin{reproposition}[Across-interval Validity]
Let $\mathbf{x}_i, \mathbf{x}_{i'}, \mathbf{x}_j, \mathbf{x}_{j'}$ be four records that differ only in their $\ell$th coordinate, with $x_{i, \ell} < x_{i', \ell} \leq x_{j, \ell} < x_{j', \ell}$. Suppose that each pair of log-perturbation probabilities $\ln z(\mathbf{y}_k \mid \mathbf{x}_i)$, $\ln z(\mathbf{y}_k \mid \mathbf{x}_{i'})$, $\ln z(\mathbf{y}_k \mid \mathbf{x}_j)$, and $\ln z(\mathbf{y}_k \mid \mathbf{x}_{j'})$ satisfy $(\epsilon, d_1)$-Lipschitz bound.

Let $\mathbf{x}_a$ and $\mathbf{x}_b$ be two additional records that differ from the above four points only in the $\ell$th coordinate, with $x_{a,\ell} \in [x_{i,\ell}, x_{i',\ell}]$ and $x_{b,\ell} \in [x_{j,\ell}, x_{j',\ell}]$. If the corresponding interpolated values $\hat{z}(\mathbf{y}_k \mid \mathbf{x}_a)$ and $\hat{z}(\mathbf{y}_k \mid \mathbf{x}_b)$ are defined via log-convex interpolation as:
\begin{eqnarray}
&& \hat{z}(\mathbf{y}_k \mid \mathbf{x}_a) \lnconv \left(z(\mathbf{y}_k \mid \mathbf{x}_i), z(\mathbf{y}_k \mid \mathbf{x}_{i'})\right) \\
&&  
\hat{z}(\mathbf{y}_k \mid \mathbf{x}_b) \lnconv \left(z(\mathbf{y}_k \mid \mathbf{x}_j), z(\mathbf{y}_k \mid \mathbf{x}_{j'})\right),
\end{eqnarray}
then the pair $\left(\hat{z}(\mathbf{y}_k \mid \mathbf{x}_a), \hat{z}(\mathbf{y}_k \mid \mathbf{x}_b)\right)$ satisfies $(\epsilon, d_1)$-Lipschitz bound between $\mathbf{x}_a$ and $\mathbf{x}_b$. This property is illustrated in Fig.~\ref{fig:property}(b).
\end{reproposition}
\begin{proof}
Since both $z(\mathbf{y}_k|\mathbf{x}_{a})$ and $z(\mathbf{y}_k|\mathbf{x}_{b})$ satisfy the logarithmic convexity condition in Eq.~(\ref{eq:logconvex})
\begin{eqnarray}
&& \ln z(\mathbf{y}_k|\mathbf{x}_{a}) \\
&=& \lambda_{\mathbf{x}_i, \mathbf{x}_a}^{\ell} \ln z(\mathbf{y}_k|\mathbf{x}_{i}) + (1- \lambda_{\mathbf{x}_i, \mathbf{x}_a}^{\ell}) \ln z(\mathbf{y}_k|\mathbf{x}_{i'}).  \\
&& \ln z(\mathbf{y}_k|\mathbf{x}_{b}) \\
&=& \lambda_{\mathbf{x}_j, \mathbf{x}_b}^{\ell} \ln z(\mathbf{y}_k|\mathbf{x}_{j}) + (1- \lambda_{\mathbf{x}_j, \mathbf{x}_b}^{\ell}) \ln z(\mathbf{y}_k|\mathbf{x}_{j'}).
\end{eqnarray}
We decompose the log-ratio between $z(\mathbf{y}_k \mid \mathbf{x}_a)$ and $z(\mathbf{y}_k \mid \mathbf{x}_b)$ into three additive segments:
\begin{eqnarray}
&& \ln z(\mathbf{y}_k|\mathbf{x}_{a}) - \ln z(\mathbf{y}_k|\mathbf{x}_{b}) \\ \nonumber 
&=& \ln z(\mathbf{y}_k|\mathbf{x}_{a}) - \underbrace{\ln z(\mathbf{y}_k|\mathbf{x}_{i'}) + \ln z(\mathbf{y}_k|\mathbf{x}_{i'})}_{=0} \\
&-& \underbrace{\ln z(\mathbf{y}_k|\mathbf{x}_{j})+\ln z(\mathbf{y}_k|\mathbf{x}_{j})}_{=0}- \ln z(\mathbf{y}_k|\mathbf{x}_{b}) \\
&=& \underbrace{\ln z(\mathbf{y}_k|\mathbf{x}_{a}) - \ln z(\mathbf{y}_k|\mathbf{x}_{i'})}_{\small \mathrm{Component}_1} \\ 
&+&  \underbrace{\ln z(\mathbf{y}_k|\mathbf{x}_{i'}) - \ln z(\mathbf{y}_k|\mathbf{x}_{j})}_{\mathrm{Component}_2} \\
&+&\underbrace{\ln z(\mathbf{y}_k|\mathbf{x}_{j})- \ln z(\mathbf{y}_k|\mathbf{x}_{b})}_{\mathrm{Component}_3}  
\end{eqnarray} 
Using the fact that all involved distributions satisfy $(\epsilon_\ell, d_1)$-Lipschitz bound, we apply distance-based upper bounds:
\begin{eqnarray}
|\mathrm{Component}_1| 
&=& |\ln z(\mathbf{y}_k \mid \mathbf{x}_a) - \ln z(\mathbf{y}_k \mid \mathbf{x}_{i'})| \\
&\leq& \epsilon_\ell \cdot |x_{i', \ell} - x_{a, \ell}|, \\
|\mathrm{Component}_2| 
&=& 
|\ln z(\mathbf{y}_k \mid \mathbf{x}_{i'}) - \ln z(\mathbf{y}_k \mid \mathbf{x}_{j})| \\
&\leq& \epsilon_\ell \cdot |x_{j, \ell} - x_{i', \ell}|, \\
|\mathrm{Component}_3| 
&=& 
|\ln z(\mathbf{y}_k \mid \mathbf{x}_{j}) - \ln z(\mathbf{y}_k \mid \mathbf{x}_b)| \\
&\leq& \epsilon_\ell \cdot |x_{b, \ell} - x_{j, \ell}|.
\end{eqnarray}
Summing these bounds:
\begin{eqnarray}
&& |\ln z(\mathbf{y}_k \mid \mathbf{x}_a) - \ln z(\mathbf{y}_k \mid \mathbf{x}_b)| \\ \nonumber 
&\leq& \epsilon_\ell \cdot \big(|x_{i', \ell} - x_{a, \ell}| + |x_{j, \ell} - x_{i', \ell}| + |x_{b, \ell} - x_{j, \ell}|\big) \\
&=& \epsilon_\ell \cdot |x_{b, \ell} - x_{a, \ell}| \\
&=& \epsilon_\ell \cdot d_1(\mathbf{x}_a, \mathbf{x}_b).
\end{eqnarray}
Thus, $z(\mathbf{y}_k \mid \mathbf{x}_a)$ and $z(\mathbf{y}_k \mid \mathbf{x}_b)$ satisfy $(\epsilon_\ell, d_1)$-Lipschitz bound, as required. The proof is completed. 
\end{proof}

\subsection{Proof of Theorem \ref{thm:composition}: Dimension-wise Composition for Lipschitz Bound Condition}
\label{subsec:proof:thm:composition}

\begin{reptheorem}
[\textbf{Dimension-wise Composition for Lipschitz Bound Condition}]
\label{thm:composition}
Let $f:\mathcal X\to\mathbb R$ be a mechanism that interpolates values in an $N$-dimensional space. Suppose that for each $\ell \in \{1, \dots, N\}$, $f$ satisfies $(\epsilon_\ell, d_1)$-Lipschitz bound when the input records differ only in the $\ell$th coordinate. If the parameters $\epsilon_1, \dots, \epsilon_N$ satisfy the following budget composition condition:
\begin{eqnarray}
\label{eq:budgetcompo1_}
\textstyle
&& \sum_{\ell=1}^{N} \epsilon_\ell^{\frac{p}{p-1}} \leq \epsilon^{\frac{p}{p-1}}, \quad \text{for } p > 1, 
\\  \label{eq:budgetcompo2_}
\text{and} && 
\max_{\ell} \epsilon_\ell \leq \epsilon, \quad \text{for } p = 1,
\end{eqnarray}
then $f$ is $(\epsilon, d_p)$-Lipschitz continuous.
\end{reptheorem}
\begin{proof}
For any pair $\mathbf{x}_a,\mathbf{x}_b\in\mathcal{X}$, we write the coordinate difference as $\Delta=\mathbf{x}'-\mathbf{x}$ with components
$\Delta_\ell = x'_\ell - x_\ell$ for $\ell\in\{1,\dots,N\}$, and let $\mathbf{e}_\ell$ denote the
$\ell$th standard basis vector. Consider the axis-aligned path that changes one coordinate at a time:
\begin{equation}
\mathbf{x}^{(0)} := \mathbf{x}, ~
\mathbf{x}^{(\ell)} := \mathbf{x}^{(\ell-1)} + \Delta_{\ell}\mathbf{e}_{\ell}
\quad\text{for } \ell=1,\dots,N,  
\end{equation}
i.e., at the $\ell$th step only coordinate $\ell$ changes while all other coordinates are held fixed at
their current values.

Given that the mechanism satisfies the
$(\epsilon_{\ell},d_1)$-Lipschitz bound along each coordinate, for each $\ell$ and for any $\mathbf{y}_k \in \mathcal{Y}$,
\begin{equation}\label{eq:onedim-step}
|\ln z(\mathbf{y}_k|\mathbf{x}^{(\ell)}) - \ln z(\mathbf{y}_k|\mathbf{x}^{(\ell-1)})
| \leq \epsilon_{\ell}|\Delta_{\ell}|
.
\end{equation}
Iterating \eqref{eq:onedim-step} along the $N$ steps and telescoping yields
\begin{eqnarray}
\nonumber 
&&\left|\ln z(\mathbf{y}_k|\mathbf{x}') - \ln z(\mathbf{y}_k|\mathbf{x})\right| \\
&\leq& \sum_{\ell = 1}^N \left|\ln z(\mathbf{y}_k|\mathbf{x}^{(\ell)}) - \ln z(\mathbf{y}_k|\mathbf{x}^{(\ell-1)})\right| \nonumber \\ \label{eq:sum-bound}
&\leq& \sum_{\ell=1}^N \epsilon_\ell|\Delta_\ell|.
\end{eqnarray}

% \paragraph{From coordinate-sum to $\ell_p$ distance via Hölder.}
Given $p\in[1,\infty]$, we let $q$ be its Hölder dual ($1/p+1/q=1$). Define the vectors $\boldsymbol{\epsilon}=(\epsilon_1,\dots,\epsilon_N)$ and $\Delta=(\Delta_1,\dots,\Delta_N)$. 
By Hölder’s inequality,
\begin{equation}
\label{eq:holder}
\sum_{\ell=1}^N \epsilon_\ell|\Delta_\ell|
\leq  \|\boldsymbol{\epsilon}\|_{q}\;\|\Delta\|_{p}.
\end{equation}
Note that according to Eq. (\ref{eq:budgetcompo1}), and $q = \frac{p}{p-1}$
\begin{eqnarray}
\sum_{\ell=1}^{N} \epsilon_\ell^{\frac{p}{p-1}} \leq \epsilon^{\frac{p}{p-1}}
&\Rightarrow& \sum_{\ell=1}^{N} \epsilon_\ell^{q} \leq \epsilon^{q} \\
&\Rightarrow&  \|\boldsymbol{\epsilon}\|_{q} \leq \epsilon,
\end{eqnarray}
and by definition 
$d_p(\mathbf{x}_a,\mathbf{x}_b) := \|\mathbf{x}_b-\mathbf{x}_a\|_{p} = \|\Delta\|_{p}$.
Therefore, we can obtain that 
\begin{eqnarray}
&& \left|\ln z(\mathbf{y}_k|\mathbf{x}') - \ln z(\mathbf{y}_k|\mathbf{x})\right| \\
&\leq& \sum_{\ell=1}^N \epsilon_\ell|\Delta_\ell| \quad \text{(according to Eq. (\ref{eq:sum-bound}))} \\
&\leq & \|\boldsymbol{\epsilon}\|_{q}\;\|\Delta\|_{p} \quad \text{(according to Eq. (\ref{eq:holder}))} \\
&\leq & \epsilon d_p(\mathbf{x}_a,\mathbf{x}_b)
\end{eqnarray}
which is precisely the $(\epsilon,d_p)$-Lipschitz bound between $\mathbf{x}$ and $\mathbf{x}'$.

\paragraph{Remarks on edge cases.}
For $p=1$ (so $q=\infty$), the bound reduces to
$\sum_{\ell}\epsilon_\ell|\Delta_\ell|\le \|\boldsymbol{\epsilon}\|_{\infty}\|\Delta\|_{1}$.
For $p=\infty$ (so $q=1$), it becomes
$\sum_{\ell}\epsilon_\ell|\Delta_\ell|\le \|\boldsymbol{\epsilon}\|_{1}\|\Delta\|_{\infty}$.
Both cases are encompassed by \eqref{eq:holder} and thus by the argument above.

This establishes that coordinating the per-dimension budgets
$\{\epsilon_\ell\}_{\ell=1}^N$ via $\|\boldsymbol{\epsilon}\|_q \leq \epsilon$ yields a mechanism
that satisfies the global $(\epsilon,d_p)$-Lipschitz bound on $\mathcal{X}$.
\end{proof}

\DEL{
\subsection{Proof of Lemma \ref{lem:convex}: Convex Combination Preservation}

\begin{lemma}[Convex combination preservation]
\label{lem:convex}
Consider any two points $\mathbf{x}_{a}$ and $\mathbf{x}_{b}$ that differ from only in the $\ell$th coordinate. Suppose that for each $j = 1, \ldots, L$, the pair of perturbation probabilities  $\left(z^{(j)}(\mathbf{y}_{k}|\mathbf{x}_{a}), z^{(j)}(\mathbf{y}_{k}|\mathbf{x}_{b})\right)$ satisfies $(\epsilon_\ell, d_1)$-mDP. Then, for any non-negative weights $\lambda^{(1)}, \ldots, \lambda^{(L)}$ such that $\sum_{j=1}^L \lambda^{(j)} = 1$, the convex combinations  
$\sum_{j=1}^L \lambda^{(j)} z^{(j)}(\mathbf{y}_{k}|\mathbf{x}_{a})$ 
and $\sum_{j=1}^L \lambda^{(j)} z^{(j)}(\mathbf{y}_{k}|\mathbf{x}_{b})$ also satisfy $(\epsilon_\ell, d_1)$-mDP.
\end{lemma}
\begin{proposition}
\label{lem:convex2}
Let $\mathbf{x}_{a}, x_{b, \ell}, x_{c, \ell}$, and $x_{d, \ell}$ be the coordinates of $\mathbf{x}_a, \mathbf{x}_b, \mathbf{x}_c$, and $\mathbf{x}_d$ in the $\ell$th dimension such that $\mathbf{x}_{a} < x_{b, \ell} \leq x_{c, \ell} < x_{d, \ell}$, with their perturbation probabilities $z(\mathbf{y}_k|\mathbf{x}_{a})$, $z(\mathbf{y}_k|\mathbf{x}_{b})$, $z(x_{c, \ell}, \mathbf{y}_k)$, and $z(x_{d, \ell}, \mathbf{y}_k)$ satisfying $(\epsilon_\ell, d_{1})$-mDP. 

For any $x_{e, \ell} \in [\mathbf{x}_{a}, x_{b, \ell}]$ and  $x_{f, \ell} \in [x_{c, \ell}, x_{d, \ell}]$ expressible as convex combinations
\begin{equation}
x_{e, \ell} = \lambda_{e} \mathbf{x}_{a} + (1 - \lambda_{e}) x_{b, \ell}, \quad 
x_{f, \ell} = \lambda_{f} x_{c, \ell} + (1 - \lambda_{f}) x_{d, \ell},
\end{equation}
where $\lambda_{e}, \lambda_{f} \in [0,1]$, suppose the perturbation probabilities $z(x_{e, \ell}, \mathbf{y}_k)$ and $z(x_{f, \ell}, \mathbf{y}_k)$ satisfy the logarithmic convexity condition   
\begin{eqnarray}
\ln z(x_{e, \ell}, \mathbf{y}_k) &=& \lambda_{e} \ln z(\mathbf{y}_k|\mathbf{x}_{a}) +  (1- \lambda_{e})\ln z(\mathbf{y}_k|\mathbf{x}_{b}), \\
\ln z(x_{f, \ell}, \mathbf{y}_k) &=& \lambda_{f} \ln z(x_{c, \ell}, \mathbf{y}_k) + (1- \lambda_{f}) \ln z(x_{d, \ell}, \mathbf{y}_k), 
\end{eqnarray}
then $z(x_{e, \ell}, \mathbf{y}_k)$ and $z(x_{f, \ell}, \mathbf{y}_k)$ satisfy $(\epsilon_\ell, d_{1})$-mDP.
\end{proposition}}

\subsection{Proof of Theorem \ref{thm:AIPOfeasible}: Correctness of Interpolation Function}
\label{prop:thm:AIPOfeasible}
\begin{reptheorem}
[Correctness of Log-Convex Interpolation $f_{\mathrm{int}}$]
\label{thm:AIPOfeasible}
Given that $(\epsilon_\ell, d_1)$-Lipschitz bound holds between each pair of $\ell$-neighbors in $\hat{\mathcal{X}}$ and $\{\epsilon_\ell\}_{\ell=1}^N$ satisfy the privacy budget composition condition formalized in Eq. (\ref{eq:budgetcompo1})(\ref{eq:budgetcompo2}), the use of the interpolation function $f_{\mathrm{int}}$ (defined by Eq.~(\ref{eq:PPI})) guarantees that any two interpolated values within the entire secret data domain $\mathcal{X}$ satisfy $(\epsilon, d_p)$-Lipschitz bound.
\end{reptheorem}
Before proving Theorem \ref{thm:AIPOfeasible}, we introduce the following \emph{Dimension-Wise Lipschitz bound}, which provides a sufficient condition for ensuring $(\epsilon, d_p)$-mDP.
\begin{definition}[Dimension-Wise Lipschitz bound (DW-Lipschitz bound)]
\label{def:dpc}
Given coordinate-specific privacy budgets $\boldsymbol{\epsilon}=\left\{\epsilon_1, \dots, \epsilon_N\right\}$, a mechanism $f$ satisfies $(\boldsymbol{\epsilon}, d_p)$-\emph{DW-Lipschitz bound} if for any two records $\mathbf{x}_a\in \mathcal{X}_m, \mathbf{x}_b \in \mathcal{X}_{m'}$ and perturbed value $\mathbf{y}_k \in \mathcal{Y}$, it holds that
\begin{equation}
\frac{f(\mathbf{x}_a, \mathbf{y}_k, \mathbf{Z}_{\hat{\mathcal{X}}_m})}{f(\mathbf{x}_b, \mathbf{y}_k, \mathbf{Z}_{\hat{\mathcal{X}}_{m'}})} \leq e^{\sum_{\ell=1}^N \epsilon_\ell |x_{a,\ell} - x_{b,\ell}|}.
\end{equation}
\end{definition}
Satisfying the DW-Lipschitz bound ensures that the overall privacy leakage between any two records is controlled by their $\ell_p$-norm distance. We formalize this connection in the following \textbf{Lemma \ref{lem:dwmdp_sufficient}}.
\begin{lemma}
\label{lem:dwmdp_sufficient}
Suppose that an interpolation mechanism $f$ satisfies $(\boldsymbol{\epsilon}, d_p)$-\emph{DW-Lipschitz bound}, where $\boldsymbol{\epsilon} = (\epsilon_1, \dots, \epsilon_N)$ are coordinate-specific privacy parameters. Then, if the dimension-wise composition condition holds:
\begin{eqnarray}
&& \sum_{\ell=1}^N \epsilon_\ell^{\frac{p}{p-1}} \leq \epsilon^{\frac{p}{p-1}} \quad \text{for} \quad p > 1, \\
&& \max_{\ell} \epsilon_\ell \leq \epsilon \quad \text{for} \quad p = 1,
\end{eqnarray}
the mechanism $\mathcal{M}$ satisfies $(\epsilon, d_p)$-Lipschitz bound.
\end{lemma}

\begin{proof}
[Proof of Lemma \ref{lem:dwmdp_sufficient}]
Applying Hölder's inequality, we have
\begin{eqnarray}
&& \sum_{\ell=1}^N \epsilon_\ell |x_{a,\ell} - x_{b,\ell}| \\
&\leq& \left( \sum_{\ell=1}^N \epsilon_\ell^{\frac{p}{p-1}} \right)^{\frac{p-1}{p}} \left( \sum_{\ell=1}^N |x_{a,\ell} - x_{b,\ell}|^p \right)^{\frac{1}{p}} \\
&=& \epsilon d_p(\mathbf{x}_a, \mathbf{x}_b).
\end{eqnarray}
Thus, $\forall \mathbf{y}_k$
\begin{eqnarray}
&& \frac{f(\mathbf{x}_a, \mathbf{y}_k, \mathbf{Z}_{\hat{\mathcal{X}}_m})}{f(\mathbf{x}_b, \mathbf{y}_k, \mathbf{Z}_{\hat{\mathcal{X}}_{m'}})} \leq e^{\sum_{\ell=1}^N \epsilon_\ell |x_{a,\ell} - x_{b,\ell}|} \\ &\Rightarrow& 
\frac{f(\mathbf{x}_a, \mathbf{y}_k, \mathbf{Z}_{\hat{\mathcal{X}}_m})}{f(\mathbf{x}_b, \mathbf{y}_k, \mathbf{Z}_{\hat{\mathcal{X}}_{m'}})} \leq e^{\epsilon d_p(\mathbf{x}_a, \mathbf{x}_b)},
\end{eqnarray}
which establishes that $f$ satisfies $(\epsilon, d_p)$-Lipschitz bound.
\end{proof}

Next, we introduce \textbf{Lemma \ref{lem:chain}--Lemma \ref{lem:1DPPI}} as a preparation of the proof for \textbf{Theorem \ref{thm:AIPOfeasible}}.   

\begin{lemma}[Chain Rule for $(\boldsymbol{\epsilon}, d_p)$-DW-Lipschitz bound]
\label{lem:chain}
Let $\hat{\mathcal{X}}$ denote the set of anchor records obtained from the $N$-orthotope partitioning of the secret domain. If an interpolation mechanism $f$ ensures that every pair of $\ell$-neighbor anchors satisfies $(\boldsymbol{\epsilon}, d_p)$-DW-Lipschitz bound, then $f$ also guarantees that any pair of anchors $\hat{\mathbf{x}}_i, \hat{\mathbf{x}}_j \in \hat{\mathcal{X}}$ satisfies $(\boldsymbol{\epsilon}, d_p)$-DW-Lipschitz bound.
\end{lemma}
\begin{proof}[Proof of Lemma \ref{lem:chain}]
For any two anchors $\hat{\mathbf{x}}_i, \hat{\mathbf{x}}_j \in \hat{\mathcal{X}}$, we construct a path $\mathcal{P}$ by sequentially changing each coordinate from $\hat{\mathbf{x}}_i$ to $\hat{\mathbf{x}}_j$. Specifically, for each dimension $\ell \in \{1, \ldots, N\}$:
\begin{itemize}
    \item Move along the $\ell$-th coordinate in steps of size $\Delta$, where each step transitions between $\ell$-neighbor anchors (i.e., points differing only in dimension $\ell$ by exactly $\Delta$),
    \item The number of steps along dimension $\ell$ is $n_\ell =  \frac{|\hat{x}_{i,\ell} - \hat{x}_{j,\ell}|}{\Delta}$.
\end{itemize}
Thus, the total path consists of $\sum_{\ell=1}^N n_\ell$ hops, each between $\ell$-neighbor anchors. By assumption, $f$ satisfies $(\boldsymbol{\epsilon}, d_p)$-DW-Lipschitz bound between every pair of $\ell$-neighbor anchors, so at each hop along dimension $\ell$, the privacy loss is at most $\epsilon_\ell \Delta$. Therefore, the cumulative privacy leakage over the path is bounded by $\sum_{\ell=1}^N n_\ell \epsilon_\ell \Delta = \sum_{\ell=1}^N \epsilon_\ell |\hat{x}_{i,\ell} - \hat{x}_{j,\ell}| = (\boldsymbol{\epsilon}, d_p)$-DW-Lipschitz bound.
\end{proof}

\begin{figure}[t]

\centering
\begin{minipage}{0.50\textwidth}
\centering
\subfigure{
\includegraphics[width=0.70\textwidth]{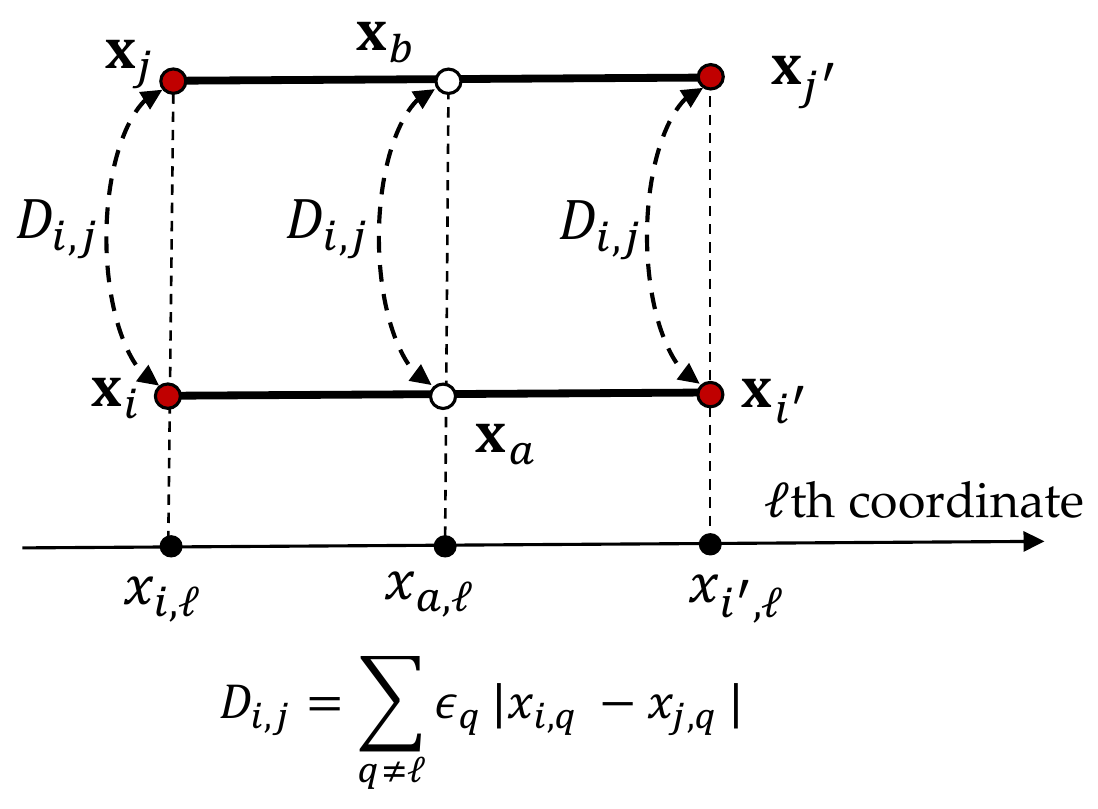}}

\caption{Proof of Lemma \ref{lem:6points}.}
\label{fig:lemma6anchors}
\end{minipage}

\end{figure}

\begin{lemma}
\label{lem:6points}
As shown in Fig.~\ref{fig:lemma6anchors}, let $\mathbf{x}_a$, $\mathbf{x}_b$, $\mathbf{x}_i$, $\mathbf{x}_{i'}$, $\mathbf{x}_j$, and $\mathbf{x}_{j'}$ be six points in $\mathbb{R}^N$ such that:
\begin{itemize}
    \item $\mathbf{x}_a$ shares the same coordinates with $\mathbf{x}_i$ and $\mathbf{x}_{i'}$ at all dimensions except $\ell$;
    \item $\mathbf{x}_b$ shares the same coordinates with $\mathbf{x}_j$ and $\mathbf{x}_{j'}$ at all dimensions except $\ell$;
    \item at dimension $\ell$, $x_{a,\ell} = x_{b,\ell}$, and $x_{i,\ell} = x_{j,\ell}$, $x_{i',\ell} = x_{j',\ell}$.
\end{itemize}
Suppose further that: both pairs $(\mathbf{x}_i, \mathbf{x}_j)$ and $(\mathbf{x}_{i'}, \mathbf{x}_{j'})$ satisfy $(\boldsymbol{\epsilon}, d_p)$-DW-Lipschitz bound; and the perturbation distributions satisfy log-convex interpolation:
\begin{eqnarray}
z(\mathbf{y}_k \mid \mathbf{x}_a) \lnconv \left(z(\mathbf{y}_k \mid \mathbf{x}_i), z(\mathbf{y}_k \mid \mathbf{x}_{i'})\right) \\
z(\mathbf{y}_k \mid \mathbf{x}_b) \lnconv \left(z(\mathbf{y}_k \mid \mathbf{x}_j), z(\mathbf{y}_k \mid \mathbf{x}_{j'})\right).
\end{eqnarray}

Then, the pair $(\mathbf{x}_a, \mathbf{x}_b)$ satisfies $(\boldsymbol{\epsilon}, d_p)$-DW-Lipschitz bound, and hence $(\epsilon, d_p)$-Lipschitz bound.
\end{lemma}

\begin{proof}[Proof of Lemma~\ref{lem:6points}]

Let $\lambda \in [0,1]$ the convex interpolation coefficient determined by the relative position of $\mathbf{x}_a$ between $\mathbf{x}_i$ and $\mathbf{x}_{i'}$ along dimension $\ell$, and similarly for $\mathbf{x}_b$ between $\mathbf{x}_j$ and $\mathbf{x}_{j'}$. Specifically, $\lambda$ is chosen such that
\begin{equation}
x_{a,\ell} = \lambda x_{i,\ell} + (1-\lambda) x_{i',\ell},
\quad
x_{b,\ell} = \lambda x_{j,\ell} + (1-\lambda) x_{j',\ell},
\end{equation}
% ensuring that the interpolation preserves the relative location along dimension $\ell$. 
The perturbation distributions at $\mathbf{x}_a$ and $\mathbf{x}_b$ are constructed via log-convex interpolation:
\begin{equation}
\ln z(\mathbf{y}_k \mid \mathbf{x}_a) = \lambda \ln z(\mathbf{y}_k \mid \mathbf{x}_i) + (1-\lambda) \ln z(\mathbf{y}_k \mid \mathbf{x}_{i'}),
\end{equation}
\begin{equation}
\ln z(\mathbf{y}_k \mid \mathbf{x}_b) = \lambda \ln z(\mathbf{y}_k \mid \mathbf{x}_j) + (1-\lambda) \ln z(\mathbf{y}_k \mid \mathbf{x}_{j'}).
\end{equation}

Applying the triangle inequality, we have
\begin{eqnarray}
\label{eq:triangle1}
&& |\ln z(\mathbf{y}_k \mid \mathbf{x}_a) - \ln z(\mathbf{y}_k \mid \mathbf{x}_b)| \\
&=& \left| \lambda \left( \ln z(\mathbf{y}_k \mid \mathbf{x}_i) - \ln z(\mathbf{y}_k \mid \mathbf{x}_j) \right) \right.
\\ 
&+& \left.(1-\lambda) \left( \ln z(\mathbf{y}_k \mid \mathbf{x}_{i'}) - \ln z(\mathbf{y}_k \mid \mathbf{x}_{j'}) \right) \right| \\ \label{eq:triangle2}
&\leq& \lambda \left| \ln z(\mathbf{y}_k \mid \mathbf{x}_i) - \ln z(\mathbf{y}_k \mid \mathbf{x}_j) \right| \\
&+& (1-\lambda) \left| \ln z(\mathbf{y}_k \mid \mathbf{x}_{i'}) - \ln z(\mathbf{y}_k \mid \mathbf{x}_{j'}) \right|.
\end{eqnarray}

By the assumption that $(\mathbf{x}_i, \mathbf{x}_j)$ and $(\mathbf{x}_{i'}, \mathbf{x}_{j'})$ satisfy DW-Lipschitz bound, we have
\begin{eqnarray}
\left| \ln z(\mathbf{y}_k \mid \mathbf{x}_i) - \ln z(\mathbf{y}_k \mid \mathbf{x}_j) \right| \leq \sum_{q=1}^N \epsilon_q |x_{i,q} - x_{j,q}|, \\
\left| \ln z(\mathbf{y}_k \mid \mathbf{x}_{i'}) - \ln z(\mathbf{y}_k \mid \mathbf{x}_{j'}) \right| \leq \sum_{q=1}^N \epsilon_q |x_{i',q} - x_{j',q}|.
\end{eqnarray}
By construction: we have $q \neq \ell$, $x_{i,q} = x_{i',q}$ and $x_{j,q} = x_{j',q}$, and hence $x_{i,q} - x_{j,q} = x_{i',q} - x_{j',q}$ for all $q \neq \ell$. Also, we have $x_{i,\ell} = x_{j,\ell}$, $x_{i',\ell} = x_{j',\ell}$. Therefore:
\begin{equation}
x_{i,\ell} - x_{j,\ell} = 0, \quad x_{i',\ell} - x_{j',\ell} = 0.
\end{equation}
and for $q \neq \ell$:
\begin{equation}
x_{i,q} - x_{j,q} = x_{i',q} - x_{j',q}.
\end{equation}

Thus, the privacy bounds simplify to:
\begin{eqnarray}
\left| \ln z(\mathbf{y}_k \mid \mathbf{x}_i) - \ln z(\mathbf{y}_k \mid \mathbf{x}_j) \right| \leq \sum_{q\neq \ell} \epsilon_q |x_{i,q} - x_{j,q}|, \\
\left| \ln z(\mathbf{y}_k \mid \mathbf{x}_{i'}) - \ln z(\mathbf{y}_k \mid \mathbf{x}_{j'}) \right| \leq \sum_{q\neq \ell} \epsilon_q |x_{i',q} - x_{j',q}|,
\end{eqnarray}
where for $q \neq \ell$, the differences match.

Substituting back into the triangle inequality in Eq. (\ref{eq:triangle1})--(\ref{eq:triangle2}), we get
\begin{equation}
|\ln z(\mathbf{y}_k \mid \mathbf{x}_a) - \ln z(\mathbf{y}_k \mid \mathbf{x}_b)|
\leq \sum_{q\neq \ell} \epsilon_q |x_{i,q} - x_{j,q}|,
\end{equation}
where the right-hand side depends only on differences at dimensions $q \neq \ell$.

Since $\mathbf{x}_a$ and $\mathbf{x}_b$ share the same value at dimension $\ell$ and differ only at other coordinates, the above bound implies that $(\mathbf{x}_a, \mathbf{x}_b)$ satisfy $(\boldsymbol{\epsilon}, d_p)$-DW-Lipschitz bound, and hence $(\epsilon, d_p)$-Lipschitz bound.
\end{proof}
\begin{lemma}
\label{lem:1DPPI}
Let $\mathbf{x}_i, \mathbf{x}_{i'} \in \mathcal{X}_m$ be any two points within an $N$-orthotope that differ only in the $\ell$th coordinate, and let $\mathbf{x}_a \in \mathcal{X}_m$ be any point such that $x_{i,\ell} \leq x_{a,\ell} \leq x_{i',\ell}$ and $\mathbf{x}_a$ also differs from $\mathbf{x}_i$ and $\mathbf{x}_{i'}$ only in the $\ell$th coordinate. Then, the interpolation function defined in Definition~\ref{def:PPI} satisfies the following log-convexity condition:
\begin{equation}
z(\mathbf{y}_k \mid \mathbf{x}_a) \lnconv \left(z(\mathbf{y}_k \mid \mathbf{x}_i), z(\mathbf{y}_k \mid \mathbf{x}_{i'})\right).
\end{equation}
\end{lemma}
\begin{proof}
[Proof of Lemma \ref{lem:1DPPI}] Let $\hat{\mathbf{x}}_{i_m}$ be the base corner of the $N$-orthotope $\mathcal{X}_m$, and let $\boldsymbol{\Delta}$ denote the side length vector. Recall from Definition~\ref{def:PPI} that the log-interpolated probability at $\mathbf{x}_a$ is:
\begin{eqnarray}
&& \ln z(\mathbf{y}_k \mid \mathbf{x}_a) \\ \nonumber 
&=& \sum_{\boldsymbol\gamma\in \left\{0,1\right\}^N}  \prod_{q=1}^N \left( (1-\gamma_q)\lambda_{\hat{\mathbf{x}}_{i_m}, \mathbf{x}_a}^{q} + \gamma_q(1 - \lambda_{\hat{\mathbf{x}}_{i_m}, \mathbf{x}_a}^{q}) \right) \\
&& \times \ln z(\mathbf{y}_k \mid \hat{\mathbf{x}}_{i_m} + \boldsymbol{\gamma} \odot \boldsymbol{\Delta}).
\end{eqnarray} 
To verify log-convexity, we compute:

\normalsize
\small 
\begin{eqnarray}
&& \ln z(\mathbf{y}_k \mid \mathbf{x}_a) - \lambda_{\hat{\mathbf{x}}_{i_m}, \mathbf{x}_a}^{\ell}  \ln z(\mathbf{y}_k \mid \mathbf{x}_i) \\ 
&-& (1 - \lambda_{\hat{\mathbf{x}}_{i_m}, \mathbf{x}_a}^{\ell})  \ln z(\mathbf{y}_k \mid \mathbf{x}_{i'}) \\ 
\nonumber
&=& \sum_{\boldsymbol\gamma\in \left\{0,1\right\}^N}  \prod_{q=1}^N ((1-\gamma_{q})\lambda_{\hat{\mathbf{x}}_{i_m}, \mathbf{x}_a}^{q}+\gamma_q(1-\lambda_{\hat{\mathbf{x}}_{i_m}, \mathbf{x}_a}^{q})) \\
&\times& \ln \left( z(\mathbf{y}_k \mid \hat{\mathbf{x}}_{i_m}+\boldsymbol{\gamma} \odot \boldsymbol\Delta)\right) \\ \nonumber
&-& {\rd \lambda_{\hat{\mathbf{x}}_{i_m}, \mathbf{x}_a}^{\ell}} \sum_{\boldsymbol\gamma\in \left\{0,1\right\}^N}  \prod_{q\neq \ell} ((1-\gamma_{q})\lambda_{\hat{\mathbf{x}}_{i_m}, \mathbf{x}_a}^{q}+\gamma_q(1-\lambda_{\hat{\mathbf{x}}_{i_m}, \mathbf{x}_a}^{q})) \\
&& \times \underbrace{{\rd((1-\gamma_{\ell})\lambda_{\hat{\mathbf{x}}_{i_m}, \mathbf{x}_i}^{\ell}+\gamma_\ell(1-\lambda_{\hat{\mathbf{x}}_{i_m}, \mathbf{x}_i}^{\ell}))}}_{\mbox{\footnotesize It equals to {\rd $(1-\gamma_{\ell})$} since $\lambda_{\hat{\mathbf{x}}_{i_m}, \mathbf{x}_i}^{\ell}= 1$}}  \\ \nonumber
&& \times \ln \left( z(\mathbf{y}_k \mid \hat{\mathbf{x}}_{i_m}+\boldsymbol{\gamma} \odot \boldsymbol\Delta)\right) \\ \nonumber
&-& {\bl (1 - \lambda_{\hat{\mathbf{x}}_{i_m}, \mathbf{x}_a}^{\ell})} \sum_{\boldsymbol\gamma\in \left\{0,1\right\}^N}  \prod_{q\neq \ell} ((1-\gamma_{q})\lambda_{\hat{\mathbf{x}}_{i_m}, \mathbf{x}_a}^{q}+\gamma_q(1-\lambda_{\hat{\mathbf{x}}_{i_m}, \mathbf{x}_a}^{q})) \\
&& \times \underbrace{({\bl (1-\gamma_{\ell})\lambda_{\hat{\mathbf{x}}_{i_m}, \mathbf{x}_{i'}}^{\ell}+\gamma_\ell(1-\lambda_{\hat{\mathbf{x}}_{i_m}, \mathbf{x}_{i'}}^{\ell})})}_{\mbox{\footnotesize It equals to {\bl $\gamma_{\ell}$} since $\lambda_{\hat{\mathbf{x}}_{i_m}, \mathbf{x}_{i'}}^{\ell}= 0$}}  \\
&& \times \ln \left( z(\mathbf{y}_k \mid \hat{\mathbf{x}}_{i_m}+\boldsymbol{\gamma} \odot \boldsymbol\Delta)\right) \\ \nonumber 
&=& \sum_{\boldsymbol\gamma\in \left\{0,1\right\}^N}  \prod_{q=1}^N ((1-\gamma_{q})\lambda_{\hat{\mathbf{x}}_{i_m}, \mathbf{x}_a}^{q}+\gamma_q(1-\lambda_{\hat{\mathbf{x}}_{i_m}, \mathbf{x}_a}^{q}))  \\
&& \times \ln \left( z(\mathbf{y}_k \mid \hat{\mathbf{x}}_{i_m}+\boldsymbol{\gamma} \odot \boldsymbol\Delta)\right) \\ \nonumber 
&-& \sum_{\boldsymbol\gamma\in \left\{0,1\right\}^N}  \prod_{q\neq \ell} ((1-\gamma_{q})\lambda_{\hat{\mathbf{x}}_{i_m}, \mathbf{x}_a}^{q}+\gamma_q(1-\lambda_{\hat{\mathbf{x}}_{i_m}, \mathbf{x}_a}^{q})) \\
&& \times {\rd \lambda_{\hat{\mathbf{x}}_{i_m}, \mathbf{x}_a}^{\ell} (1-\gamma_{\ell})} \times \ln \left( z(\mathbf{y}_k \mid \hat{\mathbf{x}}_{i_m}+\boldsymbol{\gamma} \odot \boldsymbol\Delta)\right) \\  \nonumber 
&-&  \sum_{\boldsymbol\gamma\in \left\{0,1\right\}^N}  \prod_{q\neq \ell} ((1-\gamma_{q})\lambda_{\hat{\mathbf{x}}_{i_m}, \mathbf{x}_a}^{q}+\gamma_q(1-\lambda_{\hat{\mathbf{x}}_{i_m}, \mathbf{x}_a}^{q})) \\
&&\times {\bl (1 - \lambda_{\hat{\mathbf{x}}_{i_m}, \mathbf{x}_a}^{\ell})\gamma_{\ell}} \times \ln \left( z(\mathbf{y}_k \mid \hat{\mathbf{x}}_{i_m}+\boldsymbol{\gamma} \odot \boldsymbol\Delta)\right)  \\
&=& 0, 
\end{eqnarray}
meaning that $z(\mathbf{y}_k \mid \mathbf{x}_a) \lnconv \left(z(\mathbf{y}_k \mid \mathbf{x}_i), z(\mathbf{y}_k \mid \mathbf{x}_{i'})\right)$. This concludes the proof. 
\end{proof}

\begin{figure*}[t]
\centering
\hspace{0.00in}
\begin{minipage}{1.00\textwidth}
  \subfigure[Node creation]{
\includegraphics[width=0.54\textwidth, height = 0.19\textheight]{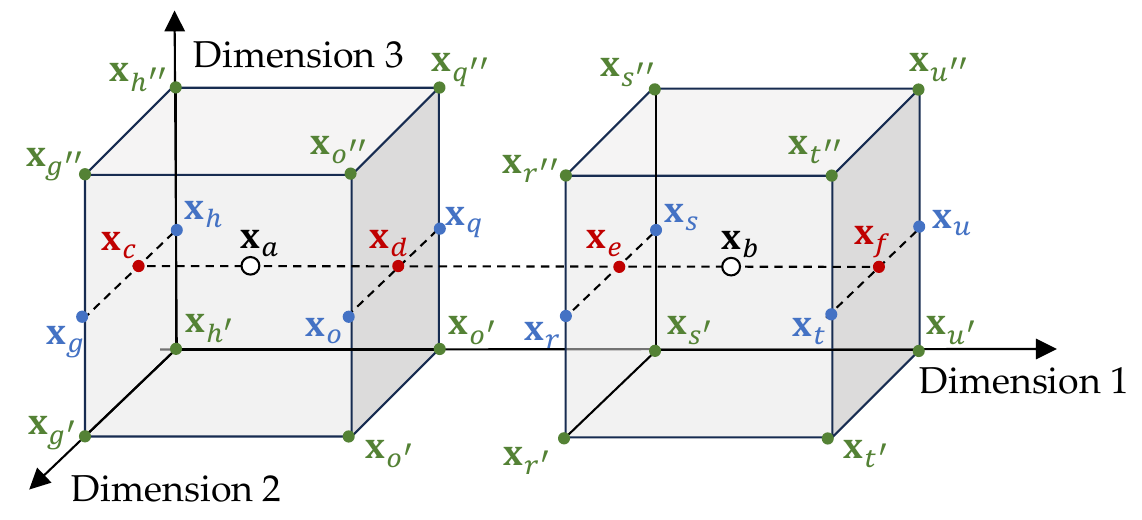}}
  \subfigure[Nodes organized by the two trees]{
\includegraphics[width=0.46\textwidth, height = 0.18\textheight]{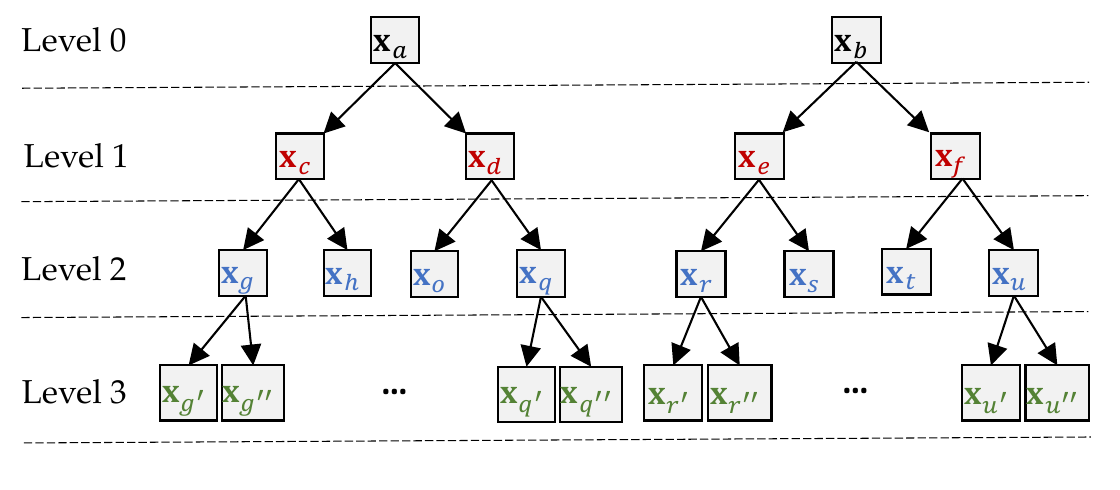}}

\end{minipage}
\caption{Illustration of node creation in Proof of Theorem \ref{thm:AIPOfeasible} (Nodes are recursively extended along dimensions 1, 2, and 3 toward cell boundaries, reaching anchor points at the leaves). }
\label{fig:nodecreation}

\end{figure*}

Next, we provide the detailed proof of \textbf{Theorem \ref{thm:AIPOfeasible}}: 
\begin{proof}[Proof of Theorem~\ref{thm:AIPOfeasible}]
Without loss of generality, assume that $\mathbf{x}_a$ and $\mathbf{x}_b$ differ only in dimension 1, with $x_{a,1} < x_{b,1}$. Let $\mathbf{x}_a \in \mathcal{C}_i$ and $\mathbf{x}_b \in \mathcal{C}_j$, where the ranges of $\mathcal{C}_i$ and $\mathcal{C}_j$ along dimension $\ell$ are $[\hat{x}_{i,\ell}, \hat{x}_{i',\ell}]$ and $[\hat{x}_{j,\ell}, \hat{x}_{j',\ell}]$, respectively.

We construct two trees rooted at $\mathbf{x}_a$ and $\mathbf{x}_b$ as follows:

Intuitively, the two trees systematically decompose the difference between $\mathbf{x}_a$ and $\mathbf{x}_b$ along one dimension at a time.  
At each level $\ell$, nodes extend along the $\ell$-th coordinate toward their cell boundaries, progressively aligning the records with anchor points.  
This recursive structure allows us to apply a dimension-wise induction, ultimately connecting the original records to anchor records through a sequence of intermediate steps that satisfy local Lipschitz bound.

\begin{itemize}
    \item \textbf{Initialization (Level 0)}: Set $\mathbf{x}_a$ and $\mathbf{x}_b$ as the roots.
    
    \item \textbf{Level 1 Construction}: Extend both $\mathbf{x}_a$ and $\mathbf{x}_b$ along dimension 1 toward the boundaries of their respective cells, producing four points $\mathbf{x}_c, \mathbf{x}_d, \mathbf{x}_e, \mathbf{x}_f$, where:
    \begin{equation}
        x_{c,1} = \hat{x}_{i,1}, \quad x_{d,1} = \hat{x}_{i',1}, \quad x_{e,1} = \hat{x}_{j,1}, \quad x_{f,1} = \hat{x}_{j',1}.
    \end{equation}
    Add $\mathbf{x}_c, \mathbf{x}_d$ as children of $\mathbf{x}_a$, and $\mathbf{x}_e, \mathbf{x}_f$ as children of $\mathbf{x}_b$ (see Fig.~\ref{fig:nodecreation}).

    \item \textbf{Level $\ell$ Construction ($\ell = 2, \ldots, N$)}: For each node $\mathbf{x}_v$ at level $\ell-1$, extend along dimension $\ell$ toward the cell boundaries, producing children $\mathbf{x}_{v'}$ and $\mathbf{x}_{v''}$, where $x_{v',\ell} = \hat{x}_{i,\ell}$ and $x_{v'',\ell} = \hat{x}_{i',\ell}$. All other coordinates are inherited from $\mathbf{x}_v$.
\end{itemize}
By construction, every node at level $\ell$ lies on the boundary along the first $\ell$ dimensions.

\begin{figure*}[h]
\centering
\hspace{0.00in}
\begin{minipage}{1.00\textwidth}
\centering
  \subfigure{
\includegraphics[width=0.80\textwidth, height = 0.17\textheight]{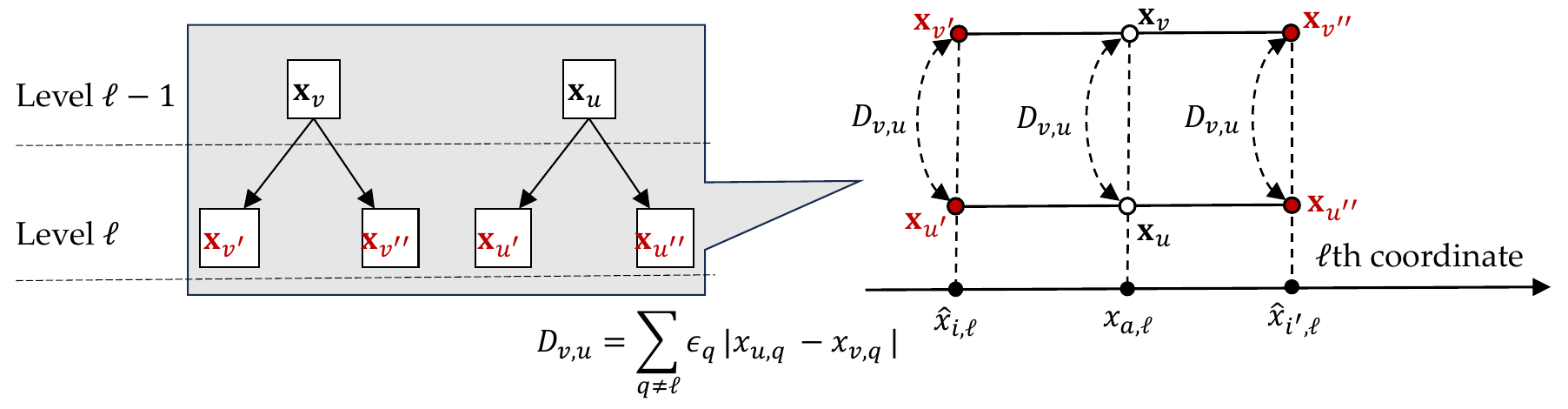}}

\end{minipage}
\caption{Illustration of Property \textbf{P-B}.}
\label{fig:3Dcube_parent}

\end{figure*}
We now prove the theorem via \textbf{induction} on the tree levels:
\begin{itemize}
    \item \textbf{Base case (Level $N$)}: Each leaf node lies on the boundaries in all $N$ dimensions, thus is an anchor record. Since 
    every pair of $\ell$-neighbor anchors satisfies $(\boldsymbol{\epsilon}, d_p)$-DW-Lipschitz bound by assumption, any pair of leaves satisfies $(\boldsymbol{\epsilon}, d_p)$-DW-Lipschitz bound  (according to \textbf{Lemma \ref{lem:chain}}).

    \item \textbf{Inductive step}: Suppose at level $\ell$ ($2 \leq \ell \leq N$), any pair of nodes satisfies $(\boldsymbol{\epsilon}, d_p)$-DW-Lipschitz bound. We show that at level $\ell-1$, their parents also satisfy $(\boldsymbol{\epsilon}, d_p)$-DW-Lipschitz bound.

    Consider two parent nodes $\mathbf{x}_v$ and $\mathbf{x}_u$ at level $\ell-1$ with children $\{\mathbf{x}_{v'}, \mathbf{x}_{v''}\}$ and $\{\mathbf{x}_{u'}, \mathbf{x}_{u''}\}$, respectively. As shown in Fig.~\ref{fig:3Dcube_parent}:
    \begin{itemize}
        \item $\mathbf{x}_v$ shares all coordinates with $\mathbf{x}_{v'}$ and $\mathbf{x}_{v''}$ except along dimension $\ell$.
        \item $\mathbf{x}_u$ shares all coordinates with $\mathbf{x}_{u'}$ and $\mathbf{x}_{u''}$ except along dimension $\ell$.
        \item Moreover, $x_{v,\ell} = x_{u,\ell} = x_{a,\ell}$, $x_{v',\ell} = x_{u',\ell} = \hat{x}_{i,\ell}$, and $x_{v'',\ell} = x_{u'',\ell} = \hat{x}_{i',\ell}$.
    \end{itemize}
    Then, by applying \textbf{Lemma \ref{lem:1DPPI}}, 
    we can obtain 
    \begin{eqnarray}
    z(\mathbf{y}_k \mid \mathbf{x}_v) &\lnconv& \left(z(\mathbf{y}_k \mid \mathbf{x}_{v'}), z(\mathbf{y}_k \mid \mathbf{x}_{v''})\right) \\
    z(\mathbf{y}_k \mid \mathbf{x}_u) &\lnconv& \left(z(\mathbf{y}_k \mid \mathbf{x}_{u'}), z(\mathbf{y}_k \mid \mathbf{x}_{u''})\right).
    \end{eqnarray}
    and then by \textbf{Lemma \ref{lem:6points}}, we can obtain the parent nodes $\mathbf{x}_v$ and $\mathbf{x}_u$ satisfy $(\boldsymbol{\epsilon}, d_p)$-DW-Lipschitz bound.

    \item \textbf{Conclusion (Level 0)}:  
    At level 1, by inducation, we have deduced that the four nodes $\mathbf{x}_c, \mathbf{x}_d, \mathbf{x}_e, \mathbf{x}_f$ satisfy $(\boldsymbol{\epsilon}, d_p)$-DW-Lipschitz bound.  
    Finally, by applying \textbf{Proposition~\ref{prop:across}} (or \textbf{Proposition~\ref{prop:intral}} if $\mathbf{x}_a$ and $\mathbf{x}_b$ lie in the same $N$-orthotope), we conclude that the original pair $(\mathbf{x}_a, \mathbf{x}_b)$ satisfies $(\boldsymbol{\epsilon}, d_p)$-DW-Lipschitz bound, and hence $\epsilon, d_p)$-mDP (according to \textbf{Lemma \ref{lem:dwmdp_sufficient}}).
\end{itemize}

This completes the proof.
\end{proof}

\subsection{Proof of Proposition \ref{prop:AIPOfeasible}}
\label{subsec:proof:prop:AIPOfeasible}
\begin{reproposition}
Given that any pair of real records $\mathbf{x}_a \in \mathcal{X}_m$ and $\mathbf{x}_b \in \mathcal{X}_{m'}$ with their perturbation probabilities interpolated by 
\begin{eqnarray}
&& z(\mathbf{y}_k \mid \mathbf{x}_a) = \overline{f}_{\mathrm{int}}(\mathbf{x}_a, \mathbf{y}_k, \mathbf{Z}_{\hat{\mathcal{X}}_m})\\
&& z(\mathbf{y}_k \mid \mathbf{x}_b) = \overline{f}_{\mathrm{int}}(\mathbf{x}_b, \mathbf{y}_k, \mathbf{Z}_{\hat{\mathcal{X}}_{m'}})
\end{eqnarray} 
then their perturbation probabilities satisfy $(2\epsilon, d_p)$-mDP. 
\end{reproposition}
\begin{proof}
We assume that the unnormalized interpolation function $f_{\mathrm{int}}$ satisfies $(\epsilon, d_p)$-Lipschitz bound, i.e.,
\begin{equation}
\label{eq:unnormalized_mdp}
e^{- \epsilon d_p(\mathbf{x}_a, \mathbf{x}_b)} \leq 
\frac{f_{\mathrm{int}}(\mathbf{x}_a, \mathbf{y}_k, \mathbf{Z}_{\hat{\mathcal{X}}_m})}
     {f_{\mathrm{int}}(\mathbf{x}_b, \mathbf{y}_k, \mathbf{Z}_{\hat{\mathcal{X}}_{m'}})} 
\leq e^{\epsilon d_p(\mathbf{x}_a, \mathbf{x}_b)}.
\end{equation}

By summing both the numerator and denominator over all $\mathbf{y}_j \in \mathcal{Y}$, the same multiplicative bounds hold for the partition function:
\begin{equation}
\label{eq:partition_bound}
e^{- \epsilon d_p(\mathbf{x}_a, \mathbf{x}_b)} \leq 
\frac{\sum_{\mathbf{y}_j \in \mathcal{Y}} f_{\mathrm{int}}(\mathbf{x}_a, \mathbf{y}_j, \mathbf{Z}_{\hat{\mathcal{X}}_m})}
     {\sum_{\mathbf{y}_j \in \mathcal{Y}} f_{\mathrm{int}}(\mathbf{x}_b, \mathbf{y}_j, \mathbf{Z}_{\hat{\mathcal{X}}_{m'}})} 
\leq e^{\epsilon d_p(\mathbf{x}_a, \mathbf{x}_b)}.
\end{equation}

Recall that the normalized output distribution is defined as:
\[
z(\mathbf{y}_k \mid \mathbf{x}) = \overline{f}_{\mathrm{int}}(\mathbf{x}, \mathbf{y}_k, \mathbf{Z}_{\hat{\mathcal{X}}_m}) 
= \frac{f_{\mathrm{int}}(\mathbf{x}, \mathbf{y}_k, \mathbf{Z}_{\hat{\mathcal{X}}_m})}
       {\sum_{\mathbf{y}_j \in \mathcal{Y}} f_{\mathrm{int}}(\mathbf{x}, \mathbf{y}_j, \mathbf{Z}_{\hat{\mathcal{X}}_m})}.
\]

Then, the ratio of normalized probabilities becomes:
\begin{align}
\frac{z(\mathbf{y}_k \mid \mathbf{x}_a)}{z(\mathbf{y}_k \mid \mathbf{x}_b)} 
&= \frac{\overline{f}_{\mathrm{int}}(\mathbf{x}_a, \mathbf{y}_k, \mathbf{Z}_{\hat{\mathcal{X}}_m})}
        {\overline{f}_{\mathrm{int}}(\mathbf{x}_b, \mathbf{y}_k, \mathbf{Z}_{\hat{\mathcal{X}}_{m'}})} \nonumber \\
&= \frac{f_{\mathrm{int}}(\mathbf{x}_a, \mathbf{y}_k, \mathbf{Z}_{\hat{\mathcal{X}}_m})}
         {f_{\mathrm{int}}(\mathbf{x}_b, \mathbf{y}_k, \mathbf{Z}_{\hat{\mathcal{X}}_{m'}})}
   \cdot
   \frac{\sum_{\mathbf{y}_j \in \mathcal{Y}} f_{\mathrm{int}}(\mathbf{x}_b, \mathbf{y}_j, \mathbf{Z}_{\hat{\mathcal{X}}_{m'}})}
        {\sum_{\mathbf{y}_j \in \mathcal{Y}} f_{\mathrm{int}}(\mathbf{x}_a, \mathbf{y}_j, \mathbf{Z}_{\hat{\mathcal{X}}_m})}.
\end{align}

Applying inequalities \eqref{eq:unnormalized_mdp} and \eqref{eq:partition_bound} to each term, we obtain:
\begin{equation}
e^{-2\epsilon d_p(\mathbf{x}_a, \mathbf{x}_b)} \leq 
\frac{z(\mathbf{y}_k \mid \mathbf{x}_a)}{z(\mathbf{y}_k \mid \mathbf{x}_b)} 
\leq e^{2\epsilon d_p(\mathbf{x}_a, \mathbf{x}_b)},
\end{equation}
which concludes the proof that the normalized distribution satisfies $(2\epsilon, d_p)$-metric differential privacy.
\end{proof}

\subsection{Proof of Proposition \ref{prop:linapprox}: Linear Surrogate For Utility Loss}
\label{subsec:proof:prop:linapprox}
\begin{reproposition}
[Linear surrogate for utility loss]
For any point 
$\mathbf{x} =\hat{\mathbf{x}}_{i_m}+\boldsymbol{\lambda}_{a}\odot\boldsymbol{\Delta} \in \mathcal{C}_m$, with the convex coefficients $\boldsymbol{\lambda} = [\lambda_{\hat{\mathbf{x}}_{i_m}, \mathbf{x}}^{1}, \ldots, \lambda_{\hat{\mathbf{x}}_{i_m}, \mathbf{x}}^{N}] \in [0,1]^N$,
approximate the perturbation probability by 
\begin{eqnarray}
&& \Pr\!\bigl[\mathcal{M}(\mathbf{x};\mathbf{Z}_{\hat{\mathcal{X}}_{m}})
             =\mathbf{y}_{k}\bigr]\\
&\approx&
\sum_{\boldsymbol{\gamma}\in\{0,1\}^{N}}
\prod_{\ell=1}^{N}\!
\bigl((1-\gamma_{\ell})\lambda_{\hat{\mathbf{x}}_{i_m}, \mathbf{x}}^{\ell}
+\gamma_{\ell}(1-\lambda_{\hat{\mathbf{x}}_{i_m}, \mathbf{x}}^{\ell})\bigr)
\\
&& \times z\!\bigl(\mathbf{y}_{k}\mid
          \hat{\mathbf{x}}_{i_m} +\boldsymbol{\gamma}\odot\boldsymbol{\Delta}\bigr)
\end{eqnarray}
the cell loss
$\mathcal{L}(\mathbf{Z}_{\hat{\mathcal{X}}_{m}})$
admits the linear surrogate $
\tilde{\mathcal{L}}(\mathbf{Z}_{\hat{\mathcal{X}}_{m}})=
\langle\tilde{\mathbf{C}}_{\hat{\mathcal{X}}_{m}}, \mathbf{Z}_{\hat{\mathcal{X}}_{m}}\rangle$, where $\tilde{\mathbf{C}}_{\hat{\mathcal{X}}_{m}} =\bigl\{\tilde{c}(\hat{\mathbf{x}}_{i},\mathbf{y}_{k})\bigr\}_{(\hat{\mathbf{x}}_{i},\mathbf{y}_{k})\in\hat{\mathcal{X}}_{m}\times\mathcal{Y}}$ is a constant coefficient matrix depending only on the prior $p(\mathbf{x})$ and the utility loss 
$\mathcal{L}(\mathbf{x},\mathbf{y}_{k})$. For each $\hat{\mathbf{x}}_j = \hat{\mathbf{x}}_{i_m} + \boldsymbol{\gamma} \odot \boldsymbol{\Delta} \in \hat{\mathcal{X}}_m$
\begin{eqnarray}
&& \tilde{c}(\hat{\mathbf{x}}_j,\mathbf{y}_{k}) \\ \nonumber 
&=&  \int_{\mathcal{C}_m}\prod_{\ell=1}^N ((1-\gamma_{\ell})\lambda_{\hat{\mathbf{x}}_{i_m}, \mathbf{x}}^{\ell}+\gamma_\ell(1-\lambda_{\hat{\mathbf{x}}_{i_m}, \mathbf{x}}^{\ell}))  p(\mathbf{x}) \mathcal{L}(\mathbf{x}, \mathbf{y}_k)\mathrm{d}\mathbf{x}. 
\end{eqnarray}
If $\mathcal{X}$ is discretized, 
\begin{eqnarray}
&& \tilde{c}(\hat{\mathbf{x}}_j, \mathbf{y}_k)
\\ \nonumber
&=& \sum_{\mathbf{x} \in \mathcal{X}_m}
\prod_{\ell=1}^N ((1-\gamma_{\ell})\lambda_{\hat{\mathbf{x}}_{i_m}, \mathbf{x}}^{\ell}+\gamma_\ell(1-\lambda_{\hat{\mathbf{x}}_{i_m}, \mathbf{x}}^{\ell}))
p(\mathbf{x})  \mathcal{L}(\mathbf{x}, \mathbf{y}_k).
\end{eqnarray}
\end{reproposition}
\begin{proof}
Let $\mathbf{x} = \hat{\mathbf{x}}_{i_m} + \boldsymbol{\lambda} \odot \boldsymbol{\Delta}$ where $\boldsymbol{\lambda} = [\lambda_{\hat{\mathbf{x}}_{i_m}, \mathbf{x}}^{1}, \ldots, \lambda_{\hat{\mathbf{x}}_{i_m}, \mathbf{x}}^{N}] \in [0,1]^N$. By definition of the approximated mechanism, we have:
\normalsize
% \small 
\begin{eqnarray}
\nonumber && \tilde{\mathcal{L}}(\mathbf{Z}_{\hat{\mathcal{X}}_m})  \\ \nonumber
&=& \sum_{\mathbf{y}_k \in \mathcal{Y}} \int_{\mathcal{C}_m}
\Pr\left[\mathcal{M}(\mathbf{x}, \mathbf{Z}_{\hat{\mathcal{X}}_m}) = \mathbf{y}_k\right] 
p(\mathbf{x}) \mathcal{L}(\mathbf{x}, \mathbf{y}_k) \mathrm{d}\mathbf{x} \\
\nonumber &\approx& \sum_{\mathbf{y}_k \in \mathcal{Y}} \int_{\mathcal{C}_m}
\left[
\sum_{\boldsymbol{\gamma} \in \{0,1\}^N} \bigl((1-\gamma_{\ell})\lambda_{\hat{\mathbf{x}}_{i_m}, \mathbf{x}}^{\ell}
\right. \\ \nonumber 
&& \left.+\gamma_{\ell}(1-\lambda_{\hat{\mathbf{x}}_{i_m}, \mathbf{x}}^{\ell})\bigr)
z(\mathbf{y}_k \mid \hat{\mathbf{x}}_{i_m} + \boldsymbol{\gamma} \odot \boldsymbol{\Delta})
\right]
p(\mathbf{x}) \mathcal{L}(\mathbf{x}, \mathbf{y}_k) \mathrm{d}\mathbf{x} \\
\nonumber  &=& \sum_{\mathbf{y}_k \in \mathcal{Y}} \sum_{\boldsymbol{\gamma} \in \{0,1\}^N}
z(\mathbf{y}_k \mid \hat{\mathbf{x}}_{i_m} + \boldsymbol{\gamma} \odot \boldsymbol{\Delta})
\\ \nonumber
&& \times \int_{\mathcal{C}_m}
\bigl((1-\gamma_{\ell})\lambda_{\hat{\mathbf{x}}_{i_m}, \mathbf{x}}^{\ell}
+\gamma_{\ell}(1-\lambda_{\hat{\mathbf{x}}_{i_m}, \mathbf{x}}^{\ell})\bigr) 
p(\mathbf{x}) \mathcal{L}(\mathbf{x}, \mathbf{y}_k) \mathrm{d}\mathbf{x} \\
&=& \sum_{(\hat{\mathbf{x}}_j, \mathbf{y}_k) \in \hat{\mathcal{X}}_m \times \mathcal{Y}}
\tilde{c}(\hat{\mathbf{x}}_j, \mathbf{y}_k)
z(\mathbf{y}_k \mid \hat{\mathbf{x}}_j),
\end{eqnarray}
where $\hat{\mathbf{x}}_j = \hat{\mathbf{x}}_{i_m} + \boldsymbol{\gamma} \odot \boldsymbol{\Delta}$ and the coefficient
\begin{eqnarray}
&& \tilde{c}(\hat{\mathbf{x}}_j, \mathbf{y}_k)
\\ 
&=& \int_{\mathcal{C}_m} \prod_{\ell=1}^N ((1-\gamma_{\ell})\lambda_{\hat{\mathbf{x}}_{i_m}, \mathbf{x}}^{\ell}+\gamma_\ell(1-\lambda_{\hat{\mathbf{x}}_{i_m}, \mathbf{x}}^{\ell})) \\
&& \times 
p(\mathbf{x}) \mathcal{L}(\mathbf{x}, \mathbf{y}_k) \mathrm{d}\mathbf{x}
\end{eqnarray}
is a constant with respect to $\mathbf{Z}$.

If $\mathcal{X}$ is discretized, then the integral becomes a summation:
\begin{equation}
\tilde{c}(\hat{\mathbf{x}}_j, \mathbf{y}_k)
= \sum_{\mathbf{x} \in \mathcal{X}_m}
\prod_{\ell=1}^N ((1-\gamma_{\ell})\lambda_{\hat{\mathbf{x}}_{i_m}, \mathbf{x}}^{\ell}+\gamma_\ell(1-\lambda_{\hat{\mathbf{x}}_{i_m}, \mathbf{x}}^{\ell}))
p(\mathbf{x})  \mathcal{L}(\mathbf{x}, \mathbf{y}_k).
\end{equation}

In both cases, the resulting loss $\tilde{\mathcal{L}}(\mathbf{Z}_{\hat{\mathcal{X}}_m})$ is a linear function of the perturbation matrix $\mathbf{Z}_{\hat{\mathcal{X}}_m}$.
\end{proof}

\section{Optimal Gap Analysis}
\label{sec:gapanalysis}
To evaluate how closely our solution can approach the true optimum, we derive a universal lower bound on the utility loss, labeled as \textbf{"LB"} in the following experiment part.

For any two cells $\mathcal{C}_m$ and $\mathcal{C}_{m'}$, we first define the following inter-cell distance:
\begin{equation}
  \hat d_p(\mathcal{C}_m, \mathcal{C}_{m'}) := \max_{\mathbf{x} \in \mathcal{C}_m,\ \mathbf{x}' \in \mathcal{C}_{m'}} d_p(\mathbf{x}, \mathbf{x}').
\end{equation}
Since $d_p(\mathbf{x}, \mathbf{x}') \leq \hat d_p(\mathcal{C}_m, \mathcal{C}_{m'})$ for all $\mathbf{x} \in \mathcal{C}_m, \mathbf{x}' \in \mathcal{C}_{m'}$, replacing $d_p$ with $\hat d_p$ yields a relaxation of the original mDP constraints. We define the aggregated decision variables:
\begin{equation}
z(\mathbf{y}_k \mid \mathcal{C}_m) := \int_{\mathcal{C}_m} \Pr(\mathcal{M}(\mathbf{x}) = \mathbf{y}_k)\; \mathrm{d}\mathbf{x},
\end{equation}
and impose the relaxed mDP constraints:
\begin{equation}
  z(\mathbf{y}_k \mid \mathcal{C}_m) -
  e^{\epsilon \hat d_p(\mathcal{C}_m, \mathcal{C}_{m'})}
  z(\mathbf{y}_k \mid \mathcal{C}_{m'}) \le 0,
  \quad \forall \mathcal{C}_m, \mathcal{C}_{m'}.
  \label{eq:relaxed-mdp}
\end{equation}
and normalization constraint:
\begin{equation}
\sum_{\mathbf{y}_k \in \mathcal{Y}} z(\mathbf{y}_k \mid \mathcal{C}_m)
= \int_{\mathcal{C}_m} \mathrm{d}\mathbf{x} = \prod_{\ell=1}^N \Delta_\ell.
\end{equation}
Finally, define the minimal utility loss over each cell:
\begin{equation}
\bar{\mathcal{L}}_m(\mathbf{y}_k) := \min_{\mathbf{x} \in \mathcal{C}_m} p(\mathbf{x}) \mathcal{L}(\mathbf{x}, \mathbf{y}_k).
\end{equation}

We now formulate the relaxed problem $\mathrm{APO}_{\mathrm{lb}}$ that minimizes a lower bound of the cell-aggregated utility loss:
\begin{eqnarray}
\min && \hat{\mathcal{L}}(\mathbf{Z}):= \sum_{m=1}^{M} \sum_{\mathbf{y}_k \in \mathcal{Y}} \bar{\mathcal{L}}_m(\mathbf{y}_k) \cdot z(\mathbf{y}_k \mid \mathcal{C}_m), \\
\mathrm{s.t.} && z(\mathbf{y}_k \mid \mathcal{C}_m) -
  e^{\epsilon \hat d_p(\mathcal{C}_m, \mathcal{C}_{m'})}
  z(\mathbf{y}_k \mid \mathcal{C}_{m'}) \le 0,
  ~ \forall \mathcal{C}_m, \mathcal{C}_{m'}, \\
  && \sum_{\mathbf{y}_k \in \mathcal{Y}} z(\mathbf{y}_k \mid \mathcal{C}_m)
= \prod_{\ell=1}^N \Delta_\ell, \forall \mathcal{C}_m
\end{eqnarray}
and denote its optimal value by $\hat{\mathcal{L}}^*$.

\begin{proposition}[Universal Lower Bound]
\label{prop:cell-lp-lower-bound}
Let $\mathcal{M} : \mathcal{X} \to \mathcal{Y}$ be any mechanism satisfying $(\epsilon, d_p)$-mDP. Let
\begin{equation}
\mathcal{L}(\mathcal{M}) :=
\int_{\mathcal{X}} \sum_{\mathbf{y}_k \in \mathcal{Y}} p(\mathbf{x}) \mathcal{L}(\mathbf{x}, \mathbf{y}_k) \Pr(\mathcal{M}(\mathbf{x}) = \mathbf{y}_k)\; \mathrm{d}\mathbf{x}
\end{equation}
denote its expected utility loss. Then,
\begin{equation}
\mathcal{L}(\mathcal{M}) \geq \hat{\mathcal{L}}^*.
\end{equation}
\end{proposition}
\DEL{
\begin{proof}[Proof of Proposition \ref{prop:cell-lp-lower-bound}]
Let $\mathcal{M}$ be an arbitrary $(\epsilon, d_p)$-mDP mechanism. For any $\mathbf{x} \in \mathcal{C}_m$, $\mathbf{x}' \in \mathcal{C}_{m'}$, the privacy guarantee implies
\begin{equation}
\frac{\Pr(\mathcal{M}(\mathbf{x}) = \mathbf{y}_k)}{\Pr(\mathcal{M}(\mathbf{x}') = \mathbf{y}_k)} 
\leq e^{\epsilon d_p(\mathbf{x}, \mathbf{x}')} 
\leq e^{\epsilon \hat d_p(\mathcal{C}_m, \mathcal{C}_{m'})}.
\end{equation}
Thus, for all such pairs, the mechanism satisfies the relaxed constraint with distance $\hat d_p(\mathcal{C}_m, \mathcal{C}_{m'})$.

This implies that the aggregated decision variables
\begin{equation}
z(\mathbf{y}_k \mid \mathcal{C}_m) := \int_{\mathcal{C}_m} \Pr(\mathcal{M}(\mathbf{x}) = \mathbf{y}_k) \mathrm{d}\mathbf{x}
\end{equation}
satisfy the relaxed mDP constraints:
\begin{equation}
z(\mathbf{y}_k \mid \mathcal{C}_m) 
\leq e^{\epsilon \hat d_p(\mathcal{C}_m, \mathcal{C}_{m'})} z(\mathbf{y}_k \mid \mathcal{C}_{m'}), 
\quad \forall \mathcal{C}_m, \mathcal{C}_{m'}.
\end{equation}
Additionally, normalization is preserved:
\begin{equation}
\sum_{\mathbf{y}_k \in \mathcal{Y}} z(\mathbf{y}_k \mid \mathcal{C}_m)
= \int_{\mathcal{C}_m} \sum_{\mathbf{y}_k \in \mathcal{Y}} \Pr(\mathcal{M}(\mathbf{x}) = \mathbf{y}_k) \mathrm{d}\mathbf{x}
= \int_{\mathcal{C}_m} \mathrm{d}\mathbf{x}
= \prod_{\ell=1}^N \Delta_\ell.
\end{equation}
Hence, the aggregated variables $\{z(\mathbf{y}_k \mid \mathcal{C}_m)\}$ induced by any valid mechanism $\mathcal{M}$ form a feasible solution to the relaxed linear program $\mathrm{APO}_{\mathrm{lb}}$.

Now we lower-bound the utility loss:
\begin{align}
\mathcal{L}(\mathcal{M}) 
&= \sum_{\mathbf{y}_k \in \mathcal{Y}} \int_{\mathcal{X}} p(\mathbf{x}) \mathcal{L}(\mathbf{x}, \mathbf{y}_k) \Pr(\mathcal{M}(\mathbf{x}) = \mathbf{y}_k) \mathrm{d}\mathbf{x} \\
&= \sum_{m=1}^M \sum_{\mathbf{y}_k \in \mathcal{Y}} \int_{\mathcal{C}_m} p(\mathbf{x}) \mathcal{L}(\mathbf{x}, \mathbf{y}_k) \Pr(\mathcal{M}(\mathbf{x}) = \mathbf{y}_k) \mathrm{d}\mathbf{x} \\
&\geq \sum_{m=1}^M \sum_{\mathbf{y}_k \in \mathcal{Y}} \bar{\mathcal{L}}_m(\mathbf{y}_k) z(\mathbf{y}_k \mid \mathcal{C}_m) \\
&\geq \hat{\mathcal{L}}^*,
\end{align}
where $\bar{\mathcal{L}}_m(\mathbf{y}_k) := \min_{\mathbf{x} \in \mathcal{C}_m} p(\mathbf{x}) \mathcal{L}(\mathbf{x}, \mathbf{y}_k)$, and the last line holds because $\hat{\mathcal{L}}^*$ is the optimal value of the relaxed problem.
This completes the proof.
\end{proof}}

\begin{proof}[Proof of Proposition~\ref{prop:cell-lp-lower-bound}]
Fix any data perturbation mechanism $\mathcal{M}$ that satisfies $(\epsilon,d_p)$-mDP.
For a fixed output symbol $\mathbf{y}_k$ let
\begin{equation}
f(\mathbf{x})
\;:=\;
\Pr\!\bigl(\mathcal{M}(\mathbf{x}) = \mathbf{y}_k\bigr),
\qquad
\mathbf{x}\in\mathcal{X}.
\end{equation}

For any two cells $\mathcal{C}_m,\mathcal{C}_{m'}$ and for every
$\mathbf{x}\in\mathcal{C}_m,\;\mathbf{x}'\in\mathcal{C}_{m'}$ the mDP guarantee gives
\begin{eqnarray}
\label{eq:A.1}
&& \frac{f(\mathbf{x})}{f(\mathbf{x}')}
\leq 
e^{\epsilon d_p(\mathbf{x}_a,\mathbf{x}_b)}
\leq 
e^{\epsilon \hat d_p(\mathcal{C}_m,\mathcal{C}_{m'})}, \\
&& \hat d_p(\mathcal{C}_m,\mathcal{C}_{m'})
\;=\;
\max_{\mathbf{u}\in\mathcal{C}_m,\;\mathbf{v}\in\mathcal{C}_{m'}}
      d_p(\mathbf{u},\mathbf{v}).
%\tag{A.1}
\end{eqnarray}
Set
$M_m := \sup_{\mathbf{x}\in\mathcal{C}_m} f(\mathbf{x}),
\quad
m_{m'} := \inf_{\mathbf{x}'\in\mathcal{C}_{m'}} f(\mathbf{x}')
$.
By Eq. (\ref{eq:A.1}) we have
\begin{equation}
\label{eq:A.2}
M_m \leq  e^{\epsilon\hat d_p(\mathcal{C}_m,\mathcal{C}_{m'})}m_{m'}.
% \tag{A.2}
\end{equation}
Since every grid cell has the same volume
$V:=\mu(\mathcal{C}_m)=\prod_{\ell=1}^{N}\Delta_\ell$,
\begin{equation}
z(\mathbf{y}_k\mid\mathcal{C}_m)
  := \int_{\mathcal{C}_m} f(\mathbf{x})\mathrm{d}\mathbf{x}
  \leq  VM_m,
~
z(\mathbf{y}_k\mid\mathcal{C}_{m'})
  \;\ge\; Vm_{m'}.
\end{equation}
Combining with Eq. (\ref{eq:A.2}) yields
\begin{equation}
z(\mathbf{y}_k\mid\mathcal{C}_m)
\leq 
e^{\epsilon\hat d_p(\mathcal{C}_m,\mathcal{C}_{m'})}
z(\mathbf{y}_k\mid\mathcal{C}_{m'}),
\end{equation}
i.e.\ the vector $\bigl\{z(\mathbf{y}_k\mid\mathcal{C}_m)\bigr\}$ satisfies the relaxed mDP
constraints~\eqref{eq:relaxed-mdp}.  Normalisation also holds because
$\sum_{\mathbf{y}_k}f(\mathbf{x})=1$ for every $\mathbf{x}$.  Hence, \textbf{the aggregated variables
induced by $\mathcal{M}$ form a feasible solution of $\mathrm{APO}_{\mathrm{lb}}$.}

%%%%%%%%%%%%%%%%%%%%%%%%%%%%%%%%%%%%%%%%%%%%%%%%%%%%%%%%%%%%%%%%%%%%
Now, we can derive the expected loss of $\mathcal{M}$ as
\begin{eqnarray}
&& \mathcal{L}(\mathcal{M}) \\ 
&=& \sum_{\mathbf{y}_k\in\mathcal{Y}}
    \int_{\mathcal{X}} p(\mathbf{x})
           \mathcal{L}\bigl(\mathbf{x},\mathbf{y}_k\bigr)
           f(\mathbf{x})\mathrm{d}\mathbf{x} \\ 
&=& \sum_{m=1}^{M} \sum_{\mathbf{y}_k\in\mathcal{Y}}
    \int_{\mathcal{C}_m} p(\mathbf{x})
           \mathcal{L}\bigl(\mathbf{x},\mathbf{y}_k\bigr)
           f(\mathbf{x})\mathrm{d}\mathbf{x} \\ 
&\geq& \sum_{m=1}^{M} \sum_{\mathbf{y}_k\in\mathcal{Y}}
\underbrace{\min_{\mathbf{x}\in\mathcal{C}_m}
                  p(\mathbf{x})
\mathcal{L}\bigl(\mathbf{x},\mathbf{y}_k\bigr)}               _{=\:\bar{\mathcal{L}}_m(\mathbf{y}_k)}
     z(\mathbf{y}_k\mid\mathcal{C}_m) \\
&=&
  \hat{\mathcal{L}}\bigl(\mathbf{Z}\bigr).
\end{eqnarray}
Because $\mathbf{Z}$ is feasible for $\mathrm{APO}_{\mathrm{lb}}$,
$\hat{\mathcal{L}}(\mathbf{Z}) \ge \hat{\mathcal{L}}^{*}$.
Therefore
$
\mathcal{L}(\mathcal{M}) \ge \hat{\mathcal{L}}^{*},
$
establishing that $\hat{\mathcal{L}}^{*}$ is a universal lower bound on the utility loss of any
$(\epsilon,d_p)$-mDP mechanism.
\end{proof}

\section{Additional Experimental Results}
\label{sec:addexperiment}
This section presents additional experimental results to support our main findings in Section \ref{sec:experiments}. 

In this section, we present supplementary experiments to further validate our framework. \textbf{Section~\ref{subsec:exp:privacy}} provides a detailed distributional analysis of the \emph{perturbation probability ratio (PPR)}, as a supplementary results for Table \ref{tab:privacy2norm}. \textbf{Section~\ref{subsec:exp:granularity}} investigates the performance of the interpolation-based method under varying grid granularities, highlighting the trade-off between accuracy and efficiency. \textbf{Section~\ref{subsec:exp:privacybudget}} conducts an ablation study that compares utility loss with and without privacy budget optimization, illustrating the benefits of joint optimization. \textbf{Section~\ref{subsec:exp:1norm}} evaluates the framework when spatial proximity is measured by the $\ell_1$-norm instead of the $\ell_2$-norm, demonstrating the generality of our approach across different distance metrics.

\DEL{
\subsection{Experiment Settings}
\label{subsec:exp:settings}
\subsubsection{Datasets}
\label{subsubsec:datasets}
We conducted experiments on road network datasets from three major cities: \emph{Rome, Italy}, \emph{New York City (NYC), USA}, and \emph{London, UK}. Each dataset comprises nodes representing intersections, junctions, and other key points in the urban road infrastructure, with edges corresponding to actual road segments. The data were obtained from OpenStreetMap~\cite{openstreetmap}.

To support our interpolation-based method, we discretize each city's geographic area into a uniform grid map. Table~\ref{tab:dataset-stats} provides a summary of the geographic boundaries, node and edge counts, grid configurations, and cell sizes for each dataset. Fig.~\ref{fig:exp:roadmap}(a)(b)(c) illustrates the spatial distribution of locations for the three cities.

\begin{figure*}[h]
\centering
\hspace{0.00in}
\begin{minipage}{1.00\textwidth}
  \subfigure[Rome]{
\includegraphics[width=0.30\textwidth, height = 0.18\textheight]{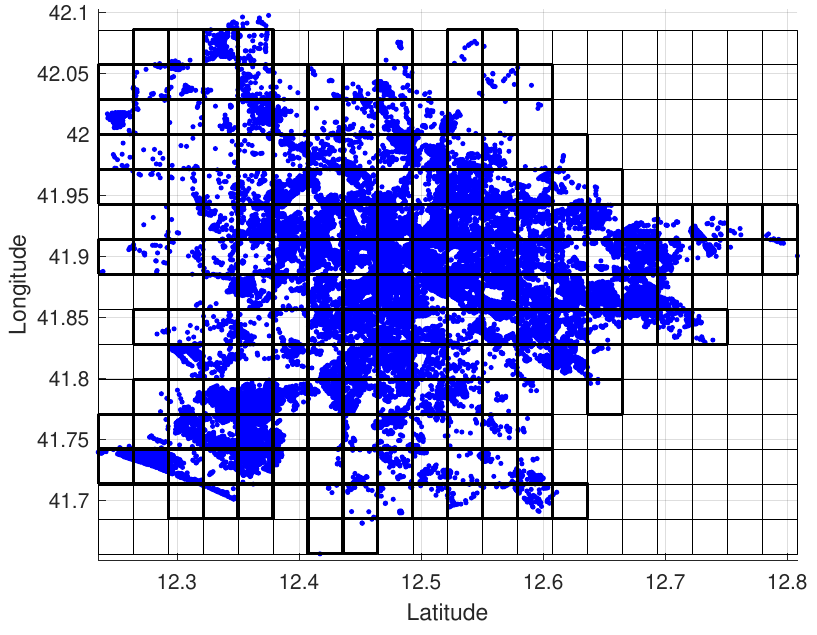}}
  \subfigure[London]{
\includegraphics[width=0.30\textwidth, height = 0.18\textheight]{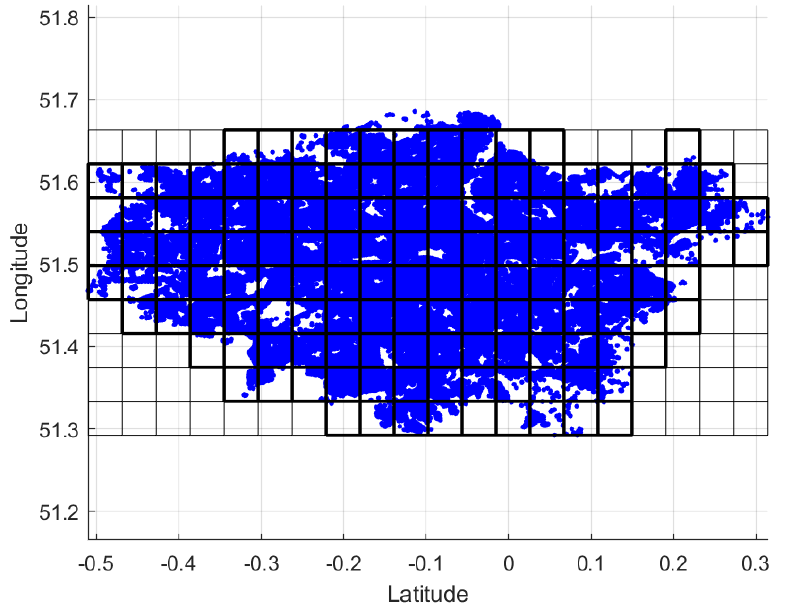}}
  \subfigure[New York City]{
\includegraphics[width=0.30\textwidth, height = 0.18\textheight]{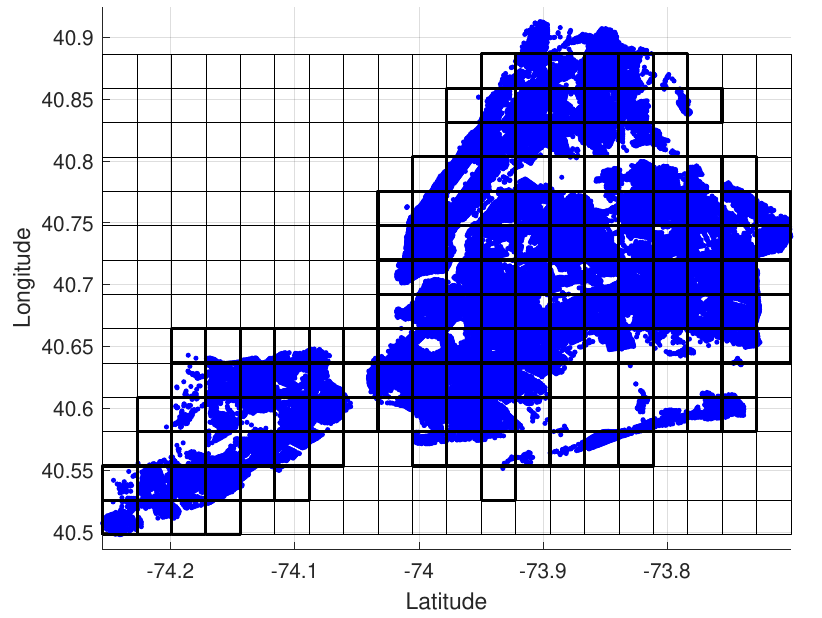}}

\end{minipage}
\caption{Spatial distribution of locations and grid maps for the three cities.}
\label{fig:exp:roadmap}
\end{figure*}

\textbf{Coordinate transformation.} Since raw latitude and longitude coordinates lie on a spherical surface, directly applying the $\ell_2$-norm to them does not reflect true geographic distances. To enable accurate computation under the $\ell_2$-norm, we convert all geographic coordinates into local Cartesian coordinates using an equirectangular projection. Specifically, we map each location $(\text{lon}, \text{lat})$ to $(x, y)$ in kilometers via:
\[
x = R \cdot (\text{lon} - \text{lon}_0) \cdot \cos\left( \frac{\text{lat} + \text{lat}_0}{2} \right), \quad
y = R \cdot (\text{lat} - \text{lat}_0),
\]
where $R = 6371\text{km}$ is the Earth’s radius, and $(\text{lon}_0, \text{lat}_0)$ is a dataset-specific reference point. For each dataset, we use the following center coordinates: Rome $(12.4964^\circ\text{E},\ 41.9028^\circ\text{N})$, London $(0.1278^\circ\text{W},\ 51.5074^\circ\text{N})$, and New York City $(74.0060^\circ\text{W},\ 40.7128^\circ\text{N})$. This transformation ensures that pairwise distances computed with the $\ell_2$-norm approximate real-world geographic distances within each local region.
}

\DEL{
\begin{eqnarray}
&& \frac{\Pr(\mathcal{M}(\mathbf{x})=\mathbf y_k)}{\Pr(\mathcal{M}(\mathbf{x}')=\mathbf y_k)} \leq e^{d_p(\mathbf{x}, \mathbf{x}')} \leq e^{d_p(\mathcal{C}_m, \mathcal{C}_{m'})}~ \forall\mathbf{x} \in \mathcal{C}_m, \mathbf{x}' \in \mathcal{C}_{m'} \\
&\Rightarrow& \frac{\frac{\int_{\mathcal{C}_m}\Pr(\mathcal{M}(\mathbf{x})=\mathbf y_k)\mathrm{d}\mathbf{x}}{\int_{\mathcal{C}_m} \mathrm{d}\mathbf{x}} }{\frac{\int_{\mathcal{C}_{m'}}\Pr(\mathcal{M}(\mathbf{x}')=\mathbf y_k)\;\mathrm{d}\mathbf{x}'}{\int_{\mathcal{C}_{m'}} \mathrm{d}\mathbf{x}} } \leq e^{d_p(\mathcal{C}_m, \mathcal{C}_{m'})}
\end{eqnarray}}

% \textbf{Proposition~\ref{prop:cell-lp-lower-bound}} shows that the value $\hat{\mathcal{L}}^{*}$ produced by the cell‑aggregated LP is a certificate of \emph{unachievable performance}; any real mechanism that obeys $(\epsilon,d_p)$‑mDP—whether defined on continuous or discrete domains—must incur \emph{at least} that much expected utility loss.  This makes $\hat{\mathcal{L}}^{*}$ a natural yardstick for evaluating the tightness of practical mechanisms in our experiments.

\subsection{Privacy Evaluation: Perturbation Probability Ratio Distributional Analysis}
\label{subsec:exp:privacy}

\begin{figure*}[h]
\begin{minipage}{1.00\textwidth}
\centering
  \subfigure[EM]{
\includegraphics[width=0.133\textwidth, height = 0.09\textheight]{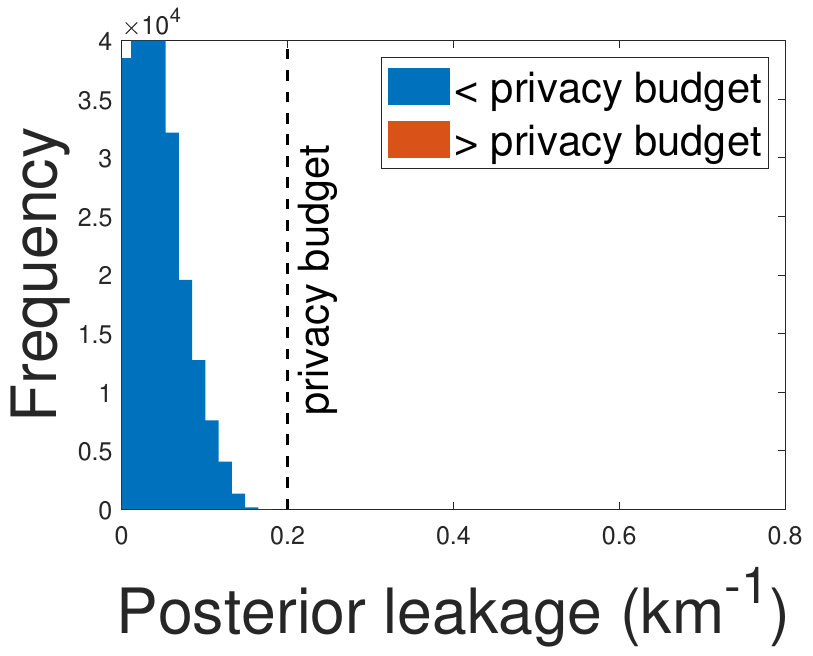}}
  \subfigure[Laplace]{
\includegraphics[width=0.133\textwidth, height = 0.09\textheight]{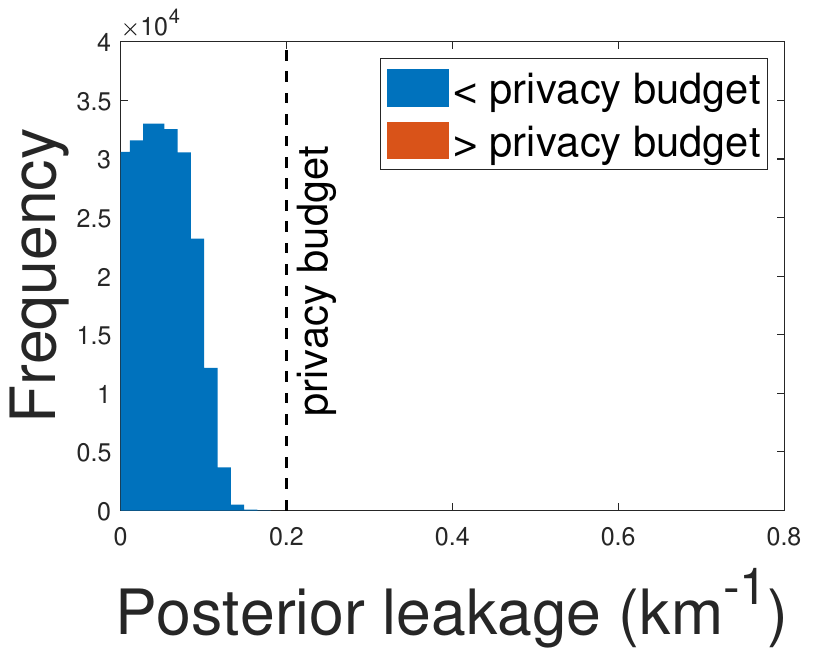}}
  \subfigure[TEM]{
\includegraphics[width=0.133\textwidth, height = 0.09\textheight]{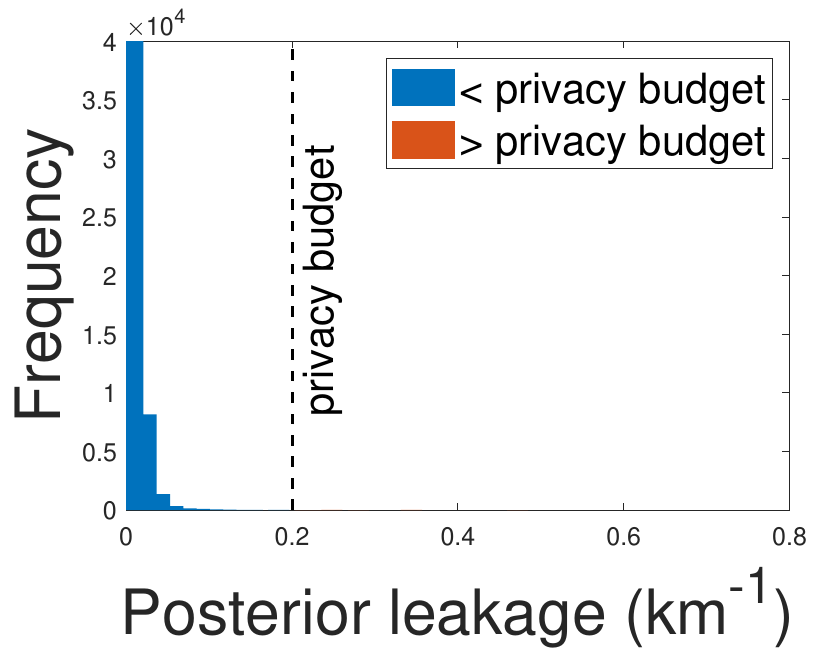}}
  \subfigure[COPT]{
\includegraphics[width=0.133\textwidth, height = 0.09\textheight]{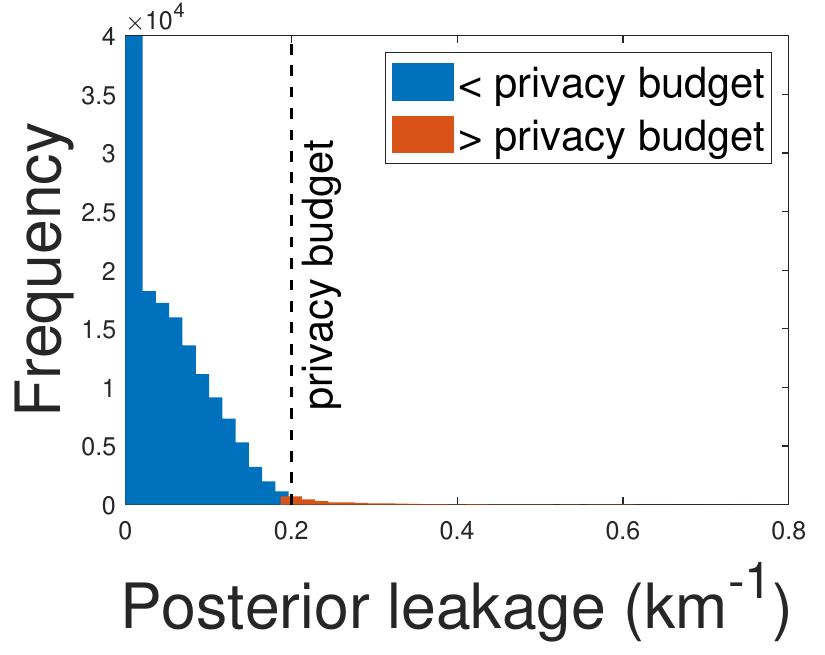}}
\subfigure[LP]{
\includegraphics[width=0.133\textwidth, height = 0.09\textheight]{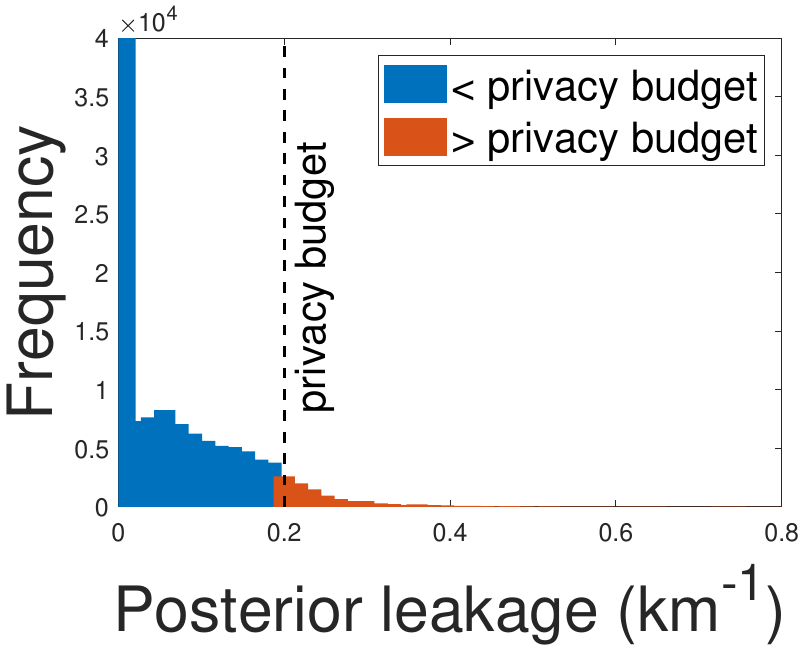}}  
\subfigure[AIPO-R]{
\includegraphics[width=0.133\textwidth, height = 0.09\textheight]{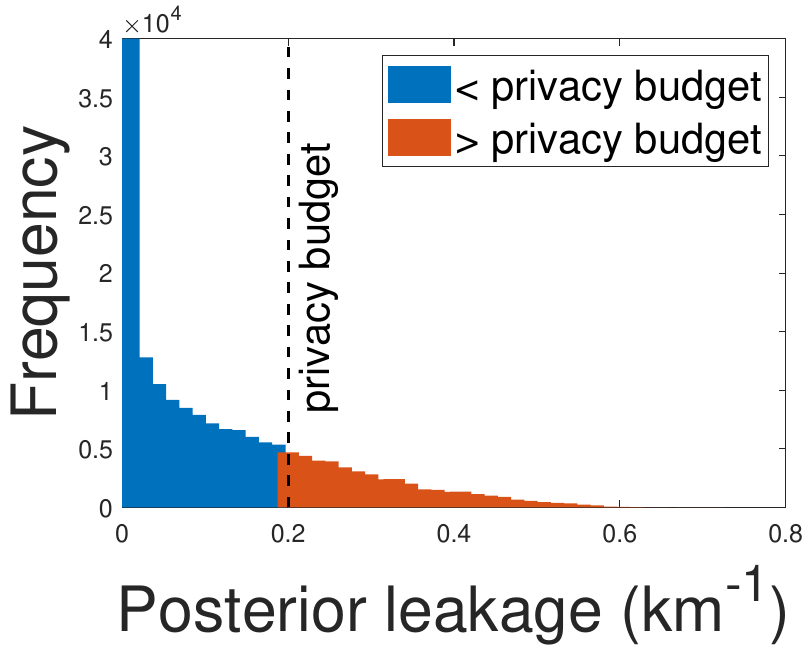}}
\subfigure[AIPO]{
\includegraphics[width=0.133\textwidth, height = 0.09\textheight]{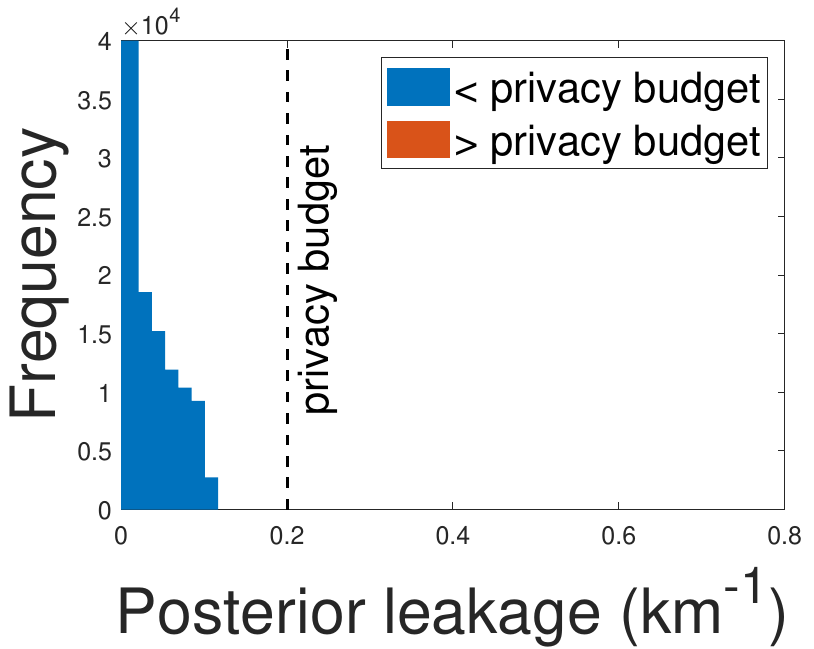}}
\end{minipage}
\caption{Example of posterior leakage distribution (Rome).}
\label{fig:PLdistRome}
\begin{minipage}{1.00\textwidth}
\centering
  \subfigure[EM]{
\includegraphics[width=0.133\textwidth, height = 0.09\textheight]{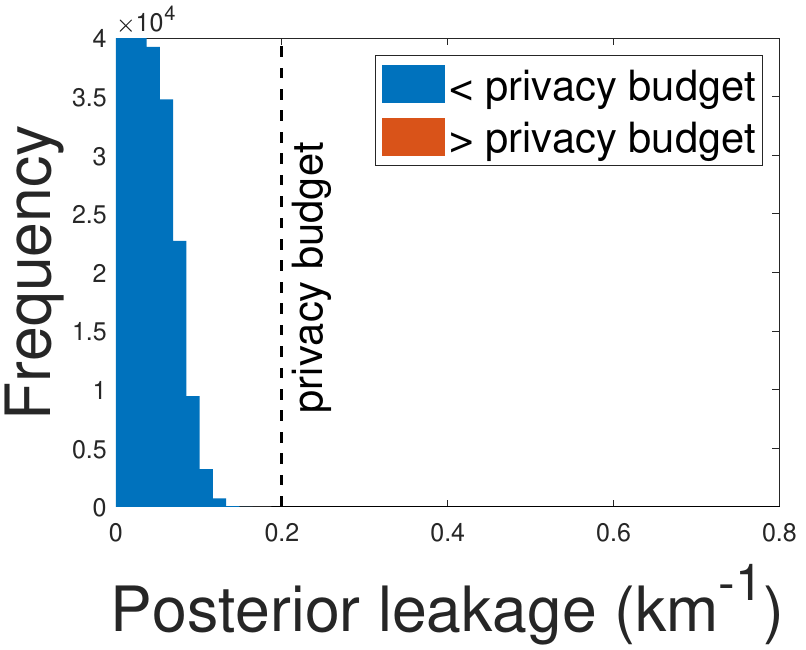}}
  \subfigure[Laplace]{
\includegraphics[width=0.133\textwidth, height = 0.09\textheight]{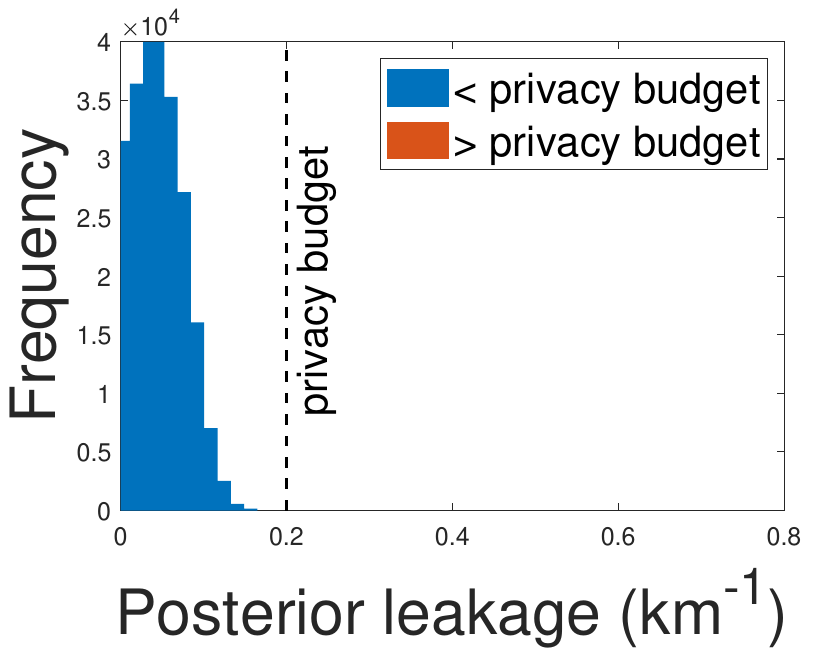}}
  \subfigure[TEM]{
\includegraphics[width=0.133\textwidth, height = 0.09\textheight]{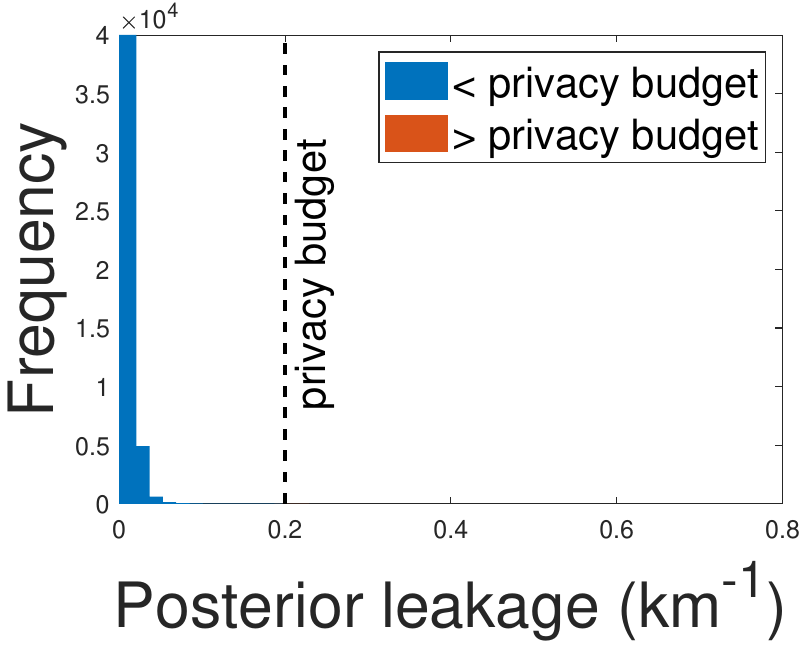}}
  \subfigure[COPT]{
\includegraphics[width=0.133\textwidth, height = 0.09\textheight]{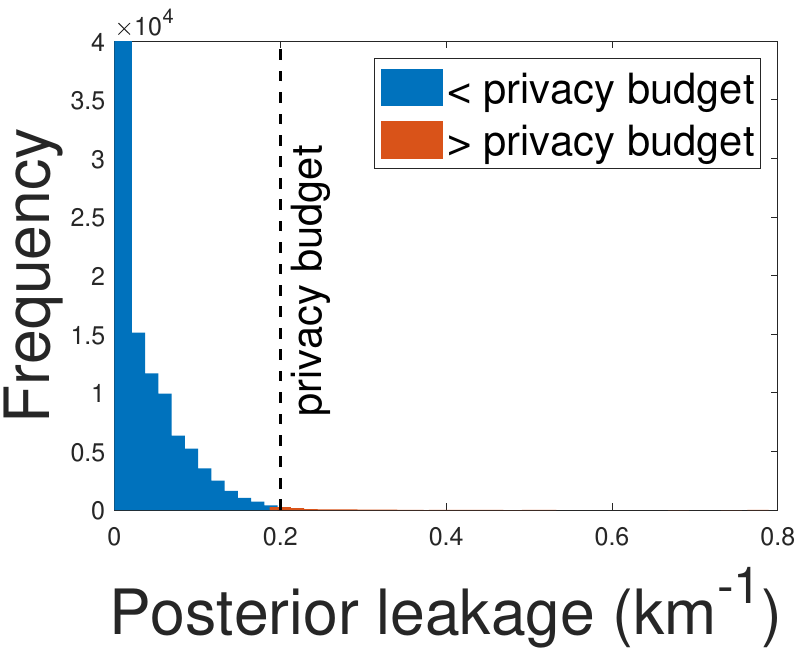}}
\subfigure[LP]{
\includegraphics[width=0.133\textwidth, height = 0.09\textheight]{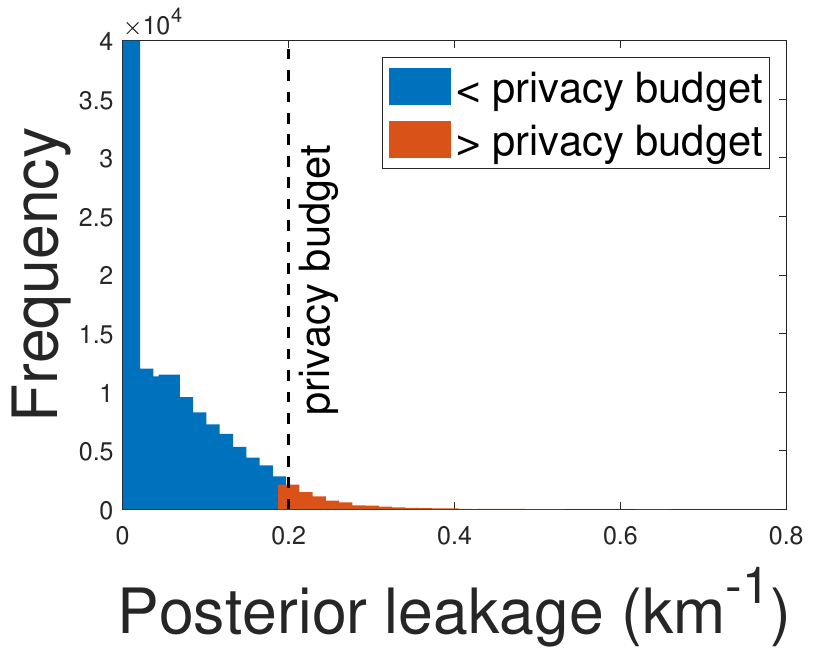}}
\subfigure[AIPO-R]{
\includegraphics[width=0.133\textwidth, height = 0.09\textheight]{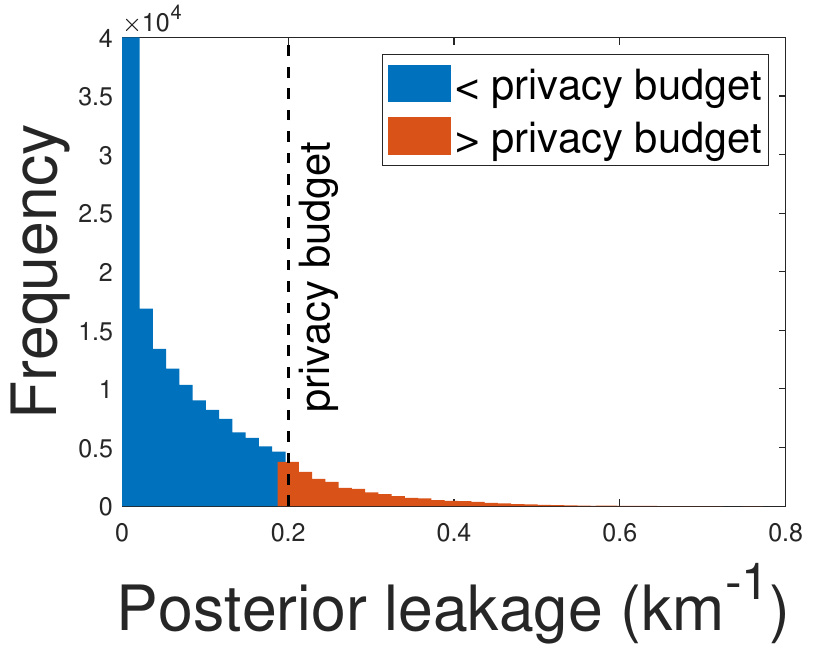}}
\subfigure[AIPO]{
\includegraphics[width=0.133\textwidth, height = 0.09\textheight]{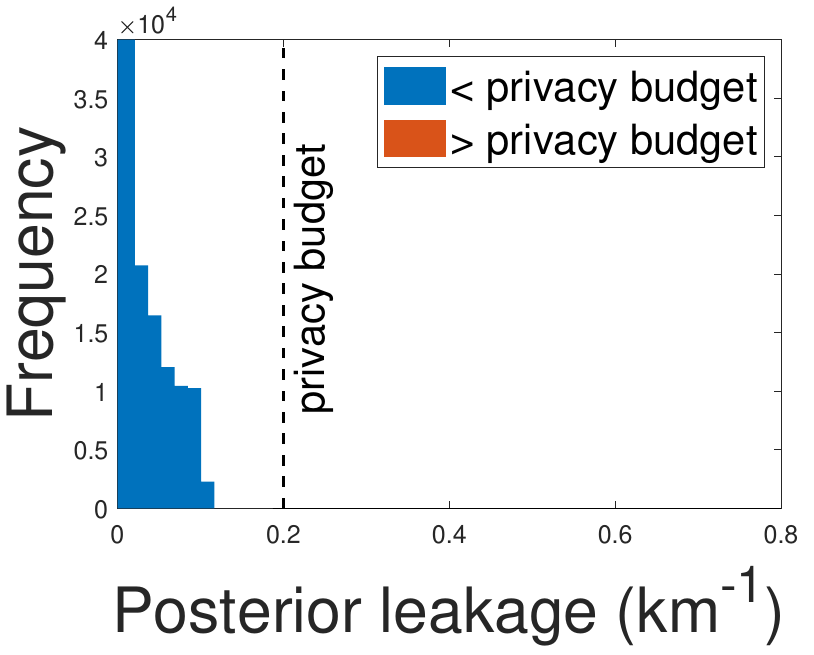}}
\end{minipage}
\caption{Example of posterior leakage distribution (London).}
\label{fig:PLdistLondon}
\begin{minipage}{1.00\textwidth}
\centering
  \subfigure[EM]{
\includegraphics[width=0.133\textwidth, height = 0.09\textheight]{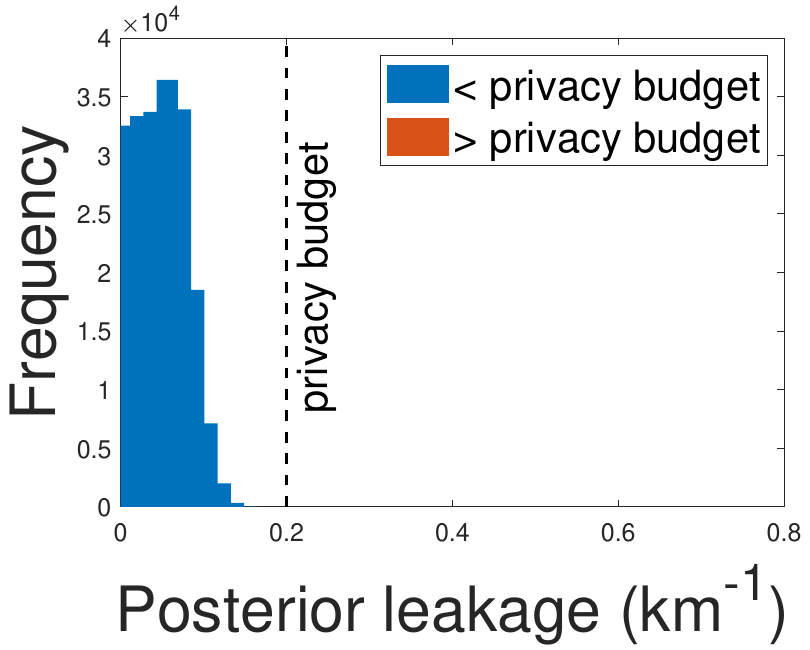}}
  \subfigure[Laplace]{
\includegraphics[width=0.133\textwidth, height = 0.09\textheight]{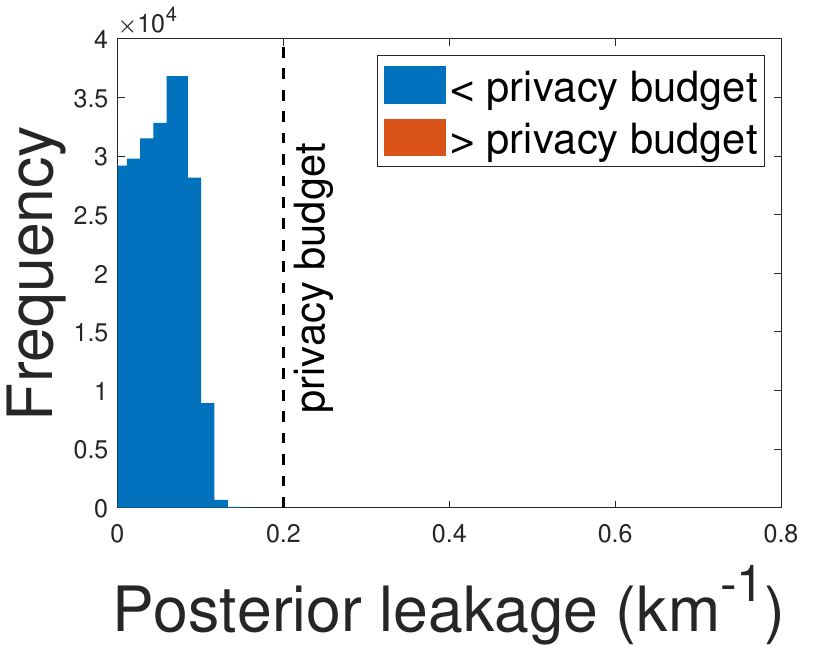}}
  \subfigure[TEM]{
\includegraphics[width=0.133\textwidth, height = 0.09\textheight]{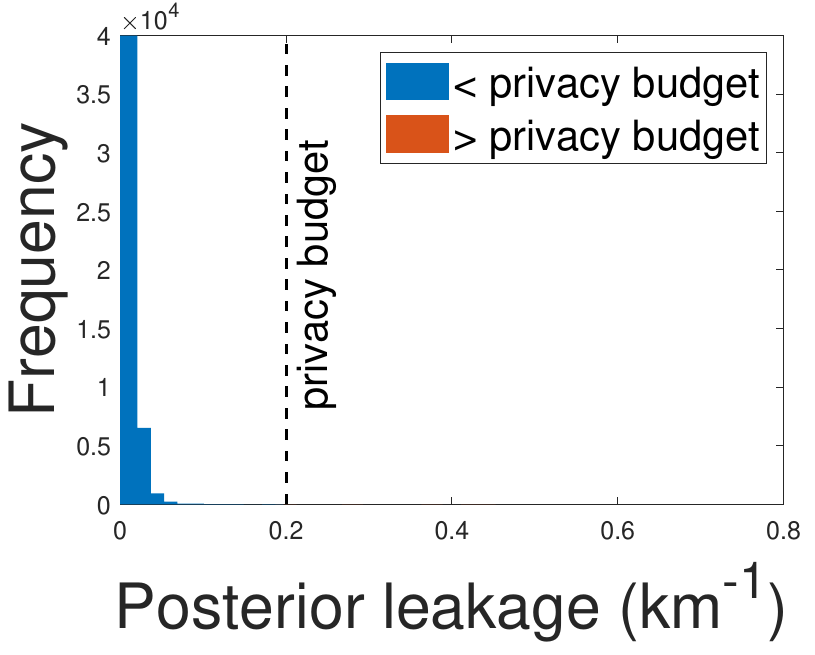}}
  \subfigure[COPT]{
\includegraphics[width=0.133\textwidth, height = 0.09\textheight]{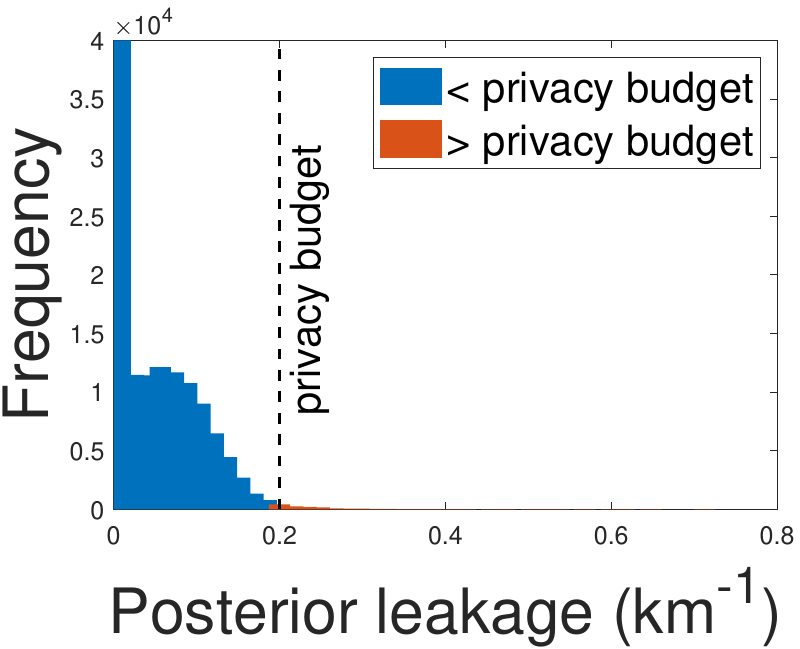}}
\subfigure[LP]{
\includegraphics[width=0.133\textwidth, height = 0.09\textheight]{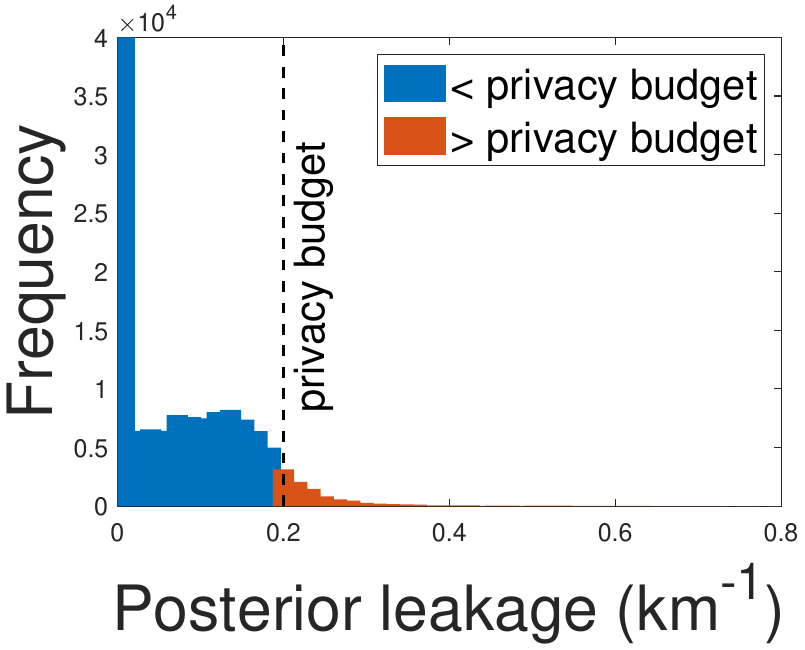}}
\subfigure[AIPO-R]{
\includegraphics[width=0.133\textwidth, height = 0.09\textheight]{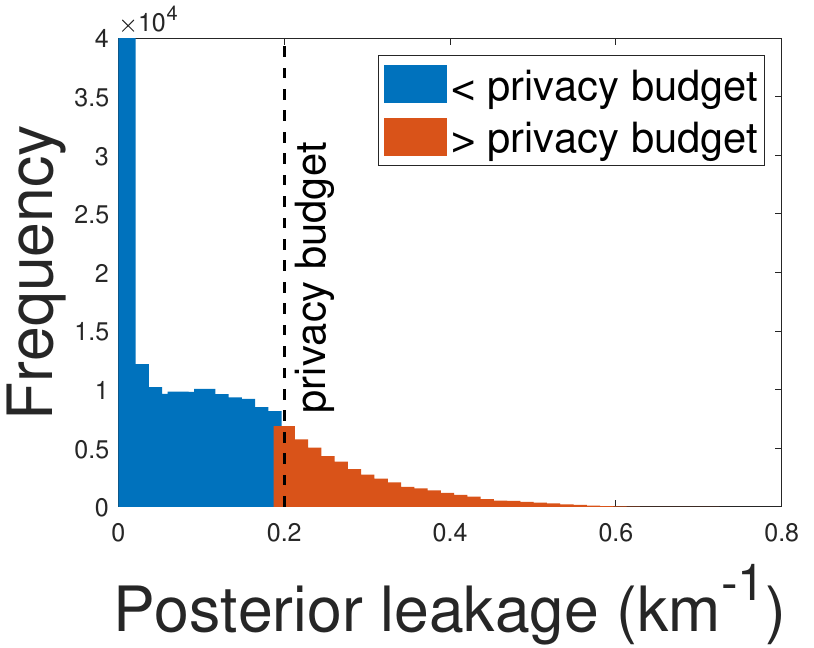}}
\subfigure[AIPO]{
\includegraphics[width=0.133\textwidth, height = 0.09\textheight]{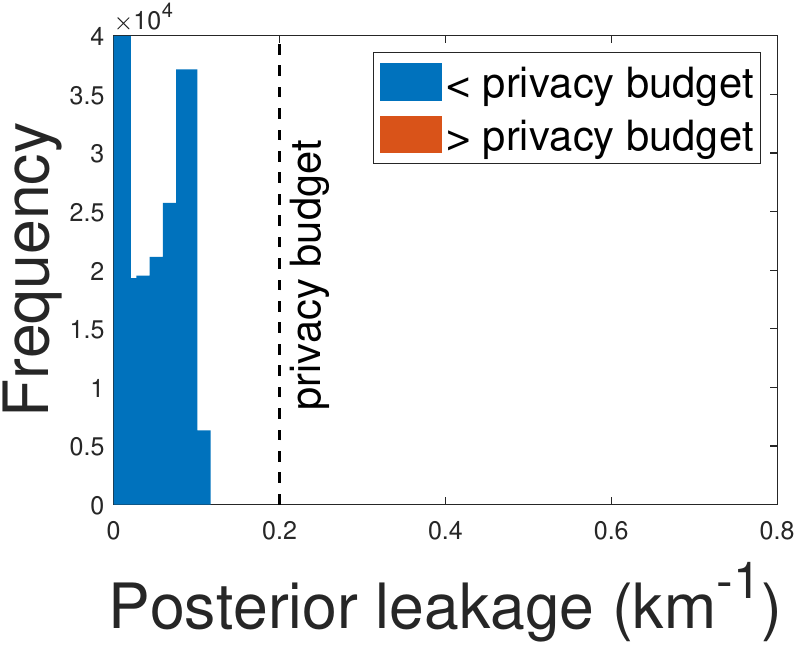}}
\end{minipage}
\caption{Example of posterior leakage distribution (New York City).}
\label{fig:PLdistNYC}
\end{figure*}

% Fig.~\ref{fig:PLdistRome}--\ref{fig:PLdistNYC} visualize the distribution of PPR values under $\epsilon = 0.2\text{km}^{-1}$ for the Rome, SF, and NYC datasets. AIPO exhibits sharply concentrated distributions with low variance and bounded tails, similar to pre-defined-noise mechanisms such as \textbf{Laplace} and \textbf{EM}. These concentrated distributions further support AIPO’s strong empirical privacy performance. In contrast, other optimization-based methods such as \textbf{LP} and \textbf{COPT} display broader distributions with heavier tails and occasional spikes, reflecting their inconsistent enforcement of mDP constraints.

Fig.~\ref{fig:PLdistRome}--\ref{fig:PLdistNYC} (corresponding to Rome, London, and New York City) illustrate the distributional behavior of the perturbation probability ratio (PPR) under different mechanisms. 

Recall that PPR measures the log-probability difference between two inputs normalized by their $\ell_p$ distance (Eq.~(\ref{eq:PPR})), and violations occur when this ratio exceeds the privacy budget $\epsilon$. Across all three datasets, AIPO exhibits a sharp concentration of PPR values strictly below the threshold, confirming zero mDP violations in practice. In contrast, LP-based optimization yields wider distributions with noticeable mass near or beyond the $\epsilon$ boundary, reflecting the effect of discretization and distance approximation errors. Hybrid methods such as COPT and RMP partially mitigate these issues, but their distributions remain more dispersed than AIPO. pre-defined mechanisms (Laplace, EM, TEM) also avoid explicit violations, yet their distributions are much flatter, indicating weaker privacy-utility trade-offs due to excessive randomization. Overall, the distributional analysis highlights the superiority of AIPO: not only does it enforce strict compliance with $(\epsilon,d_p)$-mDP, but it also produces sharply bounded PPR profiles that are consistent across cities of varying scale and density.

\begin{table*}[t]
\centering
% \small 
\footnotesize 
% \scriptsize
\begin{tabular}{ p{1.25cm} p{1.25cm} p{1.25cm} p{1.25cm} p{1.25cm} p{1.25cm} p{1.25cm} p{1.25cm} p{1.25cm}} 
\toprule
\multicolumn{9}{c }{Rome road map}\\ 
 \cline{2-9}
Method & $\epsilon = 0.2$& $\epsilon = 0.4$& $\epsilon = 0.6$& $\epsilon = 0.8$& $\epsilon = 1.0$& $\epsilon = 1.2$& $\epsilon = 1.4$& $\epsilon = 1.6$\\ 
 \hline
 \hline
\rowcolor{lightgray!30}
\textbf{AIPO}$^\dagger$ & 6.01$\pm$0.26 & 4.96$\pm$0.56 & 4.28$\pm$0.67 & 3.77$\pm$0.67 & 3.52$\pm$0.83 & 3.22$\pm$4.01 & 3.14$\pm$0.92 & 3.01$\pm$0.93 \\
AIPO-E & 6.24$\pm$0.02 & 5.02$\pm$0.26 & 4.37$\pm$0.48 & 3.86$\pm$0.59 & 3.58$\pm$0.66 & 3.37$\pm$0.70 & 3.21$\pm$0.73 & 3.09$\pm$0.74 \\
\toprule
\multicolumn{9}{c }{London road map}\\ 
 \cline{2-9}
Method & $\epsilon = 0.2$& $\epsilon = 0.4$& $\epsilon = 0.6$& $\epsilon = 0.8$& $\epsilon = 1.0$& $\epsilon = 1.2$& $\epsilon = 1.4$& $\epsilon = 1.6$\\ 
\hline
\hline
\rowcolor{lightgray!30}
\textbf{AIPO}$^\dagger$ & 5.42$\pm$0.03 & 4.34$\pm$0.07 & 3.66$\pm$0.12 & 3.20$\pm$0.14 & 2.90$\pm$0.16 & 2.66$\pm$0.14 & 2.50$\pm$0.14 & 2.37$\pm$0.17 \\
AIPO-E & 5.46$\pm$0.03 & 4.42$\pm$0.08 & 3.85$\pm$0.13 & 3.44$\pm$0.15 & 3.14$\pm$0.16 & 2.91$\pm$0.16 & 2.73$\pm$0.16 & 2.58$\pm$0.16 \\
\toprule
\multicolumn{9}{c }{New York City road map}\\ 
 \cline{2-9}
Method & $\epsilon = 0.2$& $\epsilon = 0.4$& $\epsilon = 0.6$& $\epsilon = 0.8$& $\epsilon = 1.0$& $\epsilon = 1.2$& $\epsilon = 1.4$& $\epsilon = 1.6$\\ 
 \hline
 \hline
\rowcolor{lightgray!30}
\textbf{AIPO}$^\dagger$ & 7.15$\pm$0.14 & 5.16$\pm$0.14 & 4.14$\pm$0.21 & 3.54$\pm$0.22 & 3.13$\pm$0.22 & 2.86$\pm$0.23 & 2.66$\pm$0.24 & 2.51$\pm$0.22 \\
AIPO-E & 7.19$\pm$0.13 & 5.33$\pm$0.14 & 4.39$\pm$0.22 & 3.80$\pm$0.23 & 3.40$\pm$0.24 & 3.05$\pm$0.64 & 2.90$\pm$0.25 & 2.74$\pm$0.25 \\
\hline
    \end{tabular}
\caption{AIPO with and without privacy budget optimization (Mean$\pm$1.96$\times$standard deviation).}
\label{tab:privacybudget}
\end{table*}

\subsection{Ablantion Study: Utility Loss With And Without Privacy Budget Optimization}
\label{subsec:exp:privacybudget}

\begin{figure*}[t]
\centering
\hspace{0.00in}
\begin{minipage}{1.00\textwidth}
\centering
  \subfigure[$\epsilon = 0.2$]{
\includegraphics[width=0.115\textwidth, height = 0.08\textheight]{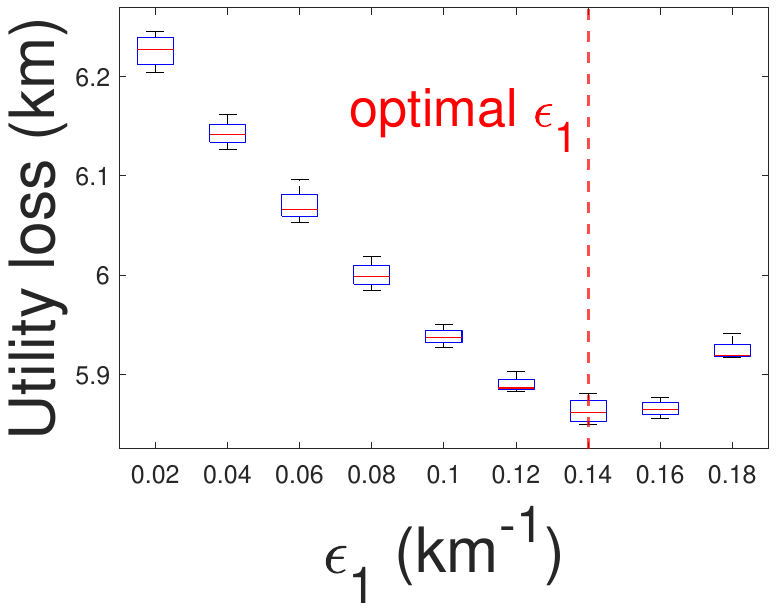}}
  \subfigure[$\epsilon = 0.4$]{
\includegraphics[width=0.115\textwidth, height = 0.08\textheight]{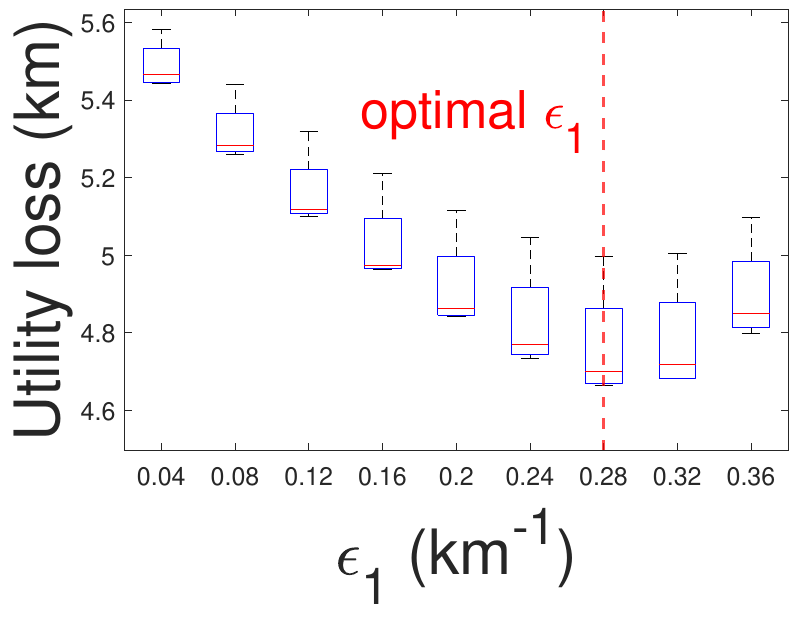}}
  \subfigure[$\epsilon = 0.6$]{
\includegraphics[width=0.115\textwidth, height = 0.08\textheight]{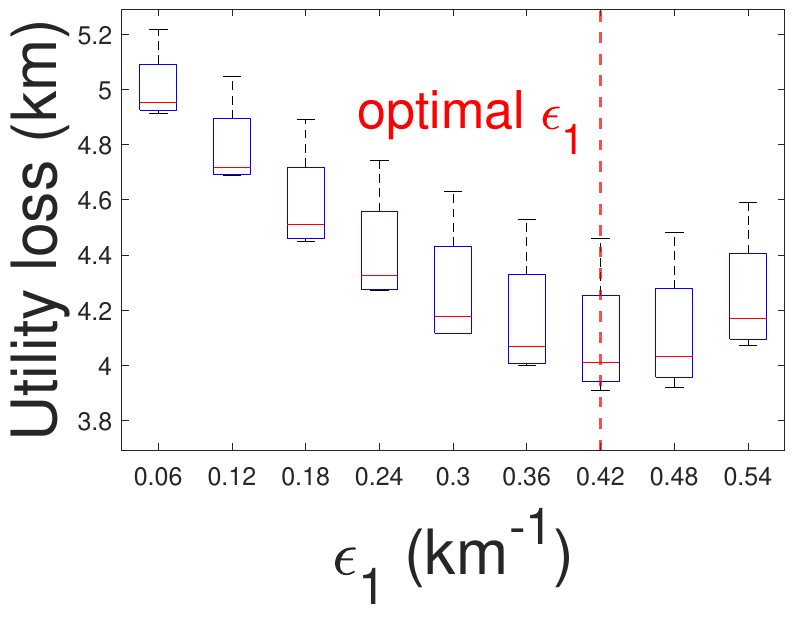}}
  \subfigure[$\epsilon = 0.8$]{
\includegraphics[width=0.115\textwidth, height = 0.08\textheight]{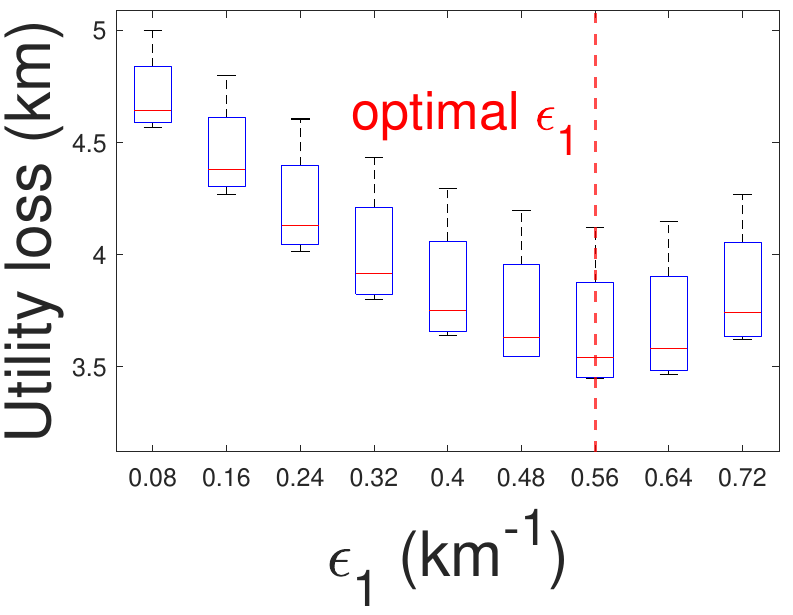}}
\subfigure[$\epsilon = 1.0$]{
\includegraphics[width=0.115\textwidth, height = 0.08\textheight]{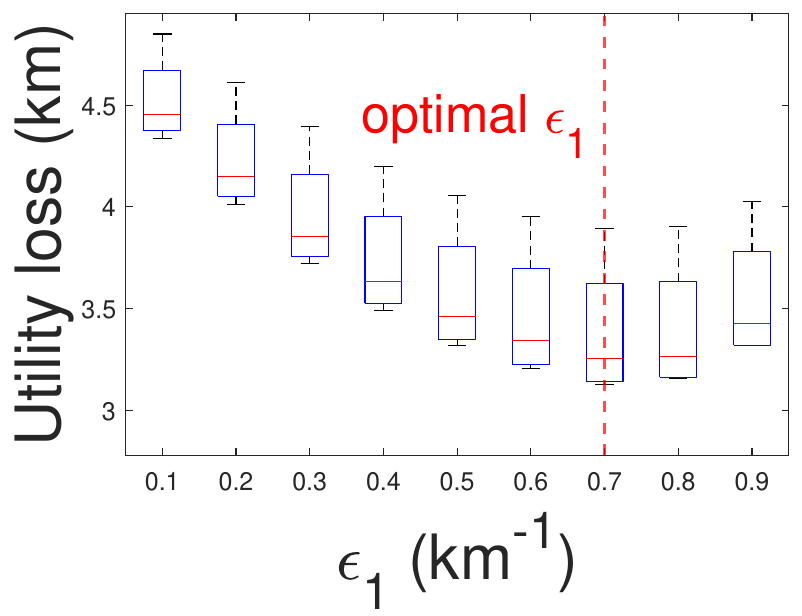}}
  \subfigure[$\epsilon = 1.2$]{
\includegraphics[width=0.115\textwidth, height = 0.08\textheight]{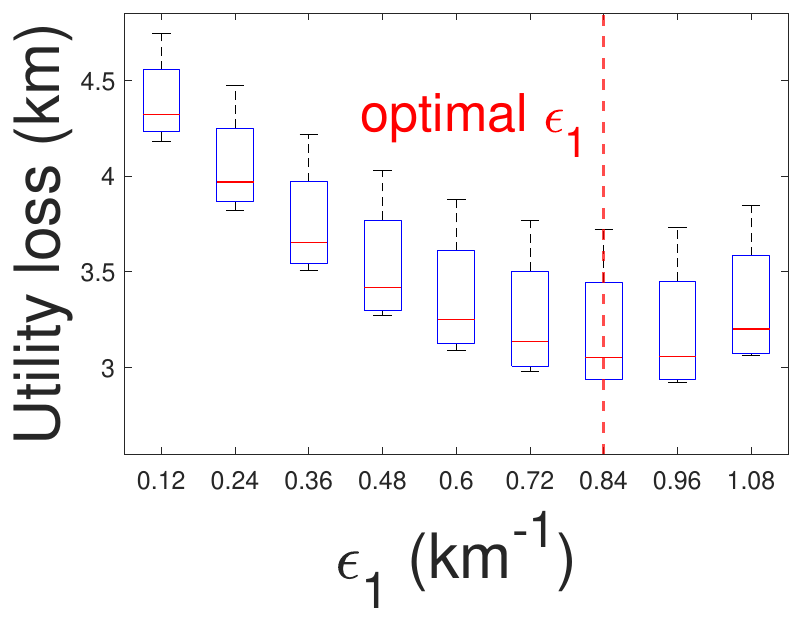}}
  \subfigure[$\epsilon = 1.4$]{
\includegraphics[width=0.115\textwidth, height = 0.08\textheight]{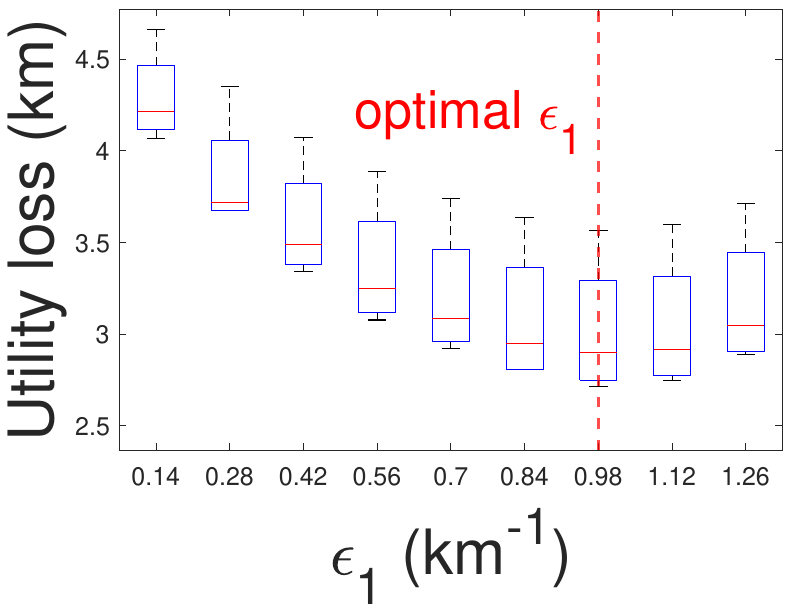}}
  \subfigure[$\epsilon = 1.6$]{
\includegraphics[width=0.115\textwidth, height = 0.08\textheight]{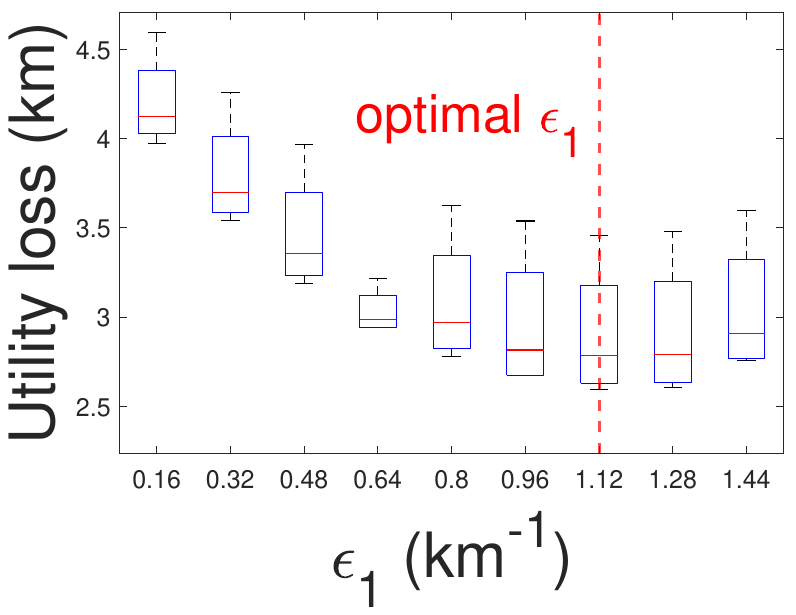}}
\end{minipage}
\caption{Utility loss vs. privacy budget assigned to dimension 1 (Rome).}
\label{fig:ULbudgetRome}
\begin{minipage}{1.00\textwidth}
\centering
  \subfigure[$\epsilon = 0.2$]{
\includegraphics[width=0.115\textwidth, height = 0.08\textheight]{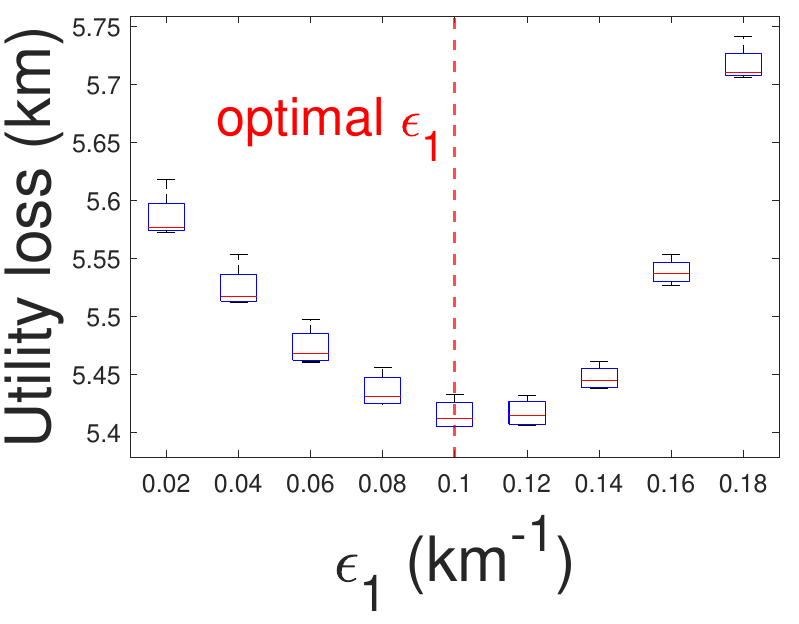}}
  \subfigure[$\epsilon = 0.4$]{
\includegraphics[width=0.115\textwidth, height = 0.08\textheight]{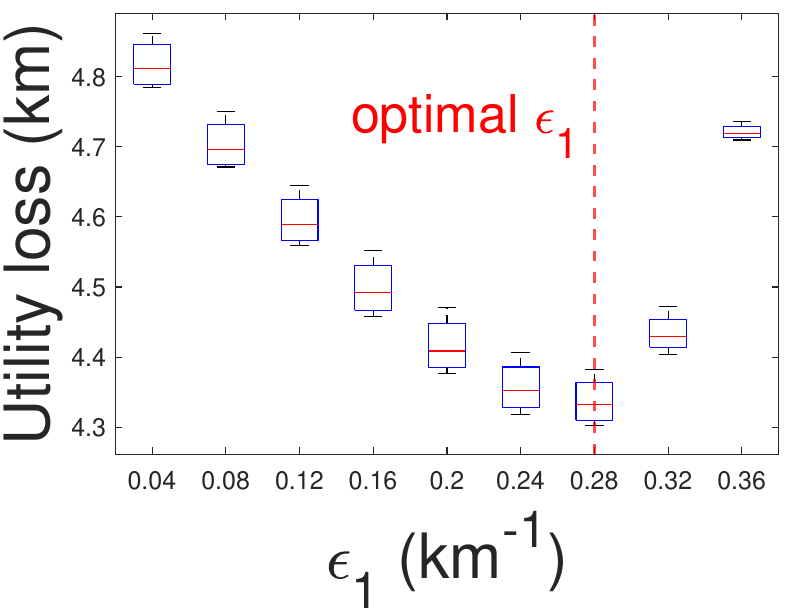}}
  \subfigure[$\epsilon = 0.6$]{
\includegraphics[width=0.115\textwidth, height = 0.08\textheight]{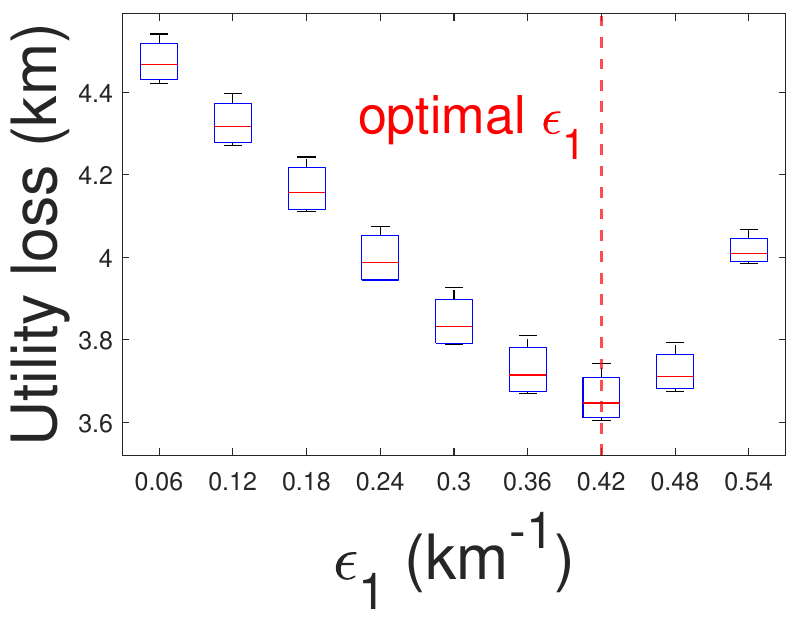}}
  \subfigure[$\epsilon = 0.8$]{
\includegraphics[width=0.115\textwidth, height = 0.08\textheight]{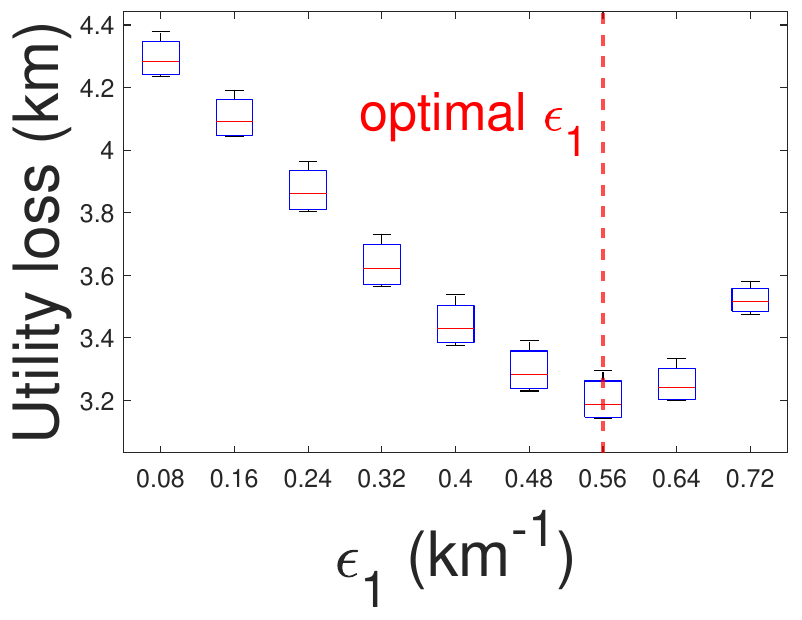}}
\subfigure[$\epsilon = 1.0$]{
\includegraphics[width=0.115\textwidth, height = 0.08\textheight]{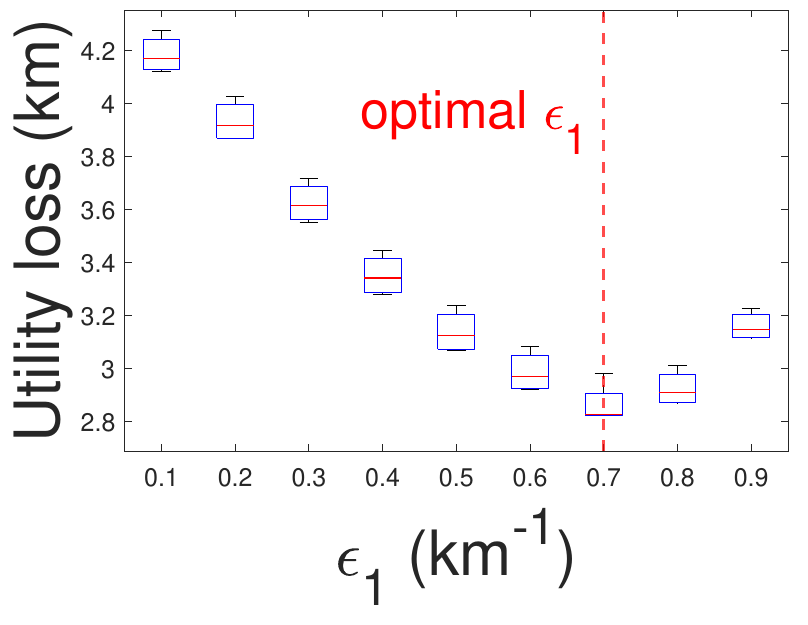}}
  \subfigure[$\epsilon = 1.2$]{
\includegraphics[width=0.115\textwidth, height = 0.08\textheight]{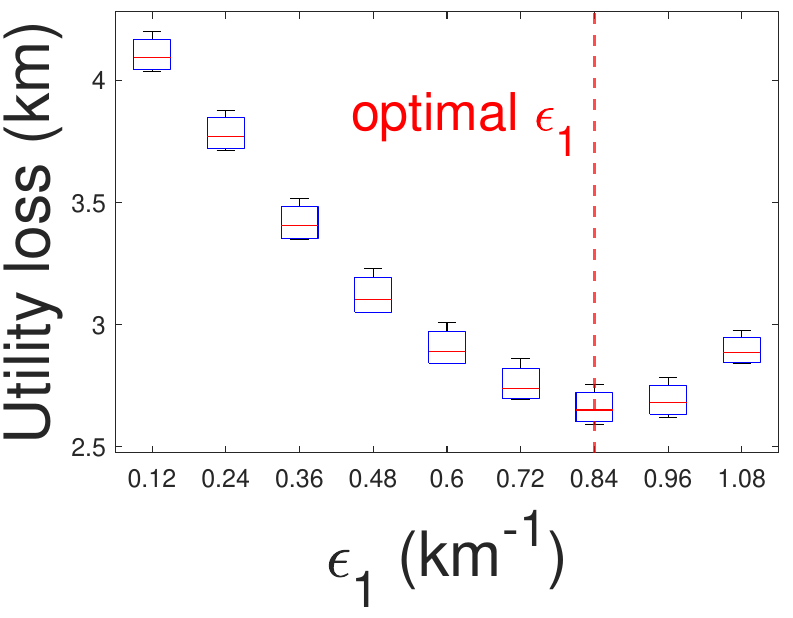}}
  \subfigure[$\epsilon = 1.4$]{
\includegraphics[width=0.115\textwidth, height = 0.08\textheight]{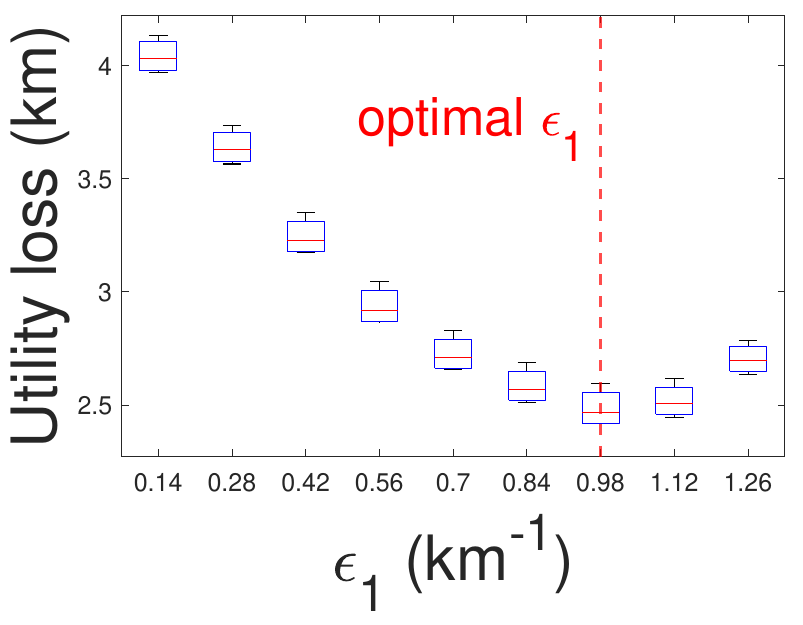}}
  \subfigure[$\epsilon = 1.6$]{
\includegraphics[width=0.115\textwidth, height = 0.08\textheight]{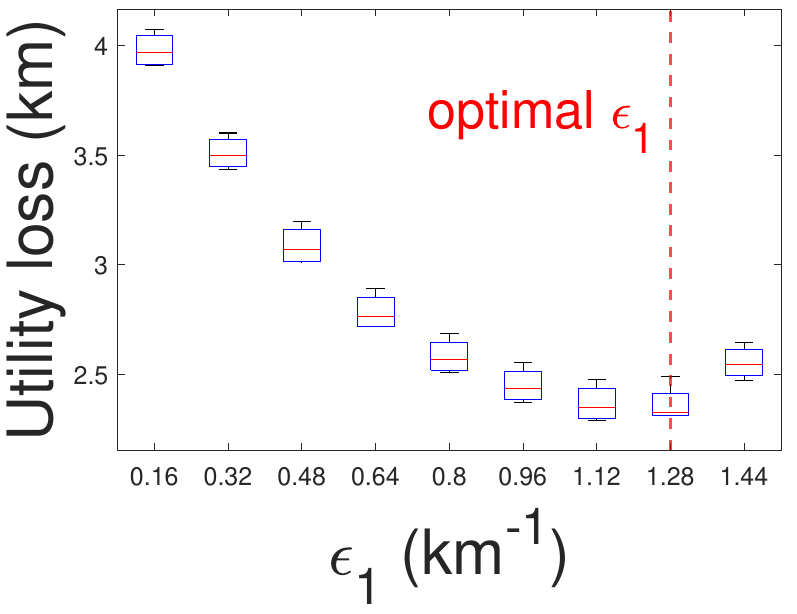}}
\end{minipage}
\caption{Utility loss vs. privacy budget assigned to dimension 1 (London).}
\label{fig:ULbudgetLondon}
\begin{minipage}{1.00\textwidth}
\centering
  \subfigure[$\epsilon = 0.2$]{
\includegraphics[width=0.115\textwidth, height = 0.08\textheight]{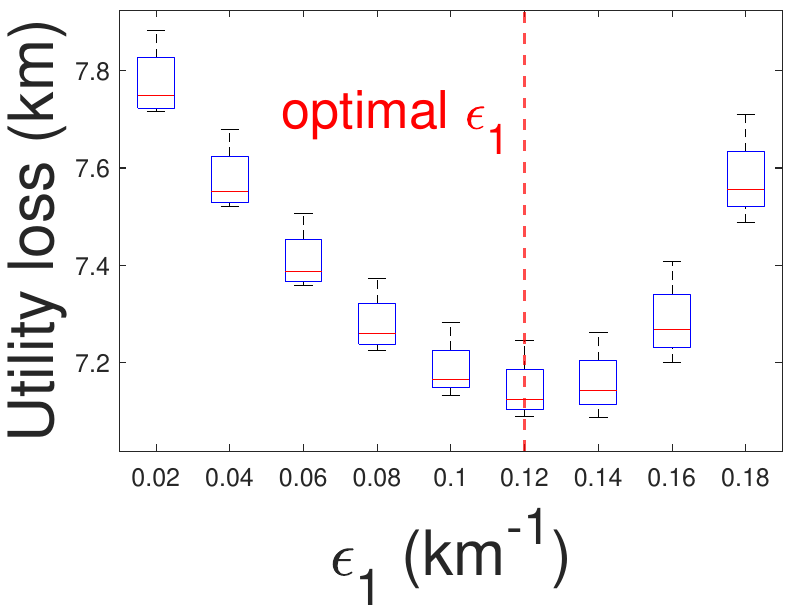}}
  \subfigure[$\epsilon = 0.4$]{
\includegraphics[width=0.115\textwidth, height = 0.08\textheight]{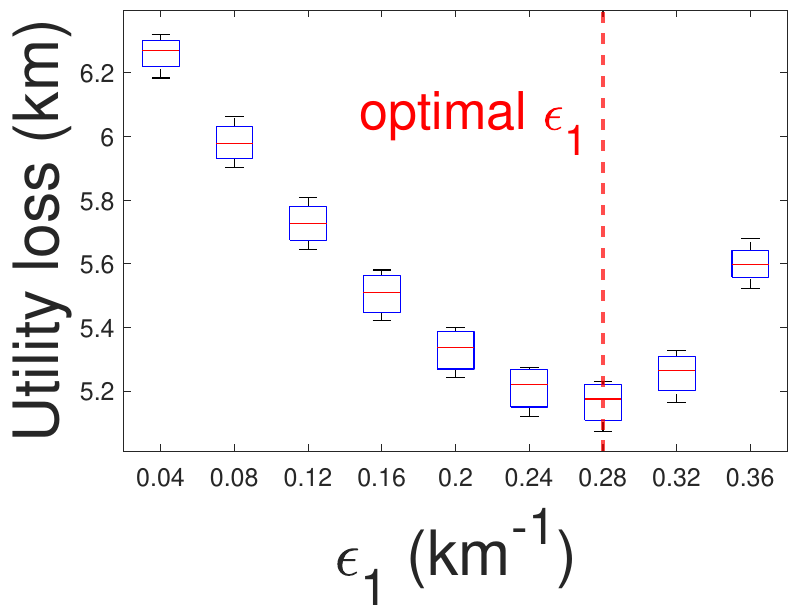}}
  \subfigure[$\epsilon = 0.6$]{
\includegraphics[width=0.115\textwidth, height = 0.08\textheight]{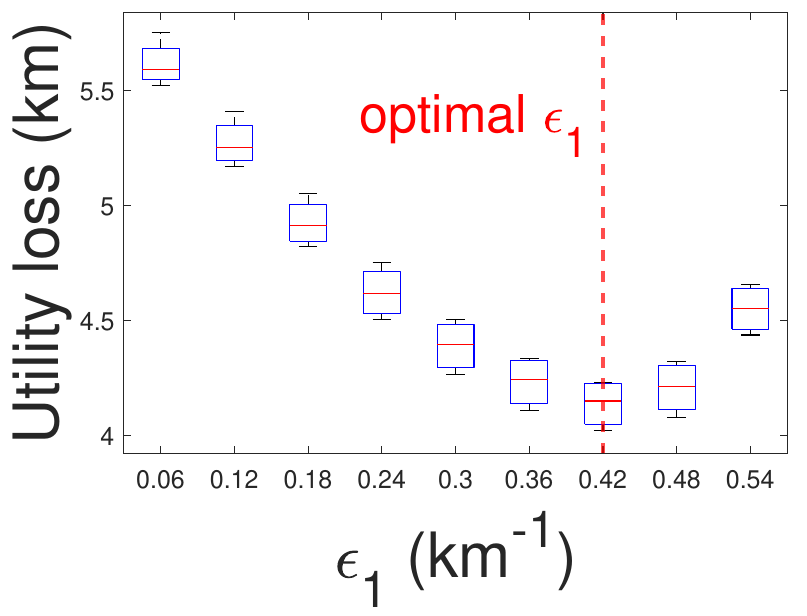}}
  \subfigure[$\epsilon = 0.8$]{
\includegraphics[width=0.115\textwidth, height = 0.08\textheight]{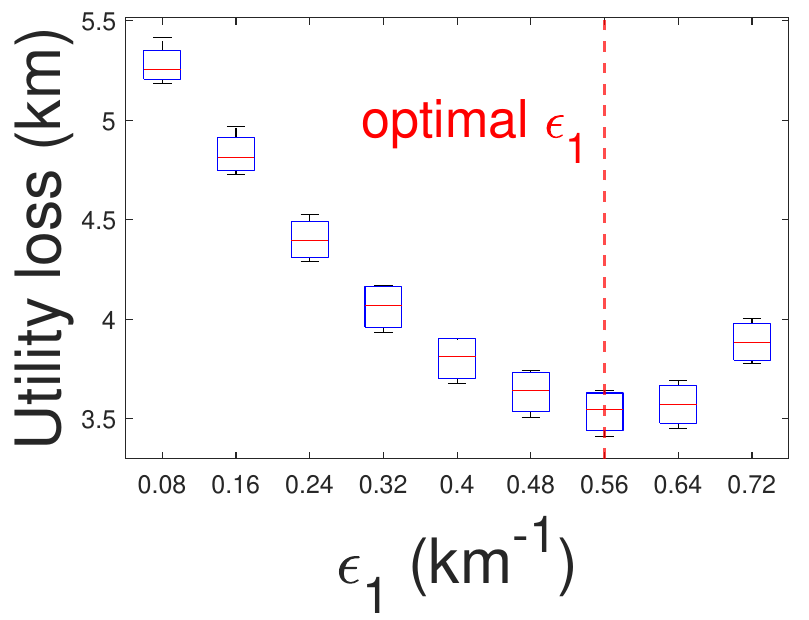}}
\subfigure[$\epsilon = 1.0$]{
\includegraphics[width=0.115\textwidth, height = 0.08\textheight]{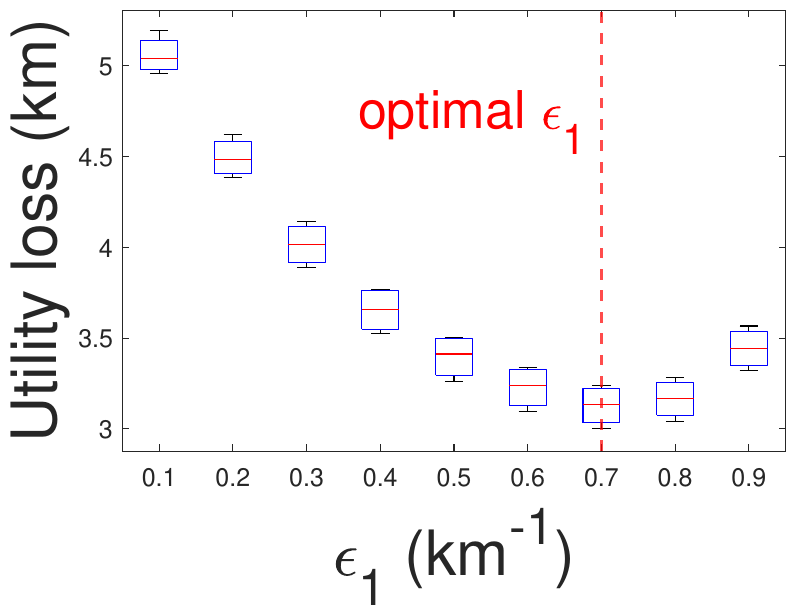}}
  \subfigure[$\epsilon = 1.2$]{
\includegraphics[width=0.115\textwidth, height = 0.08\textheight]{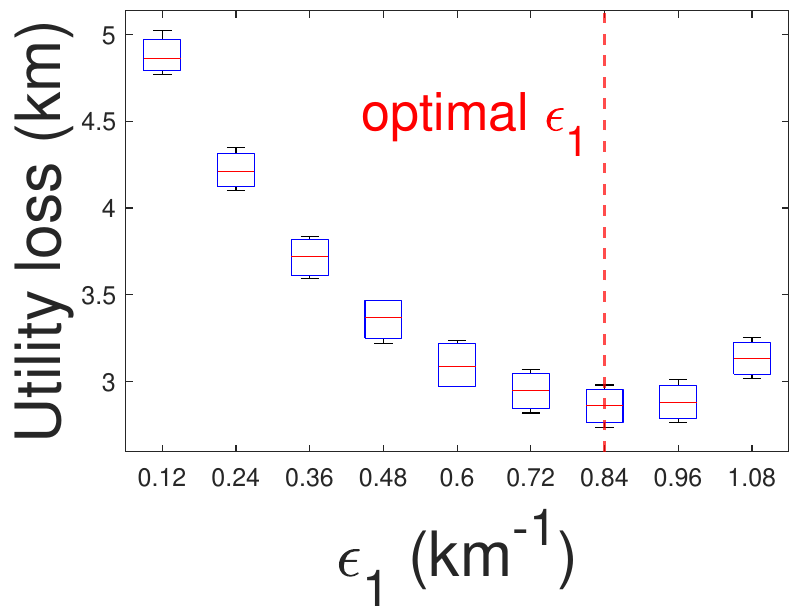}}
  \subfigure[$\epsilon = 1.4$]{
\includegraphics[width=0.115\textwidth, height = 0.08\textheight]{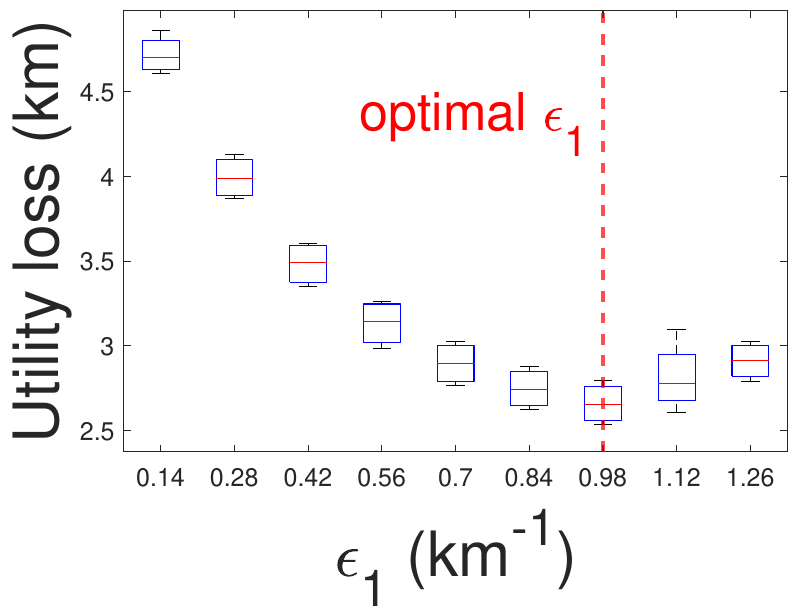}}
  \subfigure[$\epsilon = 1.6$]{
\includegraphics[width=0.115\textwidth, height = 0.08\textheight]{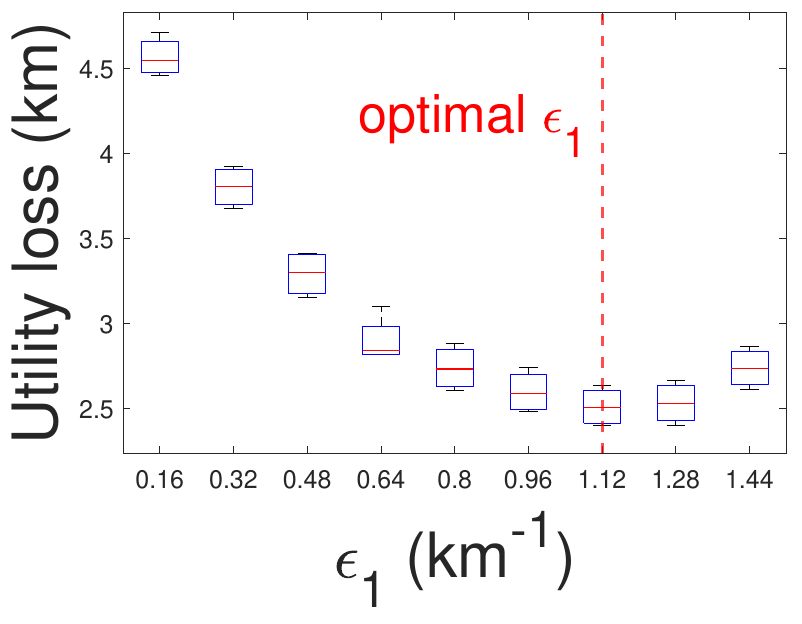}}
\end{minipage}
\caption{Utility loss vs. privacy budget assigned to dimension 1 (New York City).}
\label{fig:ULbudgetNYC}
\end{figure*}

Table~\ref{tab:privacybudget} compares \emph{AIPO} with its equal-budget counterpart \emph{AIPO-E}, which uniformly allocates the privacy budget across dimensions (e.g., $\epsilon_1=\epsilon_2=\epsilon/\sqrt{2}$ under $\ell_2$). Across Rome, London, and NYC, \emph{AIPO} consistently achieves lower utility loss than \emph{AIPO-E}, and the margin widens as the overall privacy level $\epsilon$ increases. This trend is expected: at small $\epsilon$, noise is necessarily large and budget reallocation has limited effect, whereas at larger $\epsilon$ the mechanism has more freedom to shape direction-dependent noise, and optimizing $(\epsilon_1,\epsilon_2)$ yields tangible gains. The improvement stems from \emph{AIPO}'s asymmetric allocation adapting to directional utility sensitivity and map anisotropy (e.g., corridor-like connectivity or dominant travel directions). Concretely, on the NYC dataset, \emph{AIPO} attains up to $7.8\%$ lower utility loss than \emph{AIPO-E}$\,$at $\epsilon=1.6$, highlighting the value of optimizing privacy allocation. Similar, though slightly smaller, gains are observed in Rome and London, consistent with their less anisotropic street layouts. These results validate that dimension-wise budget composition (Theorem~\ref{thm:composition}) is not merely theoretically sufficient for $(\epsilon,d_p)$-mDP but also practically important for enhancing utility.

\smallskip\noindent\textit{Budget-sweep methodology and insights.} To quantify how allocation affects utility, we fix the total budget $\epsilon$ and vary the first-dimension share $\epsilon_1$ along the feasible arc $\epsilon_1^p+\epsilon_2^p=\epsilon^p$ (with $p=2$ in our main experiments), discretizing $\epsilon_1$ over a uniform grid and setting $\epsilon_2=(\epsilon^p-\epsilon_1^p)^{1/p}$. For each candidate $(\epsilon_1,\epsilon_2)$ we solve \emph{AIPO} and record the utility loss; the best value for a given $\epsilon$ is the minimum over this sweep. Fig.~\ref{fig:ULbudgetRome}--\ref{fig:ULbudgetNYC} plot utility as a function of $\epsilon_1$: the curves are characteristically U-shaped with minima \emph{away} from the equal-split point, indicating that equal allocation is suboptimal in all three cities. Moreover, the optimal $\epsilon_1$ shifts with $\epsilon$, reflecting that the most effective directional protection depends jointly on the privacy level and the dataset’s geometric/traffic structure. Together with Table~\ref{tab:privacybudget}, these observations underscore that learning the per-dimension budgets is a key lever for improving utility while preserving $(\epsilon,d_p)$-mDP.
 % Across all cases, we observe that the utility function is non-convex in the budget allocation, and that improper assignment (e.g., uniform splitting) can lead to suboptimal performance. AIPO with budget optimization consistently identifies near-optimal allocation points that achieve the lowest utility loss. This experiment confirms that optimizing the dimension-wise budget composition is crucial for leveraging the geometry of the utility landscape and achieving tight privacy-utility tradeoffs.

\subsection{Performance of the Interpolation-Based Method under Varying Grid Granularity}
\label{subsec:exp:granularity}
% Fig.~\ref{fig:gridgranularity}(a)(b)(c) evaluates the performance of AIPO under different grid resolutions by varying the number of discretized cells in the domain. As the number of grid cells increases, the utility loss steadily decreases due to finer anchor placement and improved alignment with the domain's geometric structure. However, this improvement comes at the cost of increased computation time and a larger number of anchor points, resulting in more decision variables in the optimization problem.

% Notably, the utility gains begin to saturate beyond a certain resolution—specifically, when the number of horizontal grid cells reaches 8 for Rome and London, and 10 for New York City—while the computational cost continues to grow super-linearly. To strike a balance between computational efficiency and utility performance, we set the number of horizontal grid cells to 10 for all three datasets by default.

\begin{figure*}[t]
\centering
\begin{minipage}{1.00\textwidth}
% \centering
  \subfigure[Rome]{
\includegraphics[width=0.31\textwidth, height = 0.28\textheight]{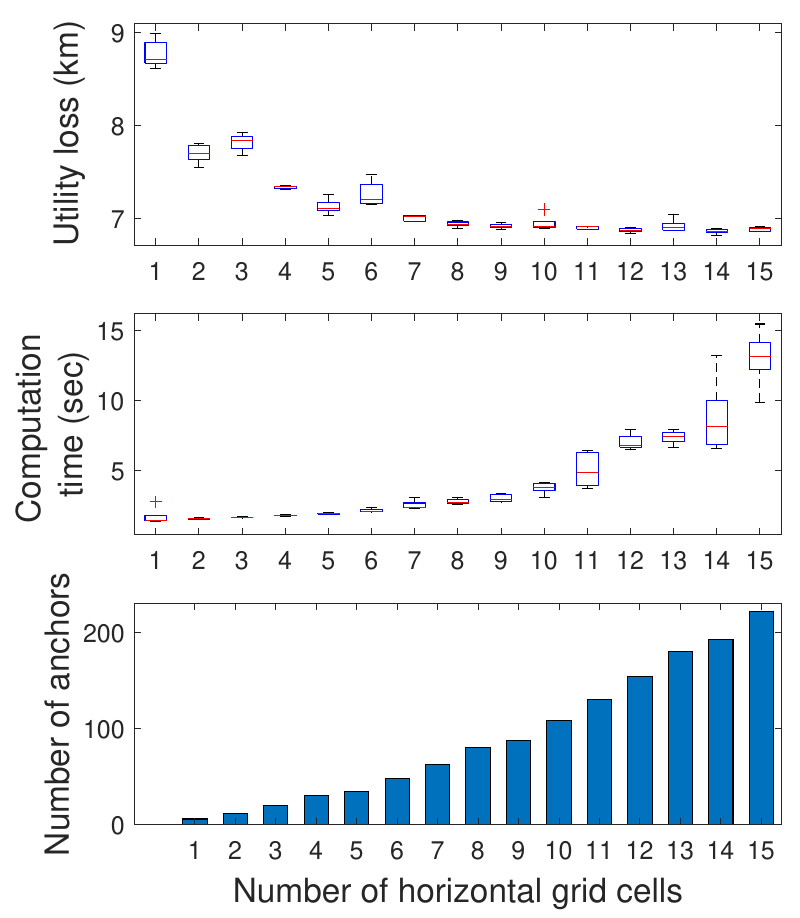}}
  \subfigure[London]{
\includegraphics[width=0.31\textwidth, height = 0.28\textheight]{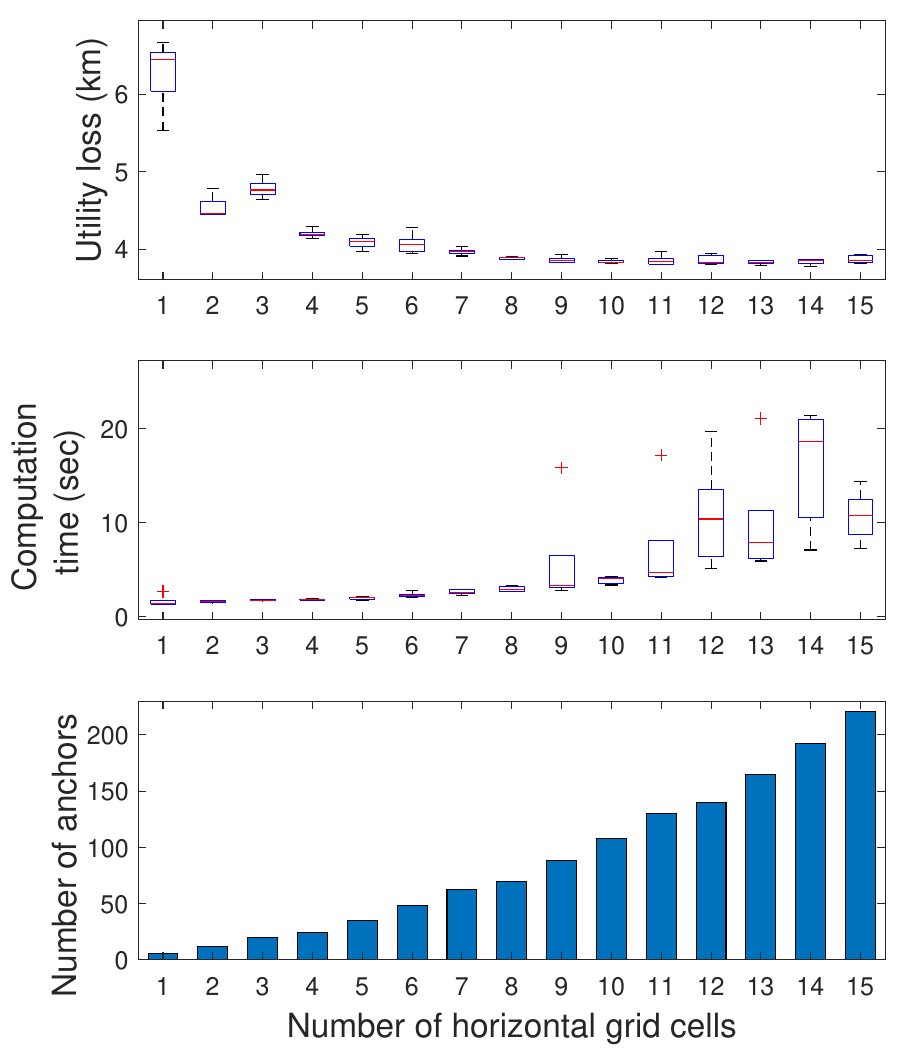}}
  \subfigure[New York City]{
\includegraphics[width=0.31\textwidth, height = 0.28\textheight]{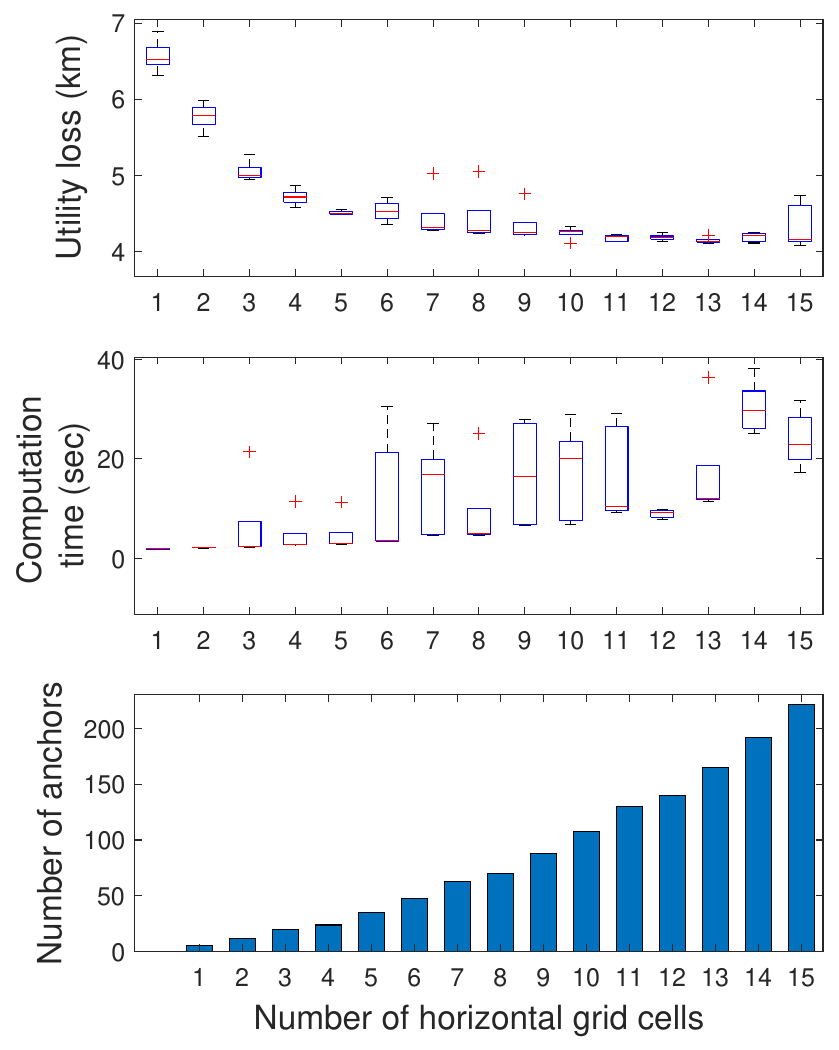}}
\end{minipage}
\caption{Performance of AIPO with varying Grid Granularity.}
\label{fig:gridgranularity}
\end{figure*}

Fig.~\ref{fig:gridgranularity}(a)–(c) evaluates AIPO across grid resolutions by varying the number of discretized cells. As resolution increases, utility loss decreases monotonically because denser anchor placement better captures local geometry (e.g., turn penalties and anisotropy along road segments), thereby reducing interpolation error. This gain, however, comes with higher computational cost: the number of decision variables grows as $\mathcal{O}(|\widehat{\mathcal{X}}|\cdot|\mathcal{Y}|)$ and the number of mDP constraints grows with anchor neighborhood pairs (super-linearly in practice due to barrier/simplex iterations), leading to longer runtimes and higher memory usage. We also note that privacy is \emph{insensitive} to grid granularity for AIPO: the PPR distributions remain tightly concentrated below $\epsilon$ at all resolutions (cf. Section~D.1), since mDP is enforced via dimension-wise composition on anchors and preserved by log-convex interpolation. Importantly, diminishing returns emerge beyond a dataset-specific threshold—utility gains saturate once horizontal cells exceed $\sim\!8$ for Rome and London and $\sim\!10$ for New York City—while runtime continues to rise super-linearly. In light of this trade-off, we default to $10$ horizontal cells for all three datasets, which is near the knee of the curve: it retains most of the achievable utility improvement while keeping the optimization tractable. Practically, coarser grids (e.g., $8$ horizontal cells) already offer near-optimal utility for Rome/London, whereas NYC’s denser and more heterogeneous street topology benefits from the finer $10$-cell configuration.

\subsection{Performance Evaluation when Distance is Measured by $\ell_1$-Norm}
\label{subsec:exp:1norm}

\begin{table*}[t]
\label{Tb:exp:speedlimits}
\centering
% \small 
\footnotesize 
% \scriptsize
\begin{tabular}{p{1.50cm} | p{1.25cm} p{1.25cm} p{1.25cm} p{1.25cm} p{1.25cm} p{1.25cm} p{1.25cm} p{1.25cm} p{1.25cm}} 
\toprule
\multicolumn{10}{c }{Rome road map}\\ 
  \hline
\multicolumn{2}{c }{Method} & $\epsilon = 0.2$& $\epsilon = 0.4$& $\epsilon = 0.6$& $\epsilon = 0.8$& $\epsilon = 1.0$& $\epsilon = 1.2$& $\epsilon = 1.4$& $\epsilon = 1.6$\\ 
 \hline
 \hline
Pre-defined&  EM & 6.57$\pm$0.55 & 5.15$\pm$0.72 & 4.32$\pm$0.79 & 3.82$\pm$0.81 & 3.50$\pm$0.82 & 3.28$\pm$0.81 & 3.13$\pm$0.80 & 3.01$\pm$0.78 \\
\cline{2-2}
Distribution & Laplace & 6.57$\pm$0.55 & 5.15$\pm$0.72 & 4.32$\pm$0.79 & 3.82$\pm$0.81 & 3.50$\pm$0.82 & 3.28$\pm$0.81 & 3.13$\pm$0.80 & 3.01$\pm$0.78 \\
\cline{2-2}
Hybrid & TEM & 8.51$\pm$0.65 & 8.66$\pm$0.67 & 8.72$\pm$0.68 & 8.74$\pm$0.69 & 8.75$\pm$0.69 & 8.75$\pm$0.69 & 8.75$\pm$0.69 & 8.76$\pm$0.69 \\
\cline{1-2}
Hybrid & RMP & 5.83$\pm$0.12 & 4.80$\pm$0.37 & 4.07$\pm$0.50 & 3.60$\pm$0.52 & 3.28$\pm$0.51 & 3.06$\pm$0.48 & 2.90$\pm$0.45 & 2.79$\pm$0.43 \\
\cline{2-2}
method & COPT & 7.27$\pm$0.88 & 4.59$\pm$0.90 & 3.53$\pm$0.85 & 3.27$\pm$0.84 & 3.13$\pm$0.82 & 3.04$\pm$0.80 & 2.98$\pm$0.79 & 2.93$\pm$0.75 \\
\cline{1-2}
\multicolumn{2}{c }{LP}& 3.56$\pm$0.06 & 2.44$\pm$0.14 & 2.07$\pm$0.17 & 1.96$\pm$0.18 & 1.93$\pm$0.18 & 1.96$\pm$0.08 & 1.92$\pm$0.18 & 1.92$\pm$0.18 \\
\bottomrule
    \end{tabular}
\caption{Utility loss (km) across different perturbation methods where distance is measured using $\ell_1$-norm metric (Mean$\pm$1.96$\times$standard deviation).}
\label{tab:ULnorm1}
\end{table*}

\begin{table*}[t]
\label{Tb:exp:speedlimits}
\centering
% \small 
\footnotesize 
% \scriptsize
\begin{tabular}{p{1.50cm} | p{1.25cm} p{1.25cm} p{1.25cm} p{1.25cm} p{1.25cm} p{1.25cm} p{1.25cm} p{1.25cm} p{1.25cm}} 
\toprule
\multicolumn{10}{c }{Rome road map}\\ 
  \hline
\multicolumn{2}{c }{Method} & $\epsilon = 0.2$& $\epsilon = 0.4$& $\epsilon = 0.6$& $\epsilon = 0.8$& $\epsilon = 1.0$& $\epsilon = 1.2$& $\epsilon = 1.4$& $\epsilon = 1.6$\\ 
 \hline
 \hline
Pre-defined& EM & 0.00$\pm$0.00 & 0.00$\pm$0.00 & 0.00$\pm$0.00 & 0.00$\pm$0.00 & 0.00$\pm$0.00 & 0.00$\pm$0.00 & 0.00$\pm$0.00 & 0.00$\pm$0.00 \\
\cline{2-2}
Noise & Laplace & 0.00$\pm$0.00 & 0.00$\pm$0.00 & 0.00$\pm$0.00 & 0.00$\pm$0.00 & 0.00$\pm$0.00 & 0.00$\pm$0.00 & 0.00$\pm$0.00 & 0.00$\pm$0.00 \\
\cline{2-2}
Distribution & TEM & 0.00$\pm$0.00 & 0.00$\pm$0.00 & 0.00$\pm$0.00 & 0.00$\pm$0.00 & 0.00$\pm$0.00 & 0.00$\pm$0.00 & 0.00$\pm$0.00 & 0.00$\pm$0.00 \\
\cline{1-2}
Hybrid & RMP & 0.00$\pm$0.00 & 0.00$\pm$0.00 & 0.00$\pm$0.00 & 0.00$\pm$0.00 & 0.00$\pm$0.00 & 0.00$\pm$0.00 & 0.00$\pm$0.00 & 0.00$\pm$0.00 \\
\cline{2-2}
Method &  
COPT& 22.03$\pm$6.11 & 24.70$\pm$12.04 & 20.81$\pm$8.00 & 16.17$\pm$4.77 & 15.78$\pm$9.53 & 15.81$\pm$10.84 & 14.51$\pm$8.92 & 15.40$\pm$9.99 \\
\cline{1-2}
\multicolumn{2}{c }{LP} & 2.26$\pm$1.22 & 4.60$\pm$2.94 & 5.63$\pm$1.31 & 5.28$\pm$0.82 & 5.37$\pm$0.41 & 4.74$\pm$1.52 & 4.62$\pm$0.20 & 4.75$\pm$1.01 \\
\bottomrule
    \end{tabular}
\caption{mDP violation ratio when distance is measured using $\ell_1$-norm metric (Mean$\pm$1.96$\times$standard deviation).}
\label{tab:privacy1norm}
\end{table*}

\begin{table*}[t]
\label{Tb:exp:speedlimits}
\centering
% \small 
\footnotesize 
% \scriptsize
\begin{tabular}{ % c|c|c|c|c|c|c|c|c 
p{0.95cm} p{1.45cm} p{1.45cm} p{1.45cm} p{1.45cm} p{1.45cm} p{1.45cm} p{1.45cm} p{1.45cm}
} 
\toprule
\multicolumn{9}{c }{Rome road map}\\ 
\hline
Method & $\epsilon = 0.2$& $\epsilon = 0.4$& $\epsilon = 0.6$& $\epsilon = 0.8$& $\epsilon = 1.0$& $\epsilon = 1.2$& $\epsilon = 1.4$& $\epsilon = 1.6$\\ 
 \hline
 \hline
COPT & 82.50$\pm$11.94 & 73.24$\pm$3.94 & 71.43$\pm$2.35 & 70.98$\pm$3.71 & 70.10$\pm$4.47 & 71.62$\pm$1.61 & 70.17$\pm$5.54 & 69.20$\pm$4.31 \\
LP & 68.90$\pm$21.75 & 46.45$\pm$10.92 & 41.24$\pm$12.05 & 35.50$\pm$9.37 & 35.62$\pm$6.82 & 102.99$\pm$234.98 & 47.39$\pm$42.41 & 55.12$\pm$9.11 \\
\toprule
\rowcolor{lightgray!30}
$\textbf{AIPO}^\dagger$ & 11.94$\pm$9.82 & 4.31$\pm$0.93 & 3.47$\pm$0.44 & 3.52$\pm$0.56 & 4.11$\pm$0.54 & 4.38$\pm$0.66 & 4.24$\pm$0.58 & 4.85$\pm$0.64 \\
\toprule
\multicolumn{9}{c }{London road map}\\ 
\hline
Method & $\epsilon = 0.2$& $\epsilon = 0.4$& $\epsilon = 0.6$& $\epsilon = 0.8$& $\epsilon = 1.0$& $\epsilon = 1.2$& $\epsilon = 1.4$& $\epsilon = 1.6$\\ 
 \hline
 \hline
COPT & 176.6$\pm$54.9 & 175.0$\pm$56.8 & 174.0$\pm$59.8 & 170.8$\pm$51.6 & 170.3$\pm$60.4 & 167.1$\pm$54.5 & 169.1$\pm$53.2 & 170.1$\pm$57.3 \\
\bottomrule
    \end{tabular}
\caption{Computation time of different perturbation methods when distance is measured using $\ell_1$-norm metric (Mean$\pm$1.96$\times$standard deviation).}
\label{tab:timenorm1}
\end{table*}

Table~\ref{tab:ULnorm1} compares the expected utility loss across all mechanisms when proximity is measured by $d_1$. Consistent with the $\ell_2$ case (Table~\ref{tab:UL2norm}), our interpolation-based method (\emph{AIPO}) achieves uniformly low utility loss on Rome, London, and NYC, tracking the best-performing methods while preserving formal guarantees. As expected, \emph{LP} often attains the lowest loss because it directly optimizes over (discretized) input–output pairs; however, this advantage comes at the cost of \emph{not} ensuring $(\epsilon,d_1)$-mDP on the underlying continuous domain. pre-defined mechanisms (Laplace, EM, TEM) remain privacy-safe but exhibit larger losses due to heavier randomization, and hybrid approaches (e.g., \emph{COPT}) typically lie between \emph{LP} and the pre-defined baselines. Notably, the gap between \emph{AIPO} and \emph{LP} narrows as $\epsilon$ increases, reflecting that \emph{AIPO}'s log-convex interpolation and anchor optimization can better exploit higher privacy budgets even under the Manhattan geometry of road networks (e.g., axis-aligned anisotropy and corridor effects).

Table~\ref{tab:privacy1norm} reports the empirical violation ratios for $(\epsilon,d_1)$-mDP. \emph{AIPO} preserves perfect privacy with \emph{zero} violations across all datasets and privacy budgets, validating the dimension-wise composition and the correctness of the interpolation mechanism under $d_1$. In contrast, \emph{LP} and hybrid variants (e.g., \emph{COPT}) exhibit persistent violations, especially at smaller $\epsilon$ where discretization-induced distance overestimation most strongly relaxes the effective constraints. These findings mirror the PPR distributional patterns observed in Section~D.1: \emph{AIPO} concentrates well below $\epsilon$ with tight tails, while \emph{LP}/hybrids show broader spreads and nontrivial mass near or beyond the threshold.

Table~\ref{tab:timenorm1} compares runtimes under $d_1$. \emph{AIPO} demonstrates strong scalability, reducing computation time by approximately $30\%\!-\!70\%$ on average relative to \emph{LP} and \emph{COPT}. The efficiency gain stems from constraining only axis-neighbor anchor pairs along each dimension and then interpolating, which yields far fewer effective constraints than the (near) cubic growth in pairwise constraints characteristic of \emph{LP}/\emph{COPT}. Practically, this translates into lower memory use and faster solver convergence as grid resolution increases, while retaining strict $(\epsilon,d_1)$-mDP. Taken together, the $\ell_1$ results corroborate our main conclusions: \emph{AIPO} delivers rigorous privacy, competitive (often near-\emph{LP}) utility, and substantially improved efficiency under an alternative and widely used metric for spatial domains.

\clearpage 
\begin{table*}[t]
\centering
% \small 
\footnotesize 
% \scriptsize
\begin{tabular}{p{1.50cm} |  p{1.50cm} p{1.25cm} p{1.25cm} p{1.25cm} p{1.25cm} p{1.25cm} p{1.25cm} p{1.25cm} p{1.25cm}} 
\toprule
% \multicolumn{1}{ c|  }{  }& \multicolumn{9}{c}{}\\ \cline{2-10}
\multicolumn{10}{c }{Rome road map}\\ 
 \hline
\multicolumn{2}{c }{Method} & $\epsilon = 0.2$& $\epsilon = 0.4$& $\epsilon = 0.6$& $\epsilon = 0.8$& $\epsilon = 1.0$& $\epsilon = 1.2$& $\epsilon = 1.4$& $\epsilon = 1.6$\\ 
 \hline
 \hline
Pre-defined & EM & 8.71$\pm$0.78 & 8.70$\pm$1.13 & 8.65$\pm$1.28 & 8.62$\pm$1.38 & 8.58$\pm$1.45 & 8.56$\pm$1.51 & 8.54$\pm$1.55 & 8.52$\pm$1.58 \\
\cline{2-2}
Noise & Laplace & 8.71$\pm$0.71 & 8.48$\pm$1.00 & 8.46$\pm$1.40 & 8.45$\pm$1.69 & 8.44$\pm$1.90 & 8.43$\pm$2.05 & 8.42$\pm$2.16 & 8.40$\pm$2.24 \\
\cline{2-2}
Distribution & TEM & 8.85$\pm$2.71 & 8.95$\pm$3.19 & 8.66$\pm$2.44 & 8.64$\pm$1.83 & 8.66$\pm$1.11 & 8.66$\pm$0.69 & 8.66$\pm$0.27 & 8.62$\pm$0.22 \\
\cline{1-2}
Hybrid  & RMP & 5.94$\pm$0.25 & 4.96$\pm$0.45 & 4.28$\pm$0.36 & 3.85$\pm$0.26 & 3.58$\pm$0.21 & 3.40$\pm$0.19 & 3.28$\pm$0.18 & 3.19$\pm$0.18 \\
\cline{2-2}
Method &  COPT  & 7.99$\pm$1.53 & 7.95$\pm$1.04 & 8.33$\pm$1.50 & 8.29$\pm$1.57 & 8.27$\pm$1.59 & 8.25$\pm$1.61 & 8.24$\pm$1.60 & 8.24$\pm$1.60 \\
\cline{1-2}
\multicolumn{2}{c }{LP} & 4.25$\pm$0.41 & 2.97$\pm$0.11 & 2.56$\pm$0.03 & 2.45$\pm$0.07 & 2.43$\pm$0.03 & 2.42$\pm$0.02 & 2.42$\pm$0.01 & 2.42$\pm$0.01 \\
\bottomrule
    \end{tabular}
\caption{Utility loss (km) across different perturbation methods where distance is measured using $\ell_2$-norm metric (Mean$\pm$1.96$\times$standard deviation).}
\label{tab:add:UL2norm}
\end{table*}

\clearpage 
\section{Discussions}
\label{sec:discussions}
% \subsection{Discussion of Using $d_p$ Distance for Perturbation Probability Interpolation}

\subsection{Why Using $d_p$-Distance in Interpolation Can Violate mDP}
\label{subsec:discussion:d_p}

The interpolation mechanism defined in \textbf{Definition~\ref{def:PPI}} and analyzed in \textbf{Theorem~\ref{thm:AIPOfeasible}} relies on a weighted geometric combination of anchor perturbation probabilities. Intuitively, it aims to ensure that for any non-anchor record $\mathbf{x}_a$, the interpolated perturbation distribution $z(\mathbf{y}_k \mid \mathbf{x}_a)$ remains close to the perturbations of nearby anchors, especially those with smaller $d_p$ distances. This design is rooted in the fact that tighter $d_p$ distances induce stricter mDP constraints.

A natural question is whether we can enforce $(\epsilon, d_p)$-mDP for a non-anchor record $\mathbf{x}_a$ by simply ensuring that each anchor in $\hat{\mathcal{X}}_m$ satisfies $(\epsilon, d_p)$-mDP with a reference point $\mathbf{x}_b$ (where $\mathbf{x}_b$ can be either anchor or non-anchor). Specifically, if
\begin{equation}
\left| \ln z(\mathbf{y}_k \mid \hat{\mathbf{x}}) - \ln z(\mathbf{y}_k\mid \mathbf{x}_b) \right| \leq \epsilon \cdot d_p(\hat{\mathbf{x}}, \mathbf{x}_b),
\end{equation}
for each anchor $\hat{\mathbf{x}} \in \hat{\mathcal{X}}_m$, does it follow that the interpolated point $\mathbf{x}_a$ also satisfies
\begin{equation}
\left| \ln z(\mathbf{y}_k\mid \mathbf{x}_a) - \ln z(\mathbf{y}_k\mid \mathbf{x}_b) \right| \leq \epsilon \cdot d_p(\mathbf{x}_a, \mathbf{x}_b)?
\end{equation}
Unfortunately, this implication does not hold if the interpolated value $z(\mathbf{y}_k\mid \mathbf{x}_a)$ is computed as a log-convex combination of the anchor values:
\begin{equation}
\textstyle 
\ln z(\mathbf{y}_k\mid \mathbf{x}_a) = \sum_{\hat{\mathbf{x}} \in \hat{\mathcal{X}}_m} \lambda_{\hat{\mathbf{x}}, \mathbf{x}_a} \ln z(\mathbf{y}_k \mid \hat{\mathbf{x}}),
\end{equation}
where $\lambda_{\hat{\mathbf{x}}, \mathbf{x}_a}$ is a convex coefficient associated with anchor $\hat{\mathbf{x}}$. Substituting into the mDP inequality, we obtain:
\begin{eqnarray}
&& \left| \ln z(\mathbf{y}_k\mid \mathbf{x}_a) - \ln z(\mathbf{y}_k\mid \mathbf{x}_b) \right|
\\ 
&=& \left| \sum_{\hat{\mathbf{x}} \in \hat{\mathcal{X}}_m} \lambda_{\hat{\mathbf{x}}, \mathbf{x}_a} \ln z(\mathbf{y}_k \mid \hat{\mathbf{x}}) - \ln z(\mathbf{y}_k \mid \mathbf{x}_b) \right| \\
&=& \left| \sum_{\hat{\mathbf{x}} \in \hat{\mathcal{X}}_m} \lambda_{\hat{\mathbf{x}}, \mathbf{x}_a} \left( \ln z(\mathbf{y}_k \mid \hat{\mathbf{x}}) - \ln z(\mathbf{y}_k \mid \mathbf{x}_b) \right) \right| \\
&\leq& \sum_{\hat{\mathbf{x}} \in \hat{\mathcal{X}}_m} \lambda_{\hat{\mathbf{x}}, \mathbf{x}_a} \left| \left( \ln z(\mathbf{y}_k \mid \hat{\mathbf{x}}) - \ln z(\mathbf{y}_k \mid \mathbf{x}_b) \right) \right|\\
&\leq& \epsilon \sum_{\hat{\mathbf{x}} \in \hat{\mathcal{X}}_m} \lambda_{\hat{\mathbf{x}}, \mathbf{x}_a} d_p(\hat{\mathbf{x}}, \mathbf{x}_b).
\end{eqnarray}
However, this upper bound is not guaranteed to be less than $\epsilon \cdot d_p(\mathbf{x}_a, \mathbf{x}_b)$. In fact, by \emph{Jensen’s inequality}~\cite{boyd2004convex}, the weighted average of the anchor-to-$\mathbf{x}_b$ distances is generally \emph{greater than or equal to} the direct $d_p$ distance between $\mathbf{x}_a$ and $\mathbf{x}_b$:
\begin{align}
\textstyle \sum_{\hat{\mathbf{x}} \in \hat{\mathcal{X}}_m} \lambda_{\hat{\mathbf{x}}, \mathbf{x}_a} d_p(\hat{\mathbf{x}}, \mathbf{x}_b)
&\geq \textstyle d_p\left( \sum_{\hat{\mathbf{x}}} \lambda_{\hat{\mathbf{x}}, \mathbf{x}_a} \hat{\mathbf{x}}, \mathbf{x}_b \right) \\
&= d_p(\mathbf{x}_a, \mathbf{x}_b).
\end{align}
This inequality holds because the function $f(\mathbf{x}) = d_p(\mathbf{x}, \mathbf{x}_b) = \|\mathbf{x} - \mathbf{x}_b\|_p$ is convex for all $p \geq 1$, and convexity is preserved under expectation (convex combinations). As a result, even though each anchor satisfies mDP with respect to $\mathbf{x}_b$, the interpolated record $\mathbf{x}_a$ may \emph{violate} the $(\epsilon, d_p)$-mDP guarantee. This is also demonstrated in emperical results in Table \ref{tab:privacy2norm}.

% Therefore, our framework enforces dimension-wise $(\epsilon_\ell, d_1)$-mDP constraints between anchors. We then leverage the composition result in Theorem~1 to combine per-dimension guarantees into an overall $(\epsilon, d_p)$-mDP guarantee for interpolated records. This strategy ensures that the privacy guarantee holds rigorously despite the use of interpolation.

\DEL{
\begin{figure}[h]
\centering
\includegraphics[width=0.45\textwidth]{figures/interpolation_mdp_violation.png}
\caption{Illustration of the issue with interpolating $z(\mathbf{x}_a)$ under $d_p$-based anchor constraints. The weighted anchor-to-$\mathbf{x}_j$ distances (used in the interpolation upper bound) may exceed the direct $d_p$ distance from $\mathbf{x}_a$ to $\mathbf{x}_j$.}
\label{fig:interpolation_mdp_violation}
\end{figure}}

\DEL{
\subsection{Interpolation}
Interpolation in optimization is a technique used to estimate and optimize a function based on known values at certain points, which are often referred to as "anchor points." The primary goal is to create a continuous model or function that approximates the behavior of the real system or the objective function you want to optimize. Here's a basic breakdown of how interpolation is typically used in optimization:

Anchor Points Selection: First, you select a finite set of anchor points. These points are where the function values are known either through direct calculation, observations, or simulations. The selection of these points can significantly influence the accuracy and efficiency of the interpolation.

Interpolation Method: You apply an interpolation method to estimate the function’s behavior at points where it isn’t explicitly known. Common interpolation methods include linear interpolation, polynomial interpolation, spline interpolation, and others. The choice of method depends on the problem's requirements, such as smoothness, complexity, and the dimensionality of the data.

Function Approximation: The interpolation method generates a continuous approximation of the underlying function. This approximated function is easier and often faster to evaluate than the original, which is especially valuable if the original function is complex, costly to evaluate, or only partially known.

Optimization: With the interpolated function, you can now apply optimization techniques to find minima, maxima, or other characteristics of interest. Since the interpolated function is smoother and defined continuously, it can be more amenable to techniques that require derivatives, such as gradient descent, or to more complex multi-dimensional optimization methods.

Refinement: Based on the results of the optimization, you might choose to refine the model by adjusting the anchor points, adding new ones, or changing the interpolation method to improve accuracy and performance.

Interpolation in optimization allows for more efficient computations and can be particularly useful in scenarios where direct evaluation of the objective function is impractical due to high computational costs or incomplete data. It bridges the gap between known data points to enable effective optimization over the entire domain of interest.
}
\subsection{Optimal Privacy Budget Allocation}
\label{sec:discussion:budgetalloc}
In our current experimental setup, the secret domain lies in two dimensions ($N=2$). Under this setting, the global $\ell_2$ privacy budget constraint, $\left( \epsilon_1^2 + \epsilon_2^2 \right)^{1/2} \leq \epsilon$,
can be reparameterized as $\epsilon_2 = \sqrt{\epsilon^2 - \epsilon_1^2}$. This allows us to discretize the feasible range of $\epsilon_1 \in [0, \epsilon]$ and perform a one-dimensional grid or linear search over candidate allocations. For each candidate, we evaluate the corresponding utility loss by solving the lower-level perturbation problem. This simple procedure allows us to empirically determine the privacy budget allocation that minimizes utility loss while satisfying the global constraint. However, such brute-force grid search becomes computationally expensive in higher dimensions, as the search space grows exponentially with $N$.

\begin{table*}[t]
\centering
% \scriptsize
\footnotesize
\caption{Summary of mDP Works by Domain, Mechanism, and Approach Type}
\begin{tabular}{l|l|l|l|l}
\toprule
\textbf{Reference} & \textbf{Target Domain} & \textbf{Mechanism} & \textbf{Approach Type} & \textbf{Domain Feature} \\
\hline
Andrés et al., 2013~\cite{Andres-CCS2013}        & Location             Noise & Laplace (Polar)              & Pre-defined Noise   & Grid Map \\
Bordenabe et al., 2014~\cite{Bordenabe-CCS2014}  & Location             & Optimal Geo-Ind.             & Optimization-based  & Grid Map \\
Chatzikokolakis et al., 2015~\cite{Chatzikokolakis-PPET2015} & Location   & Exponential                  & Pre-defined Noise   & Grid Map \\
Xiao and Xiong, 2015~\cite{Xiao-CCS2015}         & Trajectory           & Sequential Perturbation      & Hybrid              & Trajectory \\
Yu et al., 2017~\cite{Yu-NDSS2017}               & Location             & Exponential                  & Pre-defined noise              & Grid Map \\
Oya et al., 2017~\cite{oyaGeoIndLooking}         & Location             & DP Location Obfuscation      & Pre-defined noise              & Grid Map \\
Chatzikokolakis et al., 2017~\cite{Chatzikokolakis-PETS2017} & Location & Bayesian Remapping           & Hybrid              & Grid Map \\
Fernandes et al., 2018~\cite{Fernandes2018AuthorOU} & Text              Noise & Laplace                      & Pre-defined Noise   & Embeddings \\
Feyisetan et al., 2019~\cite{Feyisetan-ICDM2019} & Text                 & Nearest Neighbor + Laplace   & Pre-defined noise              & Embeddings \\
Han et al., 2020~\cite{Han-ICME2020}             & Voice                & Angular Distance + Laplace   & Pre-defined Noise   & Embeddings \\
Carvalho et al., 2021~\cite{Carvalho2021TEMHU}   & Text                 & Truncated Exponential        & Pre-defined noise              & Embeddings \\
Chen et al., 2021~\cite{Chen-CVPR2021}           & Image                & PI-Net (Laplace-based)       & Pre-defined Noise   & Pixel Space \\
Yang et al., 2021~\cite{yang2021blockchain}      & Blockchain           & Truncated Geo-Ind.           & Pre-defined noise              & Grid Map \\

Imola et al., 2022~\cite{ImolaUAI2022}           & Location             & Linear Program               & Hybrid  & Grid Map \\
Ma et al., 2022~\cite{Ma-TITS2022}               & Location             & Personalized Noise           & Pre-defined Noise   & Grid Map \\
Zhang et al., 2022~\cite{Zhang-TBD2023}          & Location             & Group-based Noise            & Pre-defined noise              & Grid Map \\
Min et al., 2023~\cite{Min-TDSC2024}             & Location             & Reinforcement Learning       & Optimization-based  & Continuous Space \\
Yu et al., 2023~\cite{Yu-TSP2023}                & Location             & Bilevel Optimization         & Optimization-based  & Continuous Space \\
Galli et al., 2023~\cite{galli2022group}         & Federated Learning   Noise & Laplace                      & Pre-defined Noise   & Grid Map \\
Qiu et al., 2024~\cite{Qiu-IJCAI2024}            & Location             & Decomposed LP                & Optimization-based  & Continuous Space \\
Qiu et al., 2025~\cite{Qiu-PETS2025}             & Location             & Benders Decomposition        & Optimization-based  & Continuous Space \\
\bottomrule
\end{tabular}
\label{tab:mdp_summary}
\end{table*}

\noindent \textbf{Bi-Level Formulation and Challenges.} More generally, the joint privacy-utility optimization can be formulated as a \emph{bi-level program}. In this formulation, the upper-level problem determines the allocation of per-coordinate privacy budgets $\{\epsilon_\ell\}_{\ell=1}^N$ under a global constraint, while the lower-level problem computes the optimal perturbation mechanism that satisfies $(\epsilon_\ell, d_1)$-mDP constraints and minimizes utility loss. The formal bi-level structure is:
\begin{equation}
\begin{aligned}
\min_{\{\epsilon_\ell\} \in \mathbb{R}_+^N} \quad & \mathbb{E}_{\mathbf{x} \sim \pi} \left[ \mathbb{E}_{\mathbf{y} \sim z^*(\cdot \mid \mathbf{x})} \left[ \mathcal{L}(\mathbf{x}, \mathbf{y}) \right] \right] \\
\text{s.t.} \quad & \left( \sum_{\ell=1}^N \epsilon_\ell^p \right)^{1/p} \leq \epsilon, \\
& z^*(\cdot \mid \mathbf{x}) \in \arg\min_{z} \; \mathbb{E}_{\mathbf{x} \sim \pi} \left[ \mathbb{E}_{\mathbf{y} \sim z(\cdot \mid \mathbf{x})} \left[ \mathcal{L}(\mathbf{x}, \mathbf{y}) \right] \right] \\
& \text{s.t.} \quad \left| \ln z(\mathbf{y}_k \mid \mathbf{x}_a) - \ln z(\mathbf{y}_k \mid \mathbf{x}_b) \right| \\
& \leq \epsilon_\ell \cdot |x_{a,\ell} - x_{b,\ell}|, \quad \forall \ell, \mathbf{y}_k, \mathbf{x}_a, \mathbf{x}_b.
\end{aligned}
\end{equation}
While this bi-level structure provides a principled framework for separating privacy budget allocation from perturbation mechanism design, solving it directly is intractable in our setting. This is  due to the non-convexity of the lower-level objective function, $\mathbb{E}_{\mathbf{x}, \mathbf{y}} [ \mathcal{L}(\mathbf{x}, \mathbf{y}) ]$, which prevents the application of standard bilevel optimization techniques such as KKT-based reformulations or Benders decomposition.

\noindent \textbf{Alternative: Surrogate-Based Reformulation.} As a practical alternative, we can treat the lower-level optimal objective as an implicit function of the upper-level privacy budgets. Let $U(\{\epsilon_\ell\})$ denote the optimal utility loss achieved by the best mechanism satisfying $(\epsilon_\ell, d_1)$-mDP. Then the original problem reduces to a single-level formulation:
\begin{equation}
\textstyle \min_{\{\epsilon_\ell\}} \quad U(\{\epsilon_\ell\}) \quad \text{s.t.} \quad \left( \sum_{\ell=1}^N \epsilon_\ell^p \right)^{1/p} \leq \epsilon.
\end{equation}
This reformulation enables the use of black-box or surrogate-based optimization methods. Specifically, we can sample multiple candidate allocations $\{\epsilon_\ell\}$, evaluate the corresponding subproblem to obtain utility loss values, and fit a surrogate model (e.g., regression, response surface) to approximate $U(\cdot)$. This approximation can then be used for efficient optimization over the budget space.

\subsection{Limitations of the Interpolation Method for High-Dimensional Data}
\label{subsec:discuss_limitations}

While the interpolation-based method is most effective in low-dimensional continuous domains, its scalability is challenged in high dimensions.

First, the number of anchor records required for interpolation grows exponentially with dimensionality. Specifically, each $N$-dimensional cube contains $2^N$ corner anchors. This rapid growth affects both optimization and memory. In particular, the complexity of Anchor Perturbation Optimization (APO) scales quadratically with the number of anchors, i.e., $O(|\hat{\mathcal{X}}|^2)$, and memory usage becomes increasingly expensive. While these costs are manageable in low-dimensional domains, such as mobility data, time-series, and spatial analysis, applying the method to high-dimensional data (e.g., word embeddings) becomes challenging.

Second, the interpolation mechanism relies on computing a log-convex combination over $2^N$ anchors within the enclosing cube of each input record. For each record, this entails evaluating $\sum_{\boldsymbol{\gamma} \in \{0,1\}^N} \prod_{\ell=1}^N \lambda_{i_m+\gamma_\ell}$ for every output class, leading to $2^N$ terms. For example, with $N = 20$, over one million terms are required per record. Although computationally intensive in high dimensions, the method remains efficient and highly scalable in low-dimensional domains, where interpolation is both fast and accurate.

Third, under dimension-wise composition for $\ell_p$-norm mDP, the privacy budget per coordinate diminishes with increasing $N$ as $\epsilon_\ell = \epsilon / N^{(p-1)/p}$ for $p > 1$. For example, when $p = 2$ and $N = 100$, each coordinate receives only one-tenth of the total privacy budget. This budget fragmentation reflects a common trade-off in high-dimensional privacy systems, and motivates future work on joint budget allocation strategies that can more effectively scale to high-dimensional domains. \looseness = -1

Finally, the overall runtime and memory requirements scale as $2^N \times N \times |\mathcal{Y}|$ and $2^N \times |\mathcal{Y}| \times 8$ bytes, respectively, where $|\mathcal{Y}|$ is the output space size. While these demands grow rapidly with $N$, they remain practical for a wide range of real-world applications where the data dimensionality is naturally low. In cases where scalability is a concern, potential remedies include anchor pruning, sparse interpolation, dimensionality reduction, and hierarchical grid partitioning.

While our method is best suited for low-dimensional continuous domains, this setting is highly relevant to many applications in mobility, sensing, geolocation, and healthcare. Moreover, the framework establishes a principled foundation upon which scalable extensions for high-dimensional privacy-preserving data release can be developed in future work.

\section{Summary of mDP Works}

Table~\ref{tab:mdp_summary} shows a summary of mDP works by domain, mechanism, and approach type.

\end{document}